\documentclass{article}

\usepackage[preprint,nonatbib]{neurips_2020}

\usepackage[utf8]{inputenc} % allow utf-8 input
\usepackage[T1]{fontenc}    % use 8-bit T1 fonts
\usepackage{url}                 % simple URL typesetting
\usepackage{booktabs}       % professional-quality tables
\usepackage{amsfonts}       % blackboard math symbols
\usepackage{nicefrac}       % compact symbols for 1/2, etc.
\usepackage{microtype}      % microtypography
\usepackage{multirow}
\usepackage{subcaption}
\usepackage{adjustbox}
\usepackage{xcolor, mdframed}
\usepackage{float}
\usepackage{scrextend}

\usepackage{algorithmic}
\usepackage[linesnumbered, ruled,vlined]{algorithm2e}
\usepackage{amsmath}
\usepackage[font=footnotesize, labelfont=bf]{caption}

\usepackage{amsthm}
\newtheorem{theorem}{Theorem}[section]
\newtheorem{prop}{Proposition}[section]

\newtheorem{lemma}[theorem]{Lemma}
\newtheorem*{remark}{Remark}
\newtheorem{defn}{Definition}[]

\usepackage{thmtools, thm-restate}
\usepackage{tikz}

\usepackage{amssymb}
\usepackage{pgfplots}
\usepackage{pgfmath}
\SetCommentSty{mycommfont}
\usepackage{bm}

\title{Rethinking Privacy Preserving Deep Learning: \\ How to Evaluate and Thwart Privacy Attacks}

\author{%
  Lixin Fan\thanks{equal contribution.} \\
  WeBank AI Lab\\
   \And
   Kam Woh Ng$^{*}$ \\
   WeBank AI Lab \\
   \And
   Ce Ju$^{*}$ \\
   WeBank AI Lab \\
   \And
   Tianyu Zhang \\
   WeBank AI Lab \\
   \AND
   Chang Liu \\
   WeBank AI Lab \\
   \And
   Chee Seng Chan \\
   University of Malaya \\
   \And
   Qiang Yang \\
   Hong Kong Univ. of Science and Tech. \\
}

\begin{document}
	
	\maketitle
	
	\begin{abstract}
		This paper investigates capabilities of Privacy-Preserving Deep Learning (PPDL) mechanisms against various forms of privacy attacks. First, we propose to quantitatively measure the trade-off between model accuracy and privacy losses incurred by \textit{reconstruction, tracing and membership attacks}. Second, we formulate reconstruction attacks as solving a noisy system of linear equations, and prove that attacks are guaranteed to be defeated if condition (\ref{eq:recon-cond}) is unfulfilled. Third, based on theoretical analysis, a novel Secret Polarization Network (SPN) is proposed to thwart privacy attacks, which pose serious challenges to existing PPDL methods. Extensive experiments showed that model accuracies are improved on average by 5-20\% compared with baseline mechanisms, in regimes where data privacy are satisfactorily protected.
	\end{abstract}
	
	\section{Introduction}

	Privacy-preserving deep learning (PPDL)  aims to collaboratively train and share a deep neural network model among multiple participants, without exposing to each other information about their private training data.  
	This typical \textit{federated learning} setting is particularly attractive to business scenarios in which raw data e.g. medical records or bank transactions are too sensitive and valuable to be disclosed to other parties \cite{communicationEfficient/mcMahan2017,yang2019federated}.  
	While \textit{differential privacy} based approaches e.g. \cite{DLDP_Abadi16,PPDL/shokri2015} attract much attentions due to its theoretical guarantee of privacy protection and low computational complexity \cite{DP/dwork2006,Calibrating_Noise_DA/dwork2006}, there is %a well-grounded consensus that %privacy preserving techniques inevitably result in drops of model {utilities}.   
	a fundamental trade-off between \textit{privacy guarantee} vs \textit{utility} of learned models, i.e. % is inevitable i.e. 
	%<<<<<<< HEAD
	overly conservative privacy protections often significantly deteriorate model utilities (\textit{accuracies} for classification models).  
	Existing solutions e.g. \cite{DLDP_Abadi16,PPDL/shokri2015} %, that follow the differential privacy theory 
	are unsatisfactory in our view --- low $\epsilon$ privacy budget value does not necessarily lead to desired levels of privacy protection. 
	For instance, the leakage of shared gradients may admit complete reconstruction of training data under certain circumstances \cite{DeepLeakage_Han19,BeyondInferClass_Wang19,Gradient_Leakage/wei2020,inverting_gradient/geiping2020},  even though substantial fraction of gradients elements are truncated \cite{PPDL/shokri2015} or large random noise are added \cite{DLDP_Abadi16}.
	
	In order to make critical analysis and fair evaluations of different PPDL algorithms, we argue that one must employ an \textit{objective} evaluation protocol to quantitatively measure privacy preserving capabilities against various forms of privacy attacks. 
	Following a privacy adversary approach \cite{attack/dwork2017,exploitingFeatLeak_Vitaly18},  we propose to
	evaluate the admitted privacy loss by three objective measures i.e.  \textit{reconstruction},  \textit{tracing} and \textit{membership} losses, with respect to the accuracies of protected models. To this end, Privacy-Preserving Characteristic (PPC) curves are used to delineate the trade-off, with Calibrated Averaged Performance (CAP) faithfully quantifying a given PPC curve. These empirical measures complement the theoretical bound of the privacy loss and constitute the first contribution of our work (see Figure \ref{fig:ppc-all} for example PPC).   

	As demonstrated by experimental results in Sect. \ref{sect:exper},  the leakage of shared gradients poses serious challenges to existing PPDL methods\cite{DLDP_Abadi16,PPDL/shokri2015,DeepLeakage_Han19}. 
	Our second contribution, therefore, is a novel \textit{secret polarization network} (SPN) and a polarization loss term, which 
	%effectively defeat the Deep Leakage attacks and, at the same time, improve the classification accuracy of the baseline CNN network. The private loss term defined in Section ? 
	%The novel network architecture, coined as \textit{secret polarization network} (SPN), 
	%<<<<<<< HEAD
	%bring about two advantages in tandem with public backbone networks --- first, SPN helps to defeat privacy attacks by adding \textit{secrete} and  \textit{element-wise adaptive} gradients to shared gradients;  second, the added polarization loss acts as a regularization term to consistently improve the classification accuracies of baseline networks in federated learning settings. This SPN based mechanism has demonstrated strong capability to thwart three types of privacy attacks without significant deterioration of model accuracies. As summarized by %CMP and 
	%CAP values in Figure \ref{fig:barchart-summary}, SPN compares favorably with existing solutions \cite{PPDL/shokri2015} and \cite{DLDP_Abadi16} with pronounced improvements of performances against reconstruction, membership and tracing attacks. 
	%=======
	bring about two advantages in tandem with public backbone networks --- first, SPN helps to defeat privacy attacks by adding \textit{secrete} and  \textit{element-wise adaptive} gradients to shared gradients;  second, the added polarization loss acts as a regularization term to consistently improve the classification accuracies of baseline networks in federated learning settings. This SPN based mechanism has demonstrated strong capability to thwart three types of privacy attacks without significant deterioration of model accuracies. As summarized by %CMP and 
	CAP values in Fig. \ref{fig:barchart-summary}, SPN compares favorably with existing solutions \cite{PPDL/shokri2015} and \cite{DLDP_Abadi16} with pronounced improvements of performances against reconstruction, membership and tracing attacks. 
	%>>>>>>> 2ccb4a654db1c6498e7bc0002c8d86a83ed5e4c1
	%Moreover, the accuracy of the backbone network is not compromised but improved, with the magnitudes of improvements actually increased with the number of participants in the federated learning setting. 
	
	%the successfulness of recent reconstruction attacks \cite{Model_Inversion/fredrikson2015,DeepLeakage_Han19,BeyondInferClass_Wang19,Gradient_Leakage/wei2020,inverting_gradient/geiping2020} and 
	Our third contribution is the formulation of reconstruction attacks as solving a \textit{noisy system of linear equations}, and it is proved that reconstructions are guaranteed to fail if the necessary condition (\ref{eq:recon-cond}) in Theorem \ref{thm:recon-lineareq} is purposely invalidated. This theoretical analysis sheds new light on the effectiveness of DP based privacy preserving mechanisms. 

	\begin{figure}[t]
		%\centering
		%\begin{minipage}[t]{0.25\textwidth}
		\centering
		%\includegraphics[scale=0.48]{imgs/legends/legend_capbar_horizontal.png}	
		%\\
		\begin{subfigure}{0.48\linewidth}
			\includegraphics[width=0.49\linewidth,height=2cm]{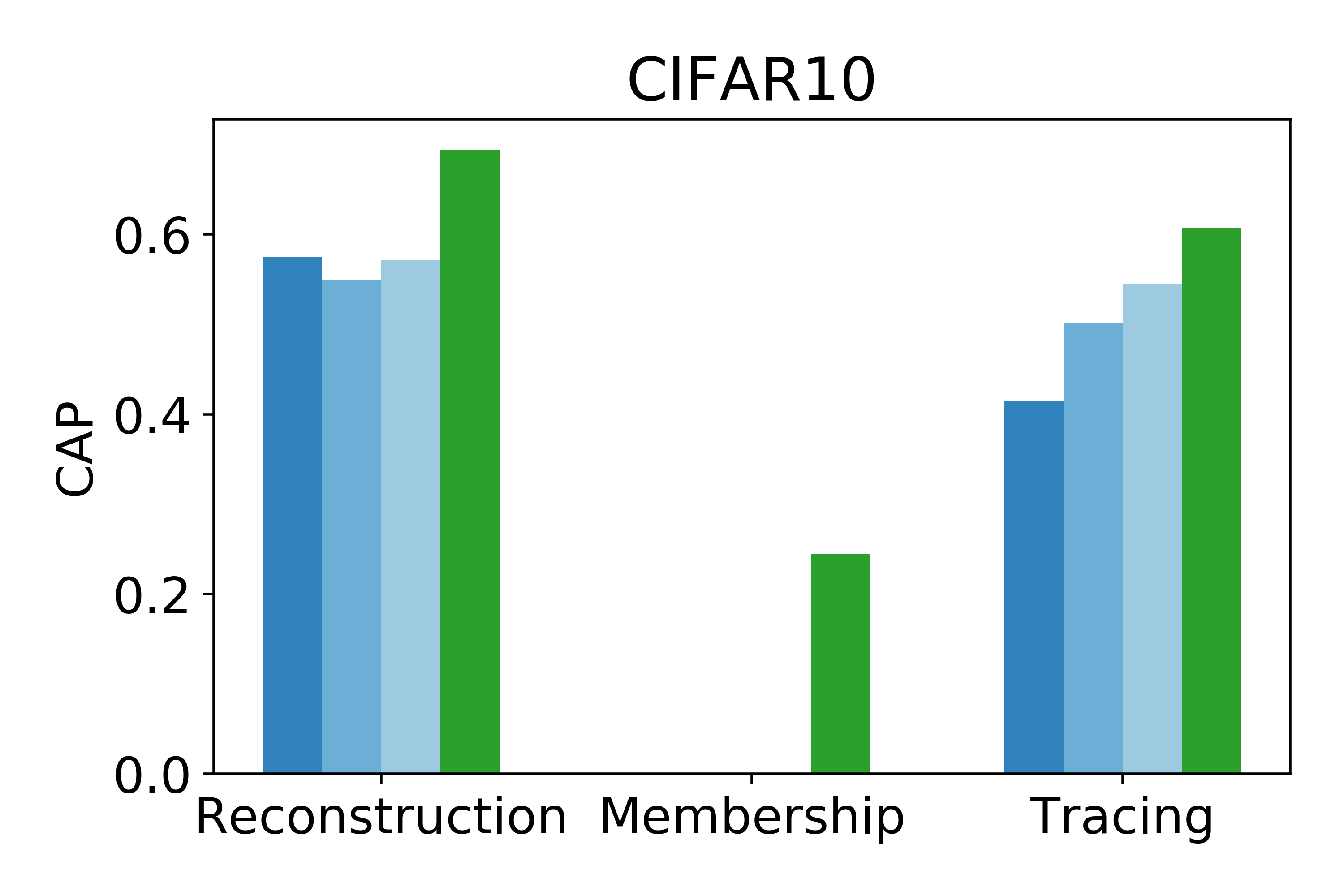}
			\includegraphics[width=0.49\linewidth,height=2cm]{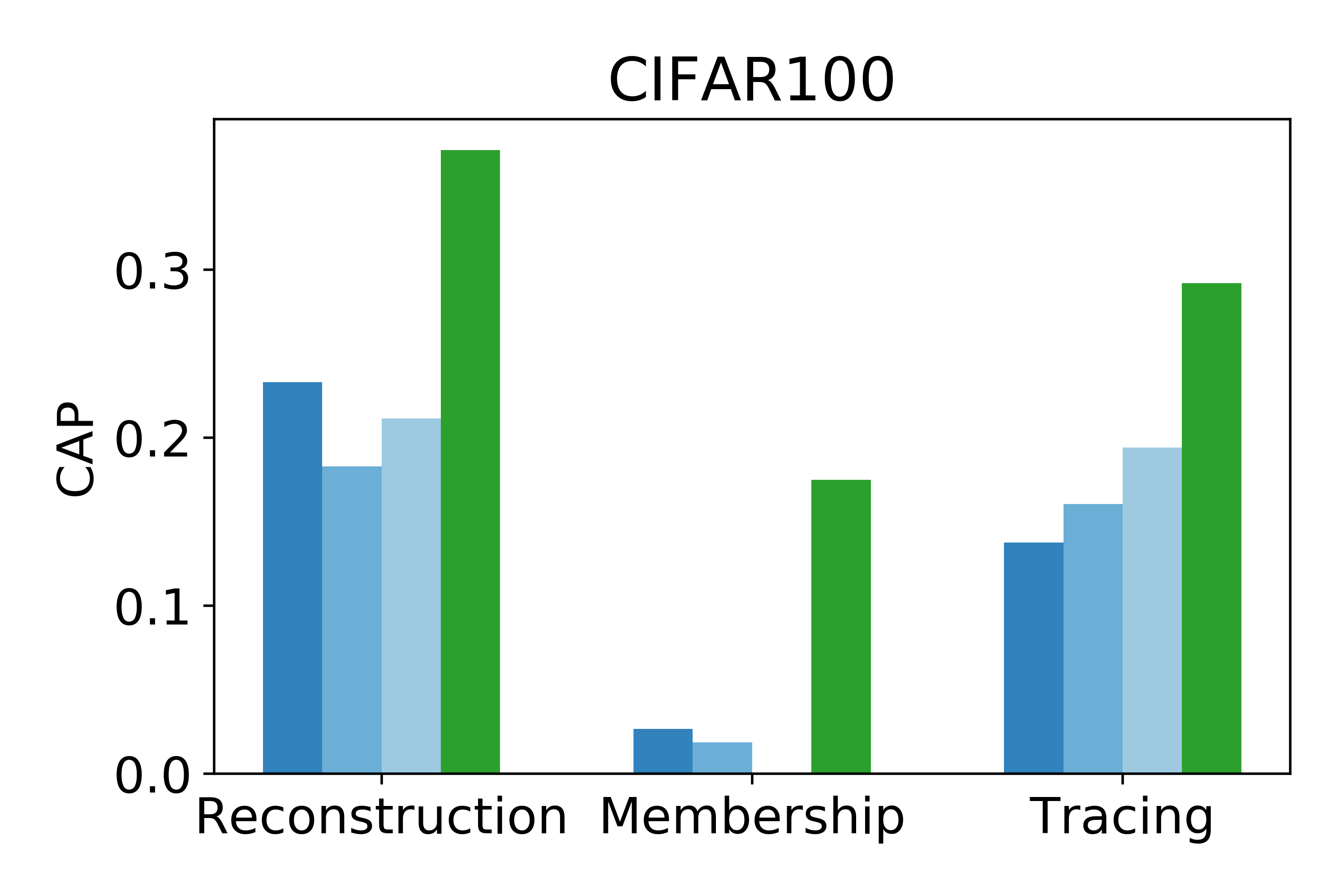}
			\caption{Attack Batch Size 1}
		\end{subfigure}
		\begin{subfigure}{0.48\linewidth}
			\includegraphics[width=0.49\linewidth,height=2cm]{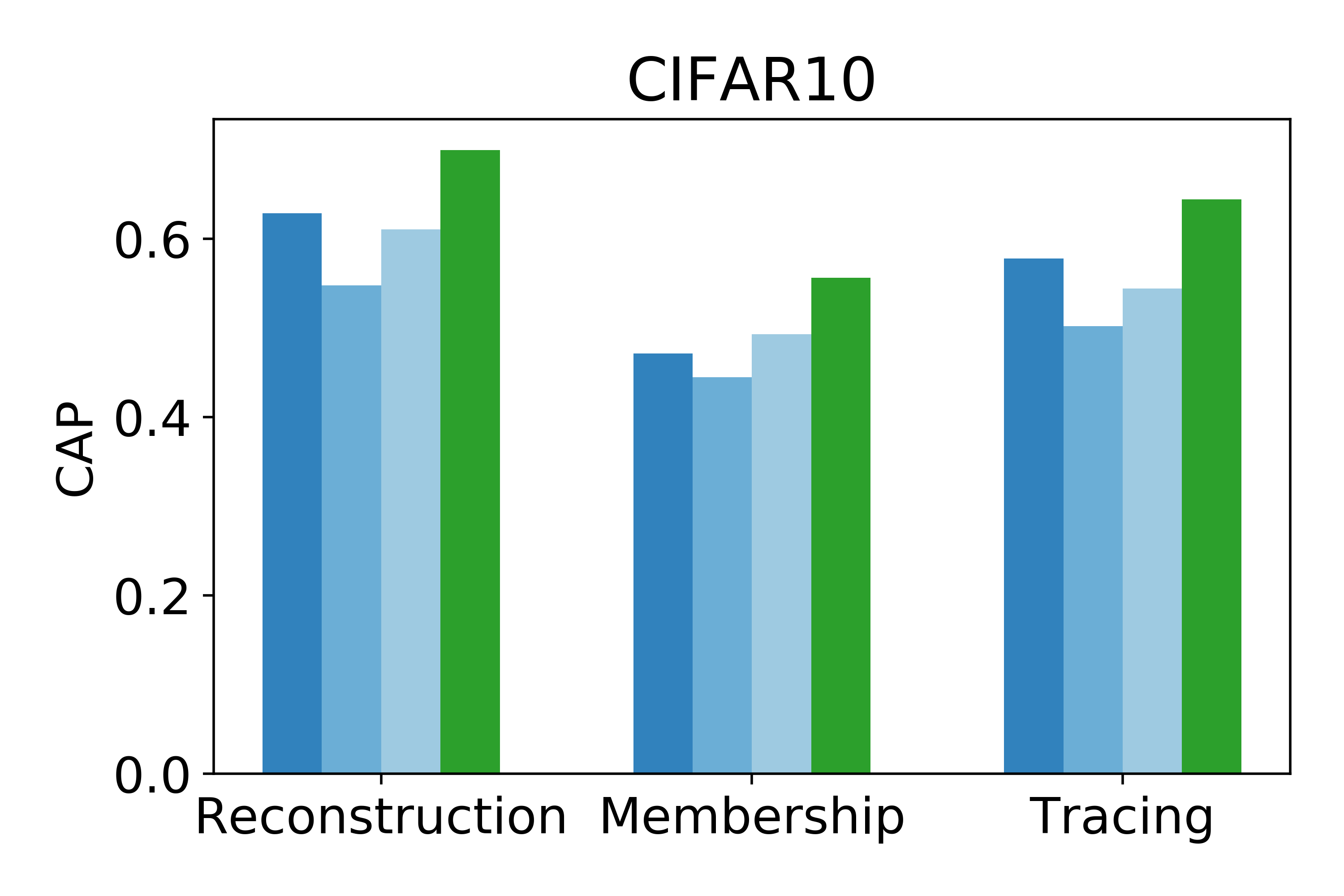}
			\includegraphics[width=0.49\linewidth,height=2cm]{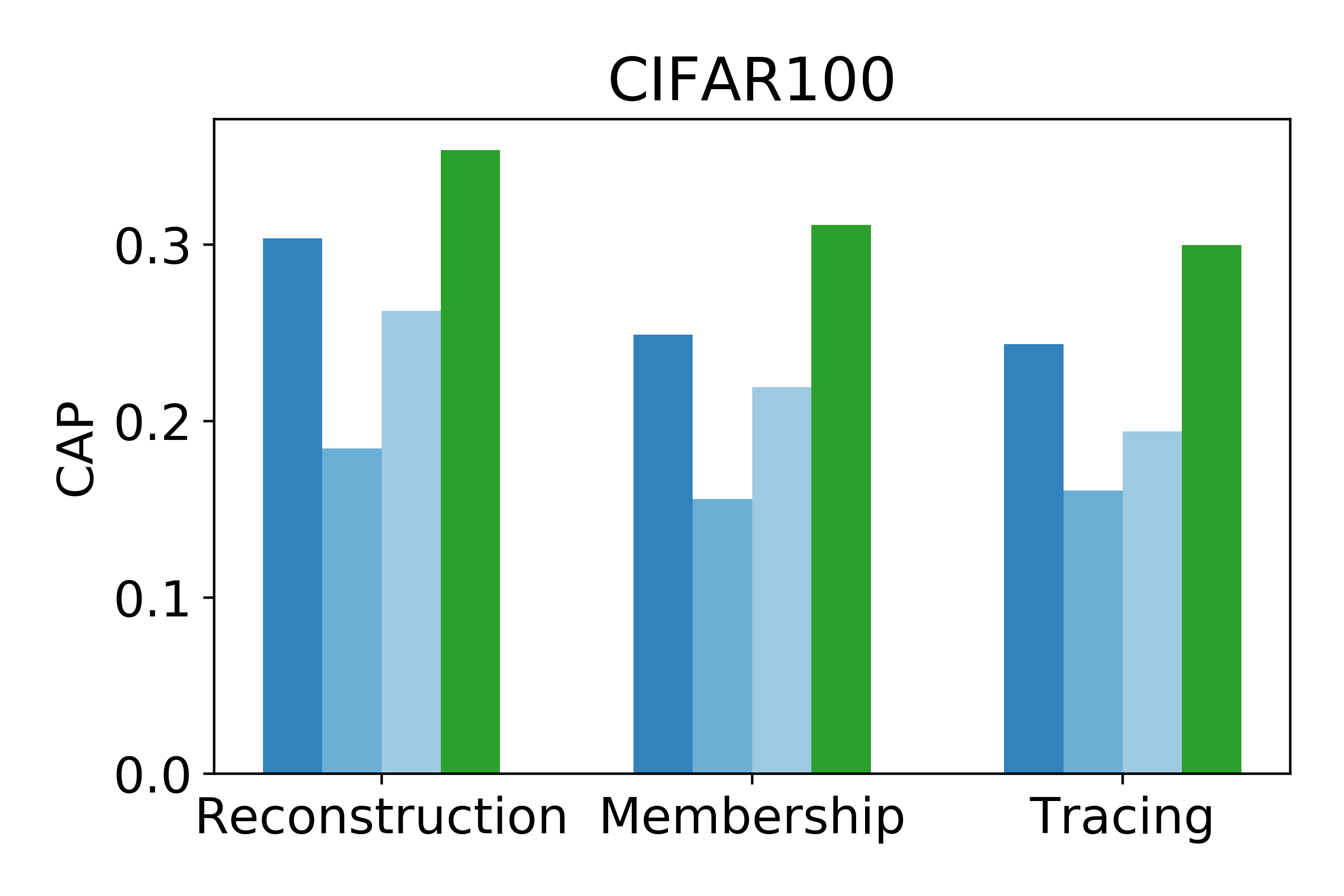}
			\caption{Attack Batch Size 8}
		\end{subfigure}
		
		%\end{minipage}
		\caption{Comparison of Calibrated %Maximum/
			Averaged Performances (%CMPs/
			CAPs) for the proposed SPN, PPDL \cite{PPDL/shokri2015} and DP \cite{DLDP_Abadi16} methods, against reconstruction, membership and tracing attacks (%CMP/
			CAP the higher the better, see threat model and evaluation protocol in Sect. \ref{subsect:protocol}). (a): CIFAR10/100 models attacked with batch size 1; (b): CIFAR10/100 models attacked with batch size 8.  \label{fig:barchart-summary}}
		%\caption{Comparison of SPN, PPDL and DP privacy-preserving mechanisms. \label{fig:barchart-summary}}
		\vspace{-0.2cm}
	\end{figure}
	
	\subsection{Related Work} \label{subsect:related}
	
	%We first review two recent privacy preserving deep learning (PPDL) methods:
	%can be broadly categorized into three types: % In terms of the underlying protection principles: 
	%first, 
	\cite{DLDP_Abadi16} demonstrated how to maintain data privacy by adding Gaussian noise to shared gradients during the training of deep neural networks.
	%second, 
	\cite{PPDL/shokri2015} proposed to randomly select and share a small fraction of gradient elements (those with large magnitudes) to reduce privacy loss. 
	Although both methods \cite{DLDP_Abadi16,PPDL/shokri2015} offered strong \textit{differential privacy} (DP) guarantees \cite{DP/dwork2006,Calibrating_Noise_DA/dwork2006}, as shown by \cite{exploitingFeatLeak_Vitaly18,DeepLeakage_Han19} and our empirical studies, 
	pixel-level reconstructions of training data and disclosing of membership information raise serious concerns about potential privacy loss. 
	%the theoretically justified privacy losses  do not necessarily correspond to desired levels of privacy protection, especially in the face of privacy attacks reviewed below. 
	
	%We therefore argue that one must employ objective evaluation of privacy breaches associated with drops in accuracy of trained models. 
	%the inevitable trade-off with the drop in accuracy made it difficult to maintain the desired level of privacy protection. 
	
	%no objective evaluation of privacy breaches were empirically evaluated. 

	%and offered an algorithmic sweet spot of the trade-off between model accuracy and privacy loss.  
	%The fundamental trade-off between privacy loss and utility is not critically analyzed; 
	
	Dwork et.al. \cite{attack/dwork2017} have formulated privacy attacks towards a database, as a series of queries maliciously chosen according to an attack strategy designed to compromise privacy.  Among three privacy attacks i.e. \textit{reconstruction},\textit{tracing} and \textit{re-identification} discussed in \cite{attack/dwork2017},  
	the detrimental reconstruction attack is formulated as solving a noisy system of linear equations,  
	and reconstruction errors are essentially bounded by the worst-case accuracies of query answers (Theorem 1 in \cite{attack/dwork2017}). However, this formulation is not directly applicable to deep learning, since queries about private training data are not explicitly answered during the training or inferencing of DNNs. 
	%in the distributed learning setting. 
	%Instead, the reconstruction of training data are made possible, by observing the shared \textit{gradients} during the learning stage. 
	%Two forms of privacy attacks i.e. \textit{reconstruction} and \textit{tracing} have been formulated from an adversary standpoint  in \cite{attack/dwork2017,Protection_private_ML_2approches/abadi2017}, 
	%with the relationship between reconstruction error and privacy losses being established (Theorem 1 in \cite{attack/dwork2017}).  
	
	In the context of deep learning, membership attacks was investigated in \cite{Membership_Inference/shokri2017} while \cite{Model_Inversion/fredrikson2015} demonstrated that recognizable face images can be recovered from confidence values revealed along with predictions. %for neural network models.
	\cite{exploitingFeatLeak_Vitaly18} demonstrated with both CNNs and RNNs that periodical gradient updates during training leaked information about \textit{training data, features} as well as class \textit{memberships}.  Possible defences such as selective gradient sharing, reducing dimensionality, and dropout were proved to be ineffective or had a negative impact on the quality of the collaboratively trained model.  Based on the assumption that activation functions are twice-differentiable, recent attacks were proposed to reconstruct training data with pixel-level accuracies %empirical evaluation and theoretical analysis presented in
	\cite{DeepLeakage_Han19,BeyondInferClass_Wang19,Gradient_Leakage/wei2020,inverting_gradient/geiping2020}. %Yet no explanations of the successfulness of reconstruction attacks against \cite{DLDP_Abadi16,PPDL/shokri2015} were provided. 
	%None of above existing work  \textit{quantitatively} investigate the trade-off between model quality (accuracy) vs privacy attacks.  To our best knowledge, the present paper is the first work that adopts privacy attacks to evaluate the privacy preserving capabilities of strategies proposed in  \cite{DLDP_Abadi16,PPDL/shokri2015,exploitingFeatLeak_Vitaly18,DeepLeakage_Han19}, with extensive experiments over different networks and datasets (see Section \ref{sect:exper}). 
	These recent reconstruction attacks were adopted in the present work to evaluate capabilities of privacy-preserving strategies proposed in  \cite{DLDP_Abadi16,PPDL/shokri2015,exploitingFeatLeak_Vitaly18,DeepLeakage_Han19}, with extensive experiments conducted over different networks and datasets (see Sect. \ref{sect:exper} and supplementary material).
	
	\begin{figure}[t]	
		%	\begin{subfigure}{0.95\linewidth}
		\centering
		\hspace*{\fill}
		\begin{subfigure}[t]{0.4\linewidth}
			\centering
			\includegraphics[width=0.99\linewidth]{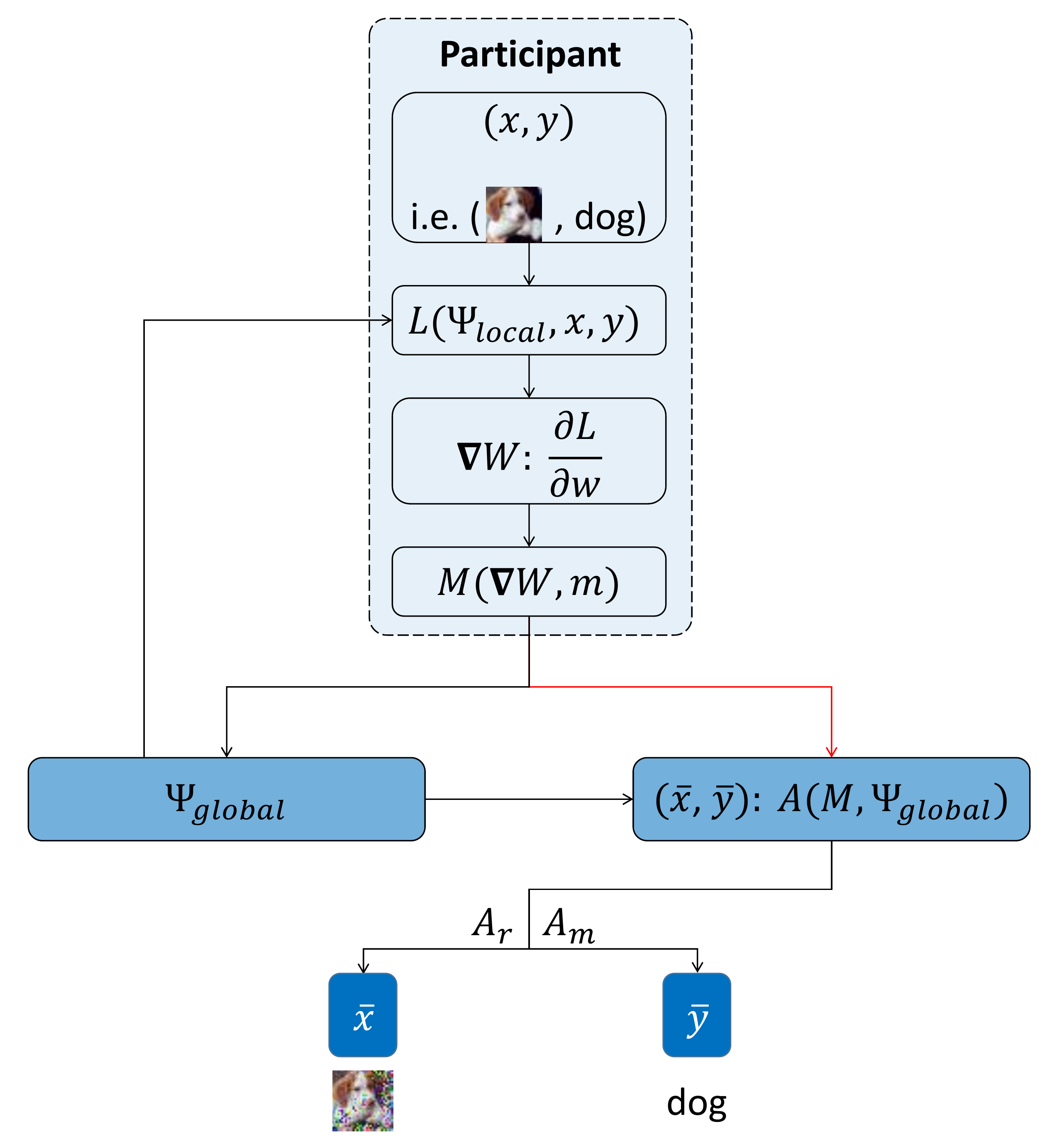}
			\caption{\textbf{Reconstruction attacks $\mathcal{A}_r$}, with relative MSE between reconstructed and original data $\frac{\|\bar{x}-x\|}{\|x\|}$.
				\textbf{Membership attacks $\mathcal{A}_m$}, with categorical distance between reconstructed and original labels $dist_m( \bar{y}, y)$.}
		\end{subfigure}
		\hfill
		\begin{subfigure}[t]{0.4\linewidth}
			\centering
			\includegraphics[width=0.99\linewidth]{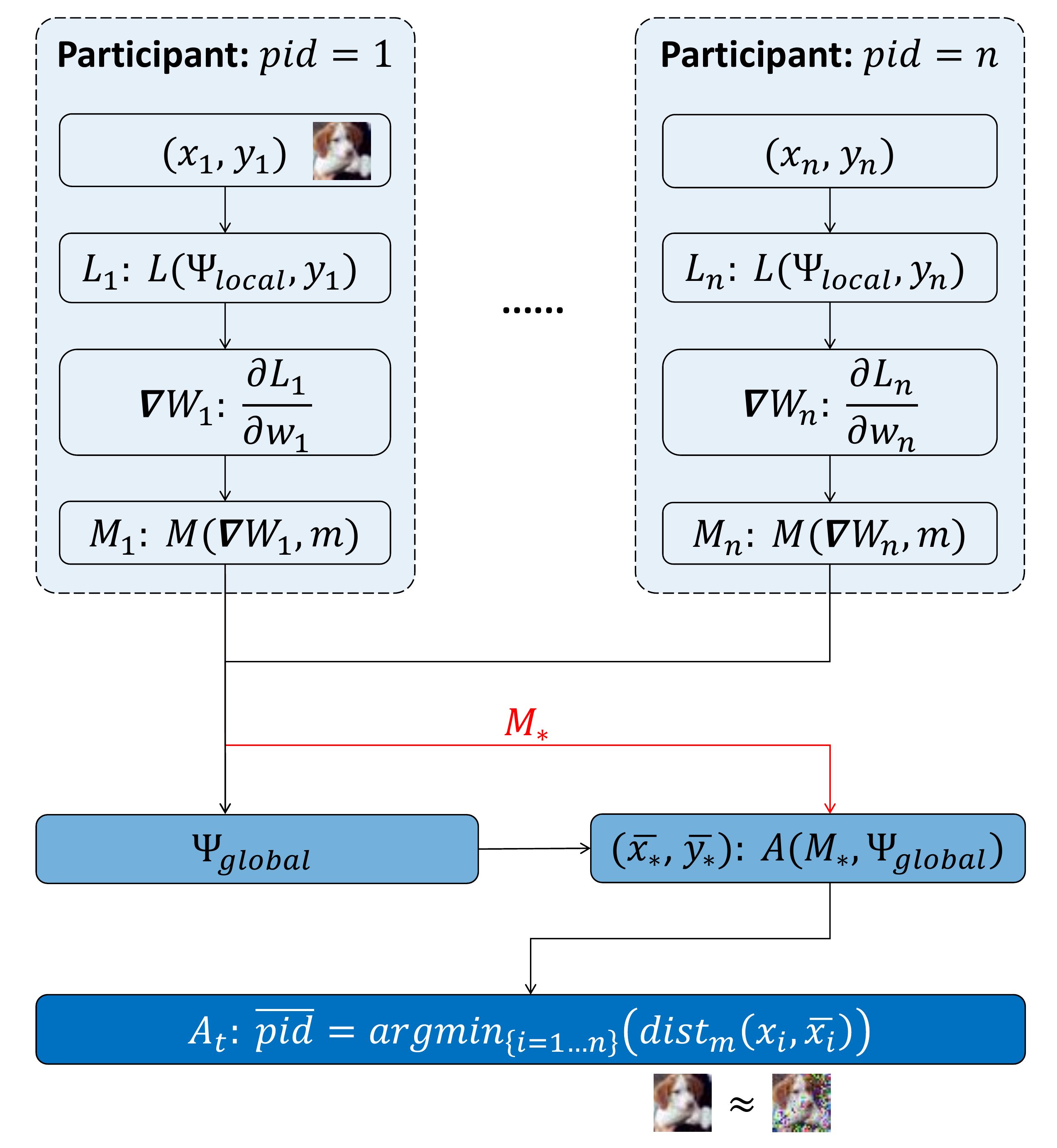}
			\caption{\textbf{Tracing attacks $\mathcal{A}_t$}, with categorical distance between recovered and actual participant IDs  $dist_m( \overline{pid}, p)$.}
		\end{subfigure}
		\hspace*{\fill}
		%\hfill
		%\includegraphics[scale=0.1]{imgs/gradient_0.0001_DPN.png}
		%	\end{subfigure}	
		\caption{Three privacy attacks considered in this work %(a) Reconstruction and Membership attacks. (b) Tracing attacks. 
			(see text in Sect. \ref{sect:PrivacyAttack}). }
		\label{fig:threat-model}
		\vspace{-10pt}
	\end{figure}
	
	Homomorphic-Encryption (HE) based \cite{Cryptonets/gilad-bachrach2016,hardy2017private,PPDL_Homomorphic_Encryption/trieu2018} and Secure Multi-Party Computation (MPC) based privacy-preserving approaches  \cite{Deepsecure/rouhani2018,ABY3/mohassel2018} 
	demonstrated strong privacy protection via encryption, but often incur significantly more demanding computational and communication costs. For instance, \cite{PPDL_Homomorphic_Encryption/trieu2018} reported 2-3 times communication overheads and \cite{ABY3/mohassel2018,HEwGPU19} had to speed up highly-intensive computation with efficient implementations. In this paper our work is only compared with Differential Privacy based mechanisms \cite{PPDL/shokri2015,DLDP_Abadi16}, and we refer readers to \cite{yang2019federated,Survey_PPDL/tanuwidjaja2019} for thorough reviews of HE and MPC based privacy-preserving methods therein. %Comparisons of HE/SMPC based methods \cite{Cryptonets/gilad-bachrach2016,hardy2017private}, with respect to differential privacy based approaches \cite{DLDP_Abadi16,PPDL/shokri2015} and the privacy-preserving method proposed in the present paper are discussed in Section \ref{sect:dis-concl}. 
	
	\section{Privacy Attacks on Training Data} \label{sect:PrivacyAttack}
	
	In this work we consider a distributed learning scenario, in which $K (K\ge 2)$ participants collaboratively learn a multi-layered deep learning model without exposing their private training data (this setting is also known as \textit{federated learning} \cite{communicationEfficient/mcMahan2017,yang2019federated}). We assume one participant is the  \textit{honest-but-curious adversary}. %, whose goal is to infer information about the training data of another participant, by analyzing periodic updates to the joint model during training. 
	The adversary is honest in the sense that he/she faithfully follows the collaborative learning protocol and does not submit any malformed messages, but he/she may launch privacy attacks on the \textit{training data} of other participants, by analyzing periodic updates to the joint model (e.g. \textit{gradients}) during training. 
	
	Fig. \ref{fig:threat-model} illustrates three privacy attacks considered in this work. 
	The goal of \textit{reconstruction attack} is to recover original training data $x$ as accurate as possible by analyzing the publicly shared gradients, which might be perturbed by privacy-preserving mechanisms. 
	Subsequent \textit{membership attack} and  \textit{tracing  attack} are  based on reconstruction attacks --- for the former, membership labels are derived either directly during the reconstruction stage or by classifying reconstructed data;  for the latter, the goal is to determine whether a given training data item belongs to certain participant, by comparing it against reconstructed data\footnote{Note that membership inference in \cite{exploitingFeatLeak_Vitaly18} is the tracing attack considered in our work.}.

	%Fig. Threat model.  (Liu Chang) 
	% \begin{figure}
	% 	\centering
	% 	\begin{minipage}[t]{0.45\textwidth}
	% 		\centering
	% 		\includegraphics[width=\textwidth]{imgs/threat_model_new.pdf}
	% 		\caption{Reconstruction $\mathcal{A}_r$ and membership attacks $\mathcal{A}_m$.}
	% 		\label{fig:threat-model}
	% 	\end{minipage}
	% 	%\hspace{0.1\textwidth}
	% 	\begin{minipage}[t]{0.4\textwidth}
	% 		\centering
	% 		\includegraphics[width=\textwidth]{imgs/tracing_attack_new.pdf}
	% 		\caption{Tracing attack  $\mathcal{A}_t$ .} 
	% 		\label{fig:tracing-attack}
	% 	\end{minipage}
	% \end{figure}

	\subsection{Evaluation of Trade-off by Privacy Preserving Mechanism}\label{subsect:protocol}
	
	We assume there is a Privacy-Preserving Mechanism (PPM)\footnote{We do not restrict ourselves to privacy mechanisms considered by differential privacy\cite{DP/dwork2006,Calibrating_Noise_DA/dwork2006,PPDL/shokri2015,DLDP_Abadi16}.} $\mathcal{M}$ that aims to defeat the privacy attacks $\mathcal{A}$ by modifying the public information from $\mathcal{G}$ to $\bar{\mathcal{G}}_m=\mathcal{M}(\mathcal{G}, m)$, that is exchanged during the learning stage and $m$ is the controlling parameter of the amount of changes exerted on $\mathcal{G}$. This modification protects the private information $x$ from being disclosed to the adversary, who can only make an estimation based on public information i.e. $\bar{x}_m=\mathcal{A}(\bar{\mathcal{G}}_m)$. Needless to say, a PPM can defeat any adversaries by introducing exorbitant modification so that \textit{ $dist( \bar{x}, x)$ is as large as possible}, where $dist()$ is a properly defined distance measure such as MSE. 
	%\footnote{ inversely,  structure similarity (SSIM) for image data.}.
	The modification of public information, however, inevitably deteriorates the performances of global models i.e. $Acc(\bar{\mathcal{G}}_m) \le Acc({\mathcal{G}}_m)$, where $Acc()$ denotes model performances such as accuracies or any other metrics that is relevant to the model task in question. 
	A well-designed PPM is expected to have \textit{$Acc(\bar{\mathcal{G}}_m)$ as high as possible}. 
	
	We propose to plot Privacy Preserving Characteristic (PPC) to illustrate the trade-off between two opposing goals i.e. to maintain high model accuracies and low privacy losses as follows, 
	\begin{defn}[Privacy Preserving Characteristic]\label{def:PPC}
		For a given Privacy-Preserving Mechanism $\mathcal{M}$, its privacy loss and performance trade-off is delineated by a set of calibrated performances i.e. $\{ Acc(\bar{\mathcal{G}}_m) \cdot dist( \bar{x}_m, x) | m \in \{m_1, \cdots, m_n\} \}$, where $Acc()$ is the model performance, $dist()$ a distance measure, $\bar{\mathcal{G}}_m=\mathcal{M}(\mathcal{G}, m)$ is the modified public information, $x$ is the private data, $\bar{x}_m=\mathcal{A}(\bar{\mathcal{G}}_m)$ is the estimation of private data by the attack and $m$ the controlling parameter of the mechanism.  
		
		Moreover, % the Calibrated Maximum Performance (CMP) and 
		Calibrated Averaged Performance (CAP) for a given PPC is defined as follows, 
		\begin{align}
		% CMP = \max_{m=m_1,\cdot,m_n} \big( Acc(\bar{\mathcal{G}}_{m}) \cdot dist( \bar{x}_{m}, x)\big),  
		CAP(\mathcal{M}, \mathcal{A}) = \frac{1}{n} \sum_{m=m_1}^{m_n} Acc(\bar{\mathcal{G}}_{m}) \cdot dist( \bar{x}_{m}, x).
		\end{align}
	\end{defn}
	Fig. \ref{fig:ppc-all} illustrates example PPCs of different mechanisms against privacy attacks. 
	%, in which the controlling parameter $m$ depends on the magnitudes of gradient perturbations being added to the original gradient components.  
	One may also quantitatively summarize PPCs with CAP ---  \textit{
		the higher the CAP value is, the better the mechanism is at preserving privacy without compromising the model performances} (see Table \ref{tab:CAP-performance}).   
	%We adopt both PPC and CAP to compare different mechanisms in terms of capabilities to defense  privacy attacks (see Figure \ref{fig:ppc-all}  and Table \ref{tab:CAP-performance} in Sect. \ref{sect:exper}).  
	%summarizes CAPs of different privacy-preserving mechanisms. 
	
	%\textbf{Accuracy vs reconstruction loss}: we adopt Deep Leakage algorithm \cite{DeepLeakage_Han19} to reconstruct training data at various stage of the training epochs, and  
	%
	%\textbf{Accuracy vs membership loss}: 
	%
	%\textbf{Accuracy vs tracing loss}: 

	\subsection{Formulation of Reconstruction Attack} \label{subsect:recon-formu}
	%Inspired by \cite{attack/dwork2017,Revealing_info_PP/dinur2003}, 
	%We give below a rigorous formulation that shows how to treat reconstruction of training data from shared (partial) gradients 
	%in deep neural networks as to \textit{solve a noisy system of equations}, in which equation parameters are determined by network parameters and shared gradients, and training data are unknown variables to be solved. 
	%This formulation admits the explanations of the successful reconstruction attacks \cite{DeepLeakage_Han19,Gradient_Leakage/wei2020,inverting_gradient/geiping2020} against Privacy-preserving deep learning methods \cite{DLDP_Abadi16,PPDL/shokri2015}.  
	Consider a neural network $\Psi(x;w, b): \mathcal{X} \rightarrow \mathbb{R}^C$, where $x \in \mathcal{X}$, $w$ and $b$ are the weights and biases of neural networks, and $C$ is the output dimension. In a machine learning task, we optimize the parameters $w$ and $b$ of neural network $\Psi$ with a loss function $\mathcal{L}\big(\Psi(x; w, b), y \big)$, where $x$ is the input data and $y$ is the ground truth labels. We denote the superscript $w^{[i]}$ and $b^{[i]}$ as the $i$-th layer weights and biases. 
	%We abbreviate the loss function by $\mathcal{L}$ in the following paragraphs. %We abbreviate the neural network and the loss function by $\Psi$ and $\mathcal{L}$ in the following paragraphs. 
	The following theorem proves that the reconstruction of input $x$ exists under certain conditions (proofs are given in Appendix A, in supplementary material due to the limited space). 
	
	\begin{restatable}[]{thm}{reconexist}
		%\label{thm:recon-exist_1}
		\begin{theorem}\label{thm:recon-exist_1}
			Suppose a multilayer neural network $\Psi:=\Psi^{[L-1]}\circ \Psi^{[L-2]}\circ \dots \circ \Psi^{[0]}(\hspace*{0.2em} \cdot \hspace*{0.3em}; w, b)$ is $\mathcal{C}^1$, where the $i$-th layer $\Psi^{[i]}$ is a fully-connected layer\footnote{Any convolution layers can be converted into a fully-connected layer by simply stacking together spatially shifted convolution kernels (see proofs in supplementary material).} %The $i$-th step forward propagation is $o^{[i+1]} = a\big(w^{[i]}\cdot o^{[i]} + b^{[i]}\big),$ where $o^{[i]}, o^{[i+1]}$, $w^{[i]}$ and $b^{[i]}$ are an input, output vector, a weight matrix and a bias vector in each layer respectively, and $a$ is an activation function.
			Then, \textbf{initial input $x^*$ of $\Psi$ exists}, provided that: if there is an $i$ $(1\leq i \leq L)$ such that
			\begin{enumerate}
				\item Jacobian matrix $D_{x} \big(\Psi^{[i-1]}\circ \Psi^{[i-1]}\circ \dots \circ \Psi^{[0]}\big)$ around $x$ is full-rank; 
				\item Partial derivative $\nabla_{b^{[i]}} \mathcal{L}\big(\Psi(x; w, b), y \big)$\footnote{We write the partial derivative as a diagonal matrix that each adjacent diagonal entries in an order are copies of each entry in $\nabla_{b^{[i]}} \mathcal{L}\big(\Psi(x; w, b), y \big)$, see proofs in Appendix for details.} is nonsingular.		
			\end{enumerate}
			%		Moreover, we have the following inequality around $x^*$, 
			%		\begin{align*}
			%		||x-x^*|| \leq M\cdot || \nabla_{w^{[i]}, b^{[i]}} \mathcal{L}\big(\Psi(x; w, b), y \big)-\nabla_{w^{[i]}, b^{[i]}} \mathcal{L}\big(\Psi(x^*; w, b), y \big) ||.
			%		\end{align*}
		\end{theorem}
	\end{restatable}
	If assumptions in Theorem \ref{thm:recon-exist_1} are met, we can pick an index set $I$ from row index set of $\nabla_{w^{[i]}, b^{[i]}} \mathcal{L}\big(\Psi(x; w, b), y \big)$ such that the following linear equation is well-posed,
	\begin{align*}
	{B}_I \cdot x = {W}_I, 
	\end{align*}
	where ${B}_I := \nabla_{b^{[i]}}^I \mathcal{L}\big(\Psi(x; w, b), y \big)$ and ${W}_I := \nabla_{w^{[i]}}^I \mathcal{L}\big(\Psi(x; w, b), y \big)$. According to Theorem \ref{thm:recon-exist_1}, the initial input $x^*$ is $\big(\Psi^{[i-1]}\circ \Psi^{[i-1]}\circ \dots \circ \Psi^{[0]}\big)^{-1}(x)$. 
	
	%Our theorem provides sufficient conditions of the initial image existence. 
	The linear system can be composed from any subsets of observed gradients elements, and the reconstruction solution exists as long as the condition of full rank matrix is fulfilled. For common privacy-preserving strategies adopted in a distributed learning scenario such as \textit{sharing fewer gradients} or\textit{ adding noisy to shared gradients} \cite{PPDL/shokri2015,DLDP_Abadi16,exploitingFeatLeak_Vitaly18}, the following theorem proves that input $x$ can be reconstructed from such a noisy linear system, if condition (\ref{eq:recon-cond}) is fulfilled. 
	\begin{restatable}[]{thm}{reconlineareq}
		\label{thm:recon-lineareq}
		\begin{theorem}
			Suppose there are perturbations $E_{{B}}, E_{{W}}$ added on ${B}_I, {W}_I$, respectively, such that observed measurements $\bar{{B}}_I =  {B}_I + E_{{B}}, \bar{{W}}_I =  {W}_I + E_{{W}} $. Then, the \textbf{reconstruction $x^*$ of the initial input $x$ can be} \textbf{determined} by solving a noisy linear system $\bar{B}_I \cdot x^* = \bar{W}_I$, provided that
			\begin{align}\label{eq:recon-cond}	
			\|{B}_{I}^{-1} \cdot E_B\|<1;
			\end{align}
			Moreover, the relative error is bounded, 
			\begin{align}
			\frac{\|x^*-x\|}{\|x\|} \leq \frac{\kappa(B_I)}{1- \|{B}_I^{-1} \cdot E_B\|} \Big( \frac{\|E_B\|}{\|B_I\|} + \frac{\|E_W\|}{\|W_I\|} \Big),
			\end{align}
			in which $B_I^{-1}$ is the inverse of $B_I$, where $\kappa(B_I)$ is the conditional number of $B_I$. 			%in which $B_I^{-1}$ is the inverse of $B_I$.
			%\begin{enumerate}
			%	\item Model parameters $w$ are known; 
			%	\item $||\bar{A}_I^{-1}\cdot E||<1$.
			%\end{enumerate}
		\end{theorem}
	\end{restatable}
	
	%\begin{remark}
	%\textbf{Remark}: 
	%1) If the $i$-th layer refers to the $L$-th layer, Theorem \ref{T1.3} theoretically confirms the deep leakage method proposed in \cite{DeepLeakage_Han19} that minimization of gradient differences yields a recovery of initial image. 	
	%1) The linear system can be composed from any subsets of observed gradients elements, and the reconstruction solution exists as long as the condition of full rank matrix is fulfilled.  %As demonstrated by experiment results (see Section \ref{sect:exper}),  this requirement is often fulfilled in practice and reconstruction attacks success even though a minor fraction of gradients are shared as suggested by \cite{PPDL/shokri2015}. 
	
	In the deep leakage approach~\cite{DeepLeakage_Han19}, the recovery of initial image requires model parameters $\mathcal{W}$ and the corresponding gradients $\nabla \mathcal{W}$ such that a minimization of gradient differences $E_p := ||\nabla \mathcal{W}' - \nabla \mathcal{W}||$ yields a recovery $\bar{x}$ of initial image.  The minimizing error $E_p$ introduces more errors to the noisy linear system. Therefore, for any iterative reconstruction algorithms like \cite{DeepLeakage_Han19} to be successful,  condition $\|{B}_I^{-1} \cdot E_B\|<1$ is \emph{necessary}.  In other words, a sufficiently large perturbation $\| E_B \| > \| {B}_I \|$ such as Gaussian noise is \textit{guaranteed} to defeat reconstruction attacks. To our best knowledge, (\ref{eq:recon-cond}) is the first analysis that  elucidates a theoretical guarantee for thwarting reconstruction attacks like \cite{DeepLeakage_Han19}. Nevertheless, existing mechanisms \cite{PPDL/shokri2015,DLDP_Abadi16} have to put up with significant drops in model accuracy incurred by high levels of added noise (see Sect. \ref{subsect:exper-compare}).  
	%  yet at the cost of deteriorated model accuracies incurred by large perturbations (see experiment results in Sect. \ref{sect:exper}).  
	%Even if $||E|| < || \bar{A}_I ||$, iterative reconstruction algorithms may converge to a solution $o_k$ that is far from the original input data $o^*$, depending on the initial setting of  $o_0$ (see discussion in \cite{Gradient_Leakage/wei2020}). 
	
	%\end{remark}
	%approach. 

	\vspace{-0.23cm}
	\section{Privacy Preserving with Secret Polarization Network} \label{sect:PP-SPN}
	
	In Sect. \ref{sect:PrivacyAttack} we have proved that 
	%reconstruction errors from shared gradients are bounded provided that ?.  
	%In particular, the convergence condition in (?) suggests that 
	%reconstruction errors of training data are unbounded if 
	the necessary condition of successful reconstruction attack is unfulfilled if sufficiently large perturbations are added. 
	%It was also shown that adaptive perturbation should be added with the increasing learning steps (see Corollary \ref{coro}).  
	%thus effectively protecting training data from being leaked to privacy adversaries. 
	We illustrate in this section a novel multi-task dual-headed networks, which leverages private network parameters and element-wise adaptive gradient perturbations to defeat reconstruction attacks and, simultaneously, maintain high model accuracies. 
	
	\vspace{-0.23cm}
	\subsection{Secret Perturbation of Gradients via Polarization Loss}
	
	\begin{figure}[t]
		\centering
		\includegraphics[width=0.7\textwidth]{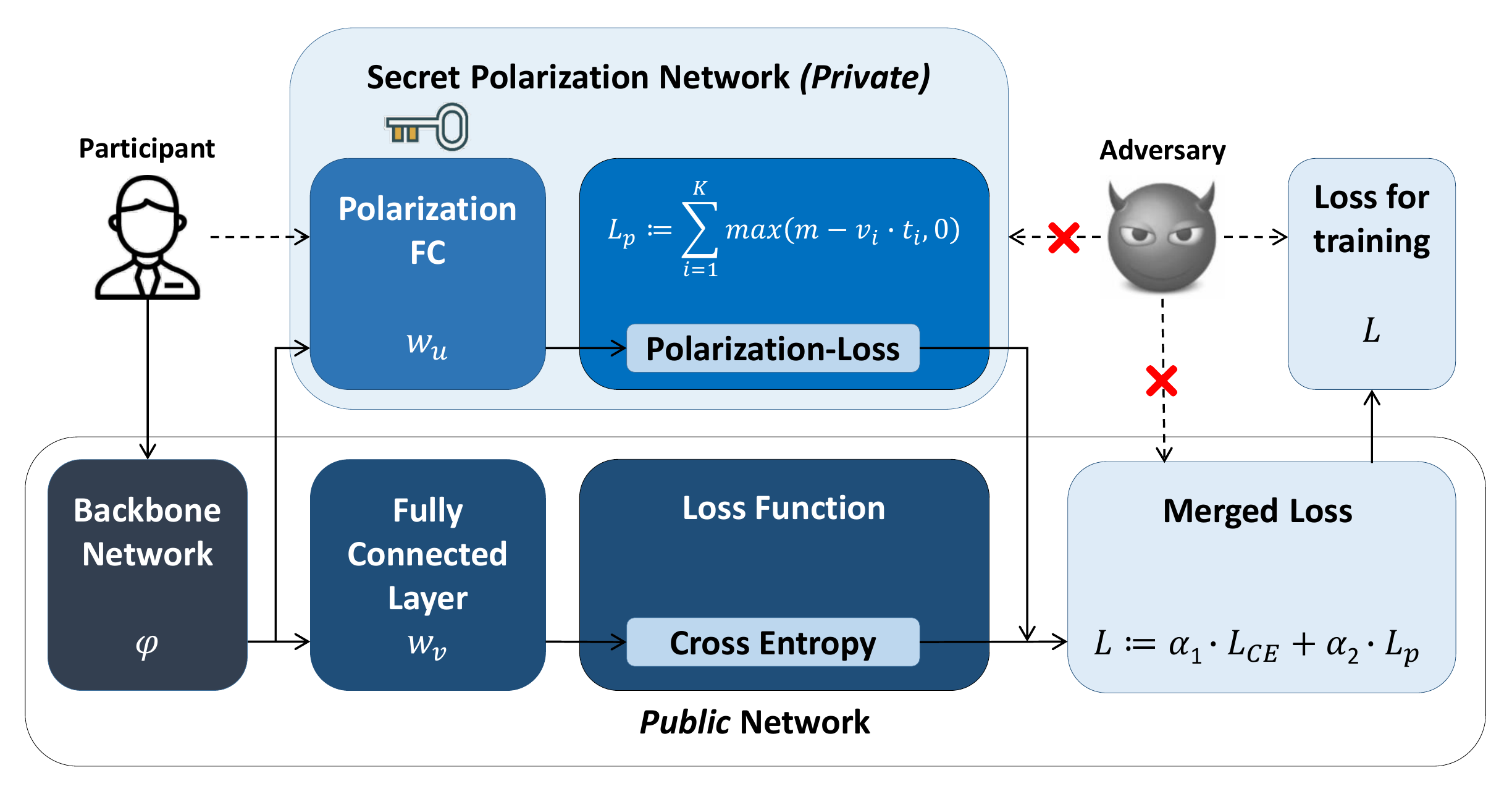}
		\caption{Our proposed SPN architecture that consists of a public and a private network (see text in Sect. \ref{sect:PP-SPN}).}
		\label{fig:SPN}
		\vspace{-10pt}
	\end{figure}

	Fig. \ref{fig:SPN} illustrates a Secret Polarization Network (SPN), 
	in which fully connected polarization layers are kept private with its \textit{parameters not shared} during the distributed learning process. Appendix shows the pseudo codes of the proposed method. 
	%Moreover, by introducing the following loss (eq. \ref{eq:loss-ce-pl}) on the outputs of polarization layers, additional gradients terms are backpropagated into the  public backbone network layers.
	%\footnote{It is worth noting that \cite{exploitingFeatLeak_Vitaly18} also adopted a multi-task learning scheme, yet for the purpose of active property inference attack.}. 

	Formally, the proposed dual-headed network consists of a public and a private SPN network based on a backbone network: $\Psi \big(\varphi(\hspace*{0.2em} \cdot \hspace*{0.3em};w, b) ;w_u, b_u \big) \oplus \Phi \big( \varphi(\hspace*{0.2em} \cdot \hspace*{0.3em};w, b) ;w_v, b_v \big): \mathcal{X} \rightarrow [0, 1]^C \oplus \mathbb{R}^K$, i.e. $ u \oplus v = \Psi \big( \varphi(x;w, b); w_u, b_u\big) \oplus \Phi \big( \varphi(x;w, b); w_v, b_v \big) \in [0, 1]^C \oplus \mathbb{R}^K$, where $\varphi(\hspace*{0.2em} \cdot \hspace*{0.3em};w, b)$ is the backbone network. The multi-task composite loss is as follows, 
	\begin{align} \label{eq:ce-spn-los} \small 
	\mathcal{L}\big(\Psi \oplus \Phi, y \oplus t \big):=&\alpha_1 \cdot \mathcal{L}_{CE}(u, y) +  \alpha_2\cdot \mathcal{L}_{P}(v, t) \\
	=& \underbrace{\alpha_1 \cdot \sum_{c=1}^C - y_c \cdot \log (u_c)}_{\text{CE loss}}  + \underbrace{\alpha_2 \cdot \sum_{c=1}^C \sum_{k=1}^K \max(m-v_k\cdot t_c^k, 0)}_{\text{polarization loss}}, 
	\end{align}
	where $\alpha_1$ and $\alpha_2$ are hyper-parameters with $\alpha_1+\alpha_2 = 1$. $y_c$ is an one-hot representation of labels for class $c$, 
	and $t_c \in \{-1, +1\}^K$ is the target $K$-bits binary codes randomly assigned to each class $c$ for $c = 1, \cdots, C$. 
	Note that by minimizing the polarization loss, Hamming distances between threshold-ed outputs $Bin(v_k)$ of \textit{intra-class} data items are minimized and, at the same time, Hamming distances are maximized for \textit{inter-class} data items (where $Bin(v_k) \in \{-1, +1\}$, see proofs in Appendix). The polarization loss therefore joints forces with the CE loss to improve the model accuracies.

	\begin{figure}[t]	
		\begin{subfigure}{0.95\linewidth}
			\centering
			\includegraphics[width=0.23\textwidth,height=0.09\textheight]{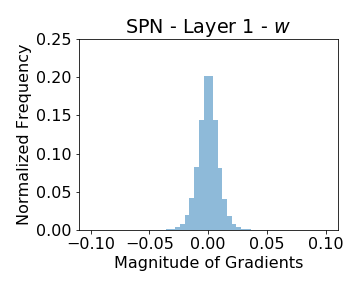}
			\includegraphics[width=0.23\textwidth,height=0.09\textheight]{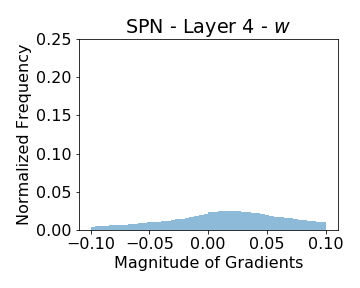}
			\includegraphics[width=0.23\textwidth,height=0.09\textheight]{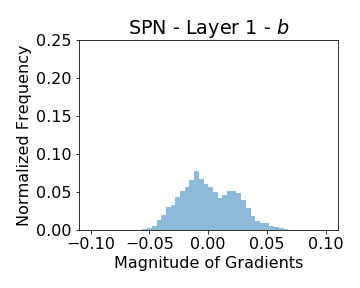}
			\includegraphics[width=0.23\textwidth,height=0.09\textheight]{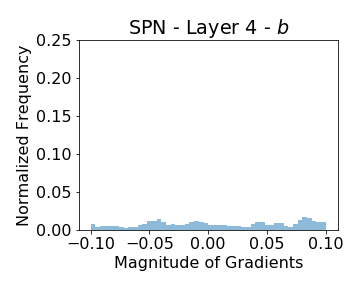}
		\end{subfigure}	
		\begin{subfigure}{0.95\linewidth}
			\centering
			\includegraphics[width=0.23\textwidth,height=0.09\textheight]{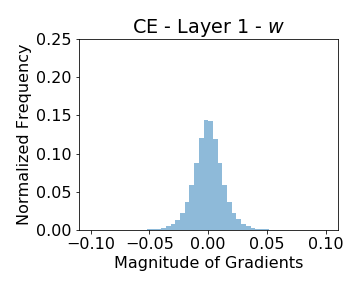}
			\includegraphics[width=0.23\textwidth,height=0.09\textheight]{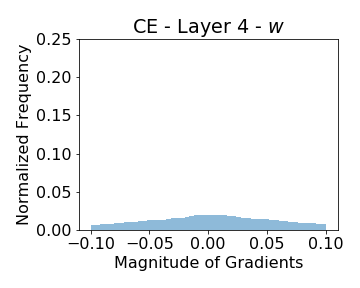}
			\includegraphics[width=0.23\textwidth,height=0.09\textheight]{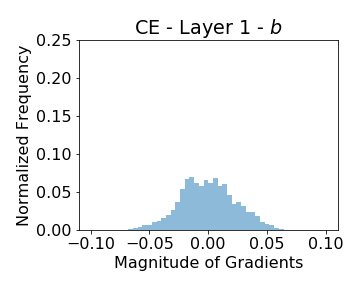}
			\includegraphics[width=0.23\textwidth,height=0.09\textheight]{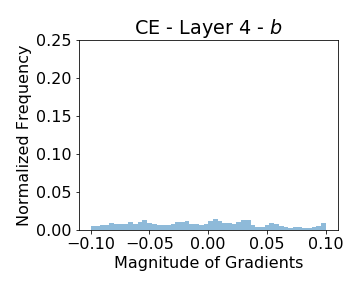}
		\end{subfigure}
		\caption{Distributions of gradients at each layer. \textbf{Left}: distributions of gradients w.r.t.  \textit{weights, w}; \textbf{Right}: distributions of gradients w.r.t. \textit{biases, b}; \textbf{Top}: gradients by polarization loss; \textbf{Bottom}: gradients by CE loss. 
			Cosine similarities  between gradients by polarization and CE losses are (from left to right): -0.0033, 0.1760, 0.0243, and 0.1861 respectively.  
			%Cosine similarities between gradients by CE loss and random Gaussian noise are (from left to right): -0.0010, -0.0004, 0.0243, and 0.0105 respectively.  
			\label{fig:grad-distr}}
		\vspace{-10pt}
	\end{figure}

	At each step of the optimization, the gradient of the loss $\bigtriangledown_{w, b} \mathcal{L}\big(\Psi \oplus \Phi, y \oplus t  \big)$ is a linear combination of gradient of CE loss and polarization loss as follows,
	\begin{align}\small
	%\bigtriangledown_{w_u, b_u} \mathcal{L} &=  \alpha_1 \cdot \sum_{c=1}^C (y_c-u_c) \cdot \frac{\partial u_c}{\partial w_u, b_u};   %\\
	%\bigtriangledown_{w_v, b_v} \mathcal{L} =   \alpha_2 \cdot  \sum_{c=1}^C \sum_{k\in \mathcal{I}^c} (-t_c^k) \cdot \frac{\partial v_k}{\partial w_v, b_v}\\
	\bigtriangledown_{w, b} \mathcal{L} 	&=  \alpha_1 \cdot \sum_{c=1}^C (y_c-u_c) \cdot \frac{\partial u_c}{\partial w, b} + \underbrace{\alpha_2 \cdot  \sum_{c=1}^C \sum_{k\in \mathcal{I}^c} (-t_c^k) \cdot \frac{\partial v_k}{\partial w, b}}_{\text{secret perturbation}},\label{eq:combined-grad}
	\end{align}
	%\underbrace{}_\textrm{secret perturbation},
	where $\mathcal{I}^c:=\Big\{k \in \{1, \cdots, K\} \Big| m-v_k\cdot t_c^k >0 \Big\}$.

	Note that $w_v$ %and $ \frac{\partial v_k}{\partial w_v, b_v}$ 
	is kept secret from other participants including the adversary. 
	%Even though the combined gradients  in (\ref{eq:combined-grad}) are shared publicly, 
	%<<<<<<< HEAD
	The summand due to the polarization loss in (\ref{eq:combined-grad}) is therefore unknown to the adversaries, and acts as perturbations to gradients ascribed to the CE loss.    Perturbations introduced by polarization loss, on the one hand, protect training data with $\alpha_{2}$ controlling the protection levels.  On the other hand, %unlike Gaussian or Laplacian noise, 
	SPN gradients back-propagated to the backbone network layers exhibit strong correlations with CE gradients (see distributions and cosine similarities between gradients by polarization and CE losses in Fig. \ref{fig:grad-distr}).  We ascribe improvements of the model accuracies brought by SPN to element-wise adaptive perturbations introduced by polarization loss. 

	\vspace{-0.2cm}
	\section{Experimental Results}\label{sect:exper}
	
	%We first illustrate below experiment settings, followed by comparison of the proposed SPN based method against existing privacy-preserving mechanisms, using Privacy-Preserving Characteristic  (PPC) defined in Sect. \ref{subsect:protocol}. Finally, improvements brought by SPNs in a federated learning setting are presented. 
	
	\vspace{-0.05cm}
	\subsection{Experiment Setup and Evaluation Metrics}
	\label{subsec:experimentsettings}
	
	\textbf{Dataset}. Popular image datasets MNIST and CIFAR10/100 are used in our experiments. %CIFAR100 with 100 classes is also used.  
	%\textbf{Network Architecture}. We use network architecture from \cite{DeepLeakage_Han19} for analysis and attacks. In federated learning setting, we used modified AlexNet and VGG16.
	Implementation of \textbf{DP} \cite{DLDP_Abadi16} method from Facebook Research Team \footnote{\url{https://github.com/facebookresearch/pytorch-dp}} 
	%where \textit{noise\_multiplier} is $m$, controlling parameters 
	is used. % and we disable the clipping feature.
	Implementation\footnote{\url{https://www.comp.nus.edu.sg/~reza/files/PPDL.zip}} of \textbf{PPDL} \cite{PPDL/shokri2015} method from Torch/Lua are re-implemented in PyTorch/Python.  PPDL is similar to gradient pruning which is one of the suggested protections in \cite{DeepLeakage_Han19}. We only show in this paper results with 5\% and 30\% of selected gradients, named respectively, as \textbf{PPDL-0.05} and \textbf{PPDL-0.3}. We refer reviewers to more results in the supplementary material. 
	Implementation of \textbf{Deep Leakage} attack \cite{DeepLeakage_Han19}, network architecture and default setting from the official released source code\footnote{\url{https://github.com/mit-han-lab/dlg}} are used in all experiments with training batch size set as $\{1, 4, 8\}$ respectively.  Following analysis in \cite{Gradient_Leakage/wei2020}, we adopt \textit{pattern-initialization} for higher reconstruction successful rates. % their network architecture for accuracy evaluation.

	\begin{figure}[t]	
		%	\begin{subfigure}{0.95\linewidth}
		\centering
		\includegraphics[scale=0.7]{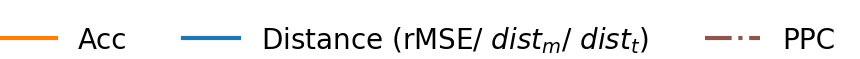}
		\\
		\includegraphics[width=3.2cm,height=2cm]{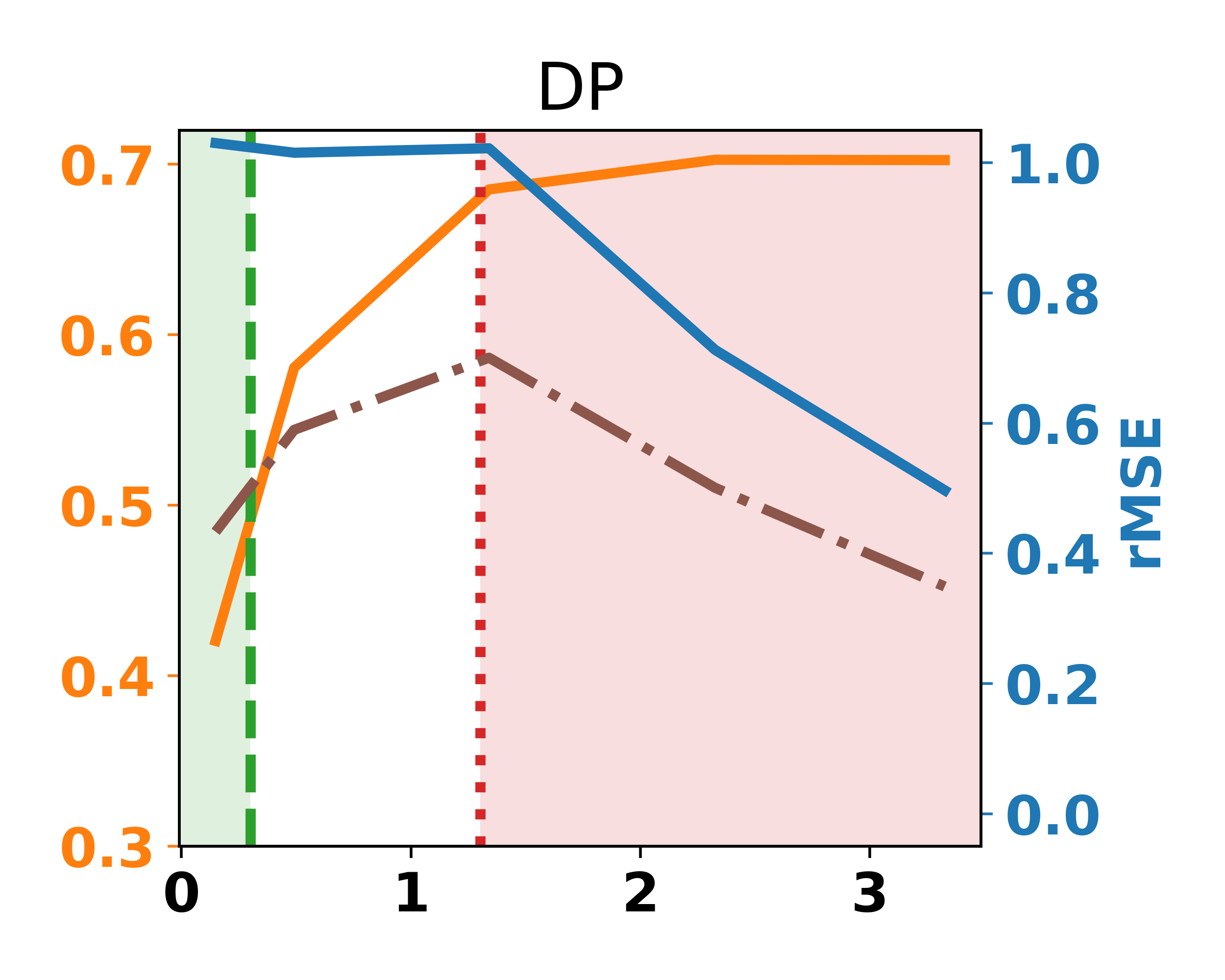}
		\includegraphics[width=3.2cm,height=2cm]{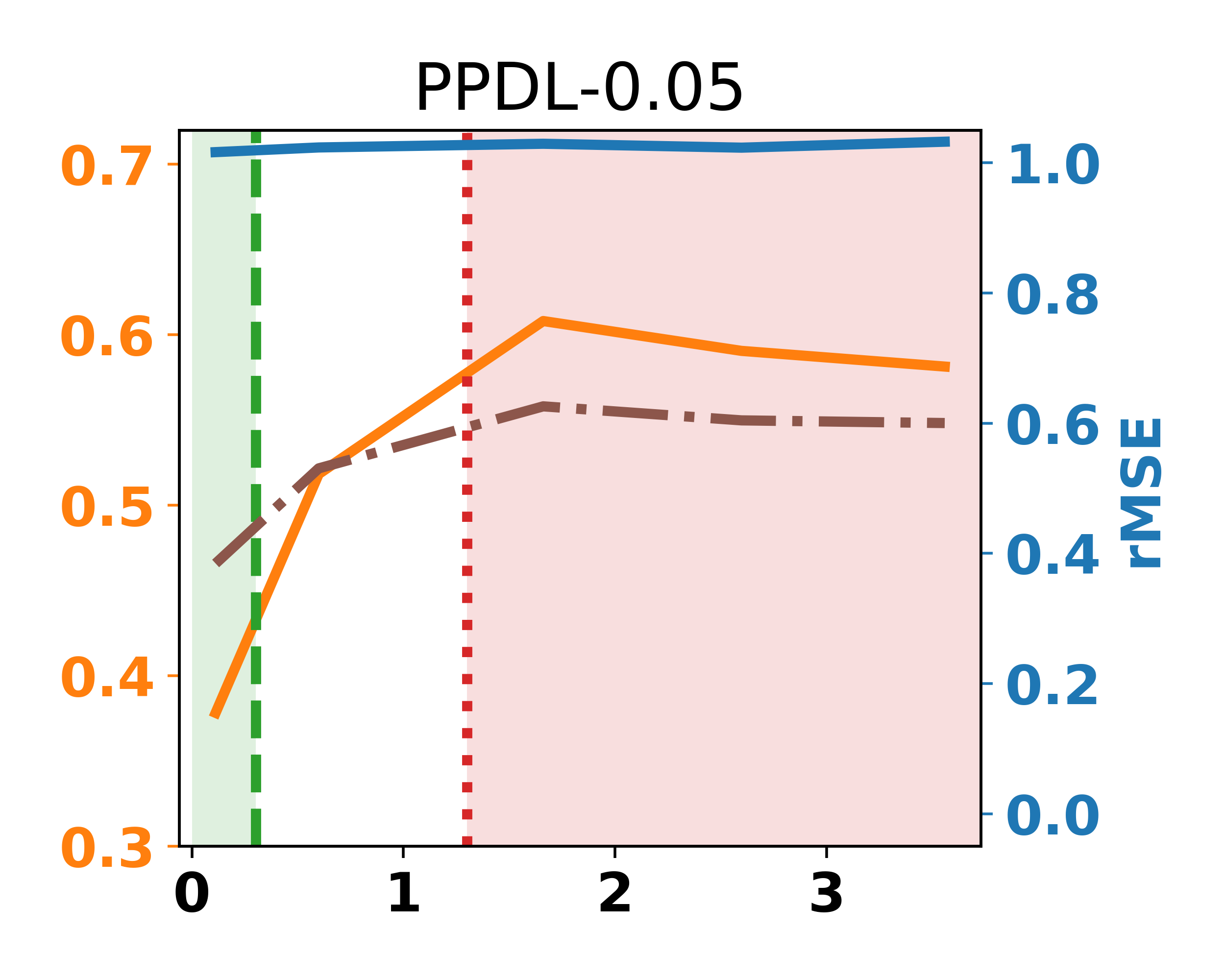}
		\includegraphics[width=3.2cm,height=2cm]{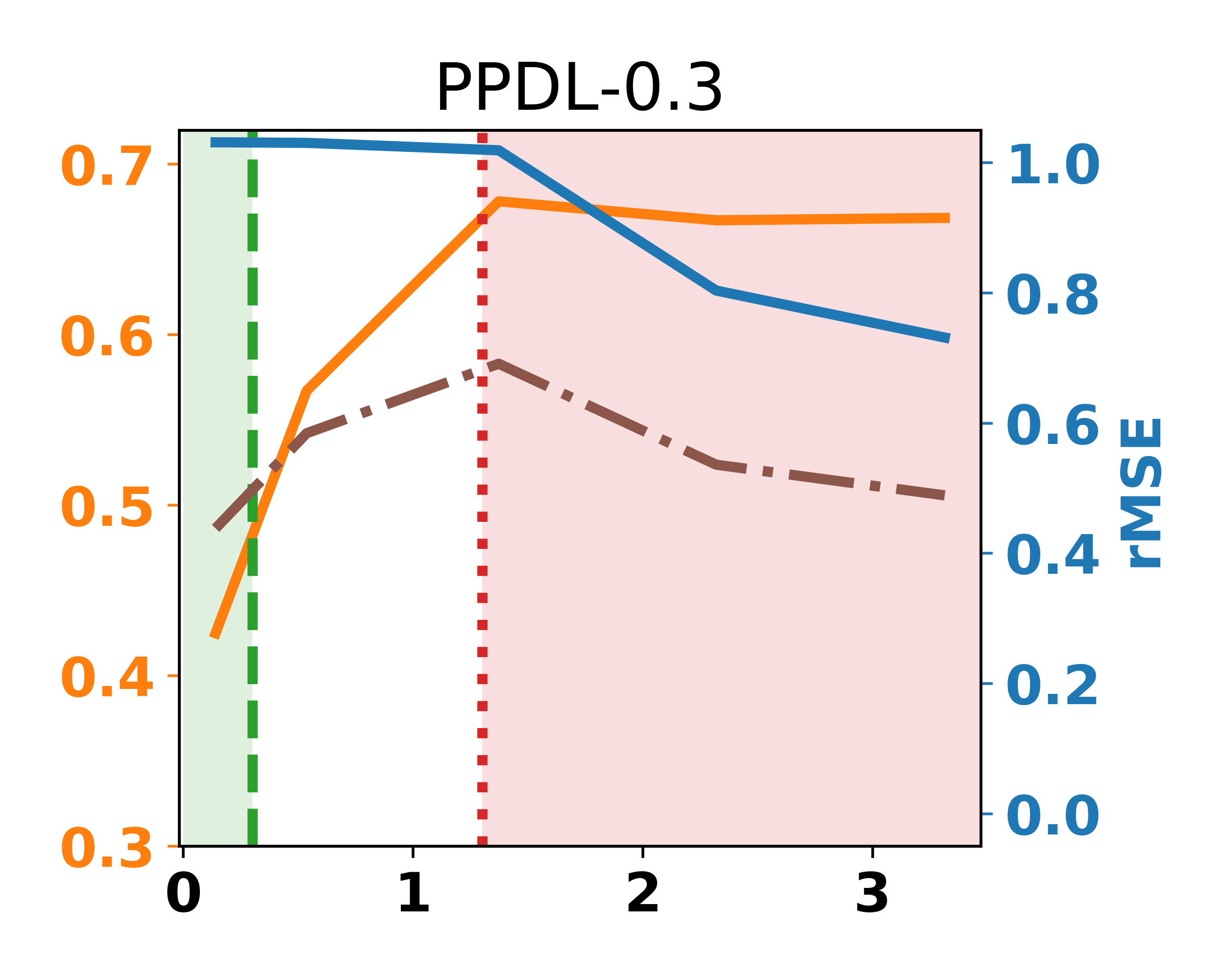}
		\includegraphics[width=3.2cm,height=2cm]{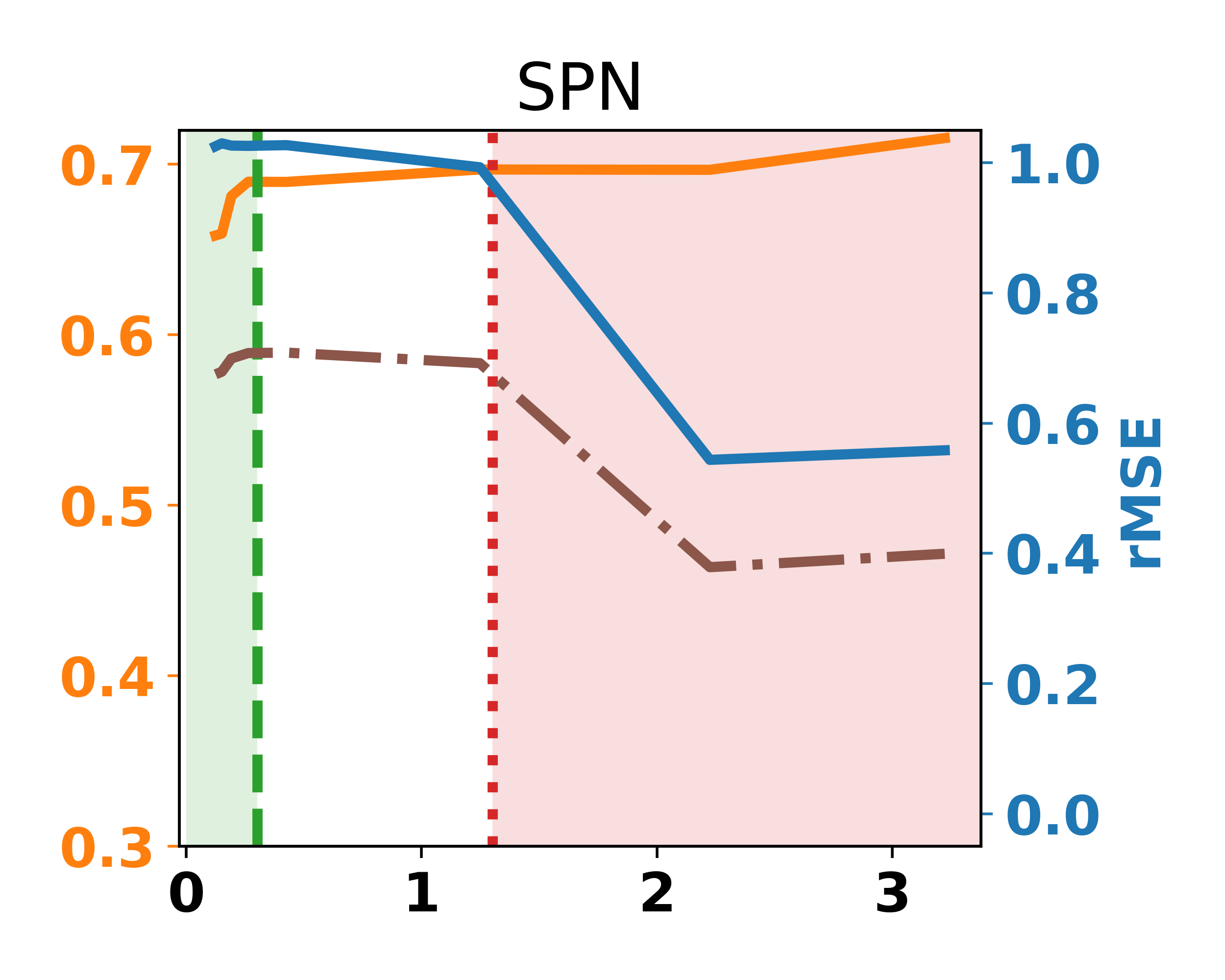}
		\\
		\includegraphics[width=3.2cm,height=2cm]{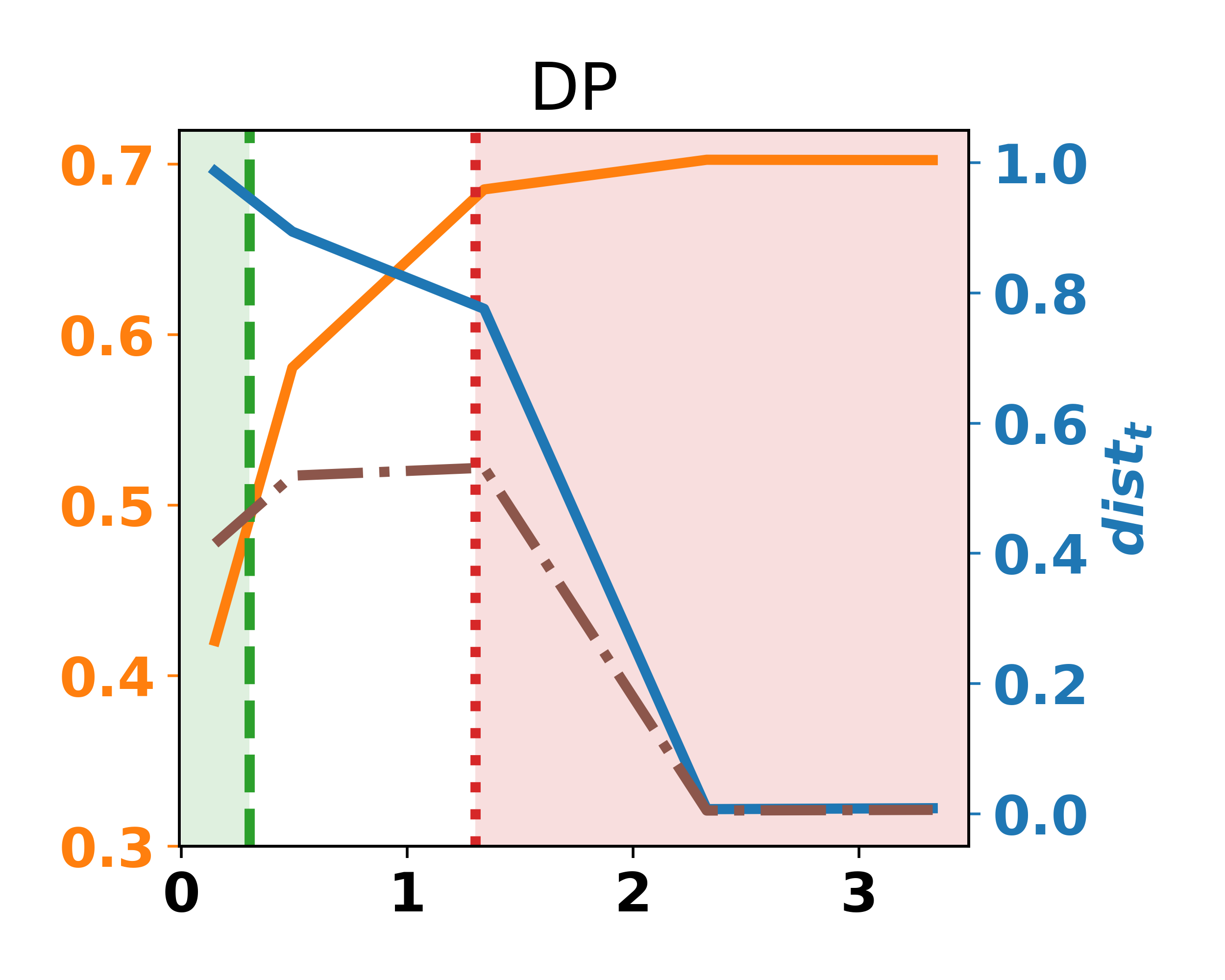}
		\includegraphics[width=3.2cm,height=2cm]{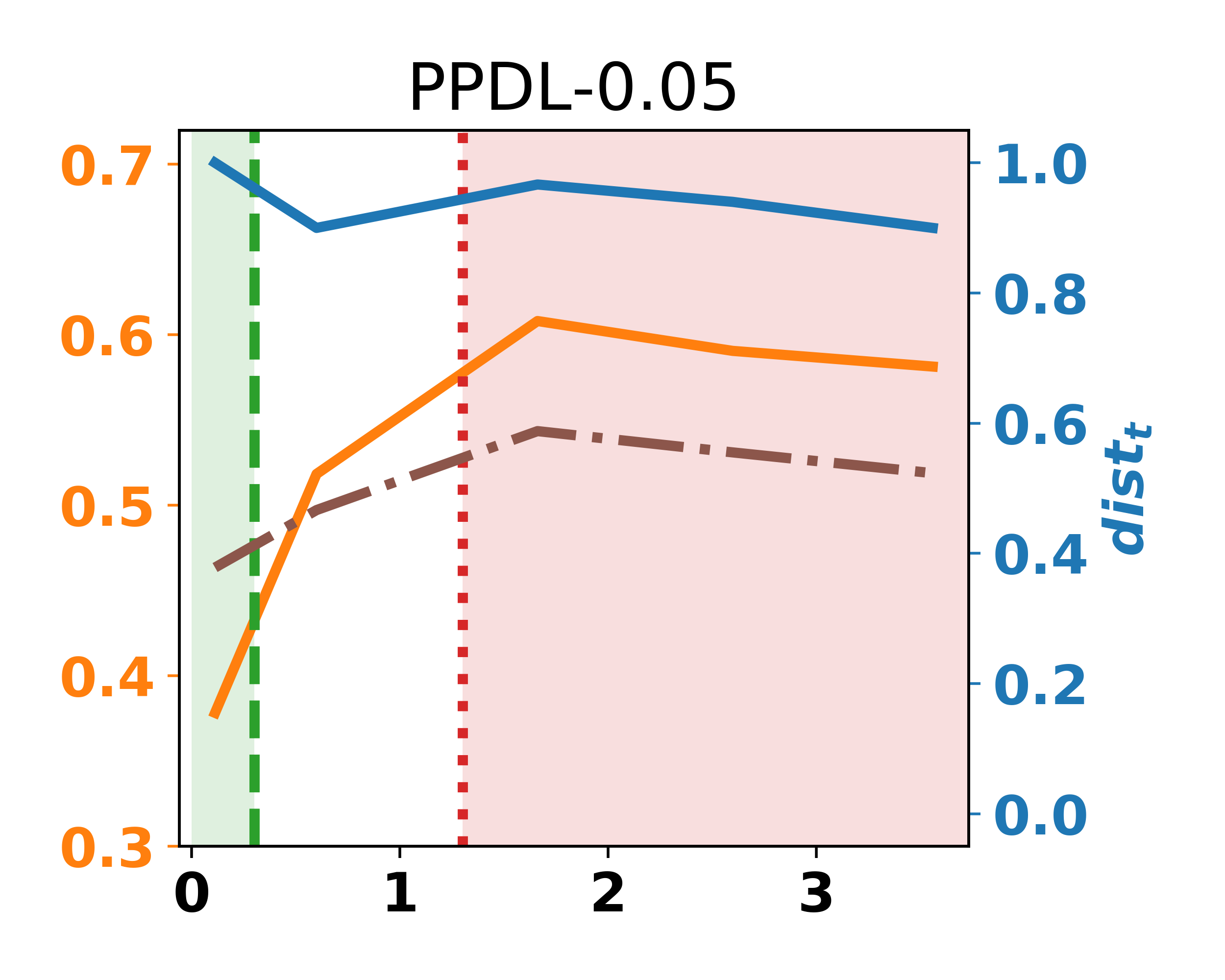}
		\includegraphics[width=3.2cm,height=2cm]{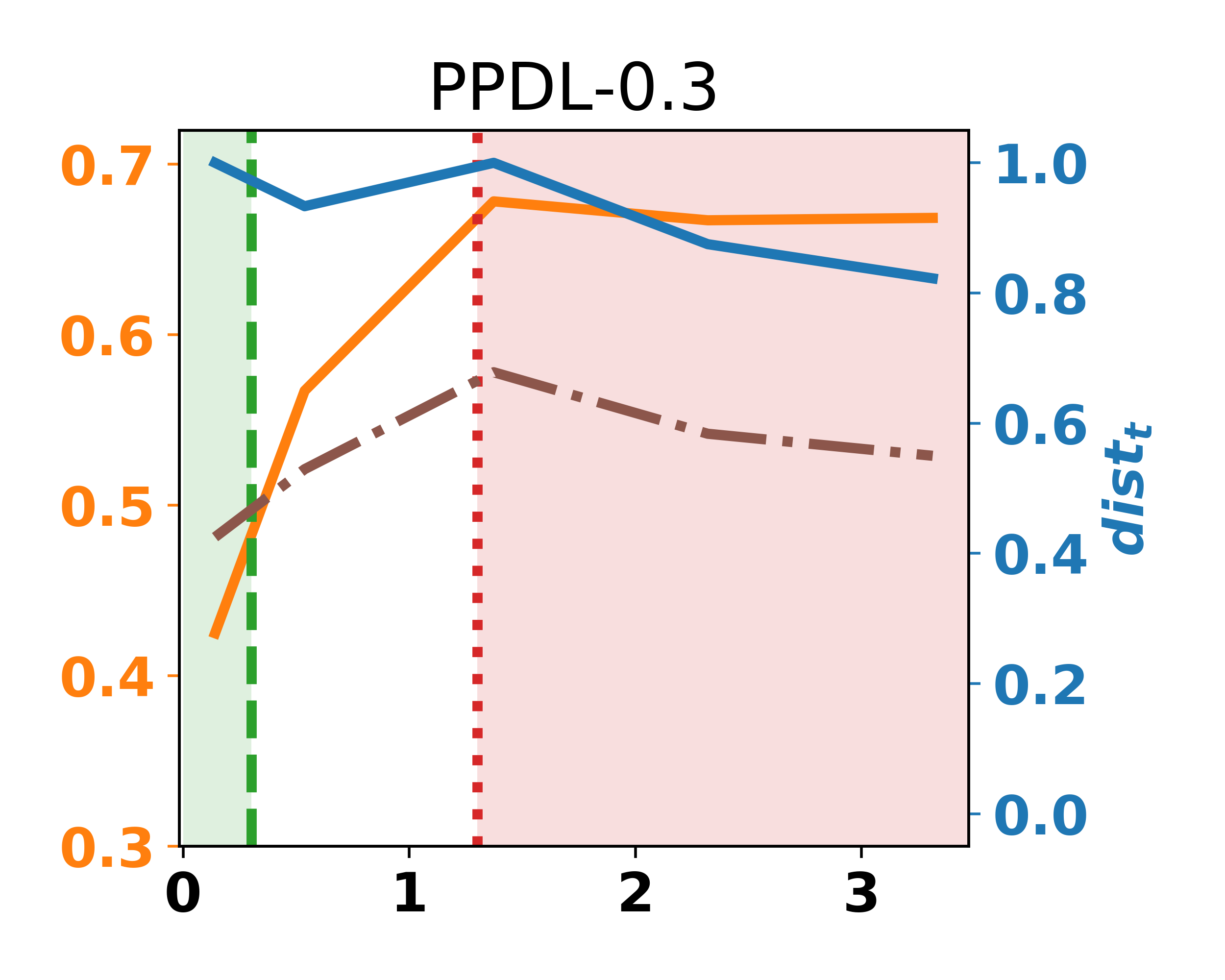}
		\includegraphics[width=3.2cm,height=2cm]{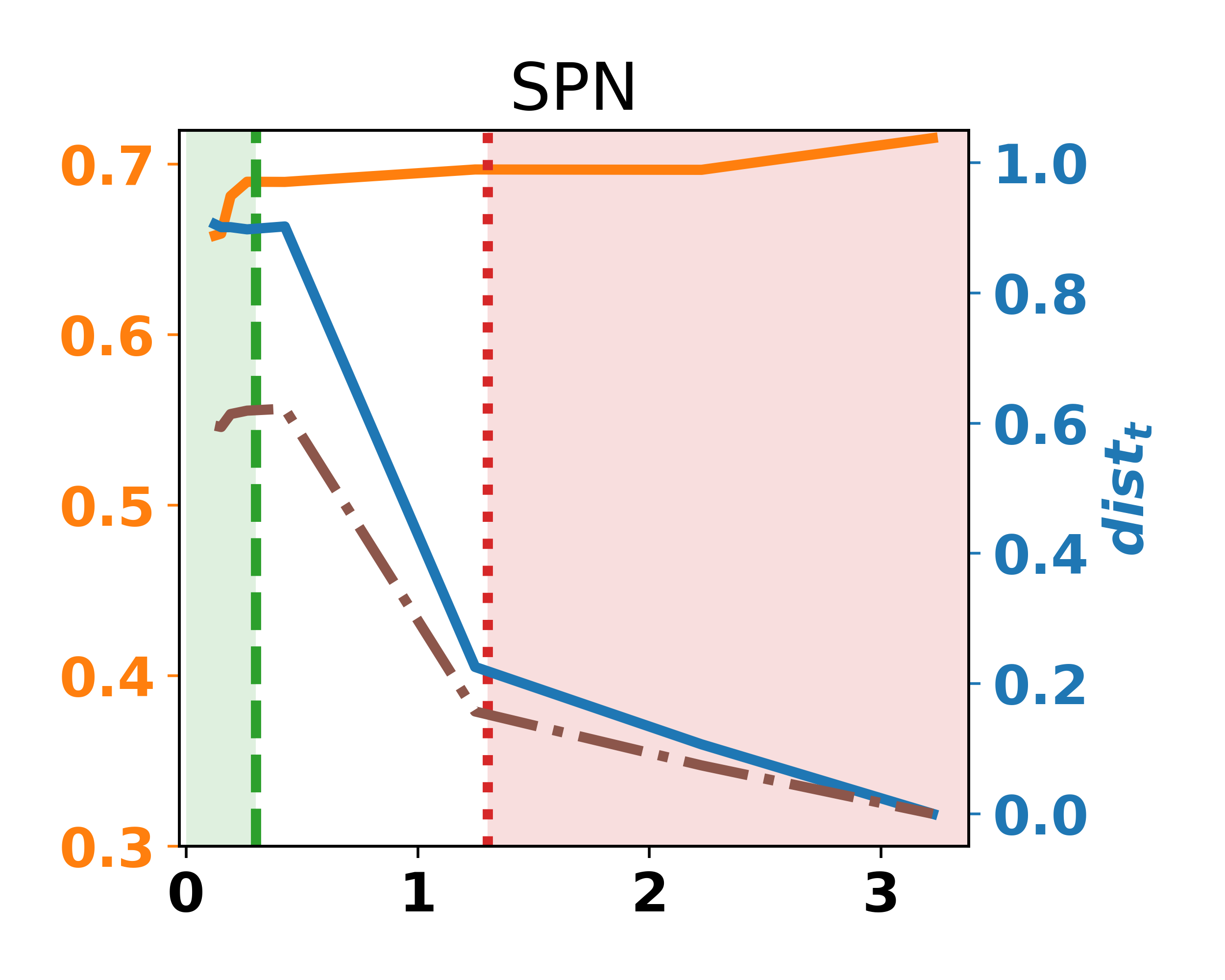}		
		\\
		\includegraphics[width=3.2cm,height=2cm]{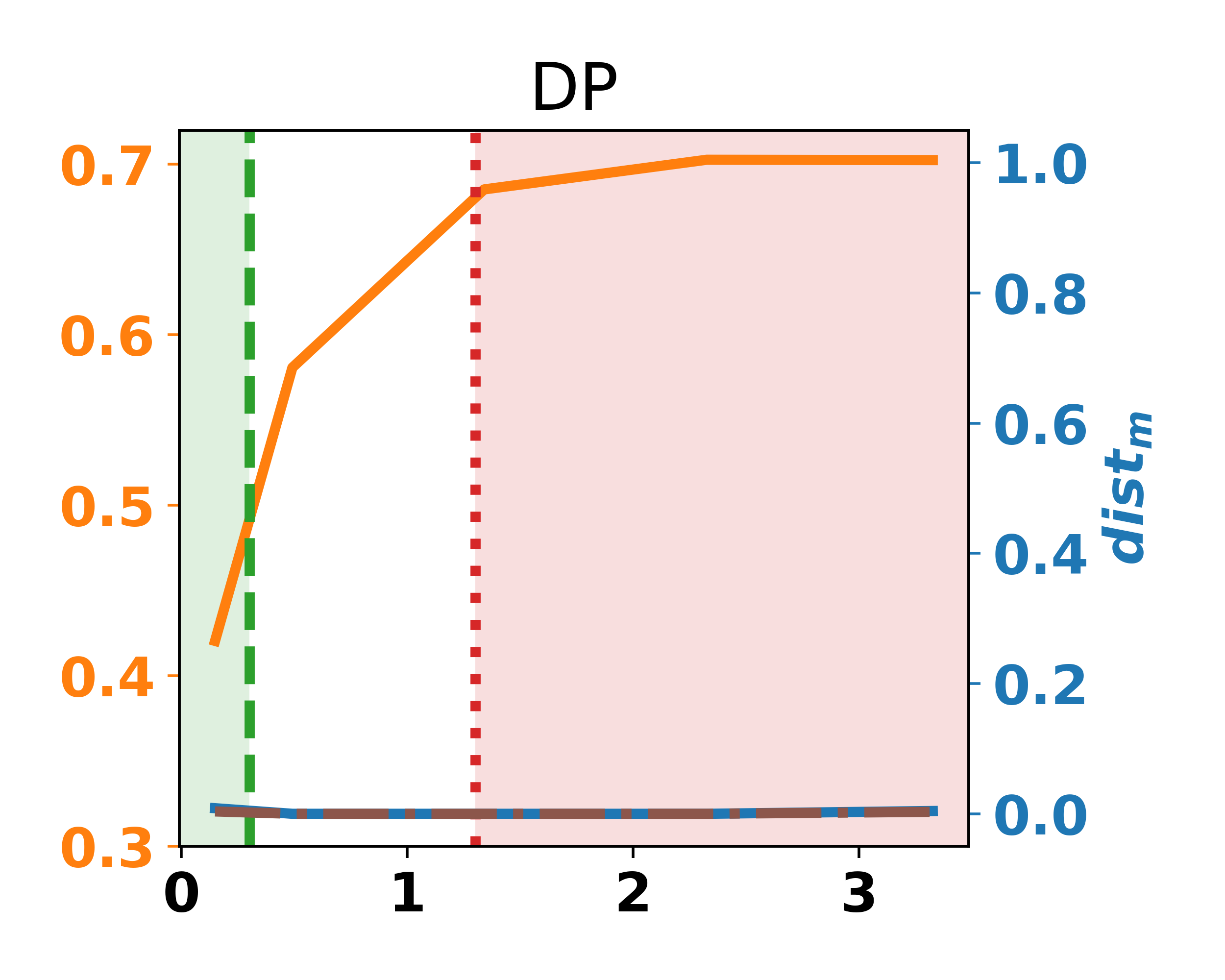}
		\includegraphics[width=3.2cm,height=2cm]{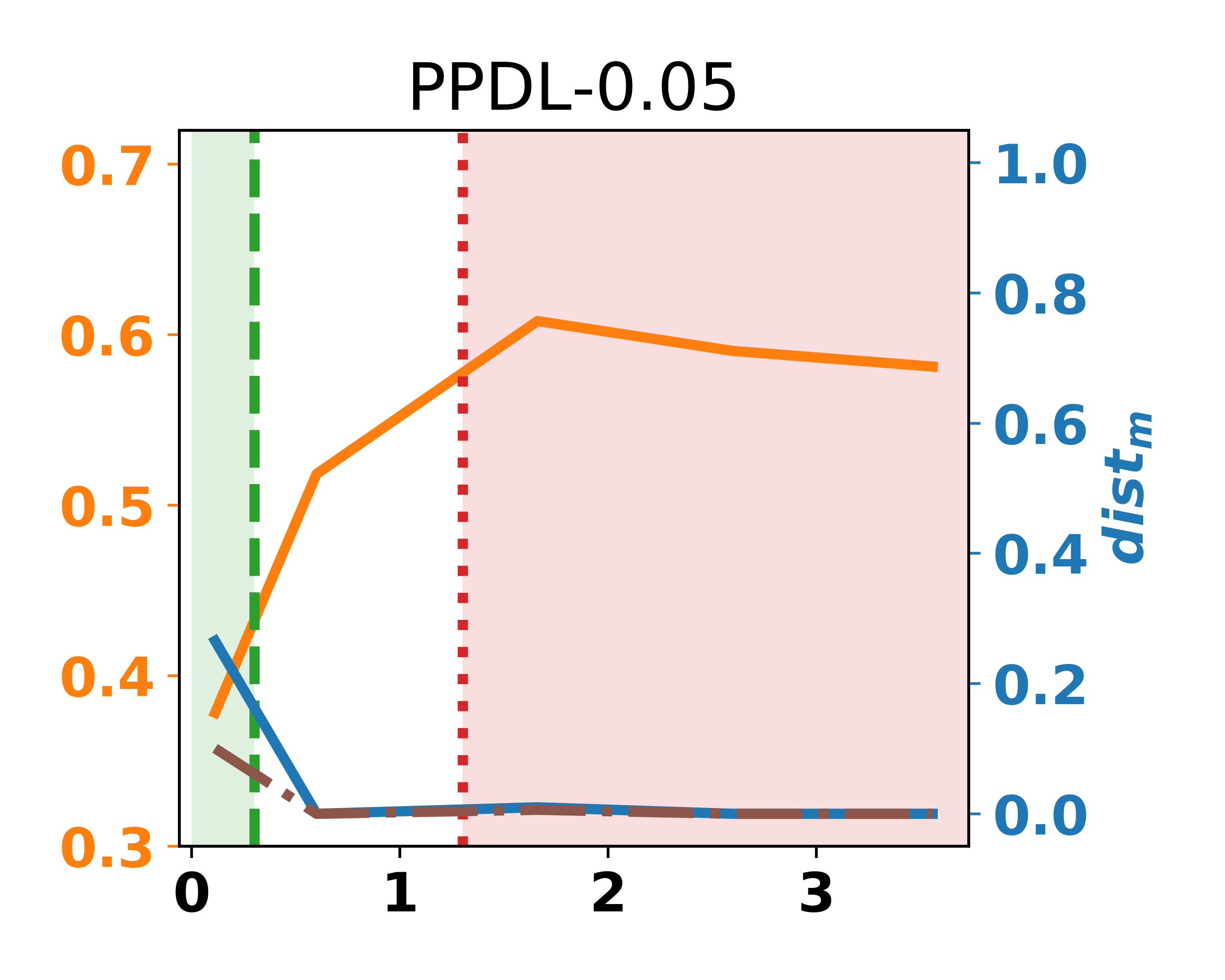}
		\includegraphics[width=3.2cm,height=2cm]{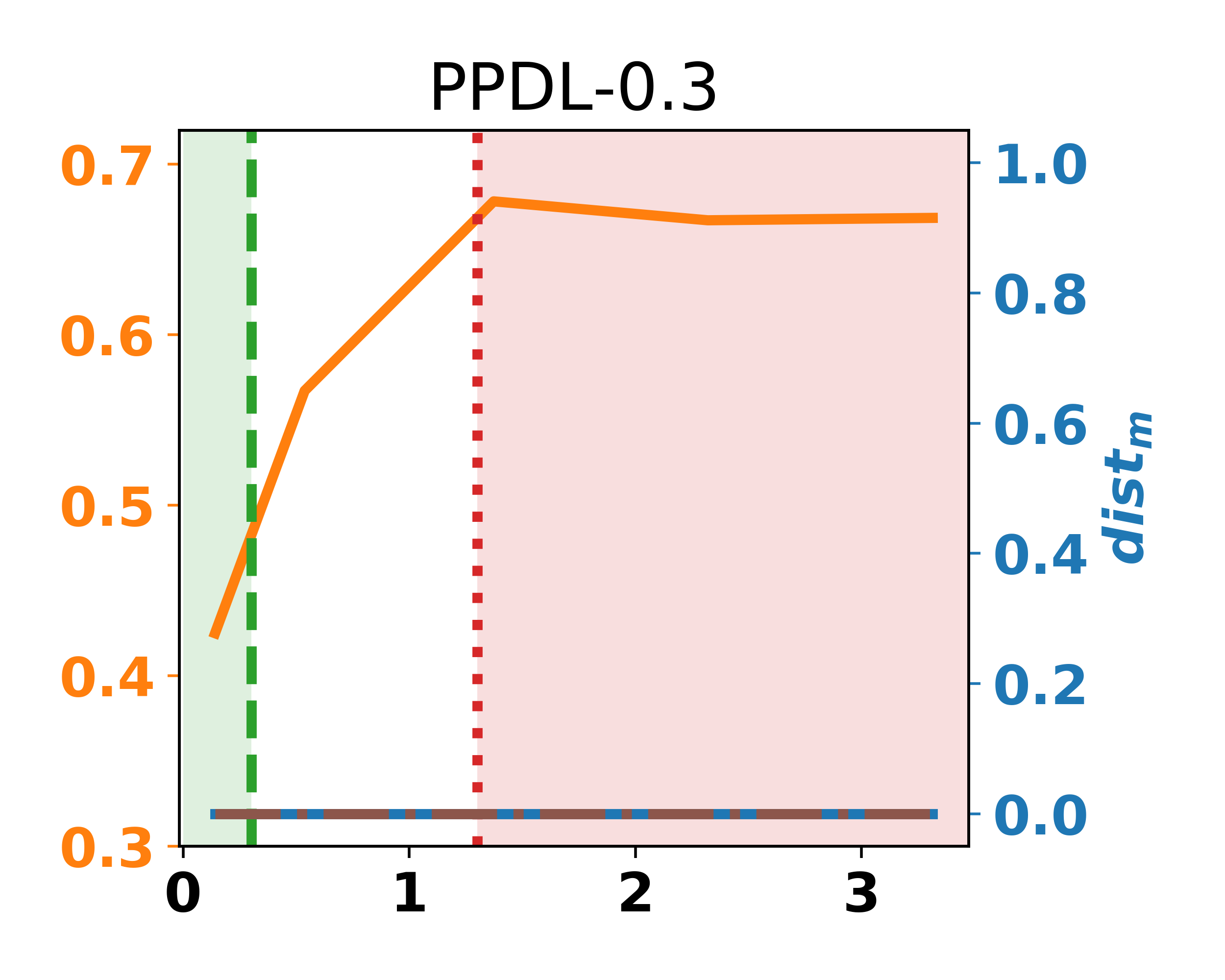}
		\includegraphics[width=3.2cm,height=2cm]{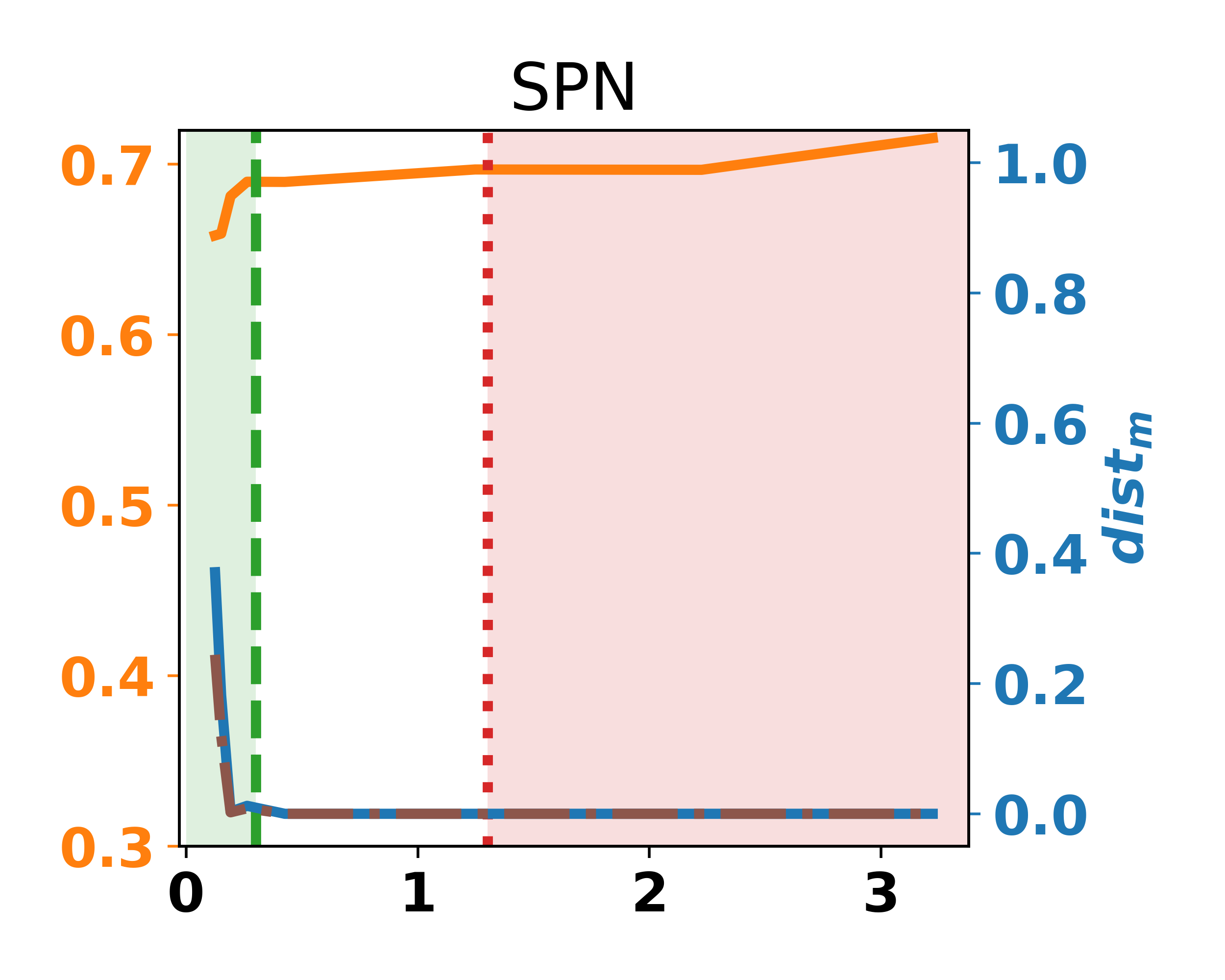}
		%\includegraphics[scale=0.1]{imgs/gradient_0.0001_DPN.png}
		%	\end{subfigure}	
		\caption{%Privacy-Preserving Characteristics  (PPC)  of different mechanisms(y-axis: model accuracies).  \textbf{Left}: Reconstruction attack (x-axis: structural similarity $ssim(x, \bar{x})$);  Middle: membership attack (x-axis: success rates);  and tracing attacks (x-axis: success rates); \textbf{Top}:  with CIFAR10 classification models; \textbf{Bottom}:  with MNIST classification models.
			Privacy-Preserving Characteristics  (PPC)  of different mechanisms (dash-dotted PPC curves); orange curves and y-axis (left): $Acc$ of models; blue curves and y-axis (right): distances for attacks;  x-axis: controlling param. $log_{10} (\frac{||B_I||}{||E_B||} + 1)$. %, see Eq. \ref{eq:recon-cond}). 
			\textbf{Left to Right}: DP, PPDL-0.05, PPDL-0.3 and SPN (Ours).
			%\textbf{Top}: Reconstruction Attack (y-axis: relative MSE $\frac{\|x^*-x\|}{\|x\|}$); \textbf{Middle}: Membership Attack (y-axis: averaged distances between reconstructed memberships and ground-truth labels); \textbf{Bottom}: Tracing Attack (y-axis: distances between reconstructed participant labels and ground-truth participant labels).
			\textbf{Top}: Reconstruction Attack; \textbf{Middle}: Tracing Attack; \textbf{Bottom}: Membership Attack.
			%Green regions to the left of vertical dashed green lines are with theoretical guarantees to defeat reconstruction attacks. 
			See Fig. \ref{fig:recon-images} for example reconstruction images. 
			% in green, white and red regions.
			\label{fig:ppc-all}}
		%\vspace{-0.22cm}
	\end{figure}

	\textbf{Relative Mean Square Error (rMSE)} (= $\frac{||x^*-x||}{||x||}$) is used to measure the distances between reconstructed and original data. 
	% We evaluate the quality of reconstructed samples using rMSE, as rMSE increased, the reconstructed image should become unvisible (see Figure \ref{fig:recon-images}). rMSE is calculated using mean square error between reconstructed image and groundtruth image divided by mean square of groundtruth image (rMSE = $\frac{||x^*-x||}{||x||}$).
	\textbf{Membership Distance ($dist_m(y^*, y)$)} is the \textit{averaged categorical distances} between recovered data labels and original labels. 
	%. We evaluate the membership attack by using averaged distances between reconstructed memberships and ground-truth labels ($dist_m = 1 - A_m$).
	\textbf{Tracing Distance ($dist_t(x)$)} is the \textit{averaged categorical distances} between recovered participant IDs and original IDs, to which the given data $x$ belongs.
	%. We evaluate the tracing attack by using averaged distances between reconstructed participant labels and ground-truth participant labels ($dist_t = 1 - A_t$).
	\begin{figure}[t]
		\centering
		
		\begin{subfigure}{0.3\linewidth}
			\centering
			\includegraphics[scale=0.8]{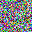}
			\hspace{3pt}
			\includegraphics[scale=0.8]{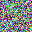}
			\caption{\textbf{Left}:1.04 (0.43); \textbf{Right}:1.13 (0.54) \label{fig:recon-images-green}}
		\end{subfigure}
		\hfill
		\begin{subfigure}{0.32\linewidth}
			\centering
			\includegraphics[scale=0.8]{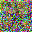}
			\hspace{3pt}
			\includegraphics[scale=0.8]{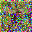}
			\caption{\textbf{Left}:1.01 (18.05); \textbf{Right}:1.00 (21.91) \label{fig:recon-images-white}}%; 1.10 (1.30)}
		\end{subfigure}
		\hfill
		\begin{subfigure}{0.32\linewidth}
			\centering
			\includegraphics[scale=0.8]{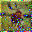}
			\hspace{3pt}
			\includegraphics[scale=0.8]{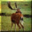}
			\caption{\textbf{Left}:0.86 (43.67); \textbf{Right}:0.02 (2186.8) \label{fig:recon-images-red}} %0.94 (1.59)}
		\end{subfigure}
		\hfill
		%\begin{subfigure}{0.2\linewidth}
		%	\centering
		%	\includegraphics[scale=0.8]{imgs/recimgs/0_0.0001_0.0151.png}
		%	\hspace{3pt}
		%	\includegraphics[scale=0.8]{imgs/recimgs/391_0.0001_0.0116.png}
		%	\caption{0.02 (3.34); 0.01 (3.28)}
		%\end{subfigure}
		\caption{Reconstructed images from different region in Fig. \ref{fig:ppc-all}. \textbf{(a)} Green region \textbf{(b)} White region \textbf{(c)} Red region. 
			Values inside bracket are $\frac{||B_I||}{||E_B||}$ and values outside are rMSE of reconstructed w.r.t. original images. %When $\frac{||B_I||}{||E_B||} < 1$ which is in the green region, it is theoretically guarantees to defeat reconstruction attacks.
		}
		\label{fig:recon-images}
		\vspace{-10pt}
	\end{figure}
	%\begin{figure}[t]
	%	\centering
	%	\includegraphics[scale=0.8]{imgs/recimgs/40_0.05_1.0406.png}
	%	\includegraphics[scale=0.8]{imgs/recimgs/30_0.01_1.0020.png} \hspace{3pt}
	%	\includegraphics[scale=0.8]{imgs/recimgs/20_0.005_0.8589.png} \hspace{3pt}
	%	\includegraphics[scale=0.8]{imgs/recimgs/10_0.001_0.2172.png} \hspace{3pt}
	%	\includegraphics[scale=0.8]{imgs/recimgs/0_0.0001_0.0151.png}
	%	\caption{Reconstructed images with different rMSEs. \textbf{Left to Right}: 1.0406, 1.002,0.8589, 0.2172 and 0.0151. \label{fig:recon-images}}
	%		\vspace{-10pt}
	%\end{figure}
	\begin{figure}[t]	
		%	\begin{subfigure}{0.95\linewidth}
		\centering
		\includegraphics[scale=0.6]{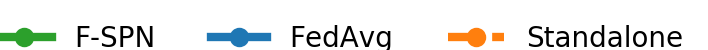}
		\\
		\begin{subfigure}{0.49\linewidth}
			\centering
			\includegraphics[width=0.49\linewidth,height=2cm]{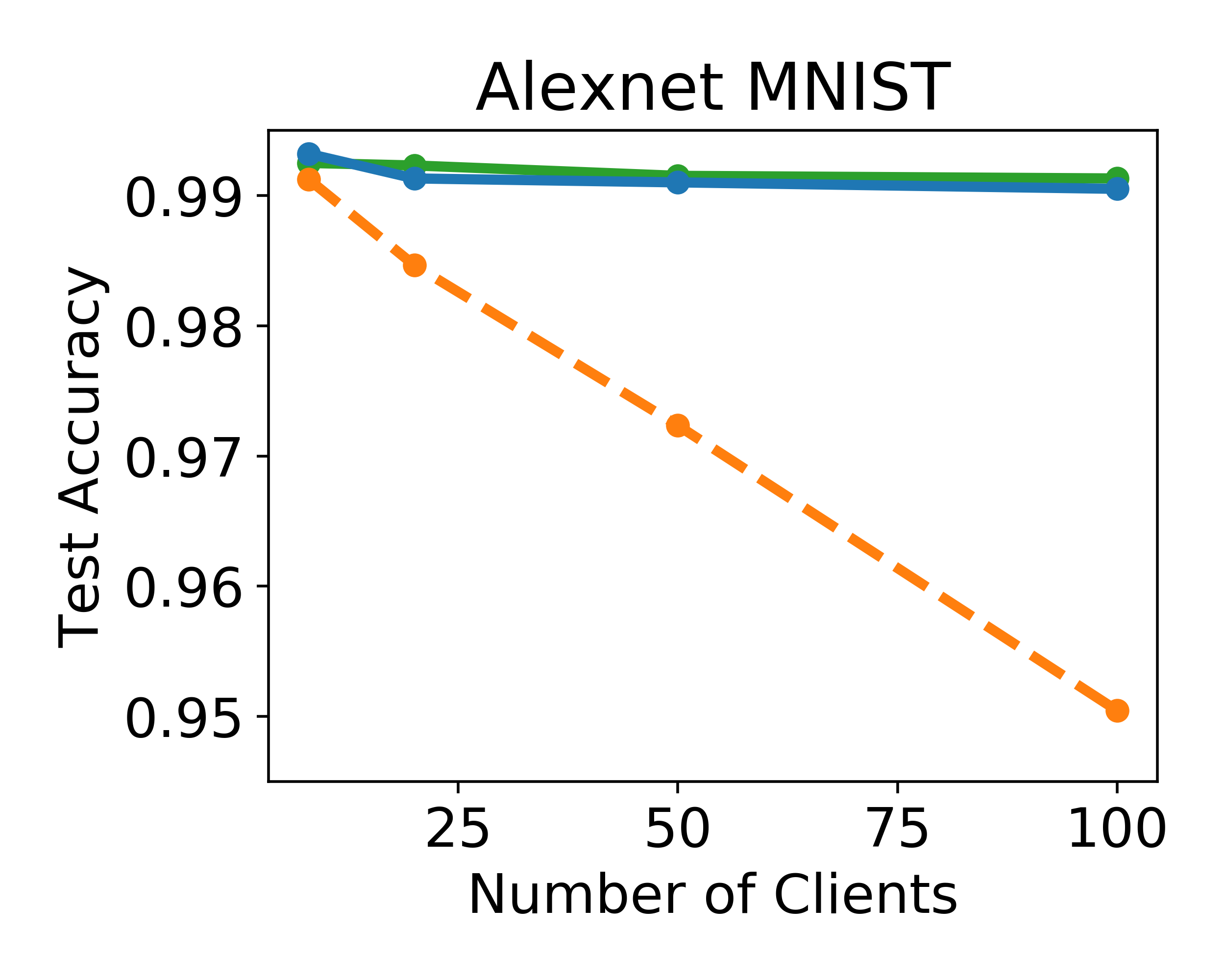}
			\includegraphics[width=0.49\linewidth,height=2cm]{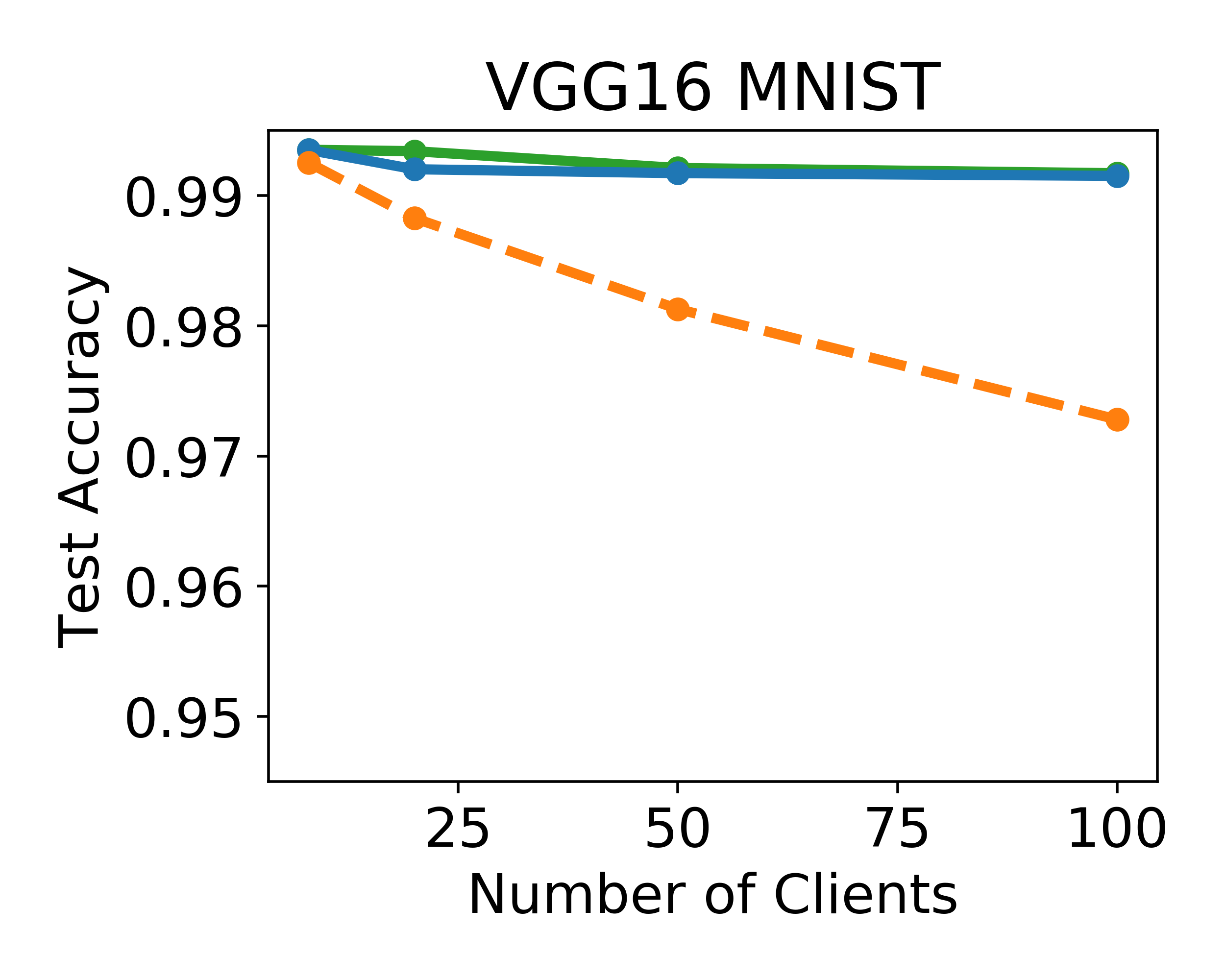}
			\caption{MNIST}
		\end{subfigure}
		\begin{subfigure}{0.49\linewidth}
			\centering
			\includegraphics[width=0.49\linewidth,height=2cm]{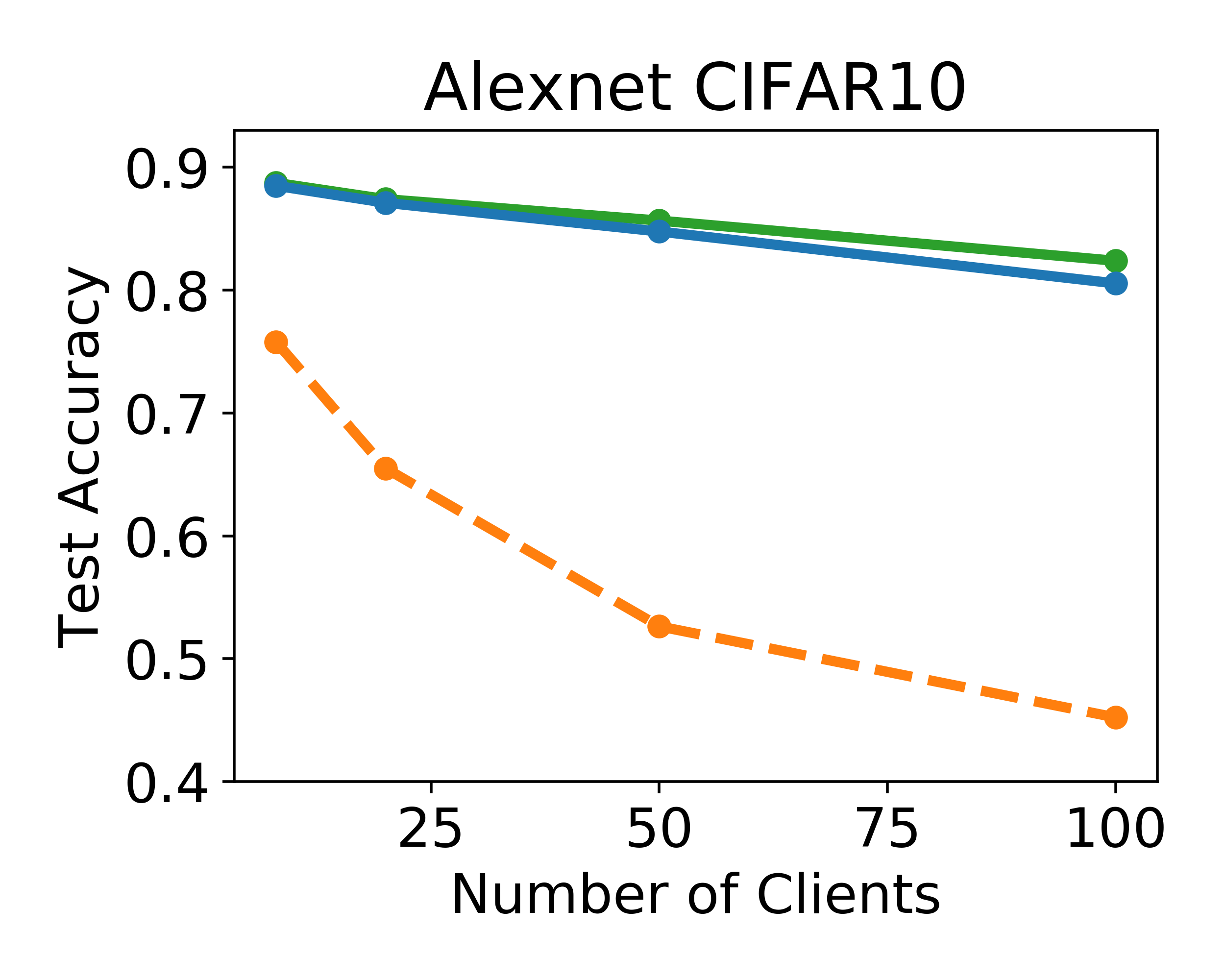}
			\includegraphics[width=0.49\linewidth,height=2cm]{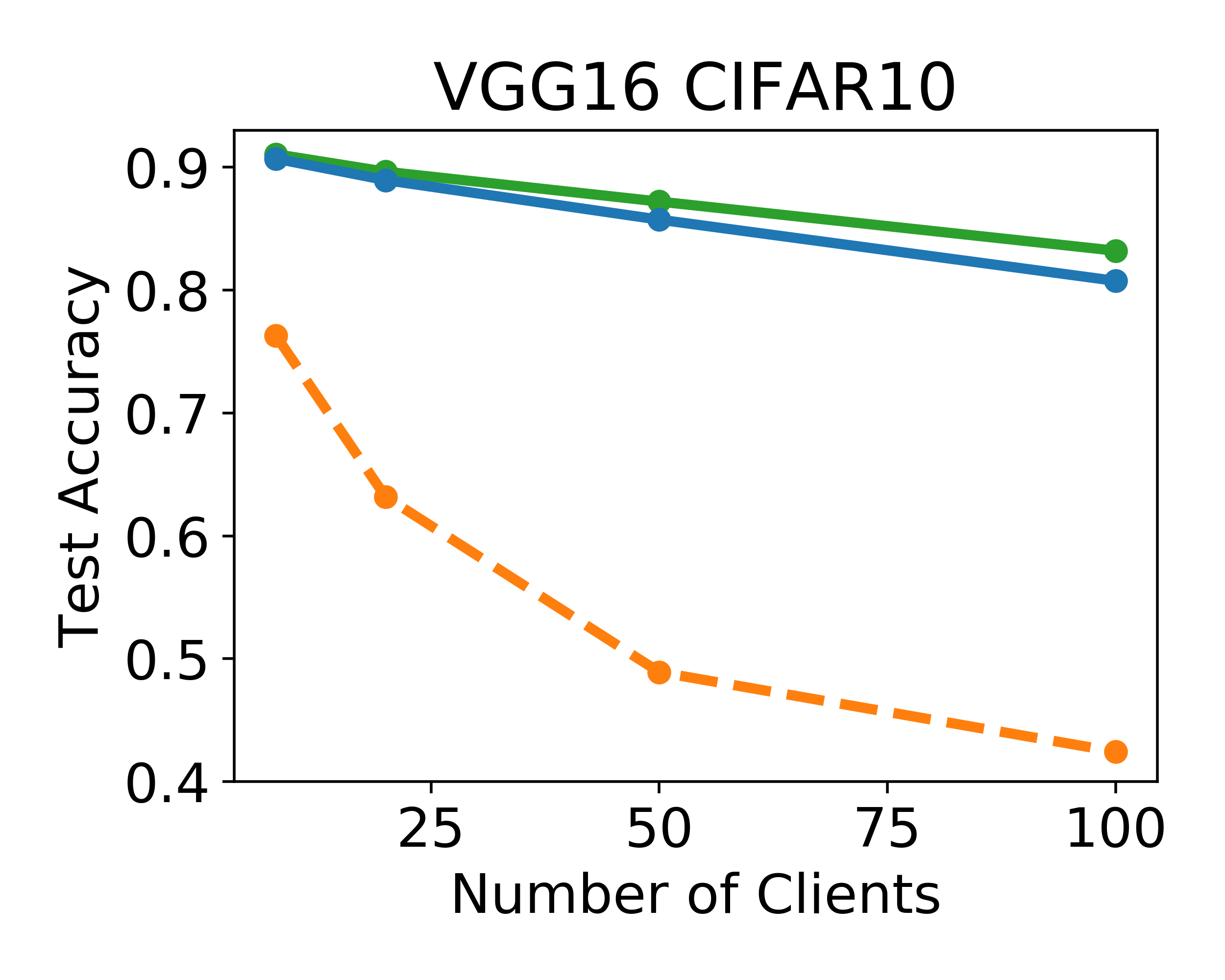}
			\caption{CIFAR10}
		\end{subfigure}
		
		%\includegraphics[scale=0.1]{imgs/gradient_0.0001_DPN.png}
		%	\end{subfigure}	
		\caption{Comparison of accuracies for \textit{standalone }local models, \textit{FedAvg} global model and \textit{Federated SPN} model. 
			Improvements over standalone models increase with the number of clients.  
			%Standalone model is trained \textbf{locally} with only limited number of data, for e.g., only 500 number of data when number of clients is 100.
			%, the lower the acc, SPN degrade slower than baseline on alexnet and vgg
		}
		\label{fig:flacc}
		%\vspace{-10pt}
	\end{figure}
	\begin{table}[t]
		\adjustbox{max width=\textwidth}{
			\centering
			\begin{tabular}{l|c|c|c|c|c|c|c|c|c|c|c|c|c|c|c|c|c|c}
				\toprule
				& \multicolumn{9}{c|}{CIFAR10} &  \multicolumn{9}{c}{CIFAR100} \\
				\hline
				& \multicolumn{3}{c|}{Reconstruction} & \multicolumn{3}{c|}{Membership} & \multicolumn{3}{c|}{Tracing} & \multicolumn{3}{c|}{Reconstruction} & \multicolumn{3}{c|}{Membership} & \multicolumn{3}{c}{Tracing} \\
				\hline
				BS & 1 & 4 & 8 & 1 & 4 & 8 & 1 & 4 & 8  & 1 & 4 & 8 & 1 & 4 & 8 & 1 & 4 & 8 \\
				\hline
				\cite{DLDP_Abadi16} & 0.57 & 0.63 & 0.63 & 0.00 & 0.45 & 0.47 & 0.42 & 0.57 & 0.58 & 0.23 & 0.31 & 0.30 & 0.01 & 0.22 & 0.25 & 0.14 & 0.24 & 0.24 \\
				\cite{PPDL/shokri2015}$\ast$ & 0.55 & 0.55 & 0.55 & 0.00 & 0.37 & 0.44 & 0.50 & 0.50 & 0.50 & 0.18 & 0.18 & 0.18 & 0.02 & 0.13 & 0.16 & 0.16 & 0.16 & 0.16 \\
				\cite{PPDL/shokri2015}$\star$ & 0.57 & 0.61 & 0.61 & 0.00 & 0.43 & 0.49 & 0.54 & 0.54 & 0.54 & 0.21 & 0.26 & 0.26 & 0.00 & 0.19 & 0.22 & 0.19 & 0.19 & 0.19 \\
				\textbf{SPN} & \textbf{0.69} & \textbf{0.70} & \textbf{0.70} & \textbf{0.24} & \textbf{0.50} & \textbf{0.55} & \textbf{0.60} & \textbf{0.62} & \textbf{0.64} & \textbf{0.35} & \textbf{0.35} & \textbf{0.36} & \textbf{0.17} & \textbf{0.28} & \textbf{0.31} & \textbf{0.29} & \textbf{0.30} & \textbf{0.30} \\
				\bottomrule
			\end{tabular}
		}
		\caption{CAP performance with different batch size and dataset for reconstruction, membership and tracing attack. Higher better. BS = Attack Batch Size, \cite{DLDP_Abadi16} = DP, \cite{PPDL/shokri2015}$\ast$ = PPDL-0.05, \cite{PPDL/shokri2015}$\star$ = PPDL-0.3}
		\label{tab:CAP-performance}
		\vspace{-10pt}
	\end{table}
	
	\subsection{Comparison of Privacy Preserving Mechanisms}\label{subsect:exper-compare}
	
	%<<<<<<< HEAD
	Fig. \ref{fig:ppc-all} illustrates example Privacy-Preserving Characteristic  (PPC)   of different mechanisms against \textit{reconstruction, membership} and \textit{tracing} attacks, in which the controlling parameter along x-axis is the ratio $m$ of gradient magnitudes $\|B_I\|$ with respect to magnitudes of added perturbations$\|E_B\|$. 
	%Table ? summarizes AUPPC of different privacy-preserving mechanisms.  We adopt both PPC and AUPPC to compare different mechanisms in terms of capabilities to defense  privacy attacks (see Sect. \ref{sect:exper}).
	It is shown that privacy attacks pose serious challenges to differential privacy based methods \textbf{DP} and \textbf{PPDL}. 
	
	\textbf{Reconstruction attacks}  (top row): when the ratio  ranges between tens to thousands in red regions,  errors decrease rapidly and  \textit{pixel-level information} about original training data are almost completely disclosed (see Fig. \ref{fig:recon-images-red}). 
	In the white regions, increased magnitudes of perturbations lead to large reconstruction errors  (rMSE $\approx 1.0$) with noticeable artifacts and random noisy dots in Fig. \ref{fig:recon-images-white}.  However, model accuracies for DP and PPDL methods also decrease dramatically.  Pronounced drops in accuracies (with more than 20\% for CIFAR10 and 5\% for MNIST)  are observed when added perturbations exceed magnitudes of original gradients (in green regions), beyond which condition (\ref{eq:recon-cond}) of reconstruction attacks is no longer fulfilled and attacks are guaranteed to be defeated (see Theorem \ref{thm:recon-lineareq} and Fig. \ref{fig:recon-images-green}). 
	
	\textbf{Tracing attacks} (middle row): similar trends were observed for distances of tracing attacks. In addition, the distance increases as the number of participants increases. We refer reviewers to ablation studies in supplementary material due to the limited space of this submission. 
	
	\textbf{Membership attacks} (bottom row): the disclosing of memberships is more detrimental, with distances between reconstructed memberships and ground truth labels almost being zero, except for PPDL-0.05 in the green region. With the increase of the number of classes (for CIFAR100) and the training batch size (8),  success rates of membership attacks dropped and the distances increased.  One may mitigate membership attacks by using even larger batch sizes, as suggested in \cite{DeepLeakage_Han19,Gradient_Leakage/wei2020}. 

	In a sharp contrast, Secret Polarization Network (SPN) based mechanism maintains consistent model accuracies, even though gradient magnitudes due to polarization loss exceed gradient magnitudes of original CE loss. Superior performances of SPN mechanism in this green region provide \textit{ theoretically guaranteed  privacy-preserving capabilities}, and at the same time, maintain decent model accuracies to be useful in practice. 
	This superiority is ascribed to the adaptive element-wise gradient perturbations introduced by polarization loss (see discussions near Eq. (\ref{eq:combined-grad})).
	%consistently outperforms DP-SGD and PPDL mechanism,  in terms of the trade-off between model accuracies and capabilities against privacy attacks.  For regimes with acceptable privacy losses (i.e. SSIM scores less than 0.1), model accuracies for SPN mechanism are 5-10\% higher than those of DP-SGD and PPDL mechanisms with varying controlling parameters. 
	%with different controlling parameters. 
	
	%is reflected by the large margin in AUPPC values of different mechanisms summarized in Table \ref{tab:CAP-performance}, in which SPN based mechanism is compared favorably in almost all experiment settings. Figure \ref{fig:barchart-summary} summarizes the comparison of AUPPC values for SPN, PPDL and DP mechanisms, in front of three privacy attacks.  We ascribe the superior SPN privacy-preserving capabilities 
	
	\vspace{-0.3cm}
	\subsection{SPN Polarization Network for Federated Learning}
	
	The dual-headed Secret Polarization Network (SPN) brought improvements in model accuracies in a federated learning setting, in which MNIST and CIFAR10 datasets are evenly distributed among all clients, resulting in small local training datasets on each client (for instance, there are only 500 CIFAR10 training data when the number of clients is 100). Substantial performances deterioration were observed for local standalone models with large numbers of  e.g. 100 clients (see Fig. \ref{fig:flacc}). Since local training data are \textit{i.i.d.}, the FedAvg algorithm \cite{communicationEfficient/mcMahan2017} effectively improved the global model accuracies about 2-4\% for MNIST and 10-40\% for CIFAR10.  The proposed SPN, once integrated with the FedAvg algorithm, consistently improved further model accuracies ranging between 2-3\% for CIFAR10 dataset and about 0.2\% for MNIST (see more results in supplementary material). The improvements are ascribed to element-wise gradients introduced by polarization losses (see discussion in Sect. \ref{sect:PP-SPN}), which in our view advocate the adoption of SPN in practical applications.

	%\vspace{-0.5cm}
	\section{Discussion and Conclusion} \label{sect:dis-concl}
	
	The crux of differential-privacy based approaches is a trade-off between privacy vs accuracy \cite{PPDL/shokri2015,DLDP_Abadi16}. As shown in \cite{exploitingFeatLeak_Vitaly18} and our experiments, existing defenses such as \textit{sharing fewer gradients} and \textit{adding Gaussian or Laplacian noise} are vulnerable to aggressive reconstruction attacks, despite the theoretical privacy guarantee. We extricated from the dilemma by hiding a fraction of network parameters and gradients from the adversary. To this end, we proposed to employ a dual-headed network architecture i.e. Secret Polarization Network (SPN), which on the one hand exerts secret gradient perturbations to original gradients under attack, and on the other hand, maintains performances of the global shared model by jointing forces with the  backbone network. %Table ? below summarizes merits of the proposed approach,  with respect to both DP and HE/SMPC based approaches.  
	This secret-public network configuration provides a theoretically guaranteed privacy protection mechanism without compromising model accuracies, and does not incur significant computational and communication overheads which HE/SMPC based approaches have to  put up with.  We find that the combination of secret-public networks provides a preferable alternative to DP-based mechanisms in application scenarios, whereas large computational and communication overheads are unaffordable  e.g. with mobile or IOT devices. As for future work, the adversarial learning nature of SPN also makes it an effective defense mechanism against adversarial example attacks. To formulate both privacy and adversarial attacks in a unified framework is one of our future directions. 
	
		\section*{Broader Impact}
	\label{impact}
	Our benchmark is likely to increase progress of federated learning and encourage more companies and people to share their data. While there will be immediate benefits resulted from the use of SPN in general, here we also advocate the impact of using our measurement tool to evaluate and thwart privacy attacks. Benefits of using such a tool include increasing transparency in federated learning applications, and mitigating data safety risks in distributed machine learning - see introduction of the paper for more details. 
	
	The \textit{sharing of local model updates} in distributed learning scenarios,  concomitantly disclose \textit{privacy of local data} if no protection measures are taken. 
	Our investigations about the trade-off between \textit{data privacy} protection and \textit{model utilities} for differential-privacy (DP) based approaches, therefore, is of interest to people who concern about the risks of reverse engineering and/or stealing of valuable private data. 
	Moreover, the theoretical guarantee (\ref{eq:recon-cond}) for the first time lays the foundation for a series of protection mechanisms, 
	one of which is instantiated by a secret polarization network (SPN) that 
	thwarts privacy attacks and maintains high model utilities at the same time. 
	The proposed secret-public network configuration, on its own, also paves the way for a novel research direction in our view. 
	Finally, source codes of this work will be made publicly available for people to reproduce and follow up. 
	
	%. table: privacy attacks;  performances; complexities; 

	%\begin{ack}
	%Use unnumbered first level headings for the acknowledgments. All acknowledgments
	%go at the end of the paper before the list of references. Moreover, you are required to declare 
	%funding (financial activities supporting the submitted work) and competing interests (related financial activities outside the submitted work). 
	%More information about this disclosure can be found at: \url{https://neurips.cc/Conferences/2020/PaperInformation/FundingDisclosure}.
	%
	%
	%Do {\bf not} include this section in the anonymized submission, only in the final paper. You can use the \texttt{ack} environment provided in the style file to autmoatically hide this section in the anonymized submission.
	%\end{ack}
	
	\clearpage
	\newpage
	
	\small
	
	\bibliographystyle{plain} %named}
	%\bibliography{ijcai20}
	\bibliography{bib/neurips2020}~
	
	%%%%%%%%%%%%%%%%%%%%%%%%%%%%%%%%%%%%%%%%%%%%%%%%%%%%%%%%%%%%%%%%%%%%%%%%%%%%%%%%
	%%%%%%%%%%%%%%%%%%%%%%%%%%%%%%%%%%%%%%%%%%%%%%%%%%%%%%%%%%%%%%%%%%%%%%%%%%%%%%%%
	%%%%%%%%%%%%%%%%%%%%%%%%%%%%%%%%%%%%%%%%%%%%%%%%%%%%%%%%%%%%%%%%%%%%%%%%%%%%%%%%
	\newpage
	\section*{Appendix A: Proofs of Reconstruction Attacks}
	Consider a neural network $\Psi(x;w, b): \mathcal{X} \rightarrow \mathbb{R}^C$, where $x \in \mathcal{X}$, $w$ and $b$ are the weights and biases of neural networks, and $C$ is the output dimension. In a machine learning task, we optimize the parameters $w$ and $b$ of neural network $\Psi$ with a loss function $\mathcal{L}\big(\Psi(x; w, b), y \big)$, where $x$ is the input data and $y$ is the ground truth labels. We abbreviate loss function as $\mathcal{L}$ and denote the superscript $w^{[i]}$ and $b^{[i]}$ as the $i$-th layer weights and biases. 
	%The following theorem proves that the reconstruction of input $x$ exists under certain conditions (proofs are given in Appendix A, in supplementary material due to the limited space). 
	
	Suppose a multilayer neural network $\Psi:=\Psi^{[L-1]}\circ \Psi^{[L-2]}\circ \dots \circ \Psi^{[0]}(\hspace*{0.2em} \cdot \hspace*{0.3em}; w, b)$ is $\mathcal{C}^1$, where the $i$-th layer $\Psi^{[i]}$ is a fully-connected layer with the step forward propagation as follows, 
	\[
	o^{[i+1]} = a\big(w^{[i]}\cdot o^{[i]} + b^{[i]}\big),
	\]
	where $o^{[i]}$, $o^{[i+1]}$, $w^{[i]}$ and $b^{[i]}$ are an input vector, an output vector, a weight matrix and a bias vector respectively, and $a$ is the activation function in the $i$-th layer.
	
	By the backpropagation, we have the matrix derivatives on $\Psi^{[i]}$ as follows, 
	\begin{align*}
	\nabla_{w^{[i]}} \mathcal{L} = \nabla_{o^{[i+1]}} \mathcal{L} \cdot  a' \big(w^{[i]}\cdot o^{[i]} + b^{[i]}\big)\cdot {o^{[i]}}^T \\ %\text{and} 
	\hspace*{0.3em} \nabla_{b^{[i]}} \mathcal{L} = \nabla_{o^{[i+1]}} \mathcal{L} \cdot a' \big(w^{[i]}\cdot o^{[i]} + b^{[i]}\big)\cdot {I},
	%\frac{\partial \mathcal{L}}{\partial w^{[i]}} = \frac{\partial \mhcal{L}}{\partial o^{[i+1]}} \cdot  a' \big(w^{[i]}\cdot o^{[i]} + b^{[i]}\big)\cdot {o^{[i]}}^T \text{and} \hspace*{0.3em} \frac{\partial \mathcal{L}}{\partial b^{[i]}} = \frac{\partial \mathcal{L}}{\partial o^{[i+1]}} \cdot a' \big(w^{[i]}\cdot o^{[i]} + b^{[i]}\big)\cdot {I},
	\end{align*}
	which yield the following output equations:
	\begin{align}\label{outputequation}
	\nabla_{w^{[i]}} \mathcal{L} = \nabla_{b^{[i]}} \mathcal{L} \cdot {o^{[i]}}^T,
	%\frac{\partial \mathcal{L}}{\partial w^{[i]}} = \frac{\partial \mathcal{L}}{\partial b^{[i]}} \cdot {o^{[i]}}^T, 
	\end{align}
	where gradients $\nabla_{w^{[i]}} \mathcal{L}$ and $\nabla_{b^{[i]}} \mathcal{L}$ are supposed to be shared in a distributed learning setting, and known to honest-and-curious adversaries who may launch reconstruction attacks on observed gradients.
	
	\begin{figure}[h]
		\centering
		\includegraphics[width=0.9\linewidth]{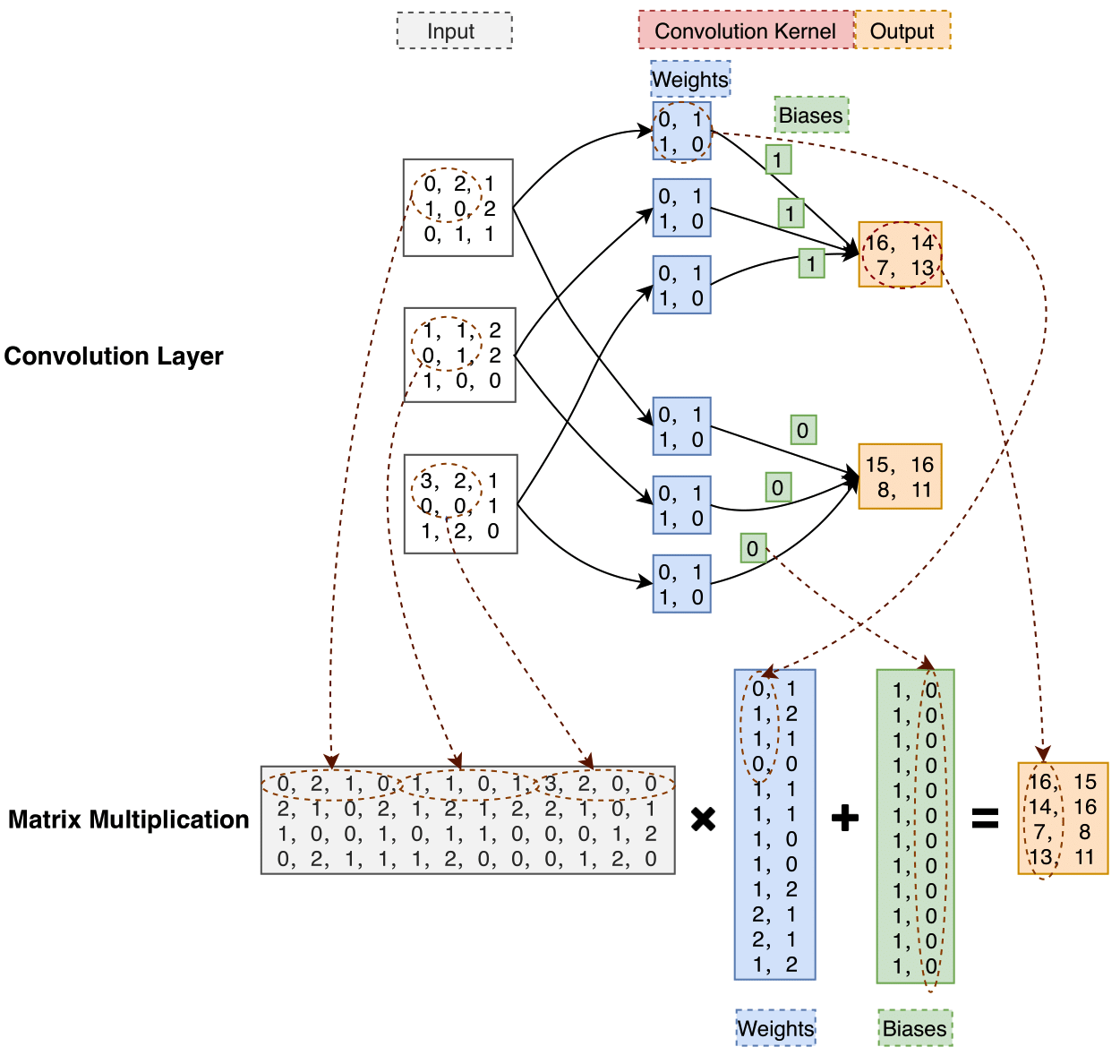}
		\caption{A pictorial example illustrating how to switch a convolution operator to a matrix multiplication.} %In an implementation of convolution operation (e.g. PyTorch), both input image and kernel will unfold into a matrix, then perform matrix multiplication. Hence, it is equivalent to a fully connected operation.} %and the difference is some values in the matrix is repeated.}
		\label{fig:convtofc}
	\end{figure}
	
	\begin{remark}
	Any convolution layers can be converted into a fully-connected layer by simply stacking together spatially shifted convolution kernels, as noted in Footnote {\color{blue}3}. A simple illustration refers to Figure \ref{fig:convtofc} and detailed algorithm refers to a technical report~\footnote{Wei Ma, Jun Lu: An Equivalence of Fully Connected Layer and Convolutional Layer. \url{https://arxiv.org/pdf/1712.01252.pdf}}.
	\end{remark}
	\begin{remark}
		Suppose $\nabla_{w^{[i]}} \mathcal{L} \in \mathbb{R}^{M\cdot N}$, $\nabla_{b^{[i]}} \mathcal{L} \in \mathbb{R}^M$ and $o^{[i]} \in \mathbb{R}^N$, we write 
		\begin{equation*}
		\nabla_{w^{[i]}} \mathcal{L}  := 
		\begin{pmatrix}
		\frac{\partial \mathcal{L}}{\partial w^{[i]}_{mn}}
		\end{pmatrix}_{\begin{subarray}{l} 
			1 \leq m \leq M; \\ 
			1 \leq n \leq N. 
			\end{subarray}},
		\hspace*{1em}
		\nabla_{b^{[i]}} \mathcal{L} := 
		\begin{pmatrix}
		\frac{\partial \mathcal{L}}{\partial b^{[i]}_{1}}, \dots, \frac{\partial \mathcal{L}}{\partial b^{[i]}_{M}}
		\end{pmatrix}^T  
		\text{, and }
		o^{[i]} := 
		\begin{pmatrix}
		o^{[i]}_{1}, \dots, o^{[i]}_{N}
		\end{pmatrix}^T. 
		\end{equation*}
		By the piecewise matrix multiplication, Equation \ref{outputequation} becomes as a linear system in a formal convention as follows, 
		\[
		\frac{\partial \mathcal{L}}{\partial w^{[i]}_{mn}} = \frac{\partial \mathcal{L}}{\partial b^{[i]}_{m}} \cdot o^{[i]}_{n}, \hspace*{0.5em} \text{ for 1 $\leq$ m $\leq$ M and 1 $\leq$ n $\leq$ N.}
		\]
		Hence, we can write the partial derivative $\nabla_{b^{[i]}} \mathcal{L}$ as an $mn\times mn$ diagonal matrix that each n adjacent diagonal entries in an order are copies of each entry, and partial derivative $\nabla_{w^{[i]}} \mathcal{L}$ as an $mn$-dimensional vector. %, as noted in Footnote {\color{blue}4}. 
	\end{remark}
	In the following paragraph, we always abbreviate equation coefficients $\nabla_{w^{[i]}} \mathcal{L}$ and $\nabla_{b^{[i]}} \mathcal{L}$ to $W^{[i]}$ and $B^{[i]}$ respectively.
	\begin{lemma}\label{l1}
		Suppose $d^{[0]}$ and $d^{[1]}, \cdots, d^{[L]}$ are dimensions of input image $x$ and output vectors $o^{[1]}, \cdots, o^{[L]}$ respectively. $x$ and $o^{[i]}$ can be estimated by solving the following $d^{[i]} \cdot d^{[i+1]}$-dimensional linear system if it is well-posed,
		\begin{align}\label{linear_e1}
		W^{[0]} &= B^{[i]} \cdot x\\
		\text{or} \hspace*{0.5em} W^{[i]}  &= B^{[i]} \cdot {o^{[i]}}^T, \hspace*{0.3em}  \text{ for }i = 1, \cdots, L-1. \label{linear_e2}
		\end{align} 
		%Moreover, the total $\sum_{i=0}^{L-1} d^{[i]}\cdot d^{[i+1]}$-dimensional linear system is written in a compact form, i.e. 
		%\begin{align*}
		%[b^{[0]},\cdots, b^{[L-1]}]^T= \diag{\{B^{[0]},\cdots, B^{[L-1]}\}} \cdot [{o^{[0]}}, \cdots, {o^{[L-1]}}]^T,
		%\end{align*}
		%where diag$\{\cdot\}$ represents the block diagonal matrix. 
	\end{lemma}
	
	\begin{remark}
		Output vectors $o^{[1]},\cdots, o^{[L]}$ are outputs of neural networks $\Psi(\hspace*{0.2em} \cdot \hspace*{0.3em}; w, b)$ on input image $x$. However, solving Linear System (\ref{linear_e2}) are always numerically unstable in that minor numerical perturbation of $B^{[i]}$ around 0 would yield the infinity solution even if it is a well-posed problem. 
		%Additionally, it is always inconsistent in most cases either that the number of equations $\sum_{i=0}^{L-1} d^{[i]}\cdot d^{[i+1]}$ are too bigger than one of unknowns $\sum_{i=0}^L d^{[i]}$. 
		Hence, it is not typically to directly recover input image $x$ and output vectors $o^{[1]}, \cdots, o^{[L]}$ by simple matrix computations in practice. 
	\end{remark}
	
	\begin{restatable}[]{thm}{restatlemmatwo}
		\begin{lemma}\label{l2}
			Assume the linear system $B\cdot x = W$ is corrupted in coefficients written as
			%with $\delta_A=A-\bar{A}, \delta_b= b-\bar{b}, \delta_x = x - \bar{x}$, and
			$\bar{B}\cdot \bar{x} = \bar{W}$. If $B$ is nonsingular, we have the following inequality,
			\[
			||x-\bar{x}|| \leq ||B^{-1}||\cdot \big( ||W-\bar{W}|| + ||B-\bar{B}||\cdot ||\bar{x}||\big).
			\] 
		\end{lemma}
	\end{restatable}
	\begin{proof}
		Obviously, we have 
		\begin{align}
		B \cdot (x-\bar{x}) = (W-\bar{W}) + (B-\bar{B})\cdot \bar{x},
		\end{align}
		which yields this lemma if $B$ is nonsingular. 
	\end{proof}
	
	According to Lemma \ref{l1} and Lemma \ref{l2}, we have the following existing theorem. 
	%\begin{theorem}\label{T1.3}
	%		Suppose a multilayer neural network $\Psi:=\Psi^{[L-1]}\circ \Psi^{[L-2]}\circ \dots \circ \Psi^{[0]}(\hspace*{0.2em} \cdot \hspace*{0.3em}; w, b)$ is $\mathcal{C}^1(\mathbb{R}^{d^{[0]}})$, where the $i$-th layer $\Psi^{[i]}: \mathbb{R}^{d^{[i]}}\rightarrow \mathbb{R}^{d^{[i+1]}}$ is fully-connected\footnote{Any convolution layers can be converted into a fully-connected layer by simply stacking together spatially shifted convolution kernels.}. If there is an $i$ $(1\leq i \leq L)$ such that Jacobian matrix $D_{o^{[0]}} \big(\Psi^{[i-1]}\circ \Psi^{[i-1]}\circ \dots \circ \Psi^{[0]}\big)$ around $o^{[0]}$ is full-rank and partial derivative $A^{[i]}$ of $\mathcal{L}\big(\Psi(o^{[0]}; w, b), y \big)$ is nonsingular, the initial image $o^{[0]*}$ exists. Moreover, we have the following inequality around $o^{[0]*}$, 
	%	\[
	%	||o^{[0]}-o^{[0]*}|| \leq M\cdot || \big\{\frac{\partial \mathcal{L}}{\partial w^{[i]}}, \frac{\partial \mathcal{L}}{\partial b^{[i]}}\big\} - \big\{\frac{\partial \mathcal{L}}{\partial w^{[i]}}, \frac{\partial \mathcal{L}}{\partial b^{[i]}}\big\}^* ||,
	%	\]
	%		\begin{align}\label{bo}
	%		||o^{[0]}-o^{[0]*}|| \leq M\cdot || \big\{A^{[i]}, b^{[i]}\big\} - \big\{A^{[i]}, b^{[i]}\big\}^* ||,
	%		\end{align}
	%		where $ \big\{A^{[i]}, b^{[i]}\big\}^* $ are partial derivatives of $\mathcal{L}\big(\Psi(o^{[0]*}; w, b), y \big)$.
	%\end{theorem}
	
	\begin{theorem}\label{thm:recon-exist_1-supp}
		Suppose a multilayer neural network $\Psi:=\Psi^{[L-1]}\circ \Psi^{[L-2]}\circ \dots \circ \Psi^{[0]}(\hspace*{0.2em} \cdot \hspace*{0.3em}; w, b)$ is $\mathcal{C}^1$, where the $i$-th layer $\Psi^{[i]}$ is a fully-connected layer. Then, \textbf{initial input $x^*$ of $\Psi$ exists}, provided that: if there is an $i$ $(1\leq i \leq L)$ such that
		\begin{enumerate}
			\item Jacobian matrix $D_{x} \big(\Psi^{[i-1]}\circ \Psi^{[i-1]}\circ \dots \circ \Psi^{[0]}\big)$ around $x$ is full-rank; 
			\item Partial derivative $\nabla_{b^{[i]}} \mathcal{L}\big(\Psi(x; w, b), y \big)$ is nonsingular.		
		\end{enumerate}
		Moreover, we have the following inequality around $x^*$, 
		\begin{align}\label{bo}
		||x-x^*|| \leq M\cdot || \nabla_{w^{[i]}, b^{[i]}} \mathcal{L}\big(\Psi(x; w, b), y \big)-\nabla_{w^{[i]}, b^{[i]}} \mathcal{L}\big(\Psi(x^*; w, b), y \big) ||.
		\end{align}
	\end{theorem}
	
	\begin{proof}
		WLOG, we suppose $i$ yields that Jacobian matrix $D_{x} \big(\Psi^{[i-1]}\circ \Psi^{[i-1]}\circ \dots \circ \Psi^{[0]}\big)$ around $x$ is full-rank. By the implicit function theorem, there exists a bounded inverse function $\big(\Psi^{[i-1]}\circ \Psi^{[i-1]}\circ \dots \circ \Psi^{[0]}\big)^{-1}(\cdot \hspace*{0.3em};w, b)$ around $x$, s.t. 
		\begin{align}\label{M}
		\big|\big(\Psi^{[i-1]}\circ \Psi^{[i-1]}\circ \dots \circ \Psi^{[0]}\big)^{-1}(\cdot \hspace*{0.3em};w, b)\big| \leq M^{[i]}.	
		\end{align}
		Since partial derivative $\nabla_{b^{[i]}} \mathcal{L}$ is nonsingular, vector $o^{[i]}$ is solved by matrix computations in Lemma \ref{l1}, and thus the initial image $x^* := \big(\Psi^{[i-1]}\circ \Psi^{[i-1]}\circ \dots \circ \Psi^{[0]}\big)^{-1}(o^{[i]})$. 
		
		By Lemma \ref{l2} and Inequality (\ref{M}), in an open neighborhood of $x^*$, we have 
		\begin{align*}
		||x-x^*|| =& ||\big(\Psi^{[i-1]}\circ \Psi^{[i-1]}\circ \dots \circ \Psi^{[0]}\big)^{-1}(o^{[i]} \hspace*{0.3em};w, b) - \big(\Psi^{[i-1]}\circ \Psi^{[i-1]}\circ \dots \circ \Psi^{[0]}\big)^{-1}(o^{[i]^*} \hspace*{0.3em};w, b)||\\
		\leq&  M^{[i]} \cdot ||o^{[i]}-o^{[i]*}||\\
		\leq&  M^{[i]} \cdot ||{\nabla_{b^{[i]}} \mathcal{L}}^{-1}||  \cdot \big( ||\nabla_{w^{[i]}} \mathcal{L}(\Psi(x; w, b), y) - \nabla_{w^{[i]}} \mathcal{L}(\Psi(x^*; w, b), y)|| \\
		&\hspace*{8em} + ||x^{*}|| \cdot|| \nabla_{b^{[i]}} \mathcal{L}(\Psi(x; w, b), y) - \nabla_{b^{[i]}} \mathcal{L}(\Psi(x^*; w, b), y)|| \big) \\
		\leq&  M \cdot ||\nabla_{w^{[i]}, b^{[i]}} \mathcal{L}\big(\Psi(x; w, b), y \big)-\nabla_{w^{[i]}, b^{[i]}} \mathcal{L}\big(\Psi(x^*; w, b), y \big)||,
		\end{align*}
		where we pick enough big number $M := M^{[i]}\cdot ||x^{*}||\cdot||{\nabla_{b^{[i]}} \mathcal{L}}^{-1}||+1$.
	\end{proof}
	
	\begin{remark}
		%1) If the $i$-th layer refers to the $L$-th layer, Theorem \ref{T1.3} theoretically confirms the deep leakage method proposed in \cite{DeepLeakage_Han19} that minimization of gradient differences yields a recovery of initial image. 	
		1) In the deep leakage approach~\cite{DeepLeakage_Han19}, the recovery of initial image requires model parameters $\mathcal{W}$ and the corresponding gradients $\nabla \mathcal{W}$ such that a minimization of gradient differences $||\nabla \mathcal{W}' - \nabla \mathcal{W}||$ yields a recovery of initial image if the initial image exists. Our theorem provides sufficient conditions of the initial image existence, and Inequality (\ref{bo}) confirms the effectiveness of the deep leakage approach. 
		
		2) Essentially, deep leakage approach is a trade-off computational technique for the matrix approach in the meaning that a loss in accuracy is trade-off with the existence of approximate solution by the optimization approach. Both approaches require model parameters $\mathcal{W}$ and the corresponding gradients $\nabla \mathcal{W}$. 
		
		3) If Jacobian matrix is not full-rank or $\nabla_{b^{[i]}} \mathcal{L}$ is singular, the inverse problem is ill-posed and a minimization of gradient differences might yield multiple solutions or an infeasibility which is observed as noisy images.   
	\end{remark}
	
	If assumptions in Theorem \ref{thm:recon-exist_1-supp} are met, we pick an index set $I$ from row index set of $B^{[i]}$ and $W^{[i]}$ such that the following linear equation is well-posed,
	\begin{align*}
	{B}_I \cdot x = {W}_I, 
	\end{align*}
	where ${B}_I := B^{[i]}_I$ and ${W}_I := W^{[i]}_I$. 
	
	%According to Theorem \ref{thm:recon-exist_1}, the initial input $x^*$ is $\big(\Psi^{[i-1]}\circ \Psi^{[i-1]}\circ \dots \circ \Psi^{[0]}\big)^{-1}(x)$. 
	
	%Suppose there is a perturbation $E$ added on $A_I$ such that measurement matrix $\bar{A}_I =  A_I + E$. We consider the following \textit{iterative} equation for Equation (\ref{oe}), 
	%\begin{align}\label{io}
	%\bar{A}_I \cdot o_{k+1} - E \cdot o_k = b_I.
	%\end{align}
	
	\begin{theorem}
		Suppose there are perturbations $E_{{B}}, E_{{W}}$ added on ${B}_I, {W}_I$, respectively, such that observed measurements $\bar{{B}}_I =  {B}_I + E_{{B}}, \bar{{W}}_I =  {W}_I + E_{{W}} $. Then, the \textbf{reconstruction $x^*$ of the initial input $x$ can be} \textbf{determined} by solving a noisy linear system $\bar{B}_I \cdot x^* = \bar{W}_I$, provided that
		\begin{align}\label{eq:recon-cond-supp}	
		\|{B}_{I}^{-1} \cdot E_B\|<1;
		\end{align}
		Moreover, the relative error is bounded, 
		\begin{align}\label{e13}
		\frac{\|x^*-x\|}{\|x\|} \leq \frac{\kappa(B_I)}{1- \|{B}_I^{-1} \cdot E_B\|} \Big( \frac{\|E_B\|}{\|B_I\|} + \frac{\|E_W\|}{\|W_I\|} \Big),
		\end{align}
		in which $B_I^{-1}$ is the inverse of $B_I$, where $\kappa(B_I)$ is the conditional number of $B_I$.
	\end{theorem}
	\begin{proof}
		According to the construction, we have 
		\[
		(\bar{{B}}_I - {B}_I)\cdot x^* + {B}_I \cdot (x^*-x) = \bar{{W}}_I - {W}_I, 
		\]
		which yields
		\begin{align}\label{error}
		x^* - x = {{B}_I}^{-1} \cdot \big( \bar{{W}}_I - {W}_I  - (\bar{{B}}_I - {B}_I)\cdot x^* \big).
		\end{align}
		Consider the relative error: since $|| W_I || \leq || B_I || \cdot || x ||$, Equation (\ref{error}) becomes
		\begin{align}\label{e15}
		\frac{|| x^*- x ||}{|| x ||} \leq \kappa(B_I) \cdot \Big( \frac{||E_B||}{||B_I||} \cdot \frac{||x^*||}{||x||} + \frac{||E_W||}{||W_I||}\Big),
		\end{align}
		where condition number $\kappa(B_I) := ||B_I||\cdot ||{B_I}^{-1}||$.
		
		Moreover, according to Lemma \ref{l2}, we have
		\[
		B_I \cdot (x - x^*) = E_B\cdot x^* -E_W. 
		\]
		A simplification of the above equation, we have
		\[
		x + ({B_I}^{-1} \cdot E_B - I)\cdot x^* = - {B_I}^{-1} \cdot E_W.
		\]
		Take a norm on both sides, we have 
		\[
		\|x \| + \|{B_I}^{-1} \cdot E_B - I\|\cdot \|x^*\| \geq 0.
		\]
		Since $\|{B}_{I}^{-1} \cdot E_B\|<1$, we have 
		\begin{align}\label{e16}
		\frac{\|x^*\|}{\| x \|} \leq \frac{1}{1- \|{B_{I}}^{-1} \cdot E_B\|}.
		\end{align}
		Combine Equation (\ref{e15}) and Equation (\ref{e16}), we get Equation (\ref{e13}).
		%Consider the iterative refinement, 
		%\begin{align}\label{io}
		%B_I \cdot x_{k+1} - E_B \cdot x_k = W_I.
		%\end{align}
		%Since $\bar{A}_I$ is nonsingular, Equation \ref{io} is equivalent to $o_{k+1} - o^* = \bar{A}_I^{-1} \cdot E \cdot (o_k-o^*).$ It implies $||o_{k+1}-o^*|| \leq ||\bar{A}_I^{-1} \cdot E||\cdot ||o_k-o^*||$. Since $||\bar{A}_I^{-1}\cdot E||<1$, $o_k$ approaches to $x^*$ when $k\rightarrow \infty $.
	\end{proof}
	\begin{remark} 
		$\|{B}_{I}^{-1} \cdot E_B\|<1$ alone is a \emph{necessary} condition for the iterative reconstruction algorithm to converge. In other words,  a big perturbation with $|| E_B || > || {B}_{I} ||$, such as Gaussian noise with a sufficiently big variance,  
		is guaranteed to defeat reconstruction attacks like \cite{DeepLeakage_Han19}. 
		%effectively defeats reconstruction attacks. 
		
		%2) For the sake of high model accuracy, %the magnitudes of $E$ should adapt to $|\bar{A}_I|$ in different ephocs during the learning process. 
		%Concretely, 
		%since $|\bar{A}_I|$ often descreases with the increasing learning steps $T$, 
		%it is preferable to set adaptive perturbations e.g. with smaller Gaussian noise at the later stage of the learning process. Improved model accuracies are empirically justified by results in Section ?. 
		%=======
		%	2) Moreover, for the sake of a balanced trade-off between privacy-preserving and model accuracy, the magnitudes of $E$ should adapt to $|\bar{A}_I|$ in different epochs during the learning process. 
		%	Concretely, since $|\bar{A}_I|$ often decreases with the increasing learning steps $T$, it is preferable to choose smaller variance $\sigma$ for additive Gaussian noise. 
		%	This adaptive setting of Gaussian noise magnitudes improves model accuracies as empirecally justified by results in Section ?. 
		%>>>>>>> 5f00d60f48375072002c1fc21a2aad072ce7783b
		
		%In sharp contrast, differential privacy based schemes suggest that larger $\sigma$ should be used with the increasing learning steps $T$ (e.g. $\sigma = \Omega(q\sqrt{T log(1/\delta)log(T/\delta)}/\epsilon)$ in \cite{DLDP_Abadi16}). 
		%	
		%3) Even if$|| E_B || > || {B}_{I} ||$, iterative reconstruction algorithms may converge to a solution $o_k$ that is far from the original input data $o^*$, depending on the initial setting of  $o_0$ (see discussion in Section ? and \cite{Gradient_Leakage/wei2020}). 
		%	
	\end{remark}
	
	%\begin{remark} A big perturbation $E$, such as Gaussian noise with a big variance, in the linear system yields the ineffectiveness of minimization of gradient differences. 
	%\end{remark}
	\section*{Appendix B: Polarization Loss}
	
	\begin{defn}[Polarization loss]
		For each data $\textbf{x} \in \mathcal{X}$ and its corresponding output vector $\textbf{v}:= {\Psi}(\textbf{x}; \textbf{w}) \in \mathbb{R}^K$, the polarization loss is defined on the vector $\textbf{v}$ with respect to a pre-set target binary code $\textbf{t} \in \mathcal{H}$ as follows,
		\begin{align}\label{p-loss}
		\mathcal{L}_{\text{P}}(\textbf{v}, \textbf{t}):= \sum_{i=1}^K \max(m-{v}_i\cdot {t}_i, 0),
		\end{align}
		where the margin threshold is pre-set, $m\ge 1$, for the bound in Lemma \ref{polarlemma} to be strict.
	\end{defn}
	
	%By minimizing the polarization loss (Eq. \ref{p-loss}) during the learning phase,  magnitudes of each DPN output channels are induced above the threshold $m$ while corresponding signs are aligned to the target vector $\textbf{t}$. %\textcolor{red} Figure \ref{fig:distr-v} illustrates the distribution of outputs $\textbf{v}$, for example images fed to a DPN. Clearly large margins push the network outputs further away from zero. It must be noted that, the outputs for correctly coded images are polarized while non-polarized outputs are more likely observed for mis-classified ones.  
	
	\begin{lemma}\label{polarlemma}
		For output vector $\textbf{v}= {\Psi}(\textbf{x}; \textbf{w})$, the Hamming distance $\mathcal{D}_h(\textbf{b}, \textbf{t}):= \frac{1}{2}( K - \textbf{b} \cdot \textbf{t})$ between $K$-bits binary hash code $\textbf{b} = \mathtt{Bin}(\textbf{v})$ and the corresponding binary vector \textbf{t} is {upper bounded} by the \textit{polarization loss}  
		\begin{align}\label{e1}
		\mathcal{D}_h(\textbf{b}, \textbf{t}) \leq \mathcal{L}_{\text{P}}(\textbf{v}, \textbf{t}),
		\end{align}
		for any $m\ge 1$ and $\textbf{v} \in \{ (v_1,\cdots,v_K) \big| v_k \in \mathbb{R} \}$.
	\end{lemma}
	\begin{proof}
		On one side, there are two cases for each coordinate of Hamming distance $\mathcal{D}_h(\textbf{b}, \textbf{t})$, 
		\[
		|b_i - t_i| \in \{0, 2\} \hspace*{0.5em} \text{ if $v_i\cdot t_i >0$ or $v_i\cdot t_i \leq 0$;}
		\]
		On the other side, both above cases are upper bounded by $\max(m-{v}_i\cdot {t}_i, 0)$ provided that if any $m\ge 1$. 
		
		Sum up the residues of each coordinate, we get this lemma. 
	\end{proof}
	
	\begin{prop}\label{p1}
		Suppose class $\mathcal{C}$ consists of data points $\{\textbf{x}_1, \cdots, \textbf{x}_{|\mathcal{C}|}\}$ associated with a pre-set target $\textbf{t} \in \mathcal{H}$ in Hamming space. The averaged \textit{intra-class} pairwise Hamming distances among the corresponding binary codes $\{\textbf{b}_1, \cdots, \textbf{b}_{|\mathcal{C}|}| \textbf{b}_i = \mathbf{\Phi}(\textbf{x}_i; \textbf{w}) \}$ is upper bounded by,
		\begin{align}\label{e4}
		\frac{1}{|\mathcal{C}|^2} \cdot \sum_{1\leq i, j \leq |\mathcal{C}|} \mathcal{D}_h (\textbf{b}_i, \textbf{b}_j) \leq \frac{2}{|\mathcal{C}|} \cdot \sum_{1\leq i \leq |\mathcal{C}|} \mathcal{L}_{\text{P}}(\textbf{v}_i, \textbf{t}).
		\end{align}
	\end{prop}
	
	\begin{proof}
		According to Lemma \ref{polarlemma} and the triangle law, we have 
		\begin{align*}
		\sum_{1\leq i, j \leq |\mathcal{C}|} \mathcal{D}_h (\textbf{b}_i, \textbf{b}_j) &\leq  \sum_{1\leq i, j \leq |\mathcal{C}|} \mathcal{D}_h (\textbf{b}_i,  \textbf{t}) +  \mathcal{D}_h (\textbf{b}_j,  \textbf{t})\\
		&\leq  \sum_{1\leq i, j \leq |\mathcal{C}|} \mathcal{L}_{\text{P}}(\textbf{v}_i, \textbf{t}) +  \mathcal{L}_{\text{P}}(\textbf{v}_j, \textbf{t}) \\
		&\leq 2 |\mathcal{C}| \cdot  \sum_{1\leq i \leq |\mathcal{C}|} \mathcal{L}_{\text{P}}(\textbf{v}_i, \textbf{t}).
		\end{align*}
		Divide $|\mathcal{C}|^2$ on both sides, we get this proposition.
	\end{proof}	
	
	\begin{prop}\label{p2}
		Suppose there are $L$ classes in the dataset, i.e. $\mathcal{C}_1, \cdots, \mathcal{C}_L$. For any two classes $\mathcal{C}_x$ and $\mathcal{C}_y$ $(1 \le x \neq y \le L)$, respectively, with associated targets binary vectors $\textbf{t}_x$ and $\textbf{t}_y$ and binary hash codes $\textbf{b}_i^x = \mathbf{\Phi}(\mathbf{x}_i; \mathbf{w}), i\in \{1,\cdots, |\mathcal{C}_x| \}$, $\textbf{b}_i^y =  \mathbf{\Phi}(\mathbf{y}_j; \mathbf{w}), j\in\{1,\cdots,  |\mathcal{C}_y|\}$, the averaged \textit{inter-class} pairwise Hamming distances among binary codes $\sum_{\substack{1\leq i \leq |\mathcal{C}_x|, \\ 1 \leq j \leq |\mathcal{C}_y|}} \mathcal{D}_h (\textbf{b}^x_i, \textbf{b}^y_j)$ is lower bounded by,
		\begin{align}
		\sum_{1\leq x\neq y \leq L}\Big( \mathcal{D}_h (\textbf{t}_x, \textbf{t}_y) - \frac{1}{|\mathcal{C}_x|\cdot |\mathcal{C}_y|}\cdot \sum_{\substack{1\leq i \leq |\mathcal{C}_x|, \\ 1 \leq j \leq |\mathcal{C}_y|}} \mathcal{D}_h (\textbf{b}^x_i, \textbf{b}^y_j)\Big) \leq \sum_{1\leq x \leq L} \frac{2\cdot(L-1)}{|\mathcal{C}_x|} \cdot \sum_{1\leq i \leq |\mathcal{C}_x|} \mathcal{L}_{\text{P}}(\textbf{v}^x_i, \textbf{t}_x). \label{e6}
		\end{align}
	\end{prop}
	
	\begin{proof}
		By the triangle law, we have 
		\[
		\mathcal{D}_h (\textbf{t}_x, \textbf{t}_y) \leq \mathcal{D}_h (\textbf{t}_x, \textbf{b}^x_i) + \mathcal{D}_h (\textbf{b}^x_i, \textbf{b}^y_j) +  \mathcal{D}_h (\textbf{b}^y_j, \textbf{t}_y) 
		\].
		Fix $x, y$ and sum over $i, j$ on both sides, we have
		\begin{align*}
		|\mathcal{C}_x| \cdot |\mathcal{C}_y| \cdot \mathcal{D}_h (\textbf{t}_x, \textbf{t}_y) - \sum_{\substack{1\leq i \leq |\mathcal{C}_x|, \\ 1 \leq j \leq |\mathcal{C}_y|}} \mathcal{D}_h (\textbf{b}^x_i, \textbf{b}^y_j) &\leq |\mathcal{C}_y| \cdot \sum_{1\leq i \leq |\mathcal{C}_x|} \mathcal{D}_h (\textbf{t}_x, \textbf{b}^x_i)  +|\mathcal{C}_x| \cdot \sum_{1\leq j \leq |\mathcal{C}_y|} \mathcal{D}_h (\textbf{b}^y_j, \textbf{t}_y) \\
		&\leq  |\mathcal{C}_y| \cdot \sum_{1\leq i \leq |\mathcal{C}_x|} \mathcal{L}_{\text{P}}(\textbf{v}^x_i, \textbf{t}_x) +|\mathcal{C}_x| \cdot \sum_{1\leq j \leq |\mathcal{C}_y|}  \mathcal{L}_{\text{P}}(\textbf{v}^y_j, \textbf{t}_y). \\
		\end{align*}
		Divide $|\mathcal{C}_x| \cdot |\mathcal{C}_y|$ and sum over $x, y$ on both sides, we have 
		\begin{align*}
		&\sum_{1\leq x\neq y \leq L}\Big( \mathcal{D}_h (\textbf{t}_x, \textbf{t}_y) - \frac{1}{|\mathcal{C}_x|\cdot |\mathcal{C}_y|}\cdot \sum_{\substack{1\leq i \leq |\mathcal{C}_x|, \\ 1 \leq j \leq |\mathcal{C}_y|}} \mathcal{D}_h (\textbf{b}^x_i, \textbf{b}^y_j)\Big) \\
		\leq &\sum_{1\leq x\neq y \leq L} \frac{1}{|\mathcal{C}_x|} \cdot \sum_{1\leq i \leq |\mathcal{C}_x|} \mathcal{L}_{\text{P}}(\textbf{v}^x_i, \textbf{t}_x) +  \frac{1}{|\mathcal{C}_y|} \cdot \sum_{1\leq j \leq |\mathcal{C}_y|} \mathcal{L}_{\text{P}}(\textbf{v}^y_j, \textbf{t}_y)\\
		\leq & \sum_{1\leq x \leq L} \frac{2\cdot(L-1)}{|\mathcal{C}_x|} \cdot \sum_{1\leq i \leq |\mathcal{C}_x|} \mathcal{L}_{\text{P}}(\textbf{v}^x_i, \textbf{t}_x).
		\end{align*}
	\end{proof}
	
	\begin{prop}\label{p3}
		The \textit{difference} between averaged \textit{intra-class} pairwise Hamming distance and averaged \textit{inter-class} pairwise Hamming distance is upper bounded, i.e. 
		\begin{align}
		& \sum_{1\leq x \leq L} \frac{1}{|\mathcal{C}_x|^2} \cdot \sum_{1\leq i, j \leq |\mathcal{C}_x|} \mathcal{D}_h (\textbf{b}_i^x, \textbf{b}_j^x) \nonumber -\sum_{1\leq x\neq y \leq L} \frac{1}{|\mathcal{C}_x|\cdot |\mathcal{C}_y|} \sum_{\substack{1\leq i \leq |\mathcal{C}_x|, \\ 1 \leq j \leq |\mathcal{C}_y|}} \mathcal{D}_h (\textbf{b}^x_i, \textbf{b}^y_j) \nonumber \\
		\leq & \sum_{1\leq x \leq L} \frac{2\cdot L}{|\mathcal{C}_x|} \cdot \sum_{1\leq i \leq |\mathcal{C}_x|} \mathcal{L}_{\text{P}}(\textbf{v}^x_i, \textbf{t}_x) -\sum_{1\leq x\neq y \leq L}\mathcal{D}_h (\textbf{t}_x, \textbf{t}_y). \label{eq-all-bound} 
		\end{align}
	\end{prop}
	
	\begin{proof}
		By Lemma \ref{p1} and Lemma \ref{p2}, we directly get this proposition. 
	\end{proof}
	
	\begin{remark}
		1) Inequality in (Eq. \ref{e4}) shows that the averaged polarization loss is a strict upper-bound of the averaged pairwise Hamming distances between points of the same class. That is to say, minimizing the RHS of (Eq. \ref{e4}) effectively \textit{minimizes} the averaged \textit{intra-class} pairwise Hamming distances.
		
		2) In terms of the computational complexity, pairwise Hamming distances on the LHS of (Eq. \ref{e4}) is $O(|\mathcal{C}|^2)$ while the polarization loss on the RHS of (Eq. \ref{e4}) is $O(|\mathcal{C}|)$ only.
		
		3) Inequality in (Eq. \ref{e6}) shows that minimizing polarization losses on the RHS of (Eq. \ref{e6}) effectively \textit{maximizes} the averaged \textit{inter-class} pair-wised Hamming distances on LHS.
		
		4) According to Proposition \ref{p3}, the optimization problem of simultaneous minimizing the intra-class and maximizing inter-class Hamming distances, i.e. 
		\begin{align*}
		\min_{\textbf{w}} \hspace*{1em} &\sum_{1\leq x \leq L} \frac{1}{|\mathcal{C}_x|^2} \cdot \sum_{1\leq i, j \leq |\mathcal{C}_x|} \mathcal{D}_h (\textbf{b}_i^x, \textbf{b}_j^x) - \sum_{1\leq x\neq y \leq L} \frac{1}{|\mathcal{C}_x|\cdot |\mathcal{C}_y|} \sum_{\substack{1\leq i \leq |\mathcal{C}_x|, \\ 1 \leq j \leq |\mathcal{C}_y|}} \mathcal{D}_h (\textbf{b}^x_i, \textbf{b}^y_j), 
		\end{align*} 
		is equivalent to the problem of  minimizing the averaged polarization loss over the whole data set, i.e.
		\begin{align*} 
		\min_{\textbf{w}}  \sum_{1\leq x \leq L} \frac{1}{|\mathcal{C}_x|} \cdot \sum_{1\leq i \leq |\mathcal{C}_x|} \mathcal{L}_{\text{P}}(\textbf{v}^x_i, \textbf{t}_x).
		\end{align*}
		
	\end{remark}

	\newpage
	\section*{Appendix C: Experiment Setup}
	
	\textbf{Dataset}
	
	In our experiments, we used MNIST, CIFAR10, CIFAR100 and SVHN, which are used in previous PPDL studies.
	
	\textbf{Network Architecture} 
	
	In our experiments, we used AlexNet, VGG16 and DLNet (from \cite{DeepLeakage_Han19}).
%	Also we have changed activation layer from ReLU to Sigmoid on AlexNet for privacy attacks.

	For AlexNet and VGG16, we slightly modified the architecture implementation from \textit{torchvision}\footnote{\url{https://pytorch.org/docs/stable/torchvision/models.html}} package to adapt a 32$\times$32 input. In VGG16, we added Group Normalization after every Convolution layer. (See Table \ref{tab:netalexnet} and \ref{tab:netvgg16})

	For privacy attack analysis, we used network architecture (DLNet) from released code \footnote{\url{https://github.com/mit-han-lab/dlg}\label{footnote:dlg}}.
	
	\textbf{Privacy-Preserving Mechanisms} 

	For \textbf{DP}, we are using implementation from \textit{pytorch-dp} package. Slightly modified to adapt to privacy attack analysis. (We disabled the gradient clipping function.)

	For \textbf{PPDL}, we reimplemented using reference from author released code \footnote{\url{https://www.comp.nus.edu.sg/~reza/files/PPDL.zip}\label{footnote:ppdl}}.
	
	For our \textbf{SPN}, we used $\alpha_{1}=1$ in all of our experiments. We random initialize the private target $\bm{t}$, and using 64-bit in all of our experiments. (See Algorithm \ref{alg:client})
	
	\textbf{Privacy Attacks} 
	
	For \textbf{reconstruction attacks}, we adopt author released code \footref{footnote:dlg} to reconstruct images. We follow their implementation which we random initialized the model for reconstruction attack. (See Algorithm \ref{alg:recon}) 
	
	For \textbf{membership attacks}, we are using same algorithm from reconstruction attacks. (See Algorithm \ref{alg:recon})
	
	For \textbf{tracing attacks}, first, we perform reconstruction attacks to recovered $X$ number of images, we used $X=1000$ in our experiments. Then we separated reconstructed images into $N$ partitions simulating $N$ participants, we used $N=10$ in our experiments. During tracing, we trace the query image from the reconstructed dataset. The query images is the dataset that used for reconstruction attacks. We are using full query dataset (e.g. 50000 images for CIFAR10) for tracing. (See Algorithm \ref{alg:tracing})
	
	\textbf{Federated Learning}. 
	
	The federated learning environment is run with both IID and Non-IID dataset. 
	For IID case, we uniformly split training datasets into $N$ partitions (with same number of data per class), respectively, for $N$ participants, and use all testing datasets for evaluation of the global model performances. 
	For Non-IID dataset, we follows their implementation \footnote{\url{https://github.com/ebagdasa/backdoor_federated_learning}} to separate the dataset into $N$ participants using Dirichlet distribution with $\alpha=0.9$.
	
	For DLNet, we are using round robin for model aggregation following implementation from \footref{footnote:ppdl} (See Algorithm \ref{alg:training-rr}). Otherwise, we are using FedAvg algorithm for model aggregation, which is following the implementation in \cite{communicationEfficient/mcMahan2017} and using source code from this \footnote{\url{https://github.com/shaoxiongji/federated-learning}} GitHub repository as reference. 
	
	Table \ref{tab:hyperparams-attack} and \ref{tab:hyperparams-fl} summarized the hyperparameters we used in this paper.
	
	\newpage
	
	\begin{table}[H]
		\begin{tabular}{l|c}
			\toprule
			Hyperparameter & Privacy Attack Analysis\\ 
			\midrule
			\multicolumn{2}{c}{Training Hyperparameters} \\ 
			\midrule
			Dataset & MNIST, CIFAR10, CIFAR100, SVHN \\
			Network Architecture & DLNet \cite{DeepLeakage_Han19} \\
			Weight Initialization & $uniform(-0.3, 0.3)$ \\
			\midrule
			Optimization method & Adam \\
			Optimizer Hyperparameter & Adam ($\beta_1=0.9, \beta_2=0.999$) \\
			Learning rate & 0.001 \\
			Learning rate decay & No decay \\
			Batch size & 32 \\
			\midrule
			Local Epochs/Global Communication Rounds & 1/300 \\
			Number of Clients & 10 \\
			\midrule
			\multicolumn{2}{c}{Privacy-Preserving Hyperparameters} \\
			\midrule
			SPN number of bit & 64 \\ 
			SPN $\alpha_{2}$ & 0.0001, 0.001, 0.01, 0.1, 0.2, 0.3, 0.4, 0.5 \\
			PPDL shared percentage & 5\%, 30\% \\
			DP noise $\sigma$ & 0.0001, 0.001, 0.01, 0.1, 0.5 \\
			\midrule
			\multicolumn{2}{c}{Deep Leakage Attack Hyperparameters} \\ 
			\midrule
			Attack Batch Size & 1, 4, 8 \\
			SPN $\alpha_2$ & 0.0001, 0.001, 0.01, 0.1, 0.2, 0.3, 0.4, 0.5 \\
			PPDL shared percentage & 5\%, 30\% \\
			DP noise $\sigma$ & 0.0001, 0.001, 0.01, 0.1, 0.5 \\
			\bottomrule
		\end{tabular}
		\caption{Hyperparameters used in our privacy attack analysis.}
		\label{tab:hyperparams-attack}
	\end{table}

	\begin{table}[H]
		\begin{tabular}{l|c}
			\toprule
			Hyperparameter & Federated Learning \\ 
			\midrule
			Dataset & MNIST, CIFAR10, CIFAR100, SVHN \\
			Network Architecture & AlexNet, VGG16 \\
			Weight Initialization & $kaiming\_uniform$ \\
			\midrule
			Optimization method & SGD \\
			Optimizer Hyperparameter & Momentum = 0.9 \\
			Learning rate & 0.01, 0.001 \\
			Learning rate decay & Decay by factor of 0.5 at round 100 and 200 \\
			Batch size & 64 \\
			\midrule
			Local Epochs/Global Communication Rounds & 1/300 \\
			Number of Clients & 8, 20, 50, 100 \\
			\midrule
			\multicolumn{2}{c}{Privacy-Preserving Hyperparameters} \\
			\midrule
			SPN number of bit & 64 \\
			SPN $\alpha_{2}$ & 0.1 \\
			PPDL shared percentage & 5\%, 30\% \\
			DP noise $\sigma$ & 0.1 \\
			\bottomrule
		\end{tabular}
		\caption{Hyperparameters used in our Federated Learning Task experiments.}
		\label{tab:hyperparams-fl}
	\end{table}

	\newpage
	
	\begin{table}[H]
		\centering
		\begin{tabular}[t]{c|c|c|c}
			\toprule
			layer name & output size & weight shape & padding \\ 
			\midrule
			Conv1 & 32 $\times$ 32 & 64 $\times$ 3 $\times$ 5 $\times$ 5 & 2 \\
			MaxPool2d & 16 $\times$ 16 & 2 $\times$ 2 &  \\
			Conv2 & 16 $\times$ 16 & 192 $\times$ 64 $\times$ 5 $\times$ 5 & 2 \\
			Maxpool2d & 8 $\times$ 8 & 2 $\times$ 2 &  \\
			Conv3 & 8 $\times$ 8 & 384 $\times$ 192 $\times$ 3 $\times$ 3 & 1 \\
			Conv4 & 8 $\times$ 8 & 256 $\times$ 384 $\times$ 3 $\times$ 3 & 1 \\
			Conv5 & 8 $\times$ 8 & 256 $\times$ 256 $\times$ 3 $\times$ 3 & 1 \\
			MaxPool2d & 4 $\times$ 4 & 2 $\times$ 2 &  \\
			Linear & 256 & 256 $\times$ 4096 &  \\ 
			Linear & 10 & 10 $\times$ 256 &  \\ 
			\bottomrule
		\end{tabular}
		\caption{Modified AlexNet}
		\label{tab:netalexnet}
	\end{table}
	
	\begin{table}[H]
		\centering
		\begin{tabular}[t]{c|c|c|c}
			\toprule
			layer name & output size & weight shape & padding \\ 
			\midrule
			Conv1-GN $\times$ 2 & 32 $\times$ 32 & 64 $\times$ 64 $\times$ 3 $\times$ 3 & 1 \\
			MaxPool2d & 16 $\times$ 16 & 2 $\times$ 2 &  \\
			Conv2-GN $\times$ 2 & 16 $\times$ 16 & 128 $\times$ 128 $\times$ 3 $\times$ 3 & 1 \\
			Maxpool2d & 8 $\times$ 8 & 2 $\times$ 2 &  \\
			Conv3-GN $\times$ 3 & 8 $\times$ 8 & 256 $\times$ 256 $\times$ 3 $\times$ 3 & 1 \\
			Maxpool2d & 8 $\times$ 8 & 2 $\times$ 2 &  \\
			Conv4-GN $\times$ 3 & 8 $\times$ 8 & 512 $\times$ 512 $\times$ 3 $\times$ 3 & 1 \\
			Maxpool2d & 8 $\times$ 8 & 2 $\times$ 2 &  \\
			Conv5-GN $\times$ 3 & 8 $\times$ 8 & 512 $\times$ 512 $\times$ 3 $\times$ 3 & 1 \\
			MaxPool2d & 4 $\times$ 4 & 2 $\times$ 2 &  \\
			Linear & 256 & 256 $\times$ 4096 &  \\ 
			Linear & 10 & 10 $\times$ 256 &  \\ 
			\bottomrule
		\end{tabular}
		\caption{Modified VGG16}
		\label{tab:netvgg16}
	\end{table}
	
	\begin{table}[H]
		\centering
		\begin{tabular}[t]{c|c|c|c|c}
			\toprule
			layer name & output size & weight shape & padding & stride \\ 
			\midrule
			Conv1 & 16 $\times$ 16 & 12 $\times$ 3 $\times$ 5 $\times$ 5 & 2 & 2 \\
			Conv2 & 8 $\times$ 8 & 12 $\times$ 12 $\times$ 5 $\times$ 5 & 2 & 2 \\
			Conv3 & 8 $\times$ 8 & 12 $\times$ 12 $\times$ 5 $\times$ 5 & 2 & 1 \\
			Conv4 & 8 $\times$ 8 & 12 $\times$ 12 $\times$ 5 $\times$ 5 & 2 & 1 \\
			Linear & 10 & 10 $\times$ 768 & & \\
			\bottomrule
		\end{tabular}
		\caption{DLNet from \cite{DeepLeakage_Han19}}
		\label{tab:netdlnet}
	\end{table}
	
	\newpage
	\section*{Appendix D: Privacy-Preserving Capability}
	
	In this section, we show experiment results of Privacy-Preserving Characteristics (PPC) and Calibrated Averaged Performance (CAP) for different dataset and different attack batch size.
	During our experiment, we found out that Sigmoid activation layer is having gradient vanishing problem, causing difficulty in training model on SVHN. Therefore, we replace Sigmoid with Tanh activation layer specifically for SVHN to measure PPC and CAP, while MNIST, CIFAR10 and CIFAR100 are measured with Sigmoid.
	
	To measure PPC and CAP, we are using model trained on 10 clients using Federated Averaged algorithm as mentioned in Appendix C with batch size of 32. We are using IID dataset which we uniformly split into 10 clients.
	
	\subsection{MNIST}
	
	\subsubsection{Privacy-Preserving Characteristics (PPC)}
	
	\begin{figure}[H]	
		%	\begin{subfigure}{0.95\linewidth}
		\centering
		\begin{subfigure}{0.99\linewidth}
			\centering
			\includegraphics[scale=0.7]{imgs/legends/legend_ppc_horizontal.png}
			\\
			\includegraphics[width=0.24\linewidth]{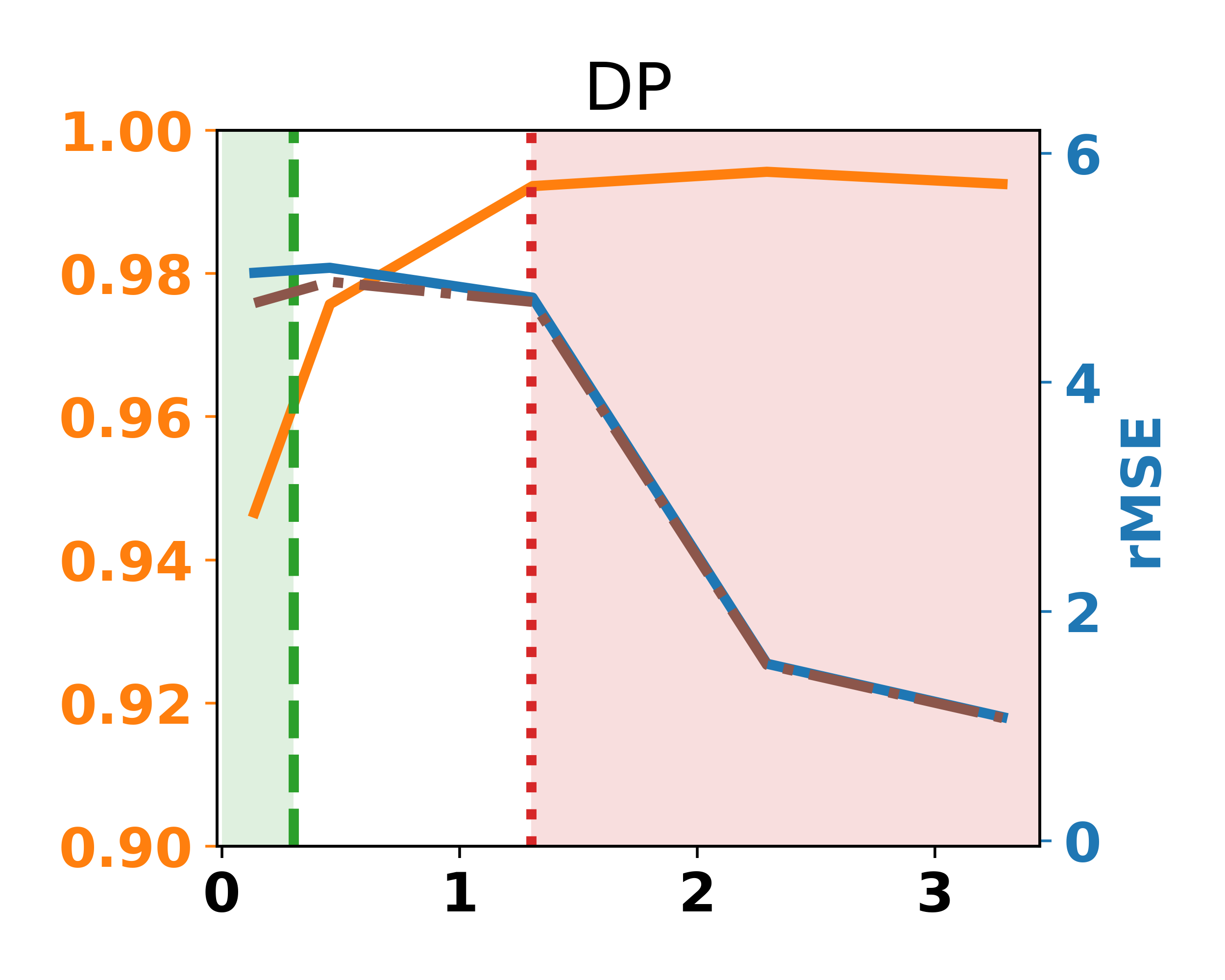}
			\includegraphics[width=0.24\linewidth]{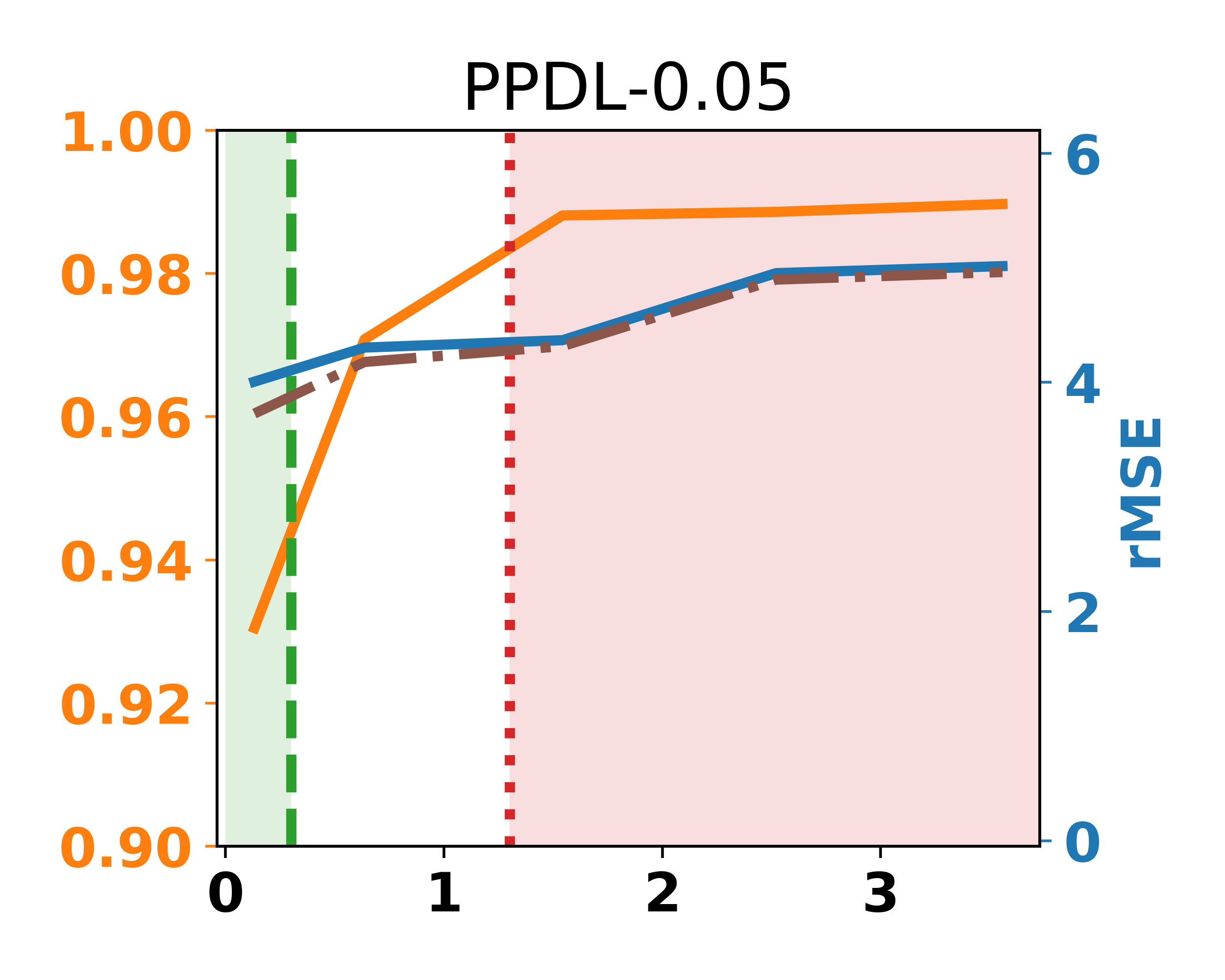}
			\includegraphics[width=0.24\linewidth]{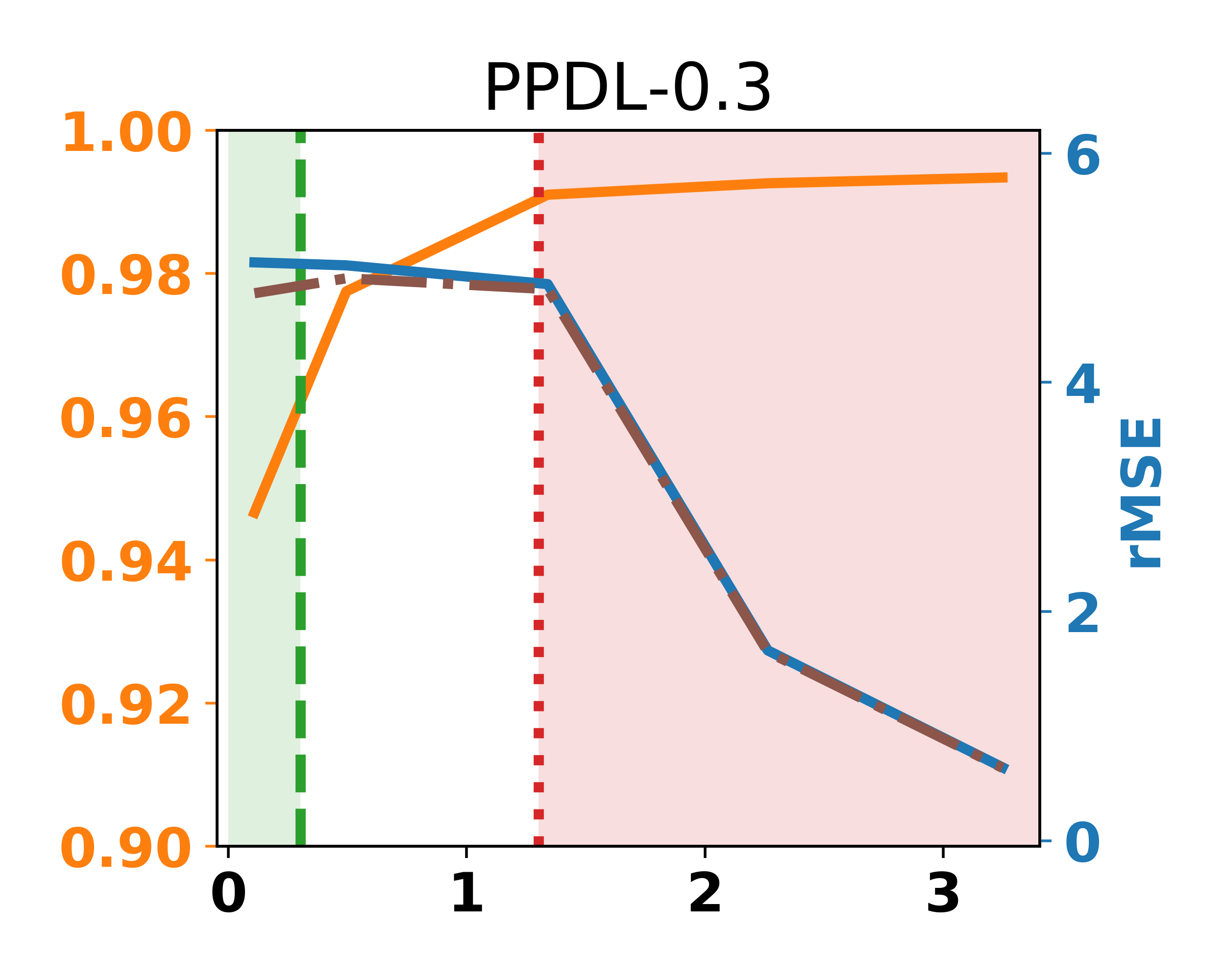}
			\includegraphics[width=0.24\linewidth]{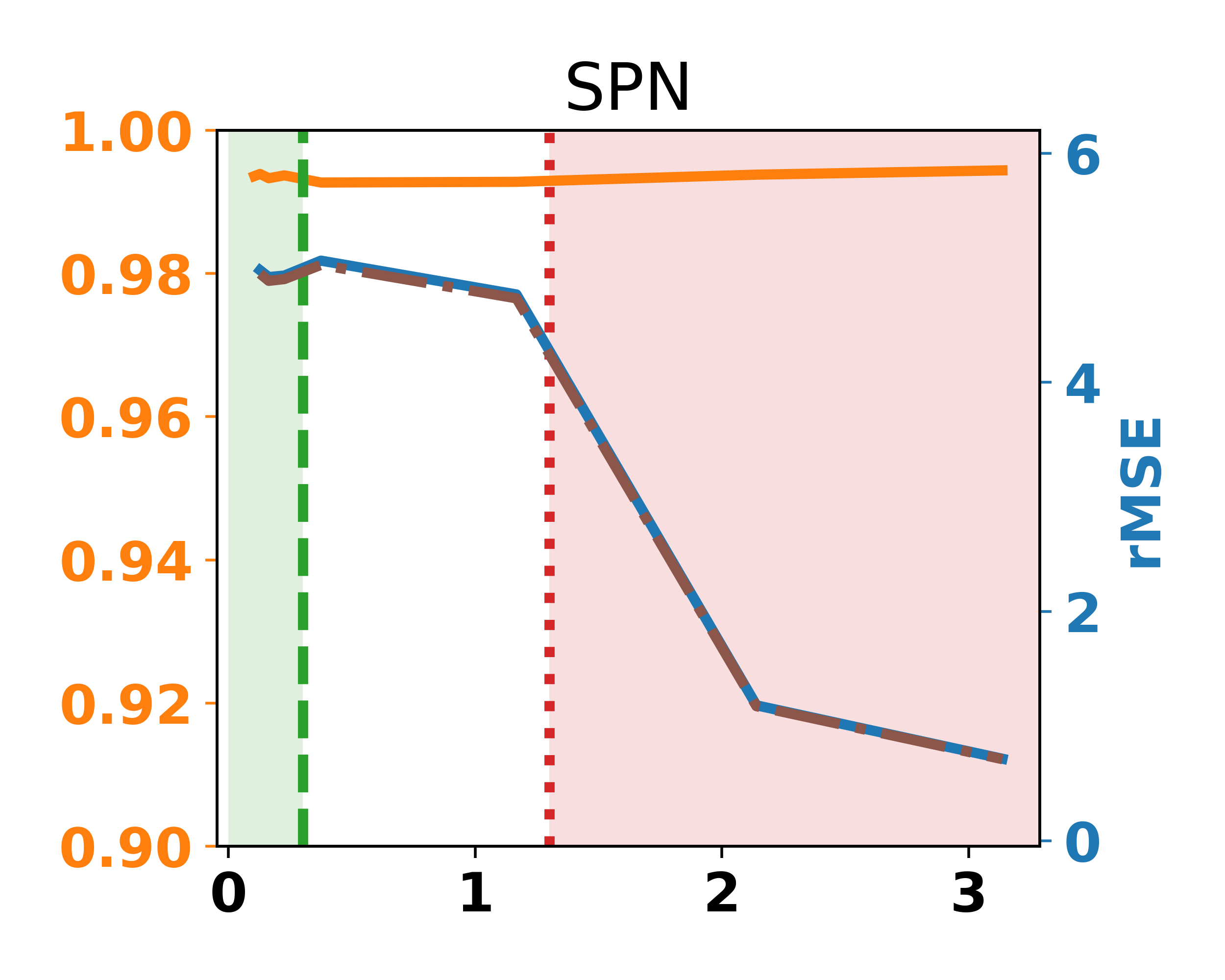}
			\caption{Reconstruction Attack}
		\end{subfigure}
		
		\begin{subfigure}{0.99\linewidth}
			\centering
			\includegraphics[width=0.24\linewidth]{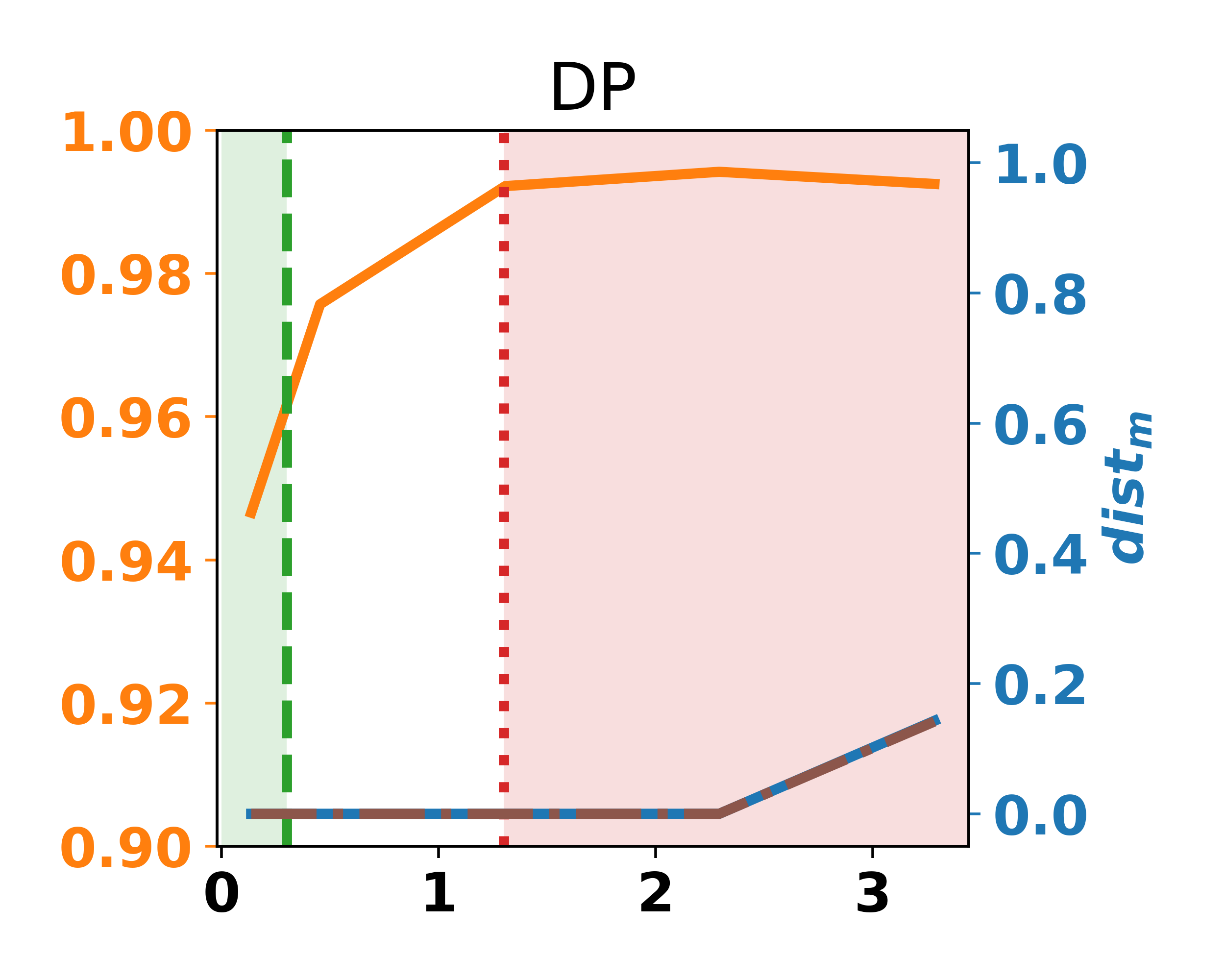}
			\includegraphics[width=0.24\linewidth]{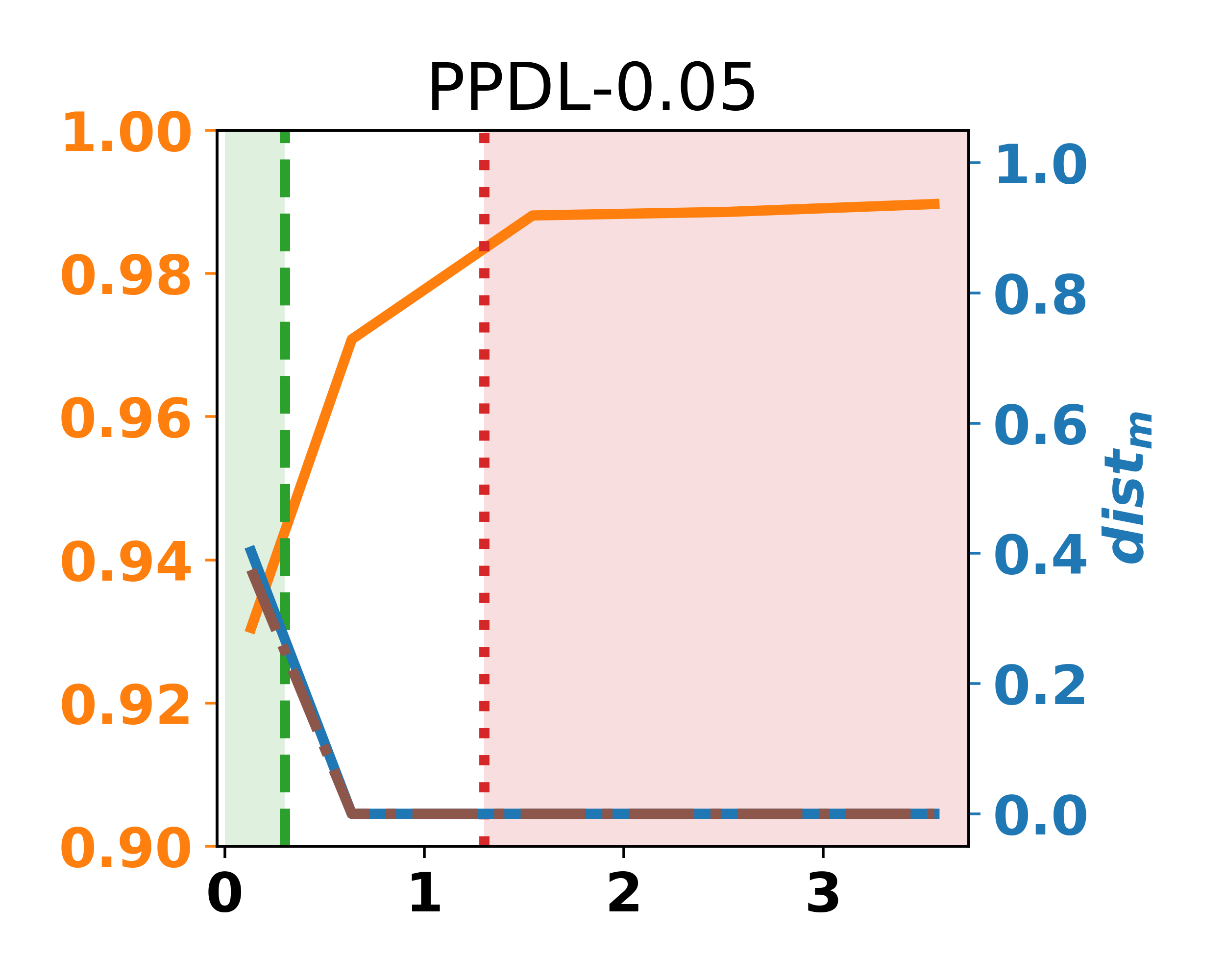}
			\includegraphics[width=0.24\linewidth]{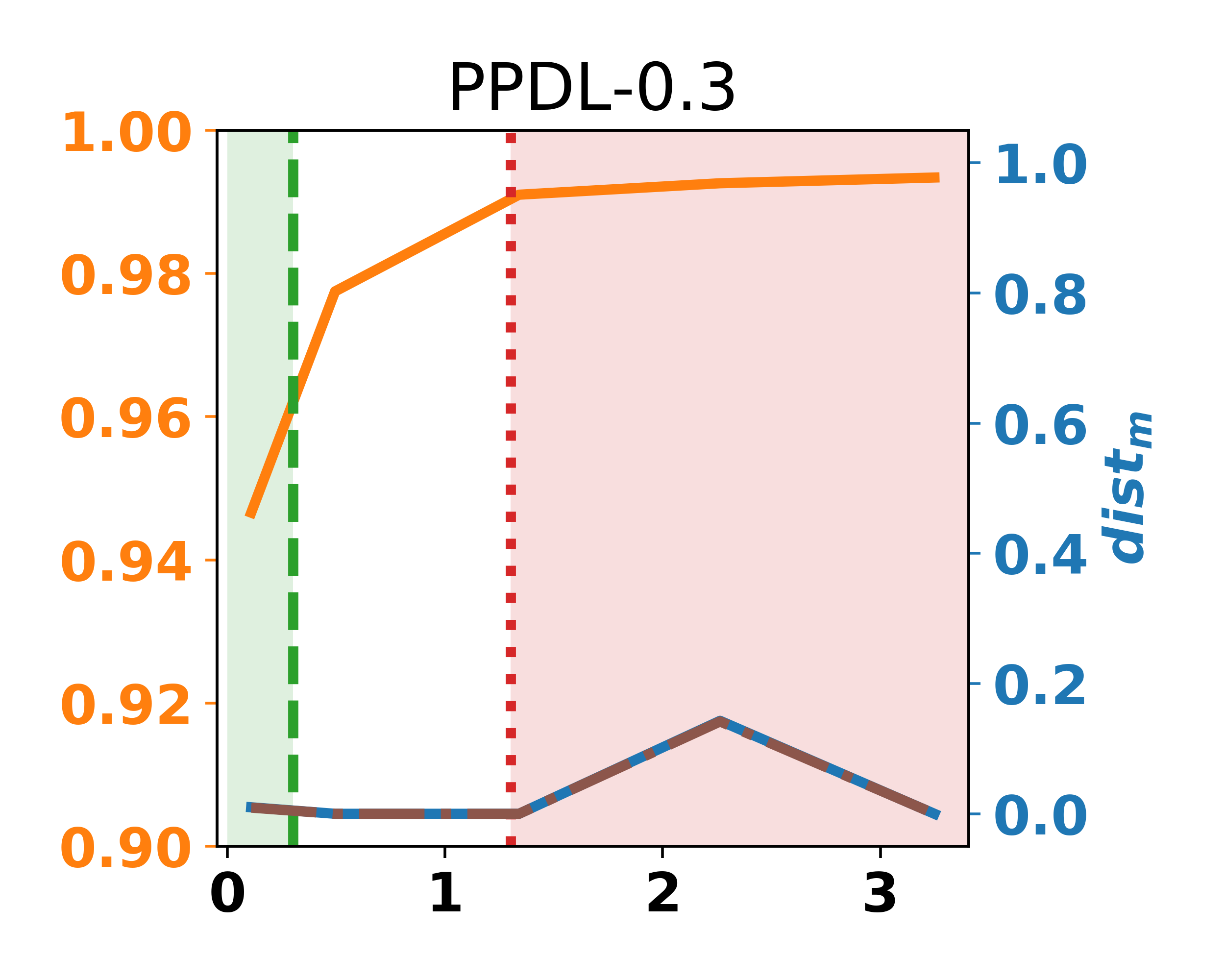}
			\includegraphics[width=0.24\linewidth]{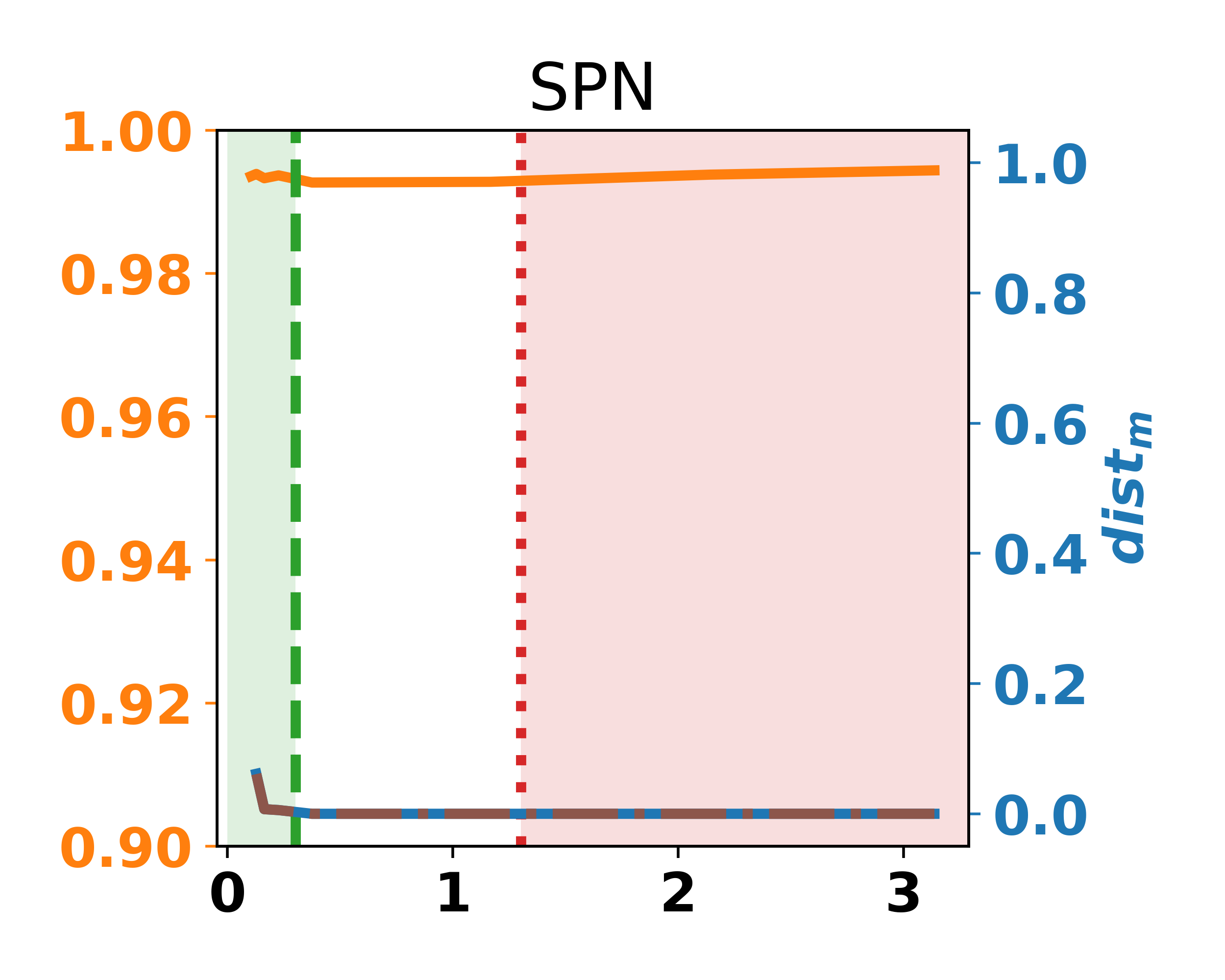}
			\caption{Membership Attack}
		\end{subfigure}
		
		\begin{subfigure}{0.99\linewidth}
			\centering
			\includegraphics[width=0.24\linewidth]{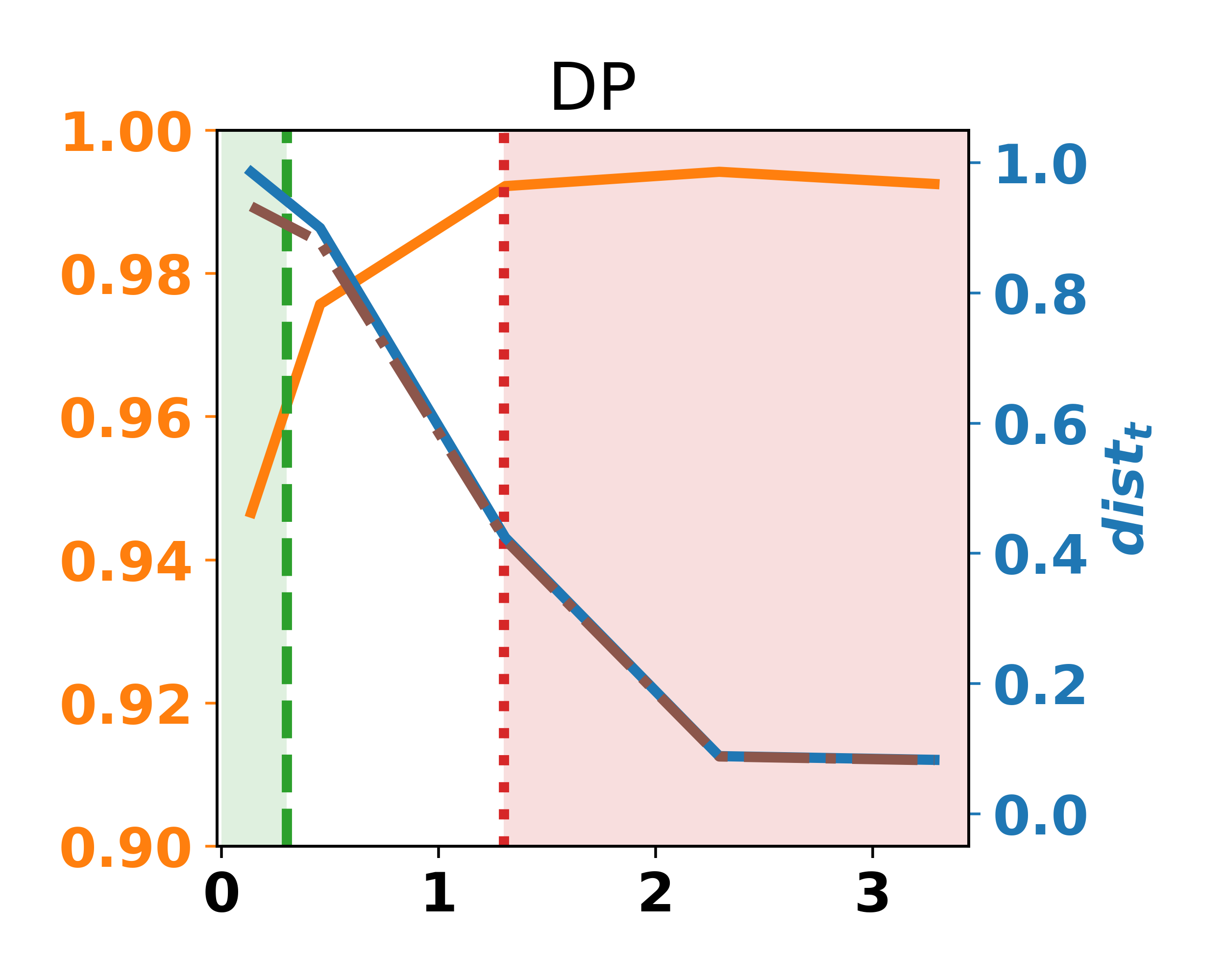}
			\includegraphics[width=0.24\linewidth]{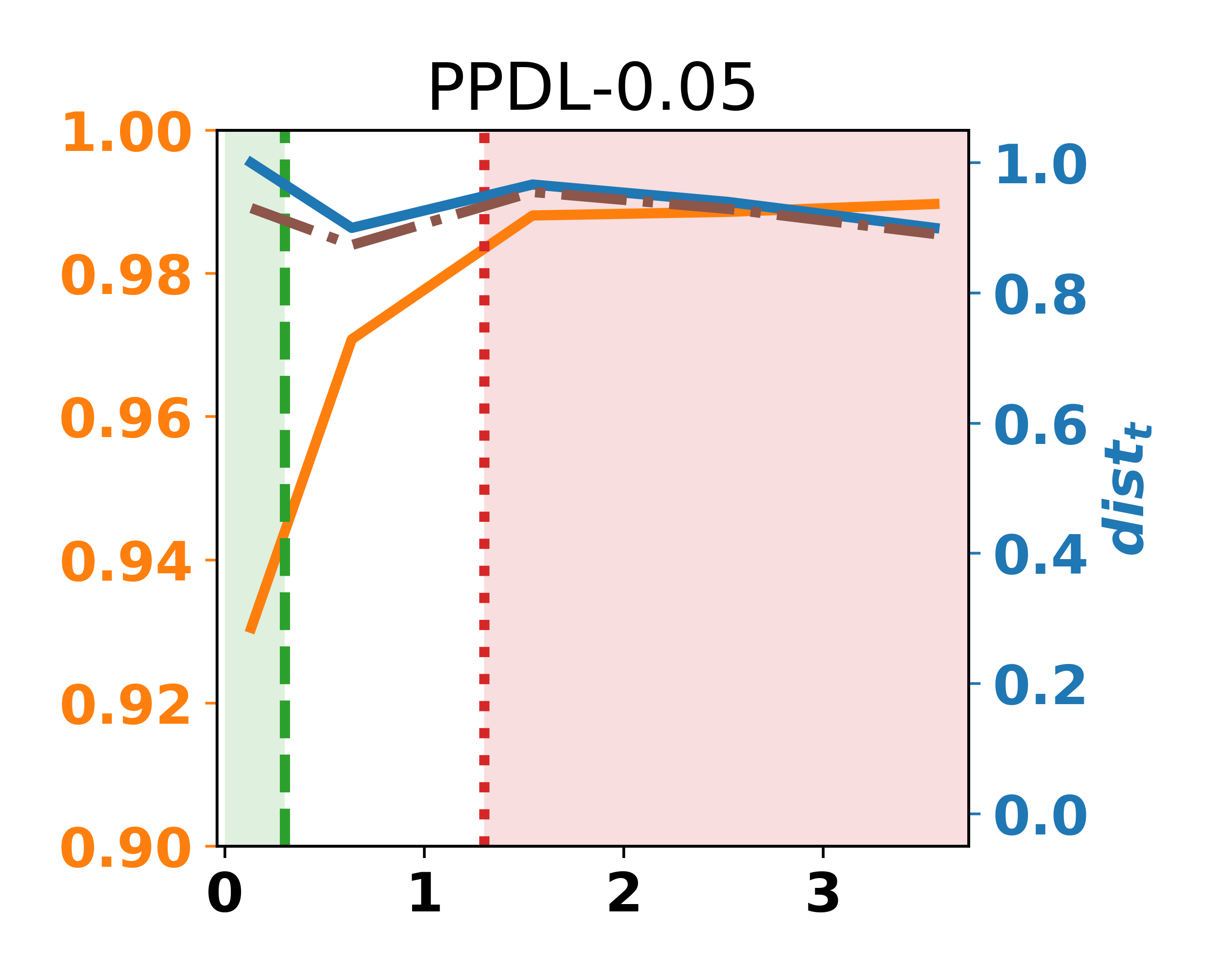}
			\includegraphics[width=0.24\linewidth]{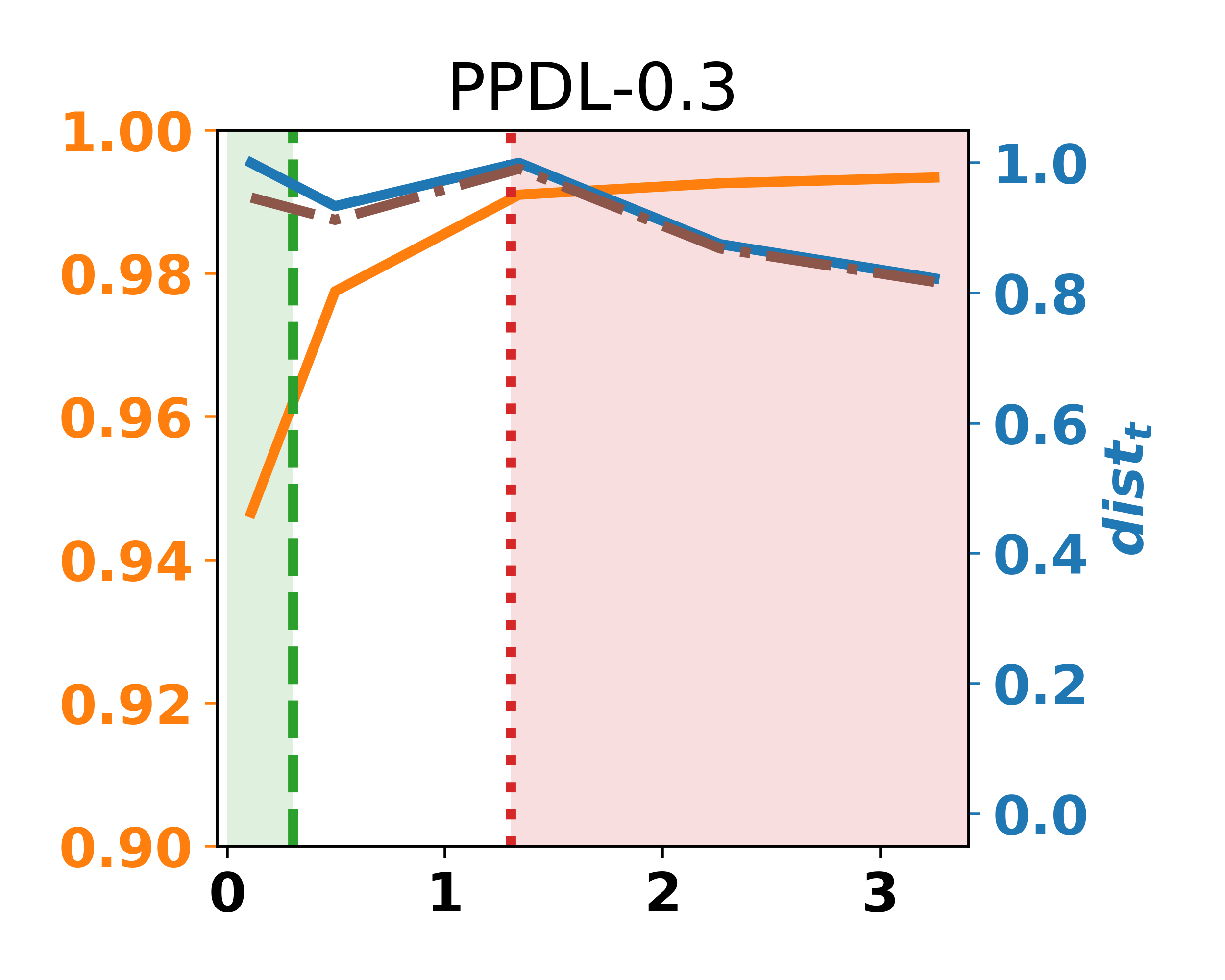}
			\includegraphics[width=0.24\linewidth]{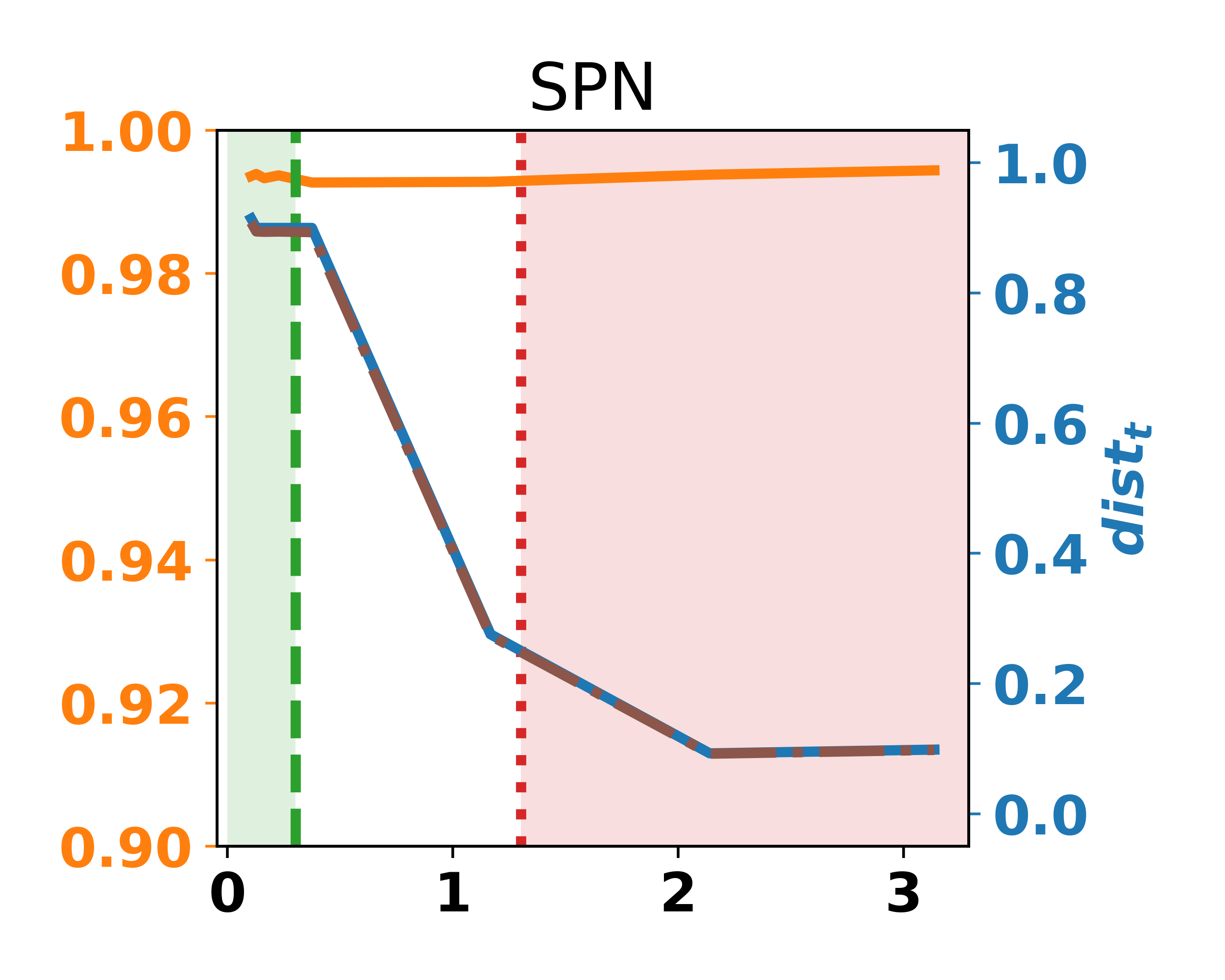}
			\caption{Tracing Attack}
		\end{subfigure}
		
		\caption{Attack with Batch Size 1}
		\label{fig:ppc-mnist-bs1}
		%\vspace{-0.22cm}
	\end{figure}

	\newpage
	
	\begin{figure}[H]	
		%	\begin{subfigure}{0.95\linewidth}
		\centering
		\begin{subfigure}{0.99\linewidth}
			\centering
			\includegraphics[scale=0.7]{imgs/legends/legend_ppc_horizontal.png}
			\\
			\includegraphics[width=0.24\linewidth]{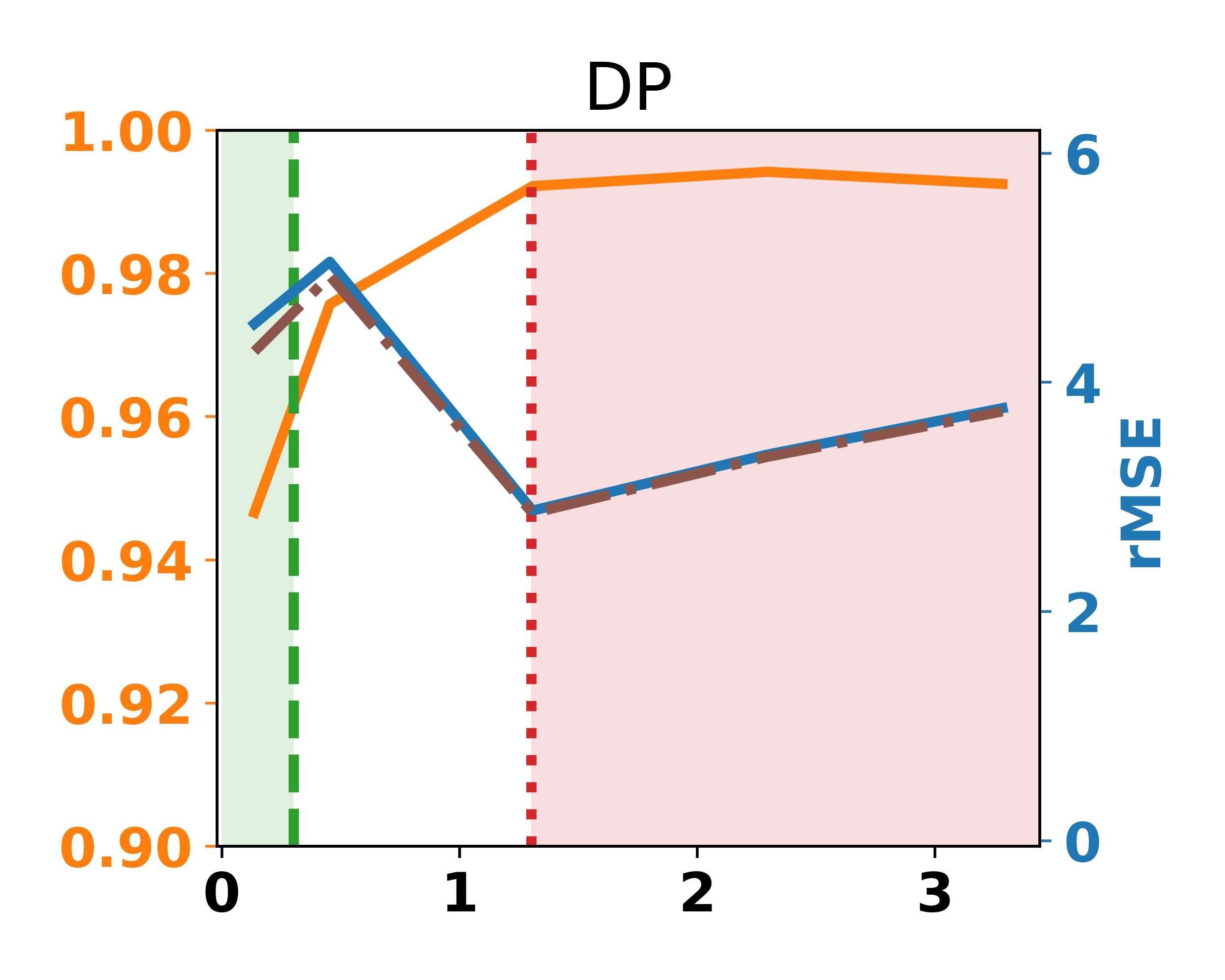}
			\includegraphics[width=0.24\linewidth]{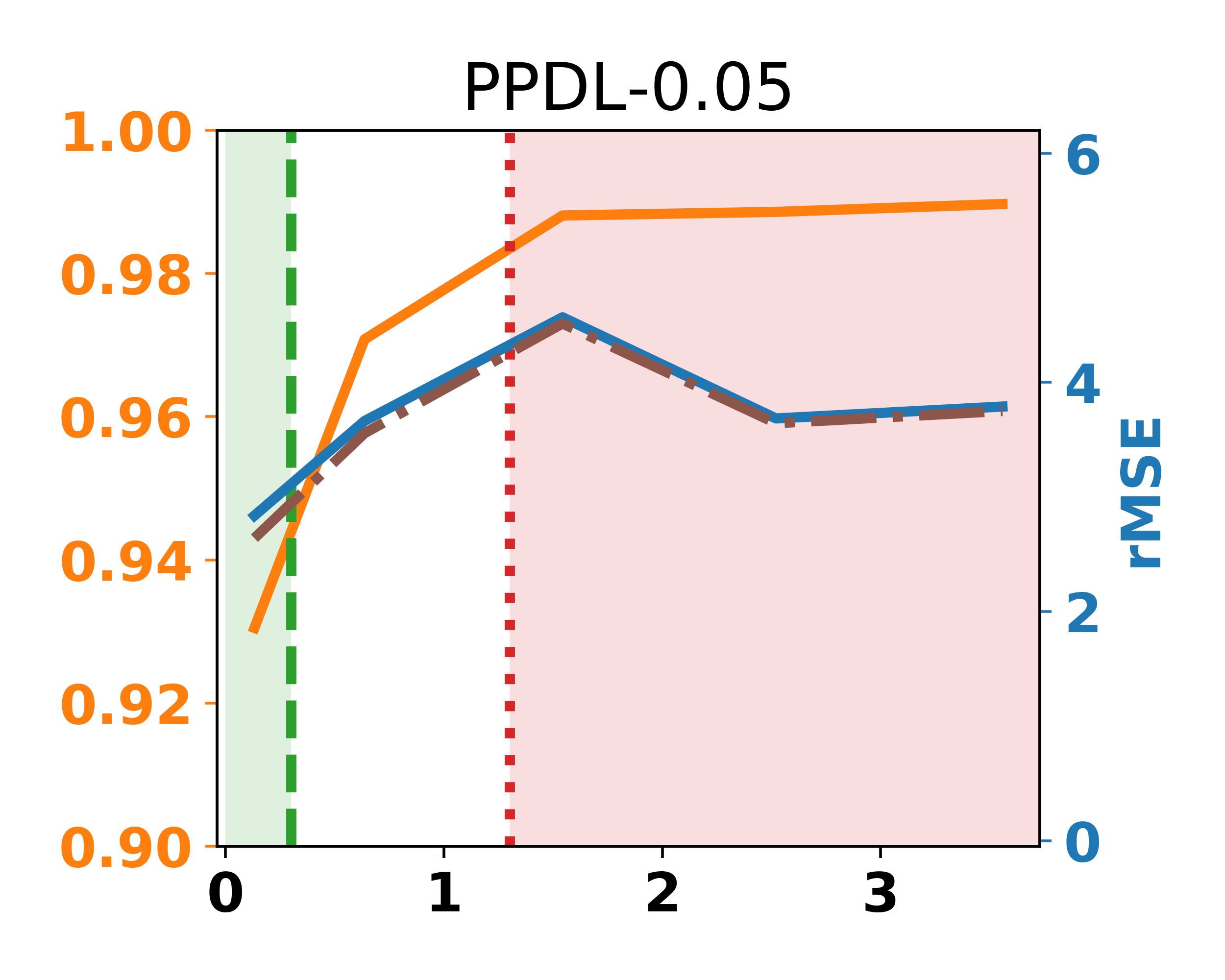}
			\includegraphics[width=0.24\linewidth]{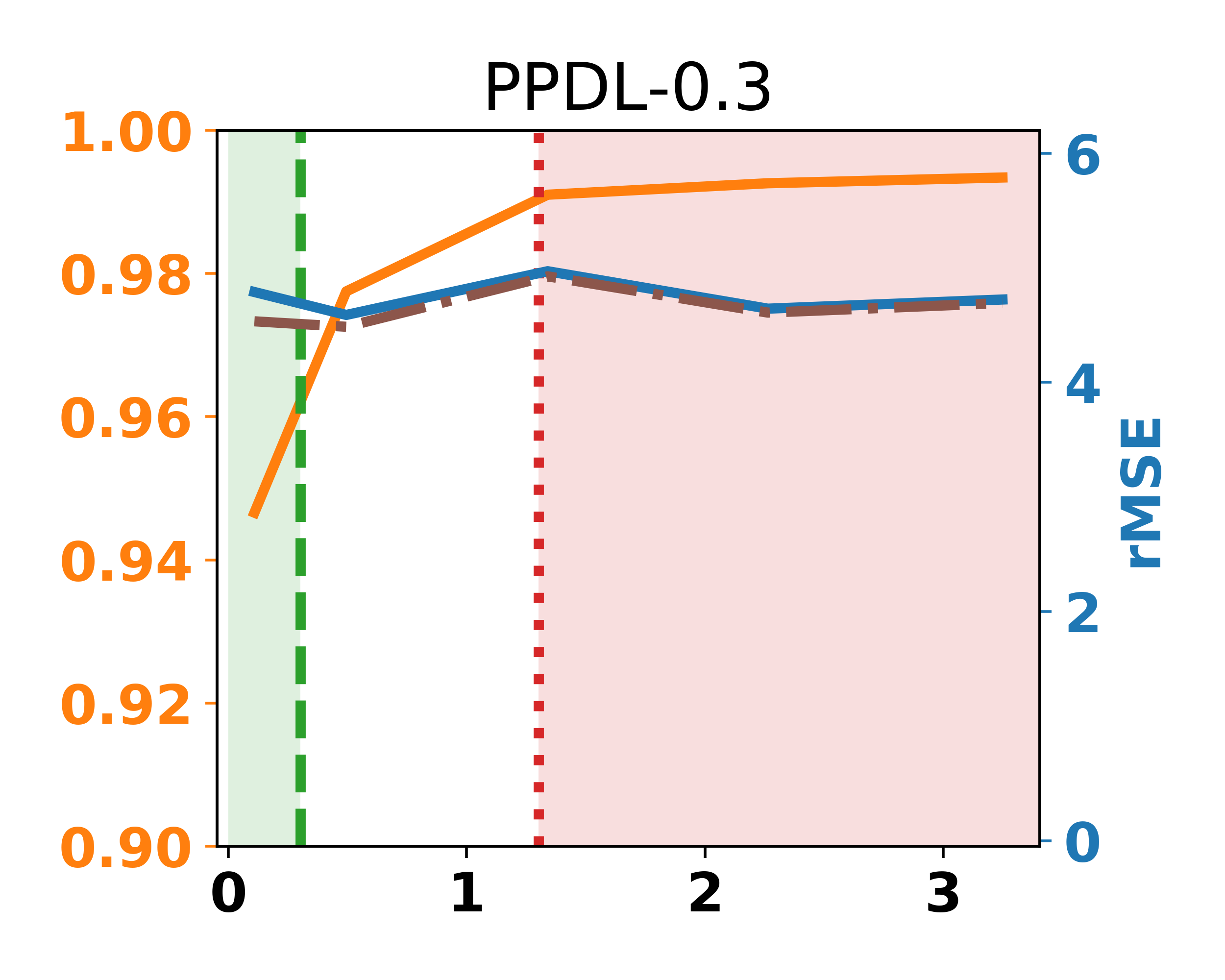}
			\includegraphics[width=0.24\linewidth]{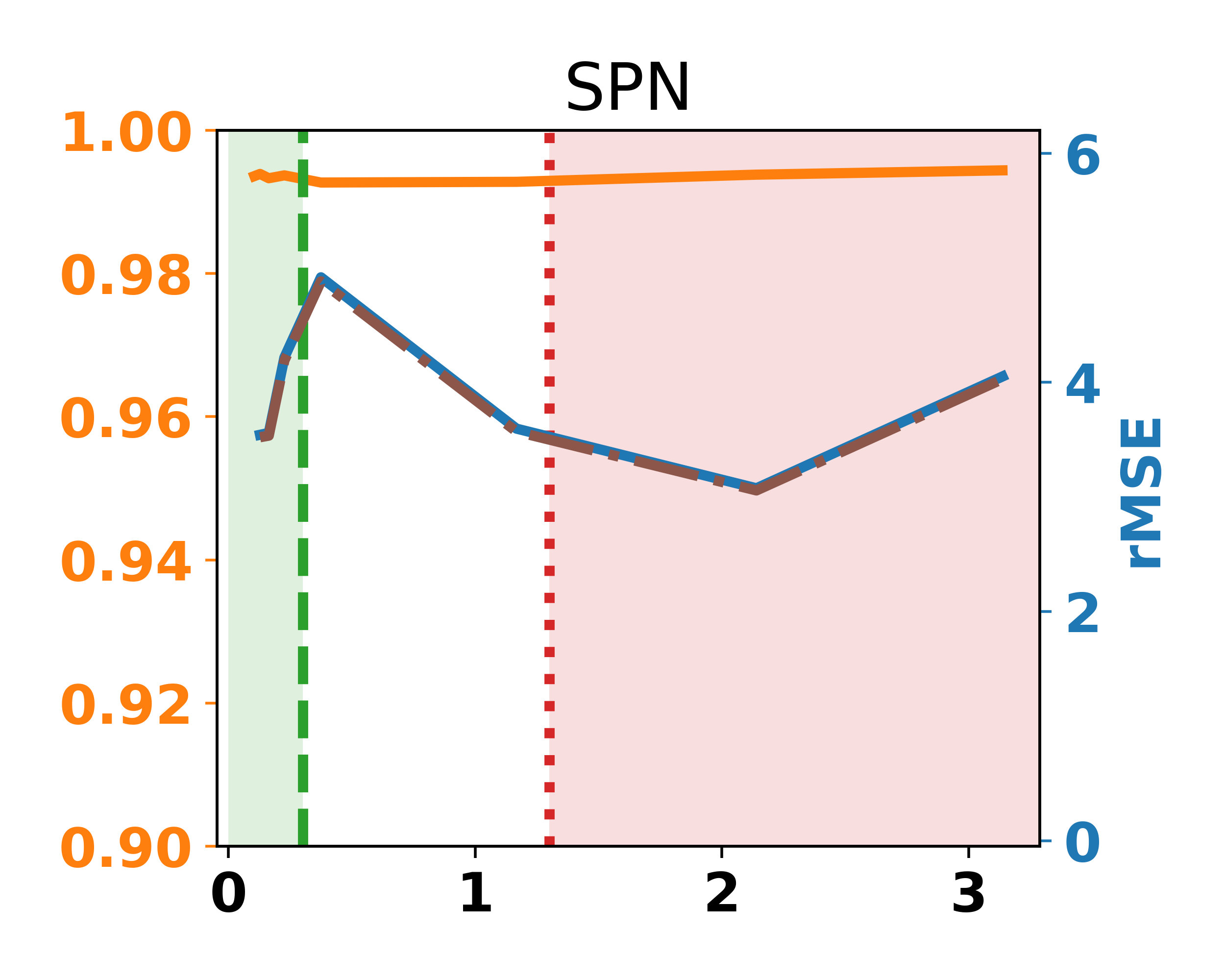}
			\caption{Reconstruction Attack}
		\end{subfigure}
		
		\begin{subfigure}{0.99\linewidth}
			\centering
			\includegraphics[width=0.24\linewidth]{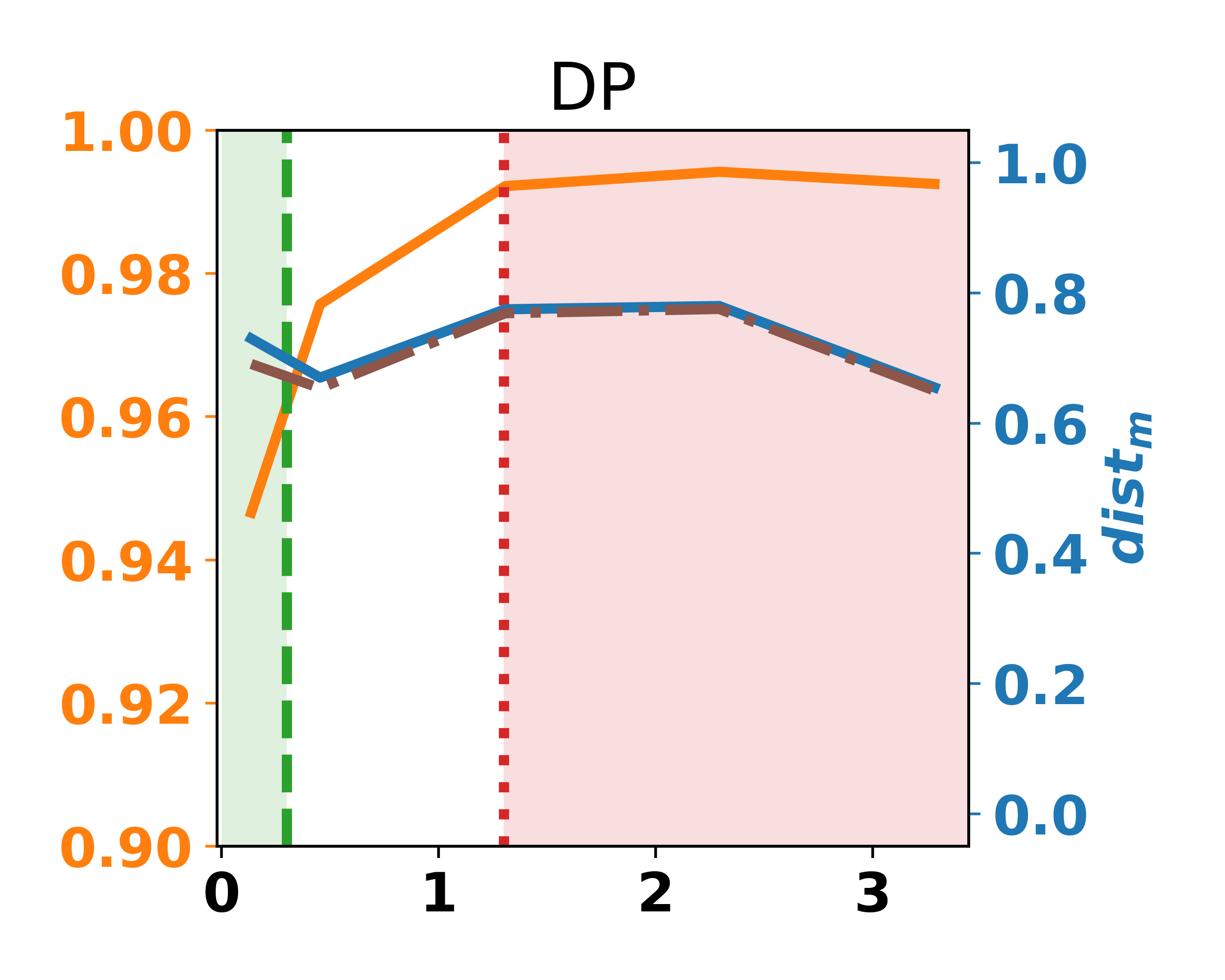}
			\includegraphics[width=0.24\linewidth]{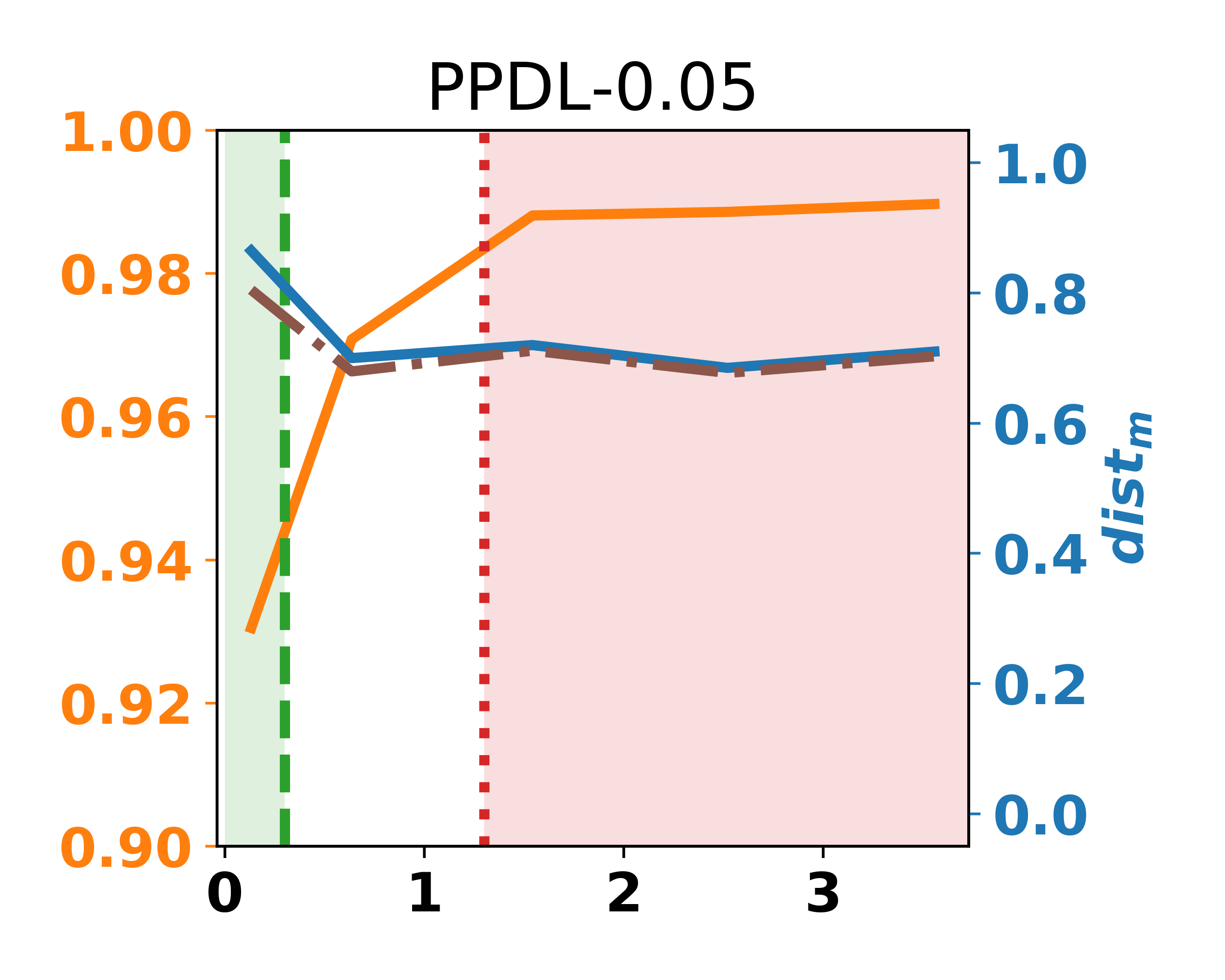}
			\includegraphics[width=0.24\linewidth]{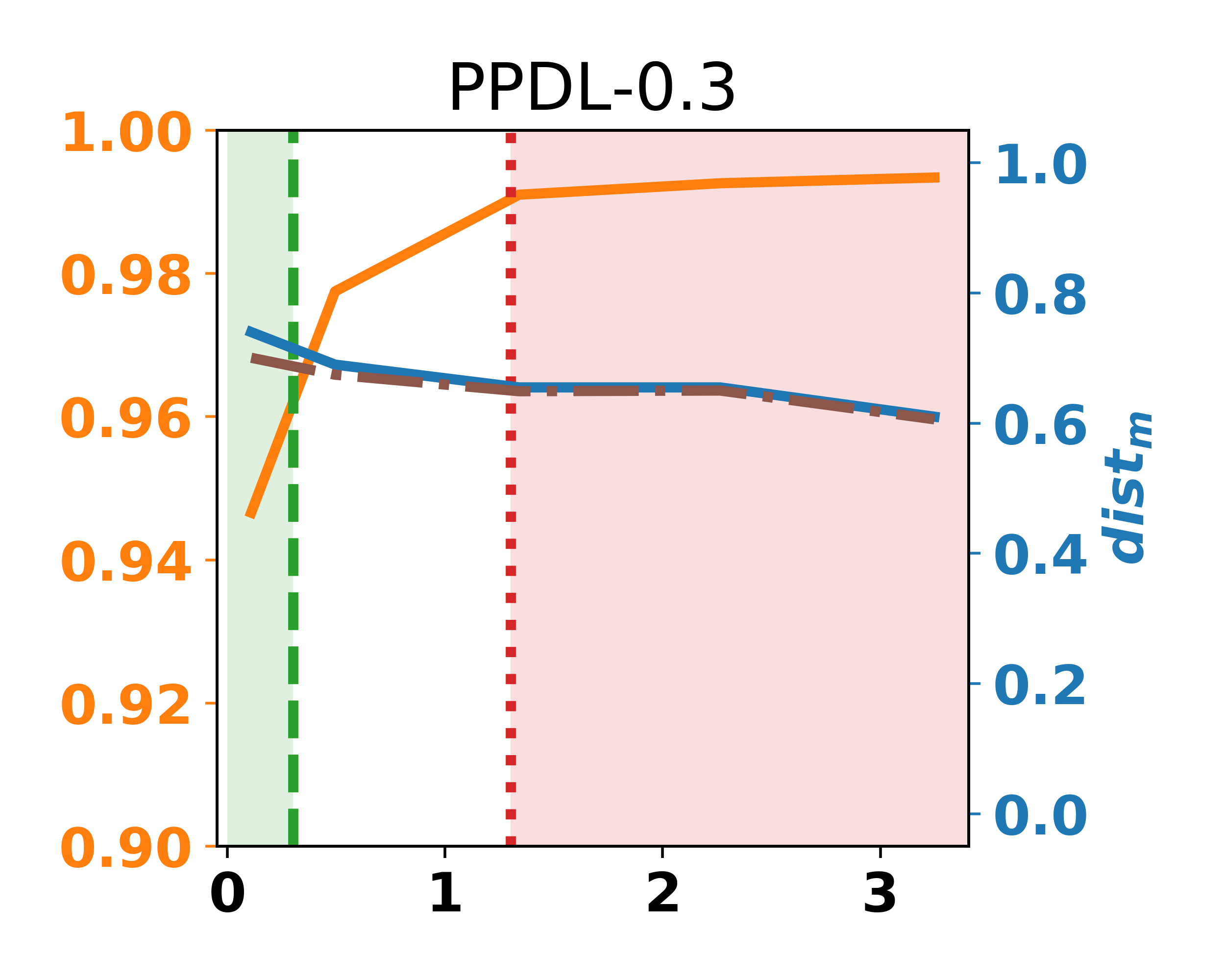}
			\includegraphics[width=0.24\linewidth]{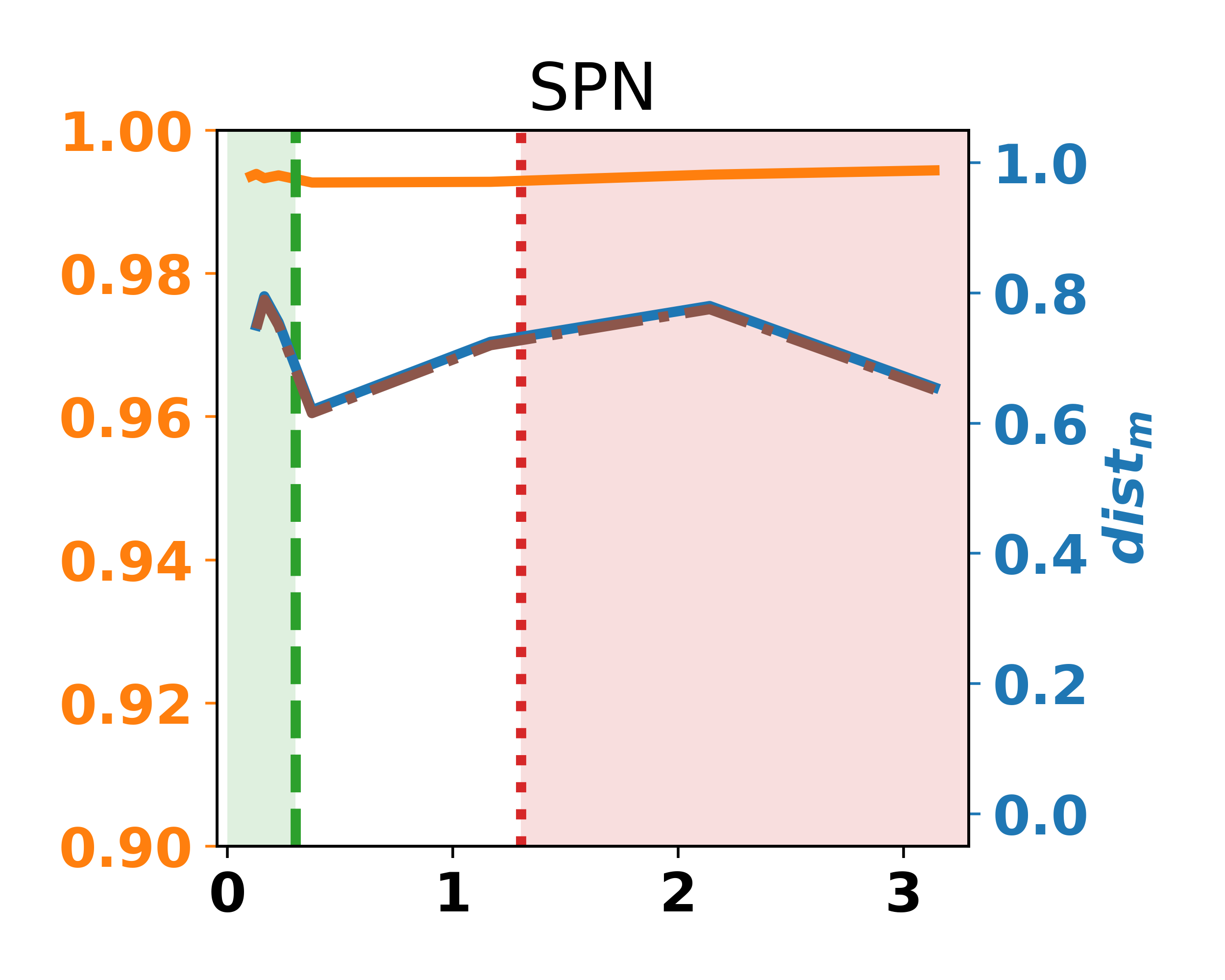}
			\caption{Membership Attack}
		\end{subfigure}
		
		\begin{subfigure}{0.99\linewidth}
			\centering
			\includegraphics[width=0.24\linewidth]{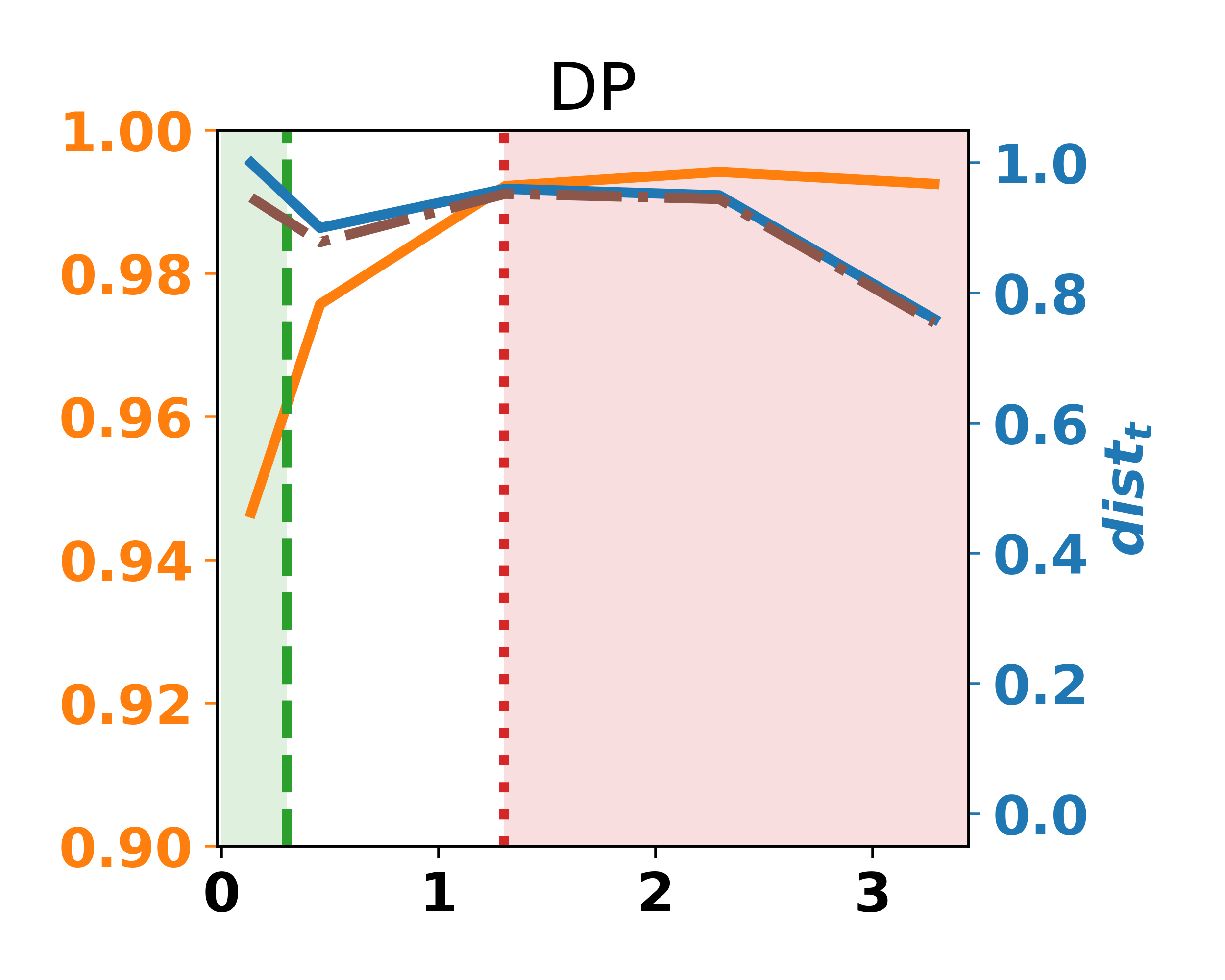}
			\includegraphics[width=0.24\linewidth]{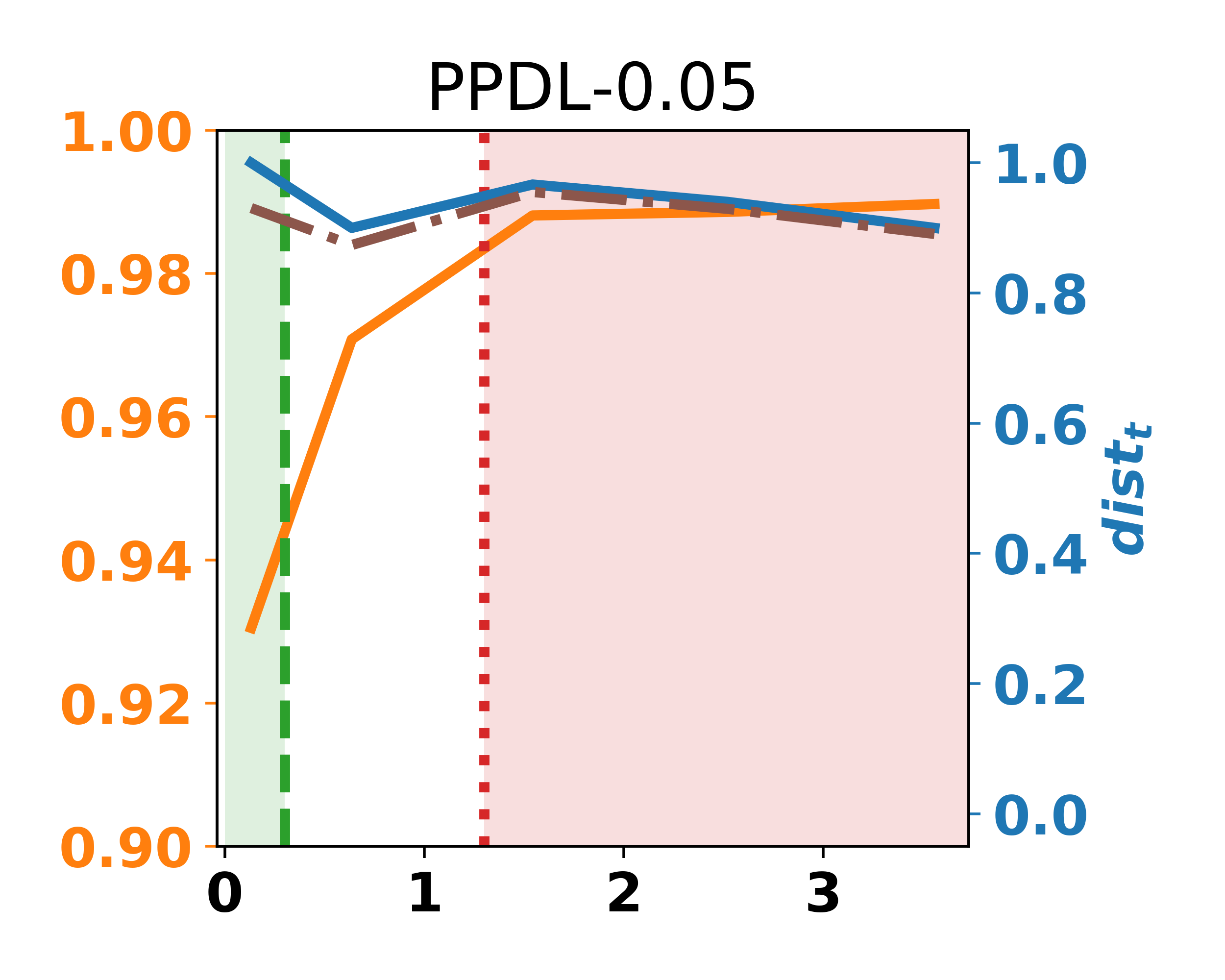}
			\includegraphics[width=0.24\linewidth]{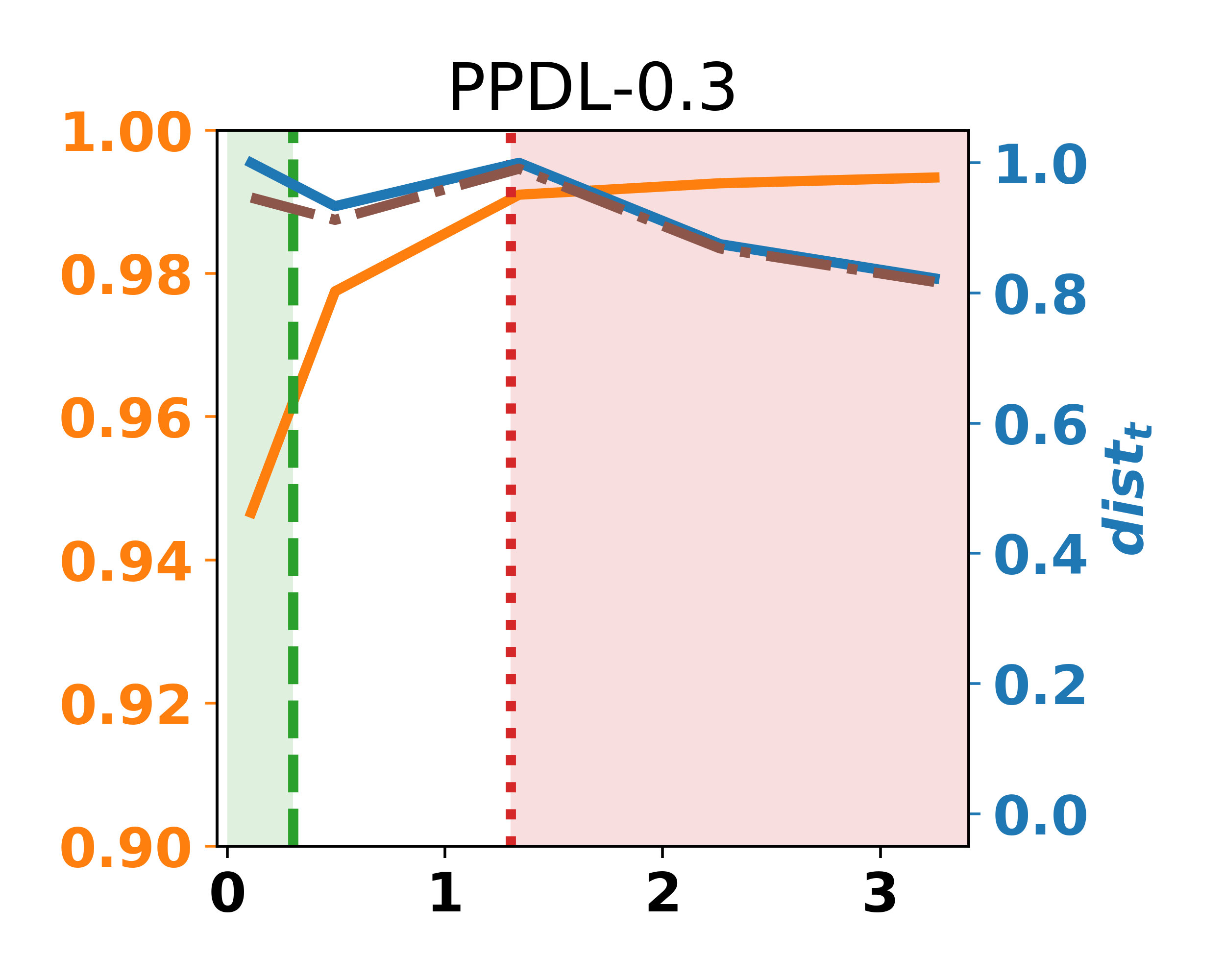}
			\includegraphics[width=0.24\linewidth]{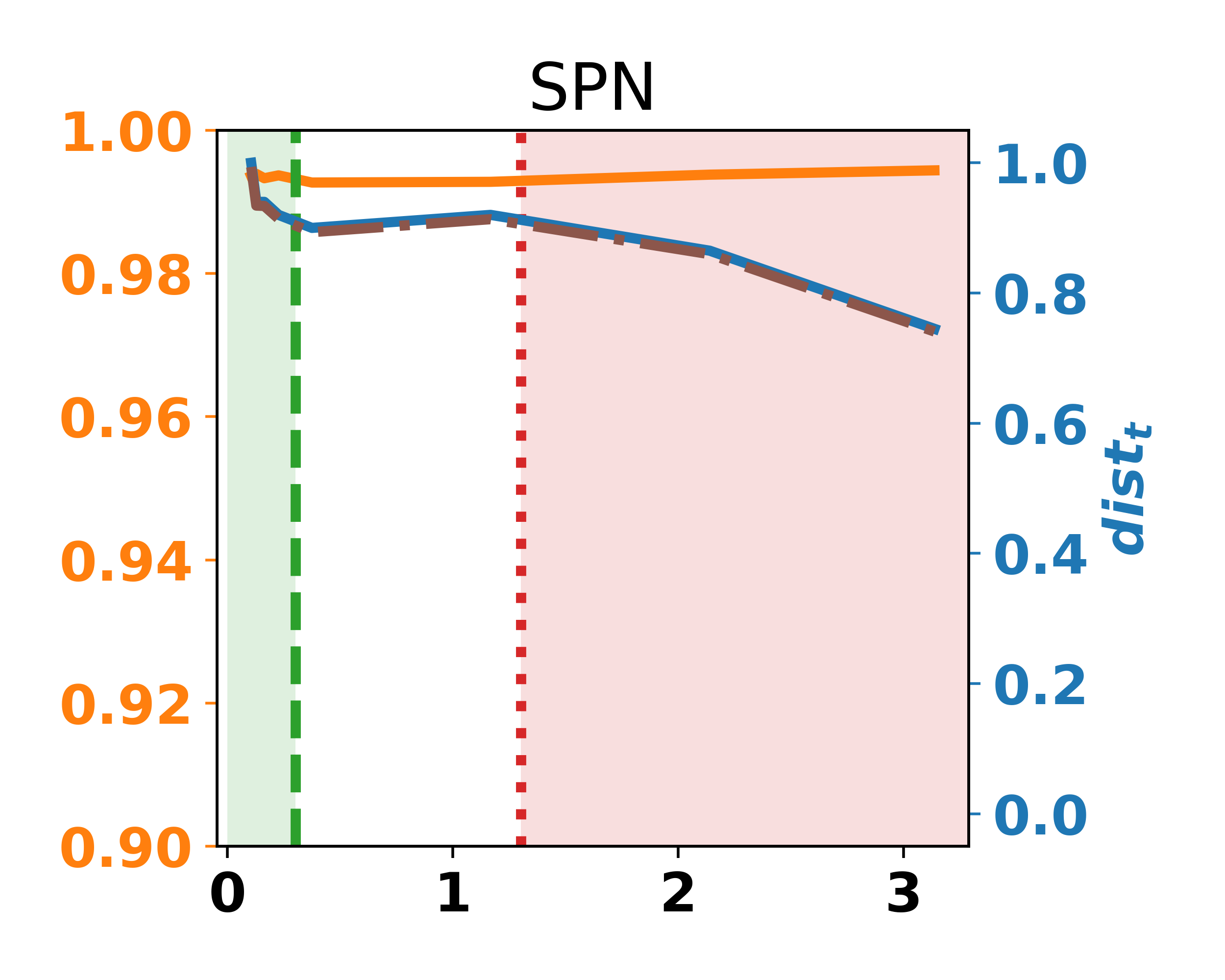}
			\caption{Tracing Attack}
		\end{subfigure}
	
		\caption{Attack with Batch Size 4}
		\label{fig:ppc-mnist-bs4}
		%\vspace{-0.22cm}
	\end{figure}
	
	\begin{figure}[H]	
		%	\begin{subfigure}{0.95\linewidth}
		\centering
		\begin{subfigure}{0.99\linewidth}
			\centering
			\includegraphics[scale=0.7]{imgs/legends/legend_ppc_horizontal.png}
			\\
			\includegraphics[width=0.24\linewidth]{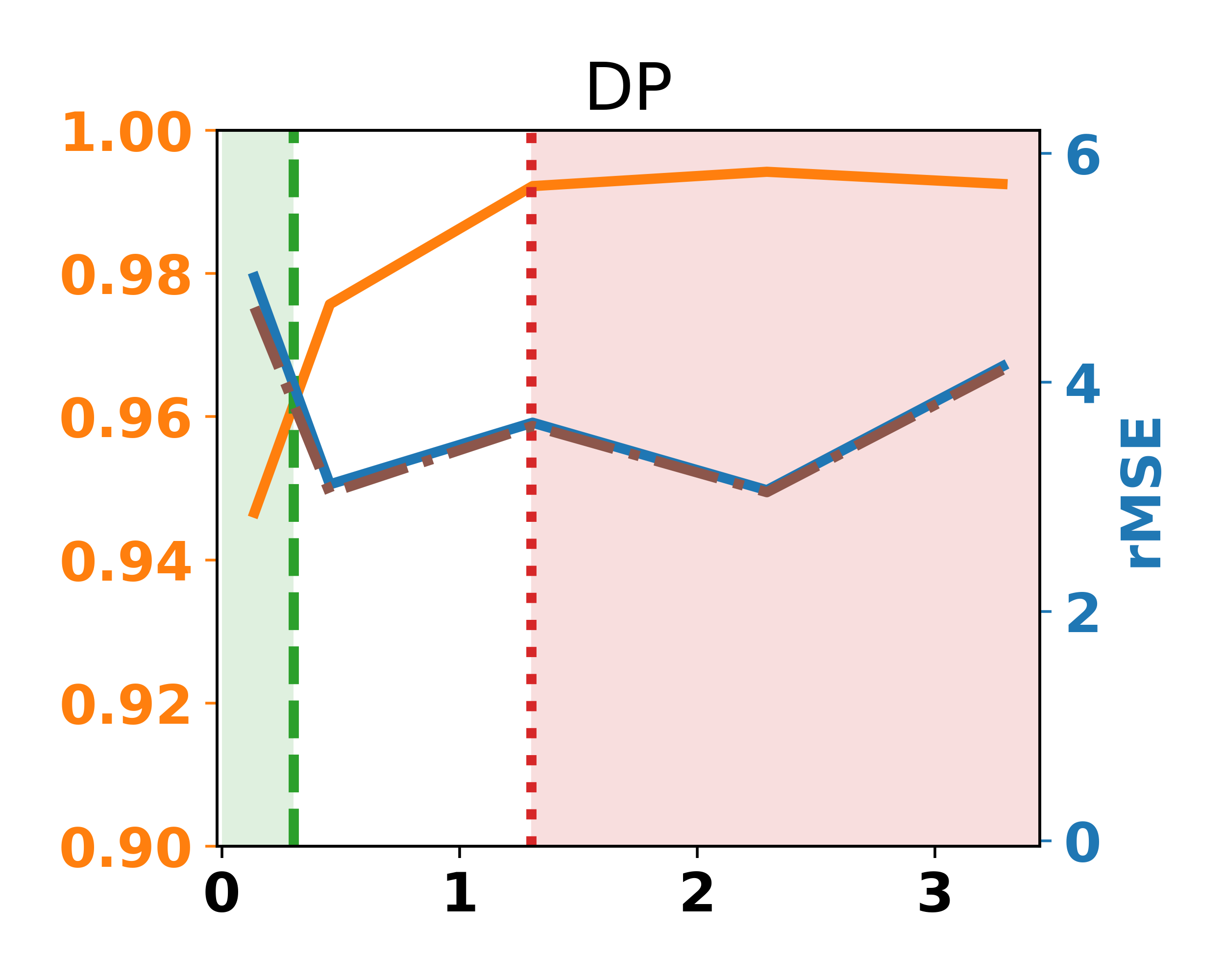}
			\includegraphics[width=0.24\linewidth]{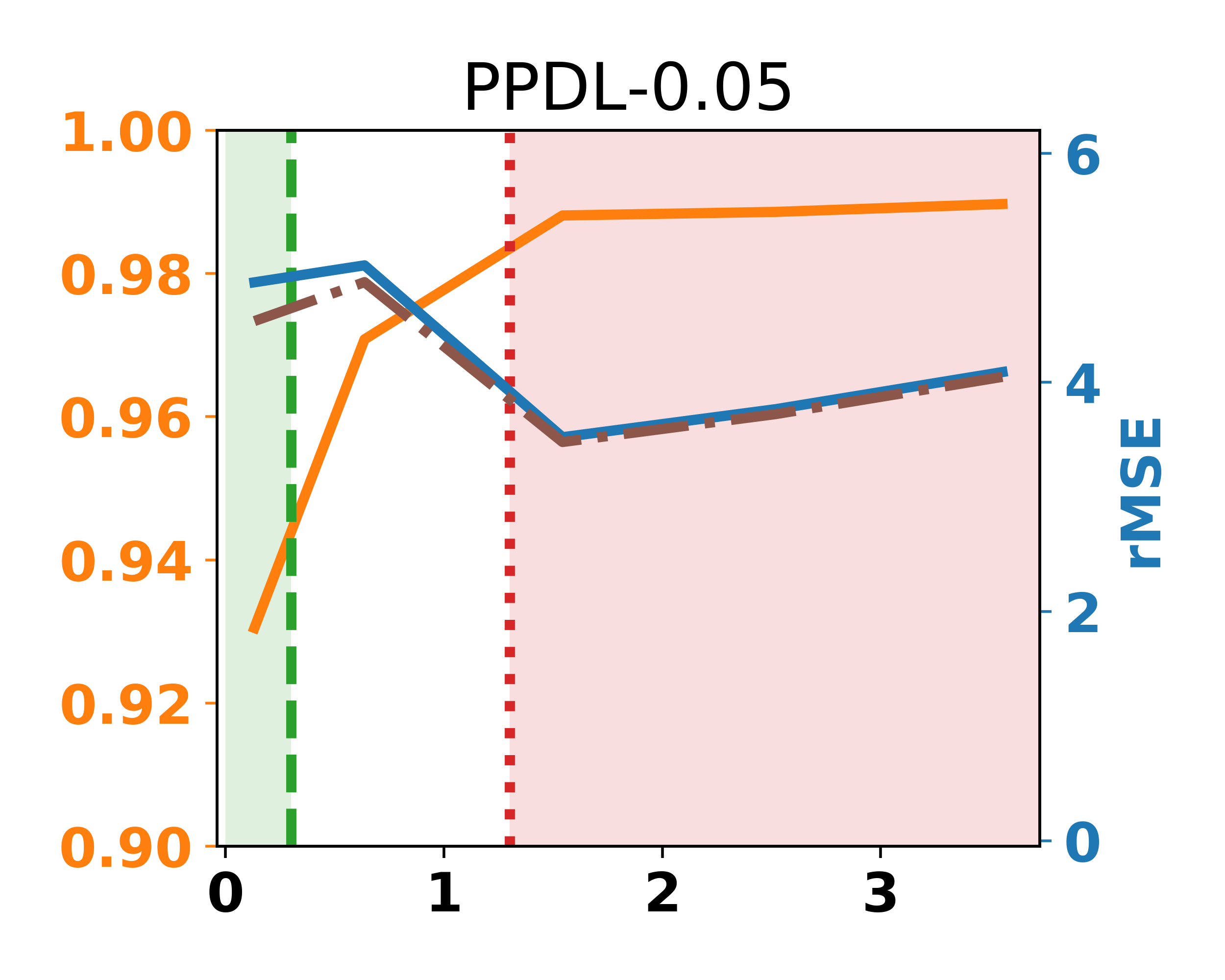}
			\includegraphics[width=0.24\linewidth]{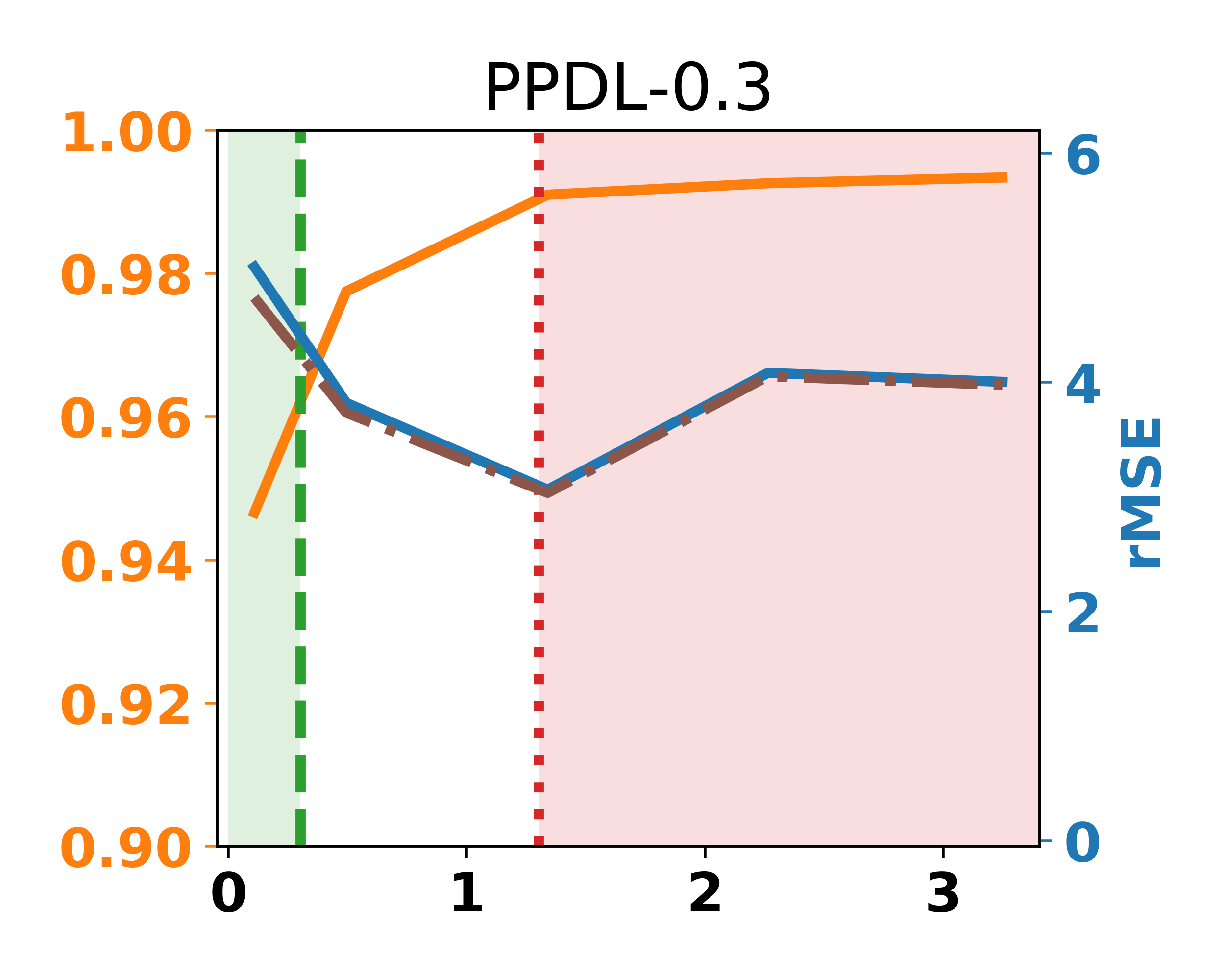}
			\includegraphics[width=0.24\linewidth]{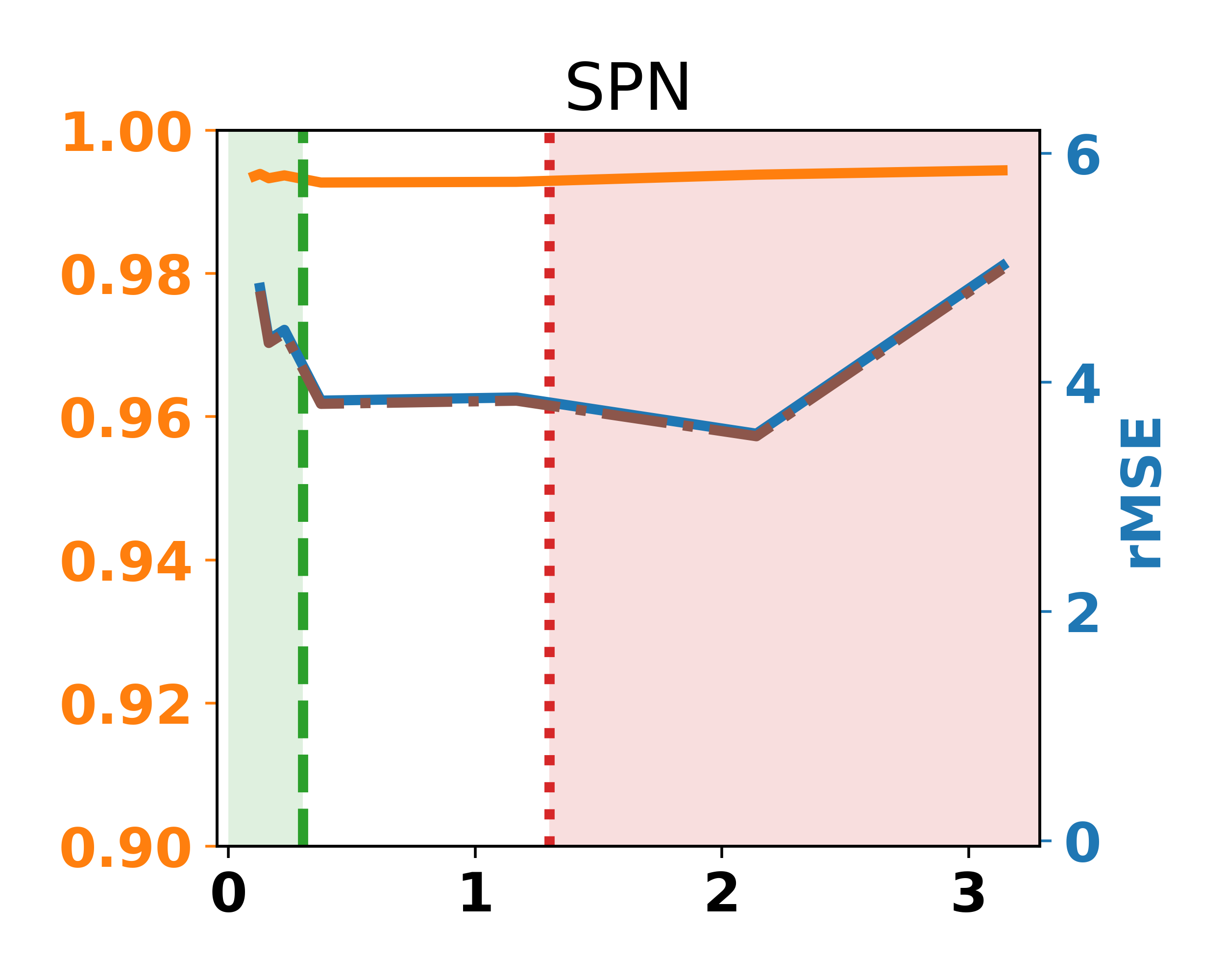}
			\caption{Reconstruction Attack}
		\end{subfigure}
		
		\begin{subfigure}{0.99\linewidth}
			\centering
			\includegraphics[width=0.24\linewidth]{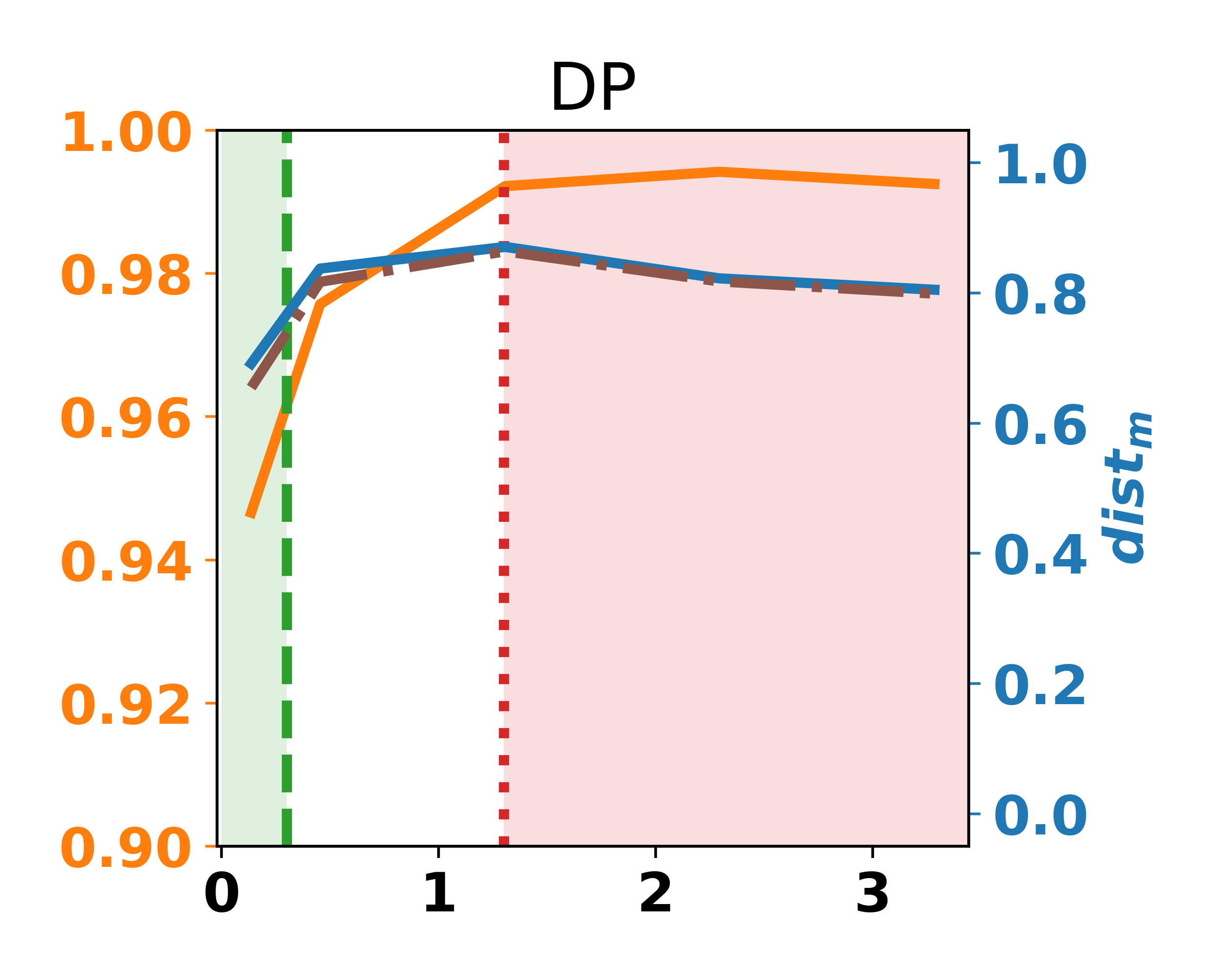}
			\includegraphics[width=0.24\linewidth]{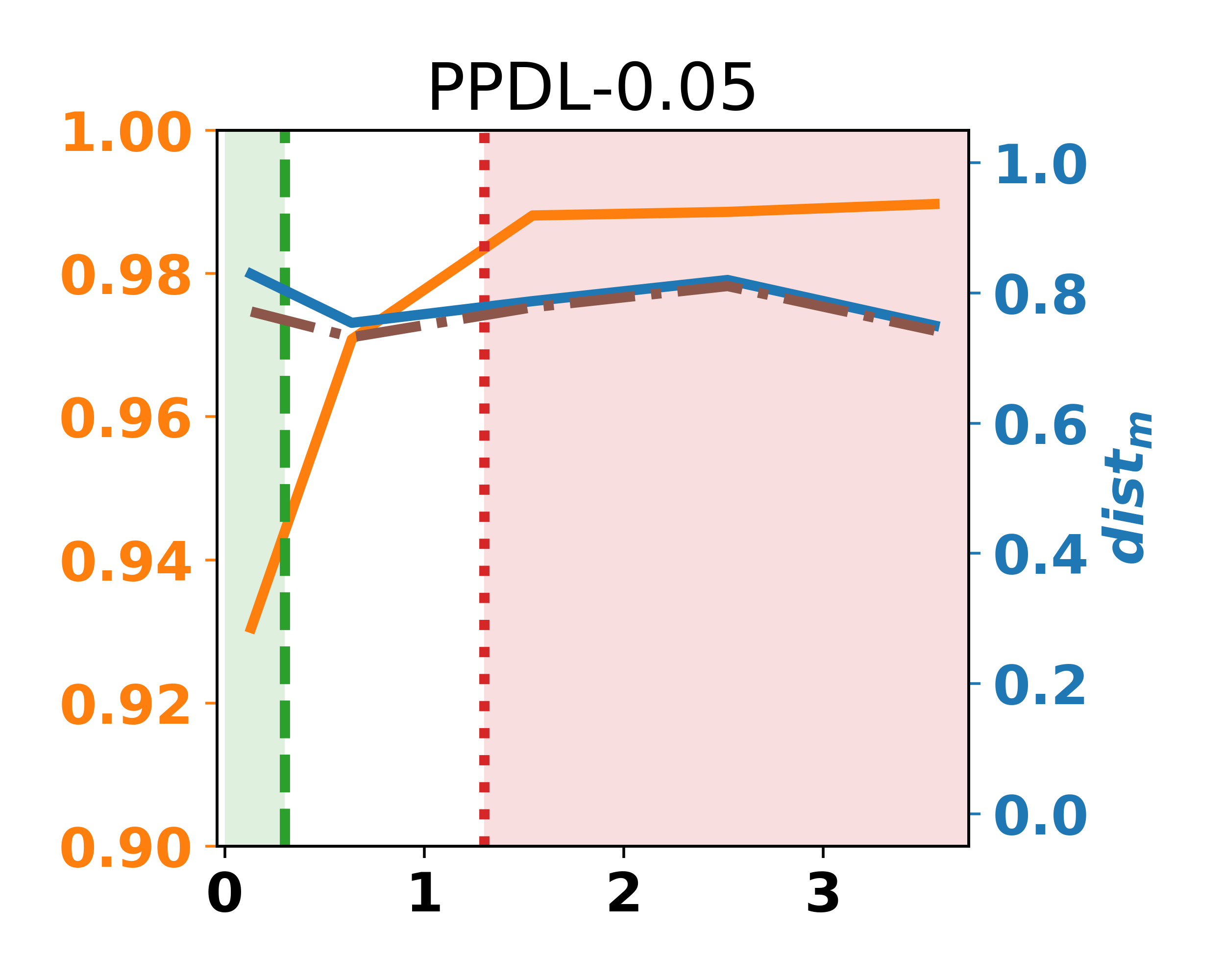}
			\includegraphics[width=0.24\linewidth]{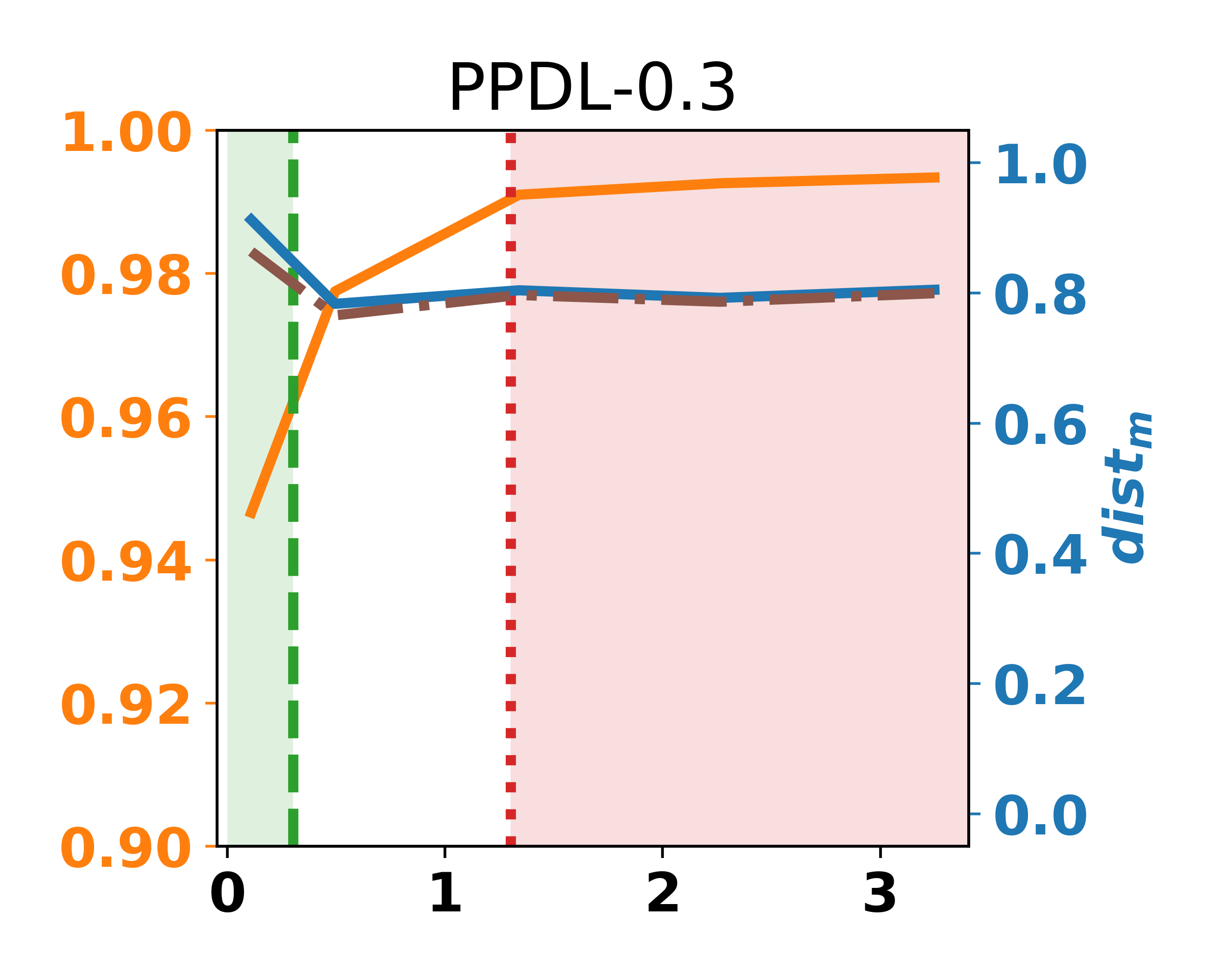}
			\includegraphics[width=0.24\linewidth]{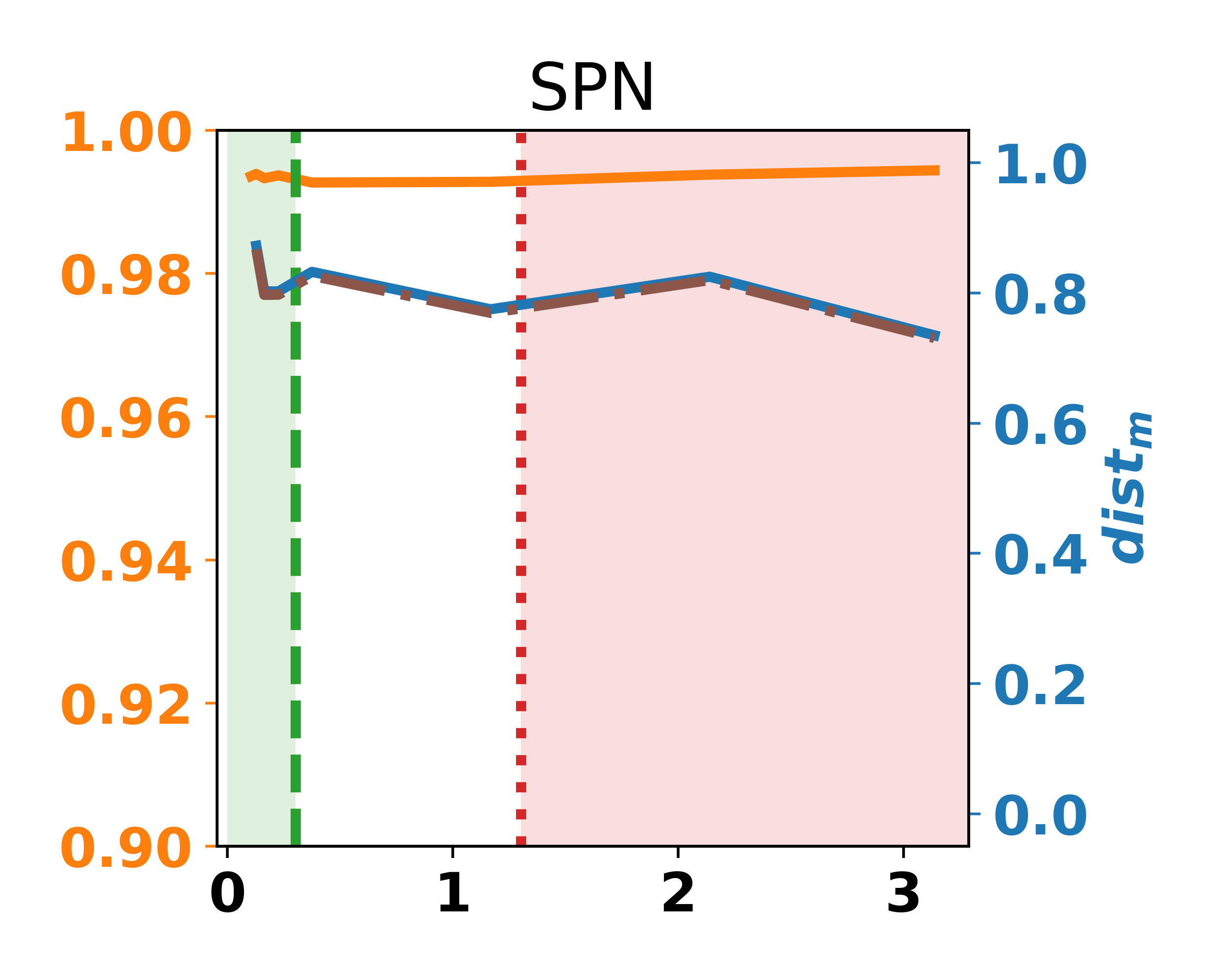}
			\caption{Membership Attack}
		\end{subfigure}
		
		\begin{subfigure}{0.99\linewidth}
			\centering
			\includegraphics[width=0.24\linewidth]{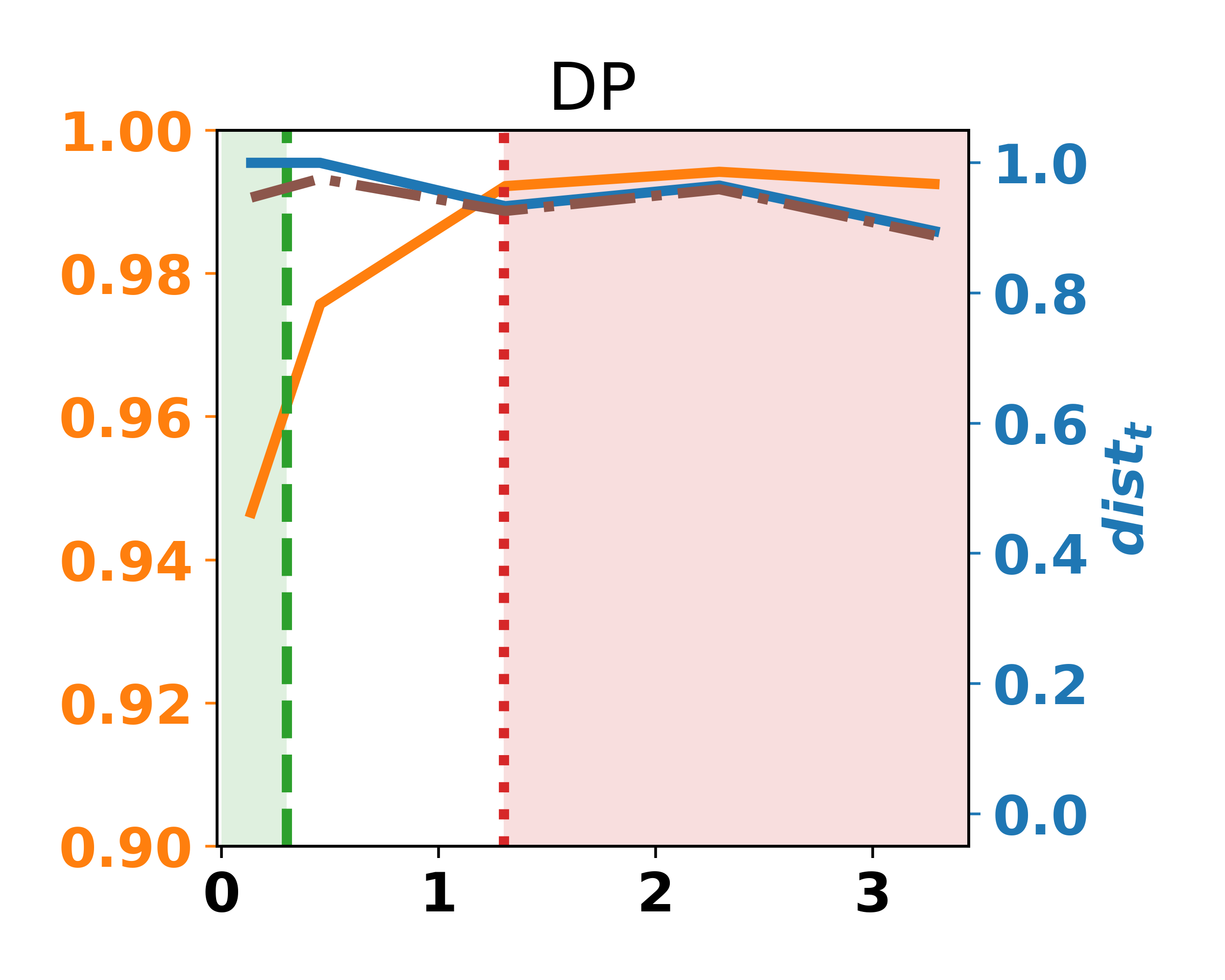}
			\includegraphics[width=0.24\linewidth]{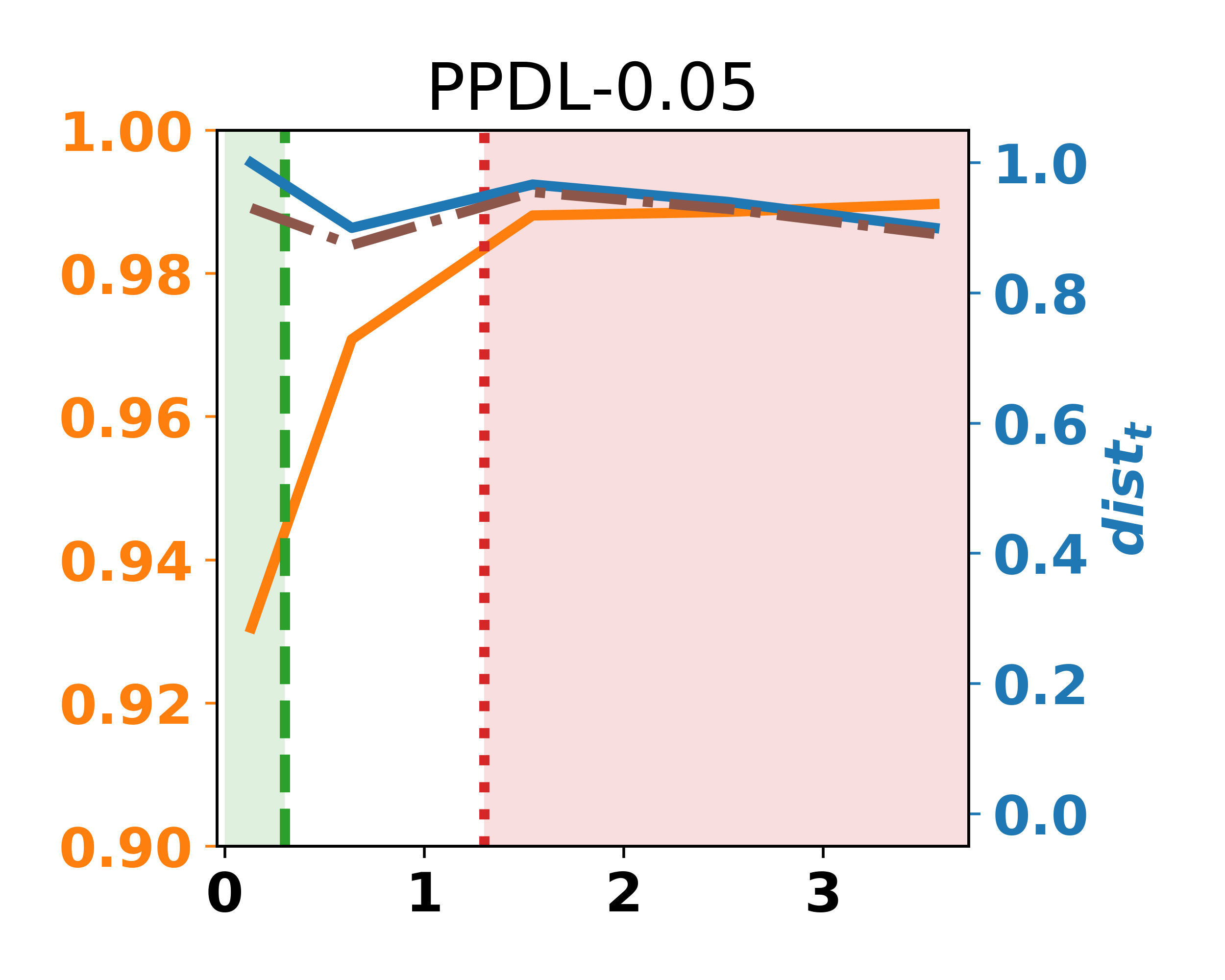}
			\includegraphics[width=0.24\linewidth]{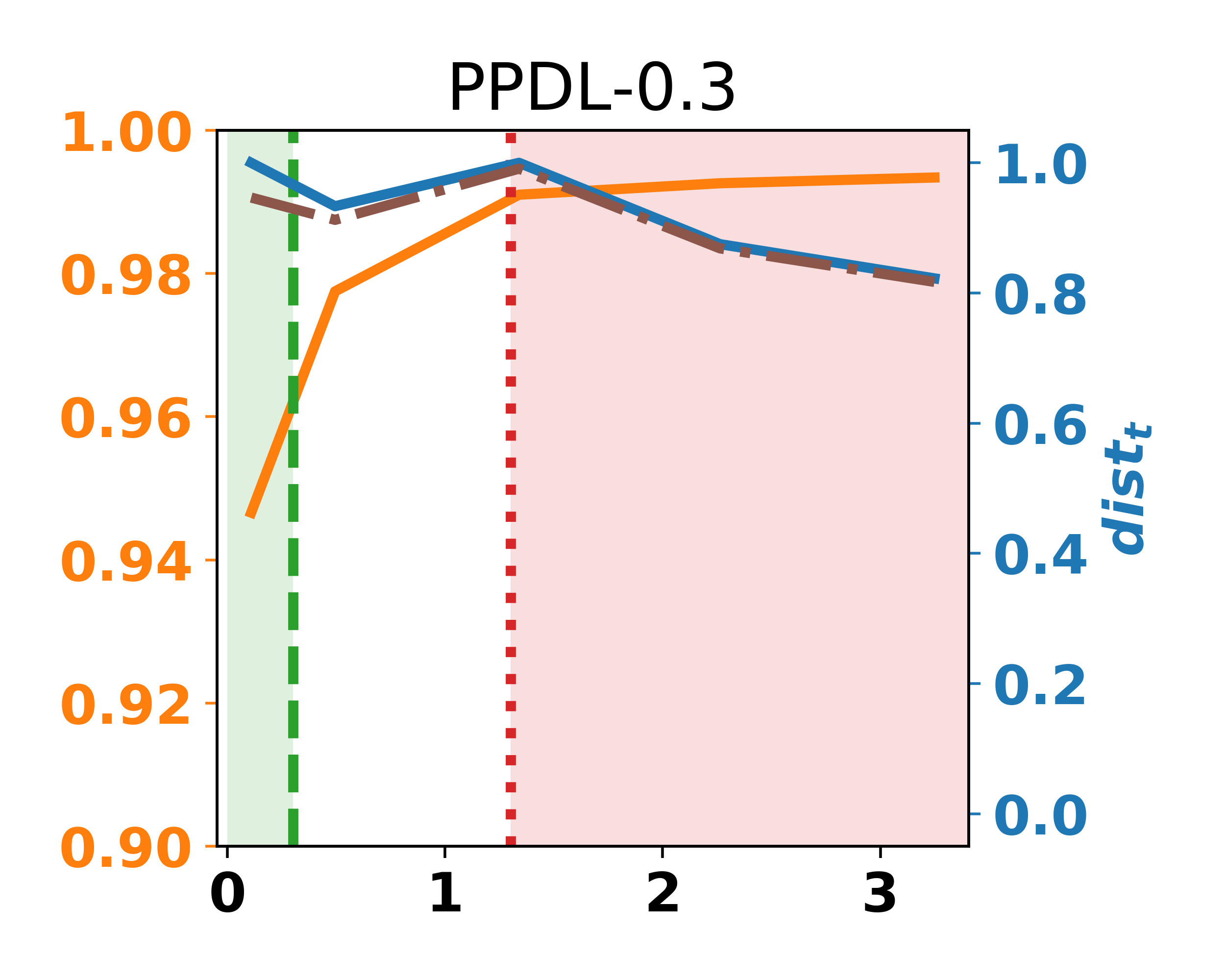}
			\includegraphics[width=0.24\linewidth]{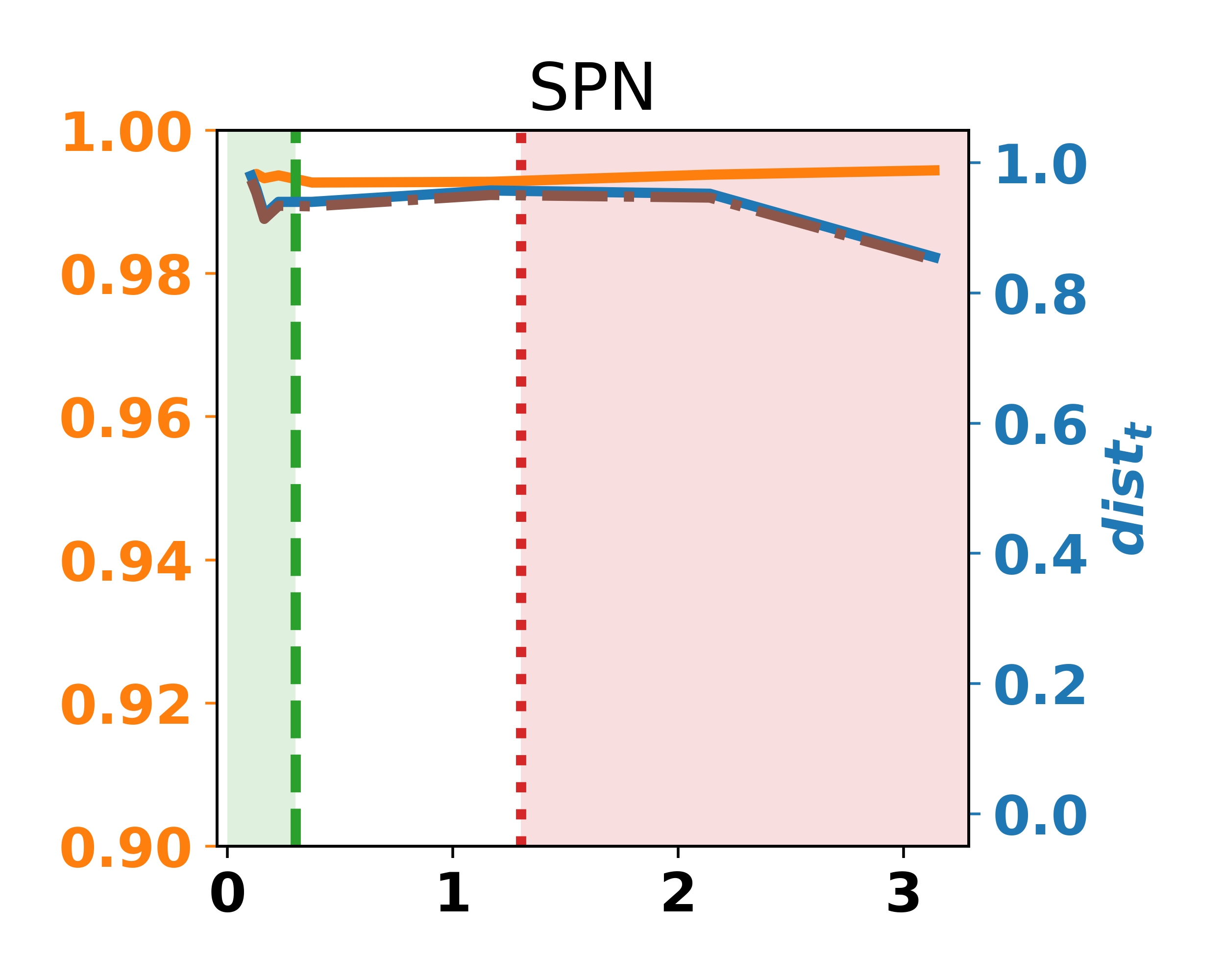}
			\caption{Tracing Attack}
		\end{subfigure}
		
		\caption{Attack with Batch Size 8}
		\label{fig:ppc-mnist-bs8}
		%\vspace{-0.22cm}
	\end{figure}
	
	\subsubsection{Calibrated Averaged Performance (CAP)}
	
	\begin{table}[H]
		\centering
		\begin{tabular}{l|c|c|c|c|c|c|c|c|c}
			\toprule
			& \multicolumn{3}{c|}{Reconstruction} & \multicolumn{3}{c|}{Membership} & \multicolumn{3}{c}{Tracing} \\
			\midrule
			BS & 1 & 4 & 8 & 1 & 4 & 8 & 1 & 4 & 8 \\
			\midrule
			DP \cite{DLDP_Abadi16} & 3.38 &\textbf{ 3.83 }& 3.69 & 0.00 & 0.71 & 0.79 & 0.91 & 0.93 & \textbf{0.95} \\
			PPDL-0.05 \cite{PPDL/shokri2015} & \textbf{4.42 }& 3.62 & 4.13 & 0.37 & 0.72 & 0.77 & 0.92 & 0.92 & 0.92 \\
			PPDL-0.3 \cite{PPDL/shokri2015} & 4.04 & 4.65 & 3.91 & 0.00 & 0.66 & 0.80 & \textbf{0.95} & \textbf{0.95} & \textbf{0.95} \\
			SPN (ours) & 4.30 & 3.72 & \textbf{4.33} & \textbf{0.37} &\textbf{0.73} & \textbf{0.81} & 0.90 & 0.93 & 0.94 \\
			\bottomrule
		\end{tabular}
		\caption{CAP performance with different batch size on MNIST for reconstruction, membership, and tracing attack. Higher better. BS = Attack Batch Size.}
		\label{tab:CAP-mnist-supp}
	\end{table}
	
	\subsubsection{Reconstructed Images}
	
	\begin{figure}[H]
		\centering
		
		\begin{subfigure}{0.32\linewidth}
			\centering
			\includegraphics[scale=0.8]{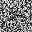}
			\hspace{1pt}
			\includegraphics[scale=0.8]{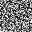}
			\hspace{1pt}
			\includegraphics[scale=0.8]{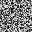}
			\hspace{1pt}
			\includegraphics[scale=0.8]{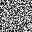}
			\\
			\vspace{2pt}
			\includegraphics[scale=0.8]{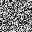}
			\hspace{1pt}
			\includegraphics[scale=0.8]{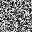}
			\hspace{1pt}
			\includegraphics[scale=0.8]{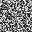}
			\hspace{1pt}
			\includegraphics[scale=0.8]{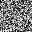}
			\\
			\vspace{2pt}
			\includegraphics[scale=0.8]{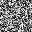}
			\hspace{1pt}
			\includegraphics[scale=0.8]{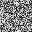}
			\hspace{1pt}
			\includegraphics[scale=0.8]{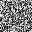}
			\hspace{1pt}
			\includegraphics[scale=0.8]{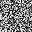}
			\\
			\vspace{2pt}
			\includegraphics[scale=0.8]{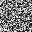}
			\hspace{1pt}
			\includegraphics[scale=0.8]{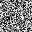}
			\hspace{1pt}
			\includegraphics[scale=0.8]{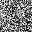}
			\hspace{1pt}
			\includegraphics[scale=0.8]{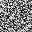}
			\caption{$5.17 (0.37)$ \label{fig:recon-images-green-mnist-supp}}
		\end{subfigure}
		\hfill
		\begin{subfigure}{0.32\linewidth}
			\centering
			\includegraphics[scale=0.8]{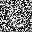}
			\hspace{1pt}
			\includegraphics[scale=0.8]{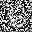}
			\hspace{1pt}
			\includegraphics[scale=0.8]{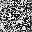}
			\hspace{1pt}
			\includegraphics[scale=0.8]{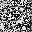}
			\\
			\vspace{2pt}
			\includegraphics[scale=0.8]{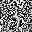}
			\hspace{1pt}
			\includegraphics[scale=0.8]{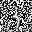}
			\hspace{1pt}
			\includegraphics[scale=0.8]{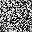}
			\hspace{1pt}
			\includegraphics[scale=0.8]{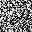}
			\\
			\vspace{2pt}
			\includegraphics[scale=0.8]{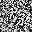}
			\hspace{1pt}
			\includegraphics[scale=0.8]{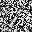}
			\hspace{1pt}
			\includegraphics[scale=0.8]{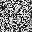}
			\hspace{1pt}
			\includegraphics[scale=0.8]{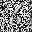}
			\\
			\vspace{2pt}
			\includegraphics[scale=0.8]{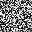}
			\hspace{1pt}
			\includegraphics[scale=0.8]{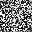}
			\hspace{1pt}
			\includegraphics[scale=0.8]{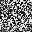}
			\hspace{1pt}
			\includegraphics[scale=0.8]{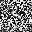}
			\caption{$4.70 (16.89)$ \label{fig:recon-images-white-mnist-supp}}%; 1.10 (1.30)}
		\end{subfigure}
		\hfill
		\begin{subfigure}{0.32\linewidth}
			\centering
			\includegraphics[scale=0.8]{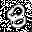}
			\hspace{1pt}
			\includegraphics[scale=0.8]{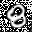}
			\hspace{1pt}
			\includegraphics[scale=0.8]{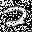}
			\hspace{1pt}
			\includegraphics[scale=0.8]{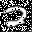}
			\\
			\vspace{2pt}
			\includegraphics[scale=0.8]{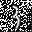}
			\hspace{1pt}
			\includegraphics[scale=0.8]{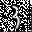}
			\hspace{1pt}
			\includegraphics[scale=0.8]{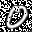}
			\hspace{1pt}
			\includegraphics[scale=0.8]{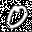}
			\\
			\vspace{2pt}
			\includegraphics[scale=0.8]{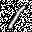}
			\hspace{1pt}
			\includegraphics[scale=0.8]{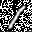}
			\hspace{1pt}
			\includegraphics[scale=0.8]{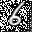}
			\hspace{1pt}
			\includegraphics[scale=0.8]{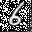}
			\\
			\vspace{2pt}
			\includegraphics[scale=0.8]{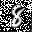}
			\hspace{1pt}
			\includegraphics[scale=0.8]{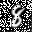}
			\hspace{1pt}
			\includegraphics[scale=0.8]{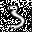}
			\hspace{1pt}
			\includegraphics[scale=0.8]{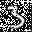}
			\caption{$1.46 (1011.30)$ \label{fig:recon-images-red-mnist-supp}} %0.94 (1.59)}
		\end{subfigure}
		
		%\begin{subfigure}{0.2\linewidth}
		%	\centering
		%	\includegraphics[scale=0.8]{imgs/recimgs/0_0.0001_0.0151.png}
		%	\hspace{3pt}
		%	\includegraphics[scale=0.8]{imgs/recimgs/391_0.0001_0.0116.png}
		%	\caption{0.02 (3.34); 0.01 (3.28)}
		%\end{subfigure}
		\caption{Reconstructed images from different region mentioned in the main paper. \textbf{(a)} Green region \textbf{(b)} White region \textbf{(c)} Red region. 
			Values inside bracket are mean of $\frac{||B_I||}{||E_B||}$ and values outside are mean of rMSE of reconstructed w.r.t. original images. High rMSE (e.g. 1.46 in the red region) is due to original image is having a lot of zero valued pixel, hence getting smaller $||x||$ and higher rMSE. Also, found out that images that have solid color pixels (e.g. fully dark (0,0,0) or fully white (255,255,255)) are more difficult to attack.
		}
		\label{fig:recon-images-mnist-supp}
	\end{figure}

	\newpage
	
	\subsection{CIFAR10}
	
	\subsubsection{Privacy-Preserving Characteristics (PPC)}
	
	\begin{figure}[H]	
		%	\begin{subfigure}{0.95\linewidth}
		\centering
		\begin{subfigure}{0.99\linewidth}
			\centering
			\includegraphics[scale=0.7]{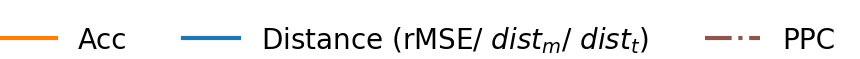}
			\\
			\includegraphics[width=0.24\linewidth]{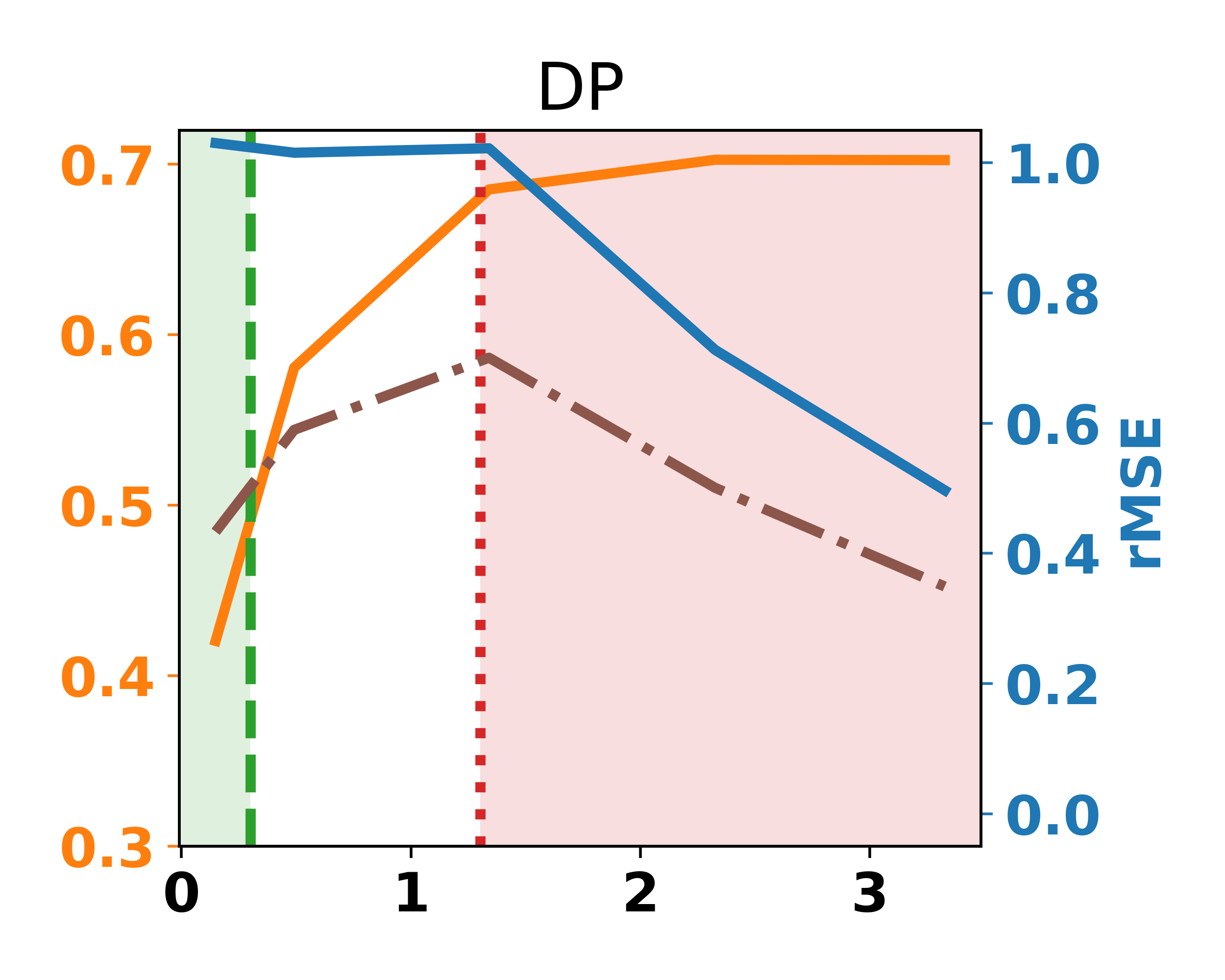}
			\includegraphics[width=0.24\linewidth]{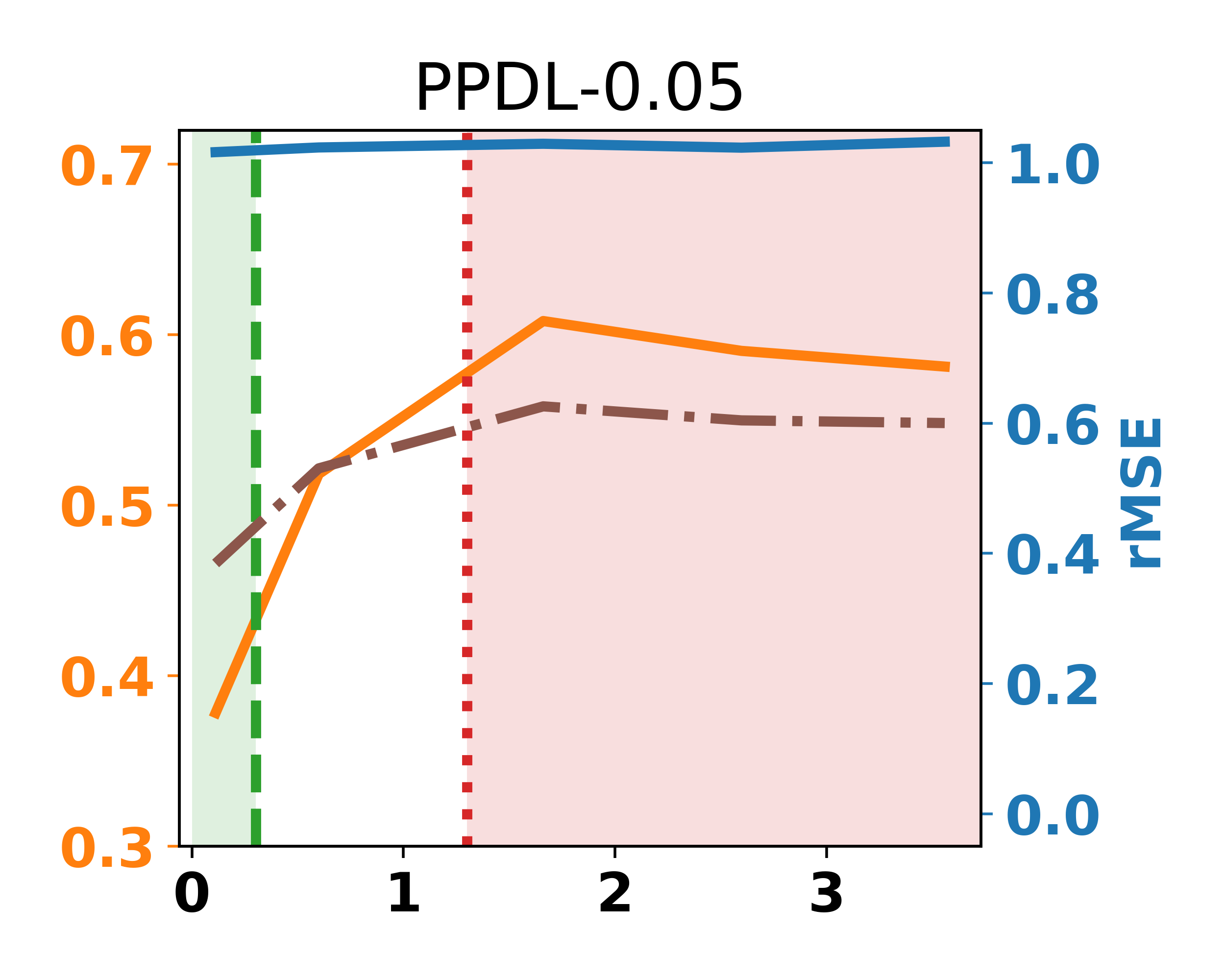}
			\includegraphics[width=0.24\linewidth]{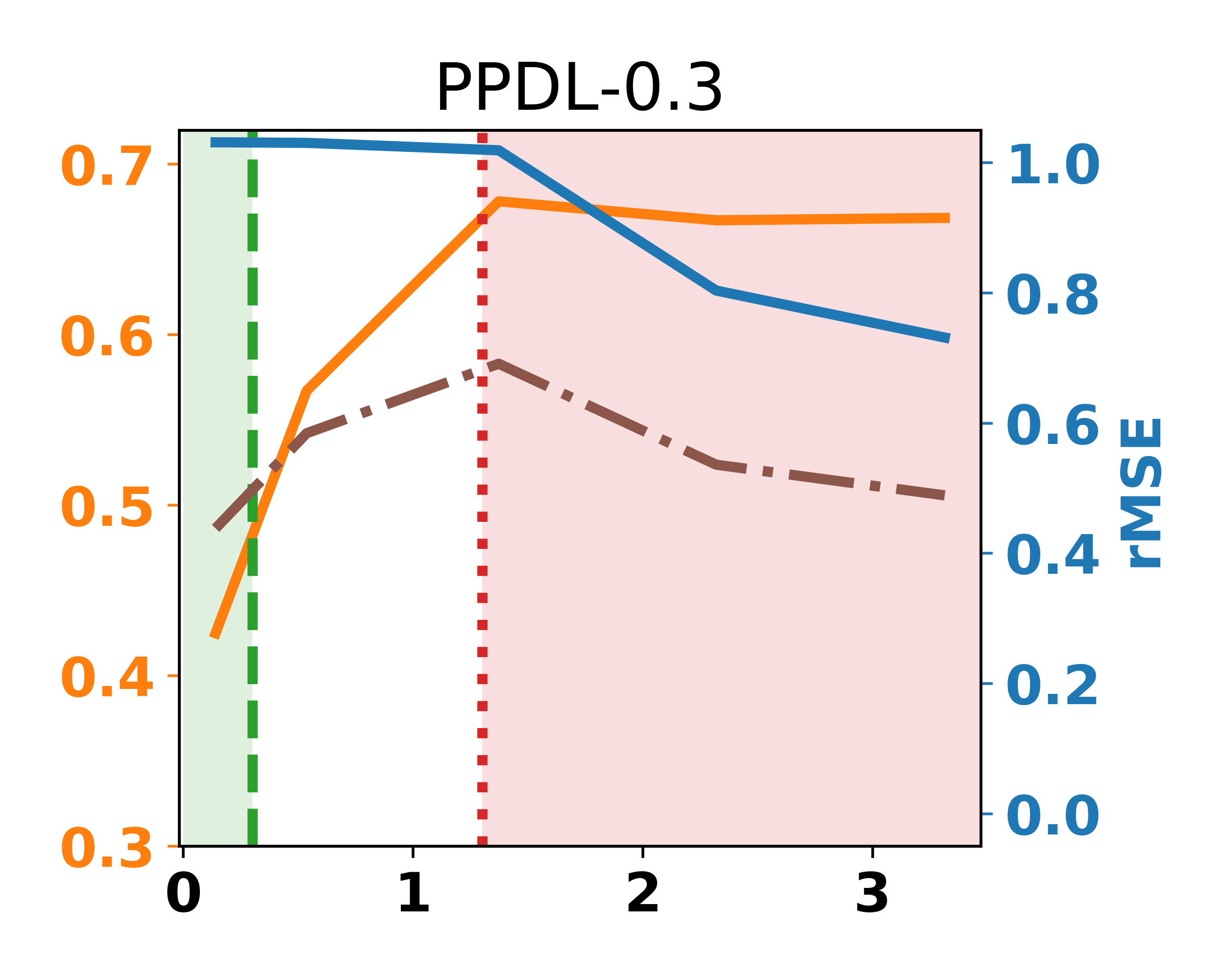}
			\includegraphics[width=0.24\linewidth]{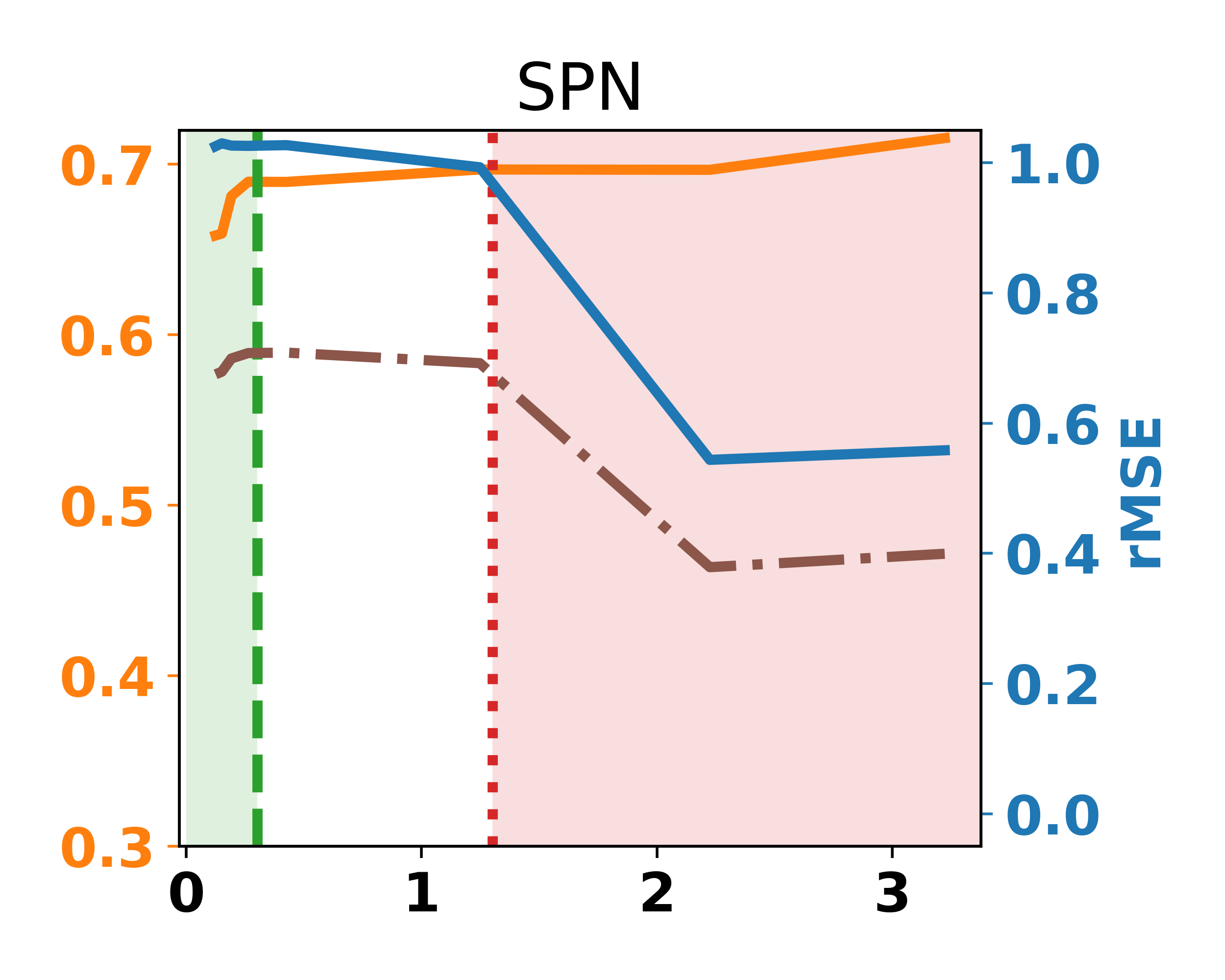}
			\caption{Reconstruction Attack}
		\end{subfigure}
		
		\begin{subfigure}{0.99\linewidth}
			\centering
			\includegraphics[width=0.24\linewidth]{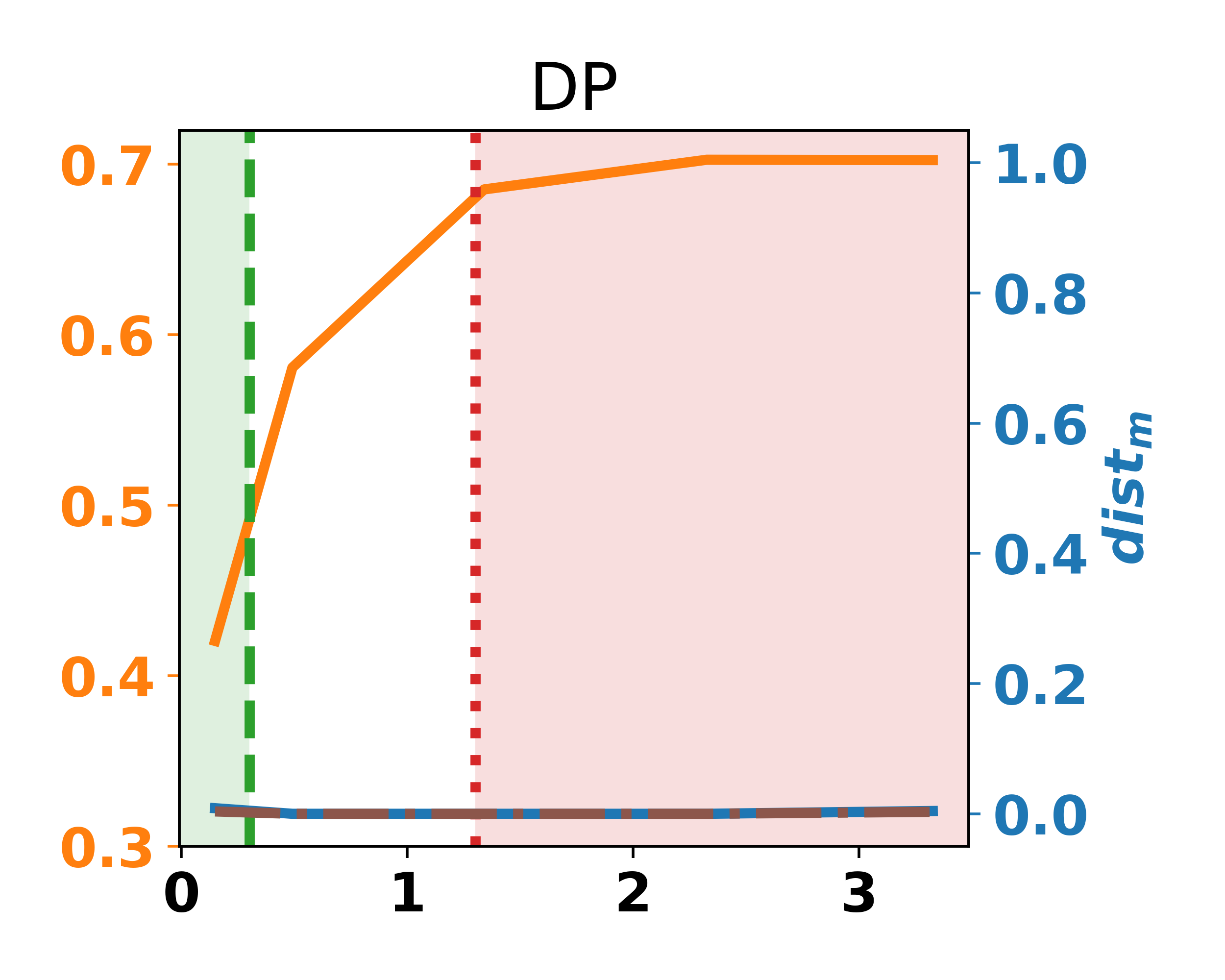}
			\includegraphics[width=0.24\linewidth]{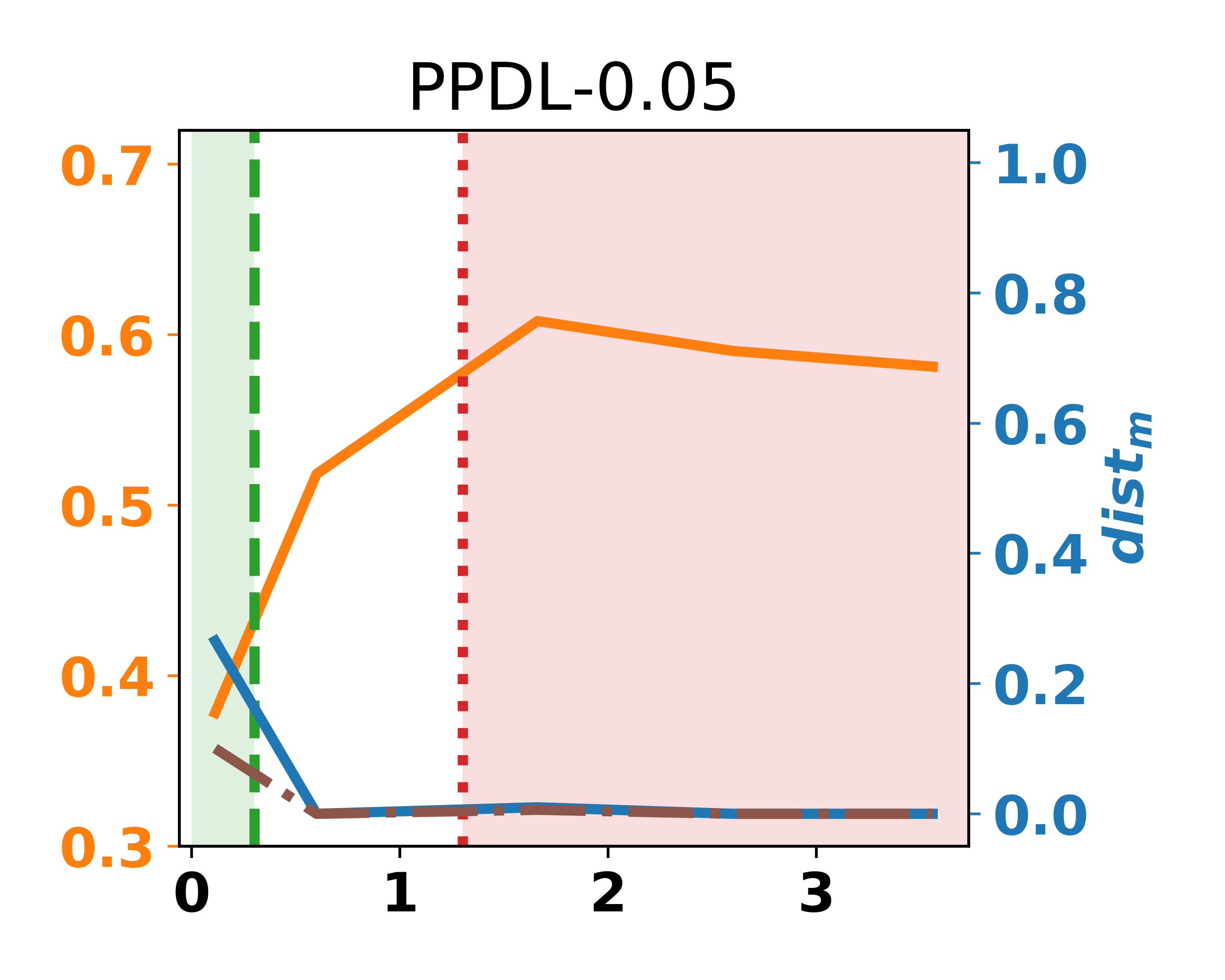}
			\includegraphics[width=0.24\linewidth]{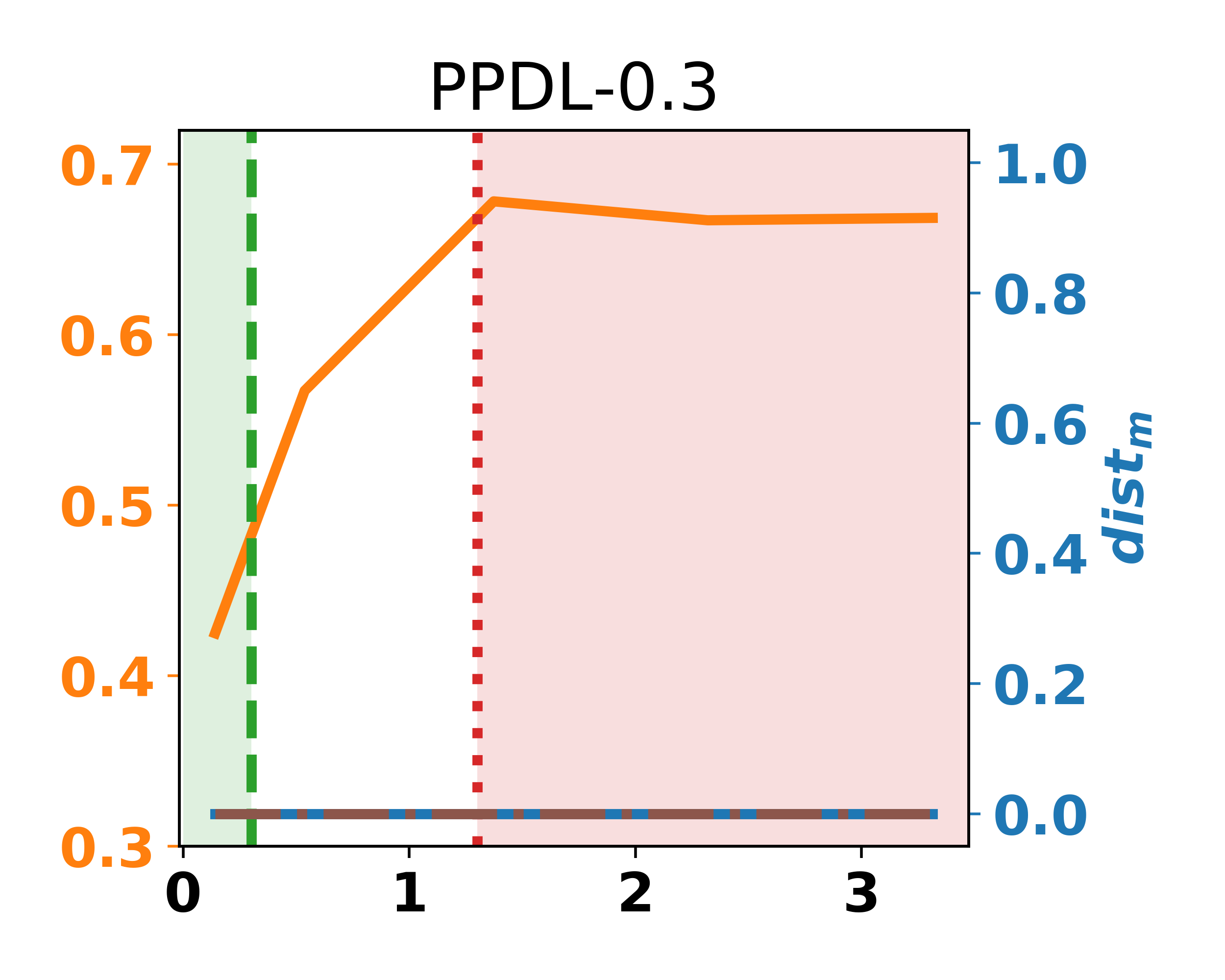}
			\includegraphics[width=0.24\linewidth]{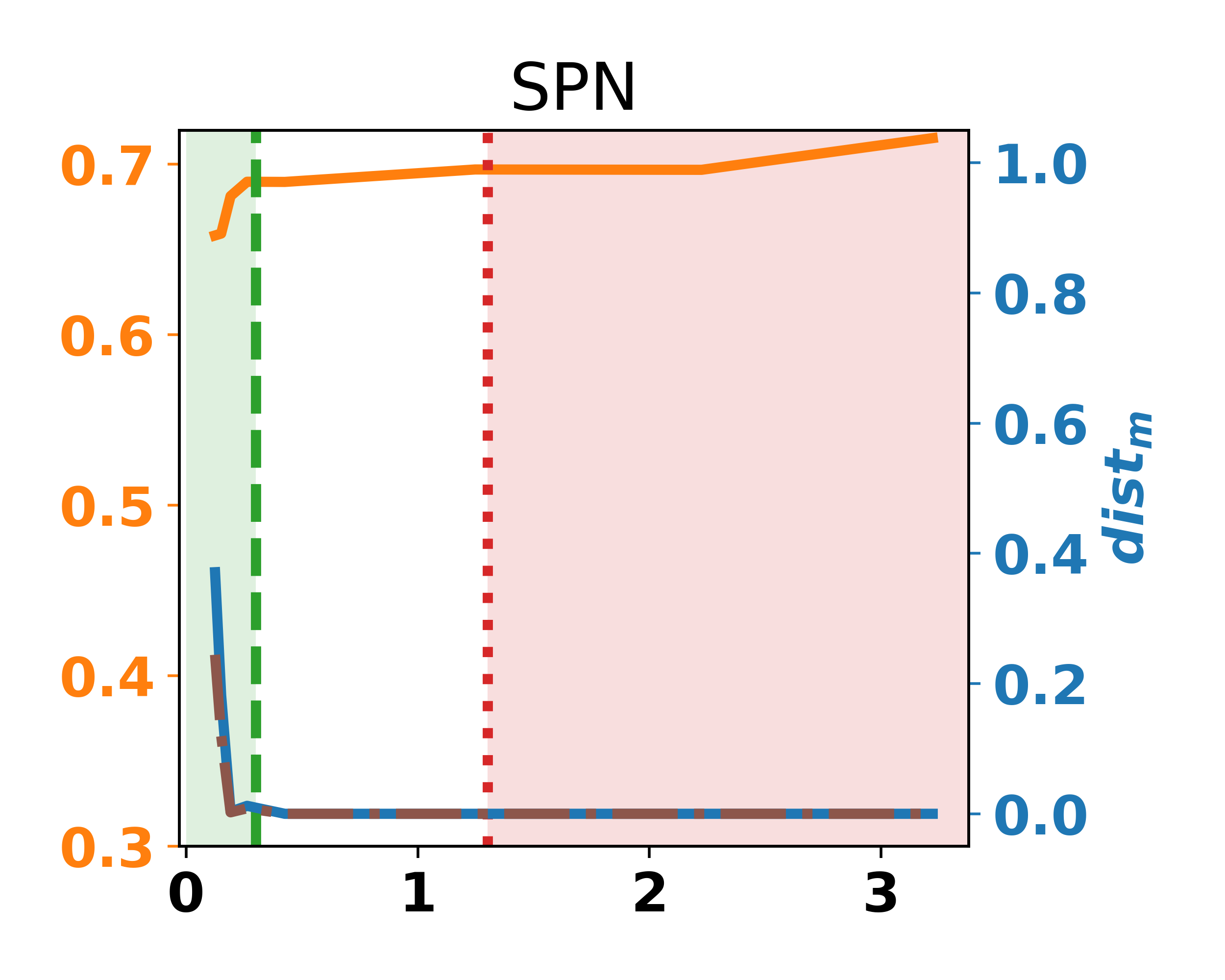}
			\caption{Membership Attack}
		\end{subfigure}
		
		\begin{subfigure}{0.99\linewidth}
			\centering
			\includegraphics[width=0.24\linewidth]{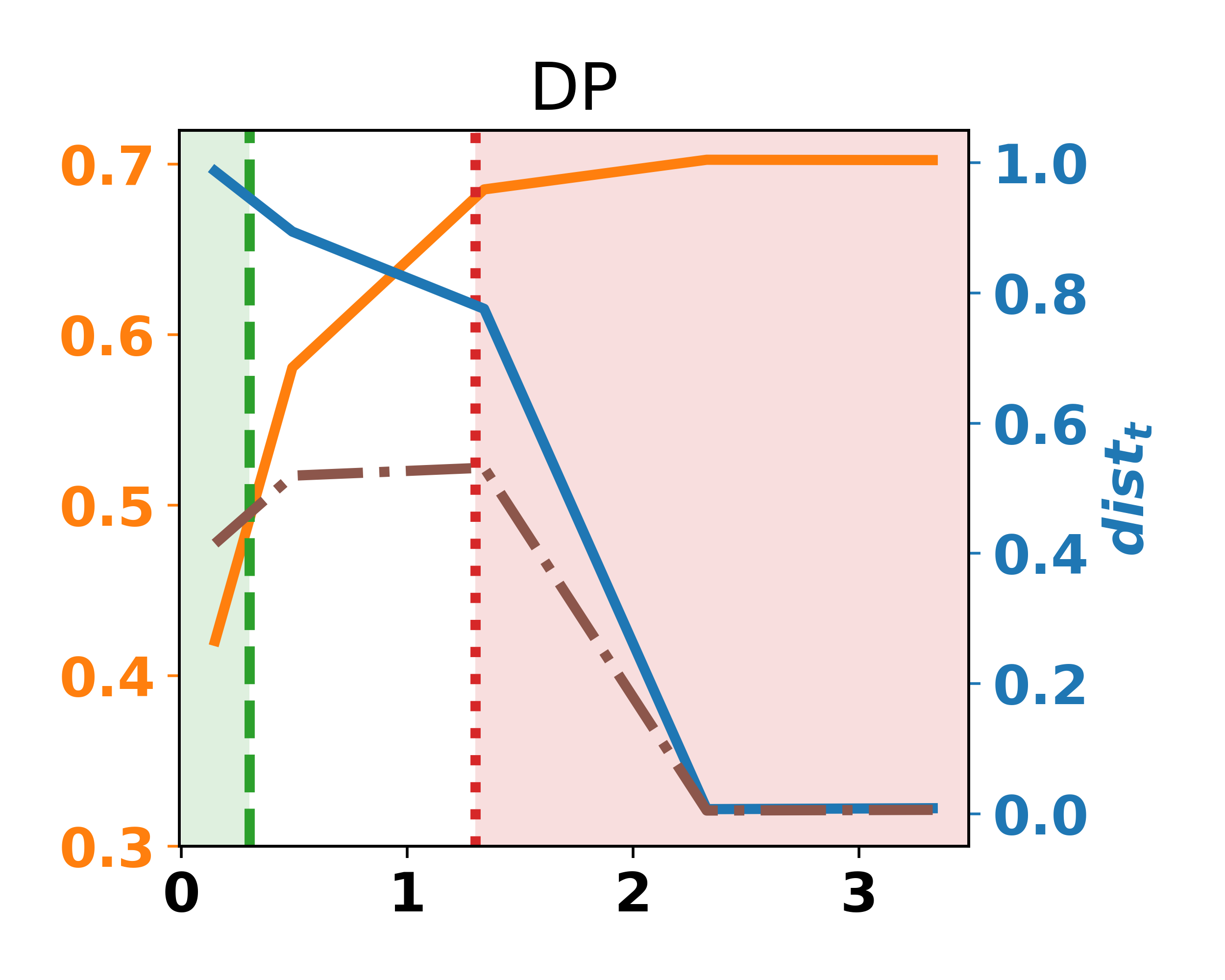}
			\includegraphics[width=0.24\linewidth]{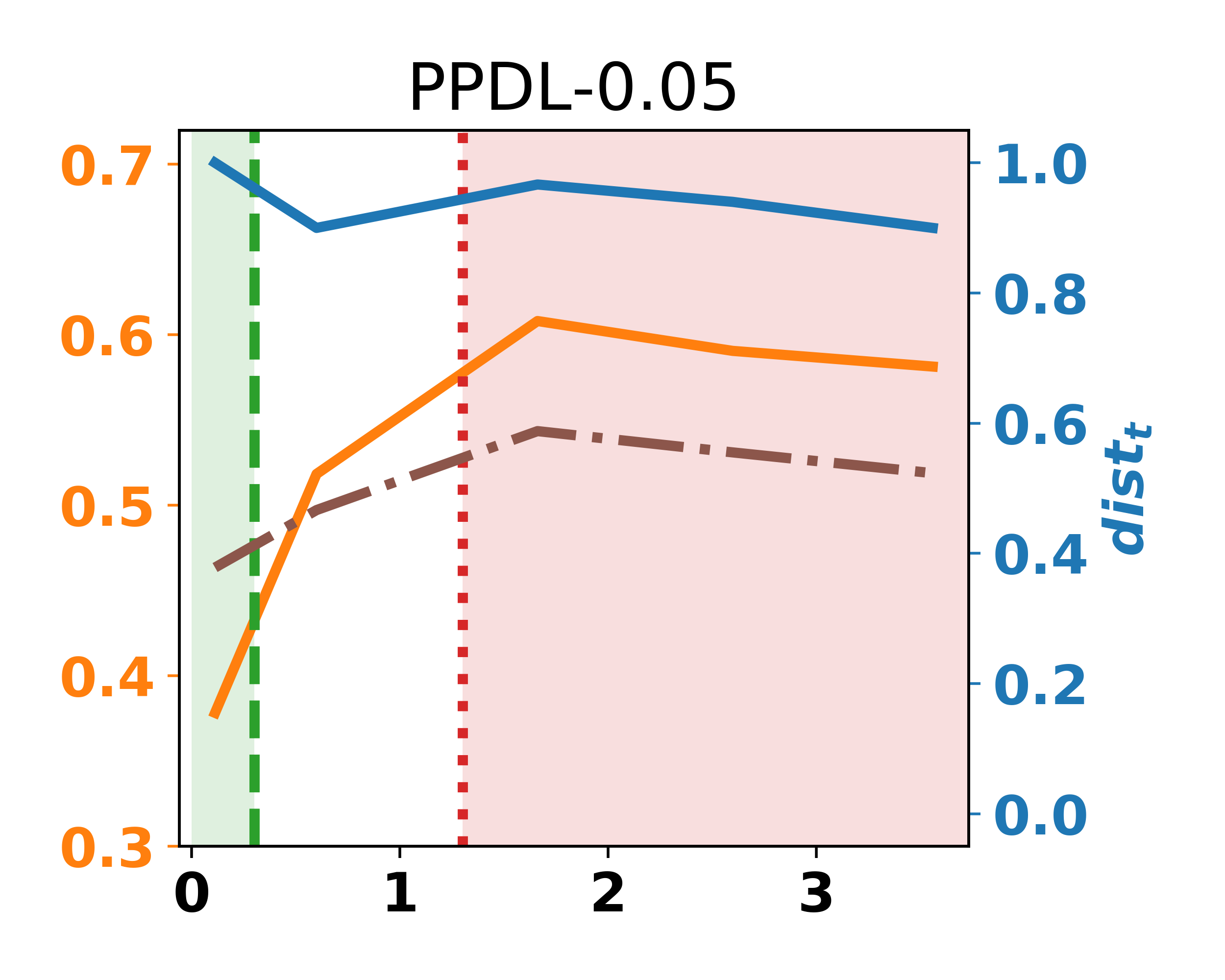}
			\includegraphics[width=0.24\linewidth]{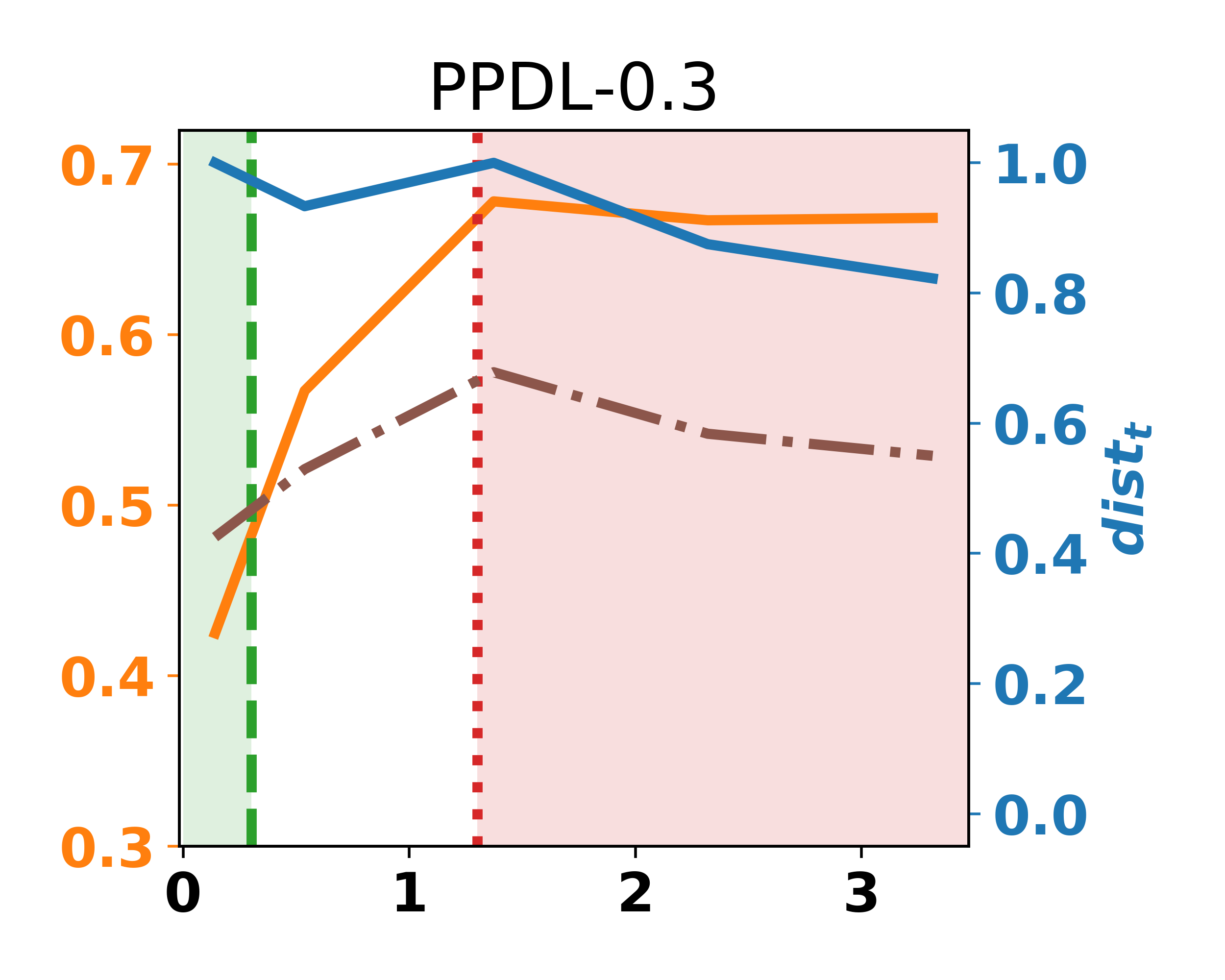}
			\includegraphics[width=0.24\linewidth]{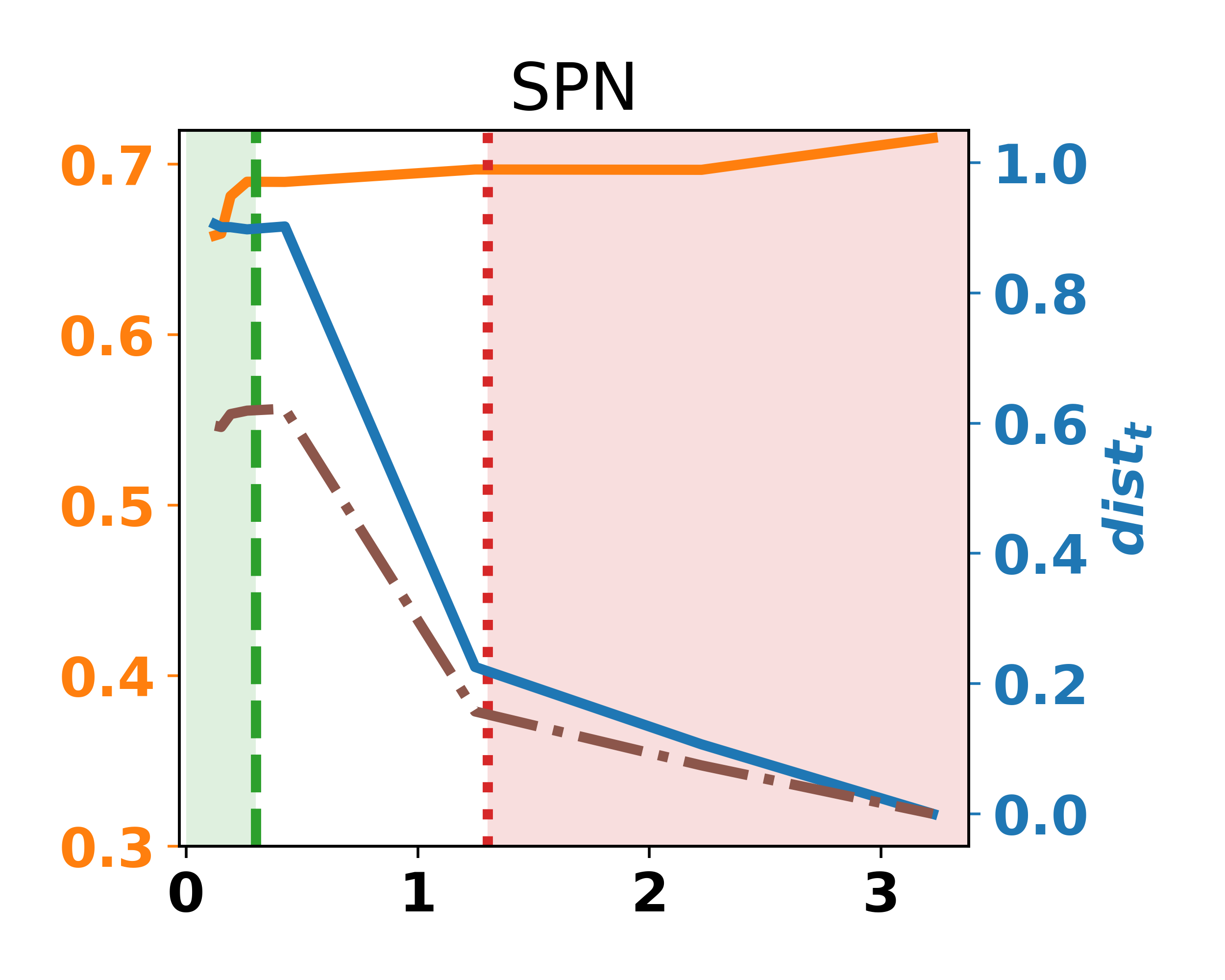}
			\caption{Tracing Attack}
		\end{subfigure}
		
		\caption{Attack with Batch Size 1}
		\label{fig:ppc-cifar10-bs1}
		%\vspace{-0.22cm}
	\end{figure}
	
	\newpage
	
	\begin{figure}[H]	
		%	\begin{subfigure}{0.95\linewidth}
		\centering
		\begin{subfigure}{0.99\linewidth}
			\centering
			\includegraphics[scale=0.7]{imgs/legends/legend_ppc_horizontal.png}
			\\
			\includegraphics[width=0.24\linewidth]{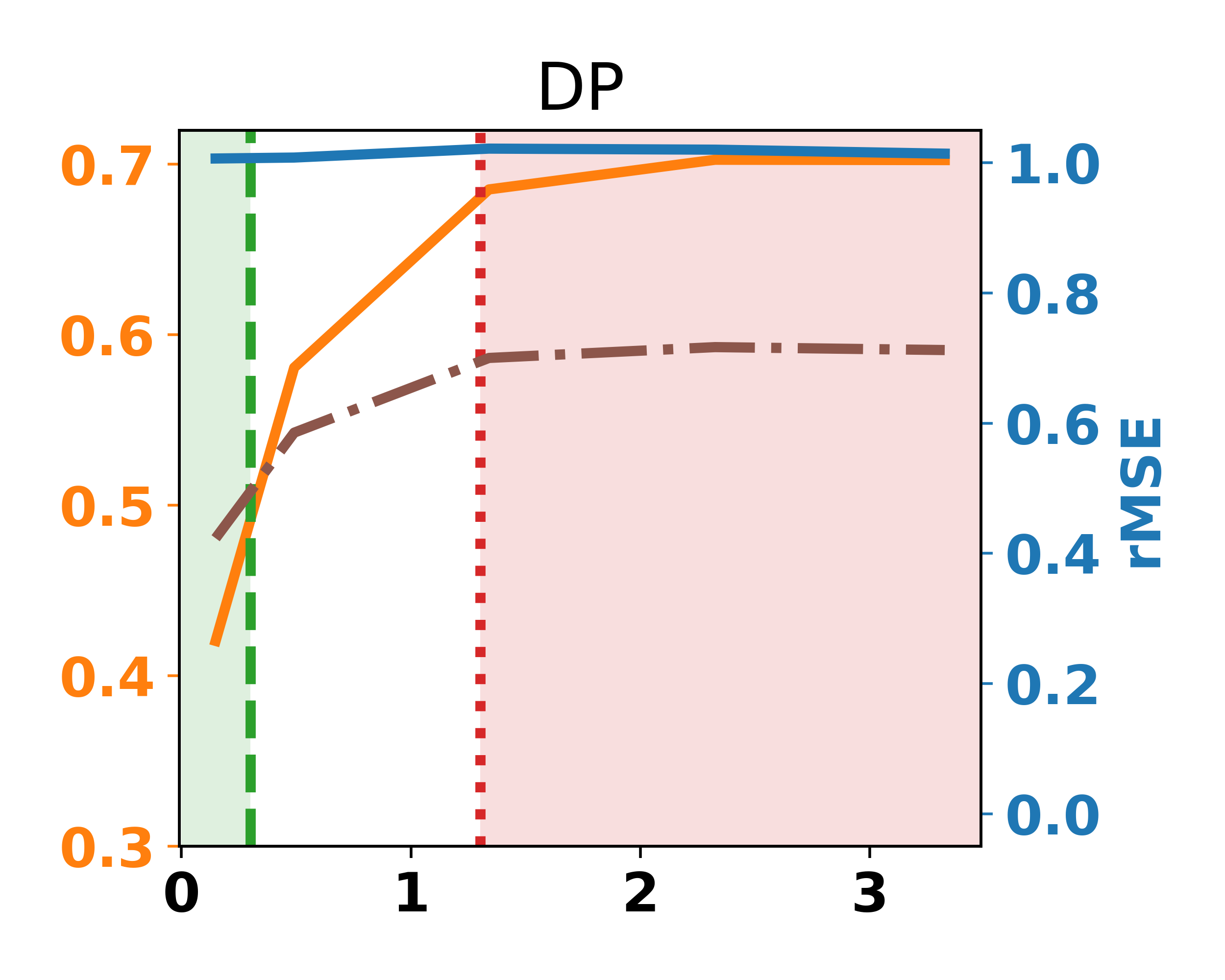}
			\includegraphics[width=0.24\linewidth]{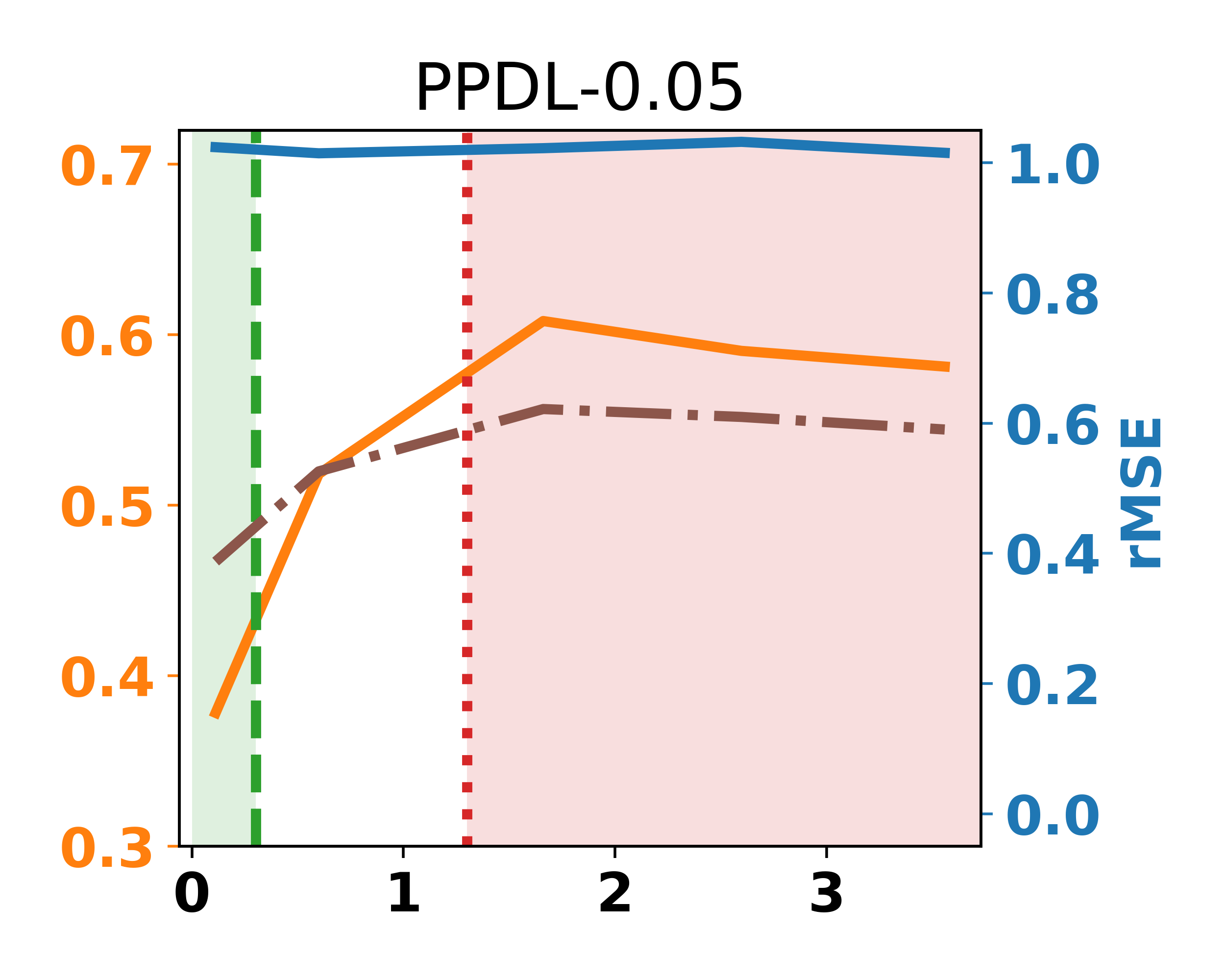}
			\includegraphics[width=0.24\linewidth]{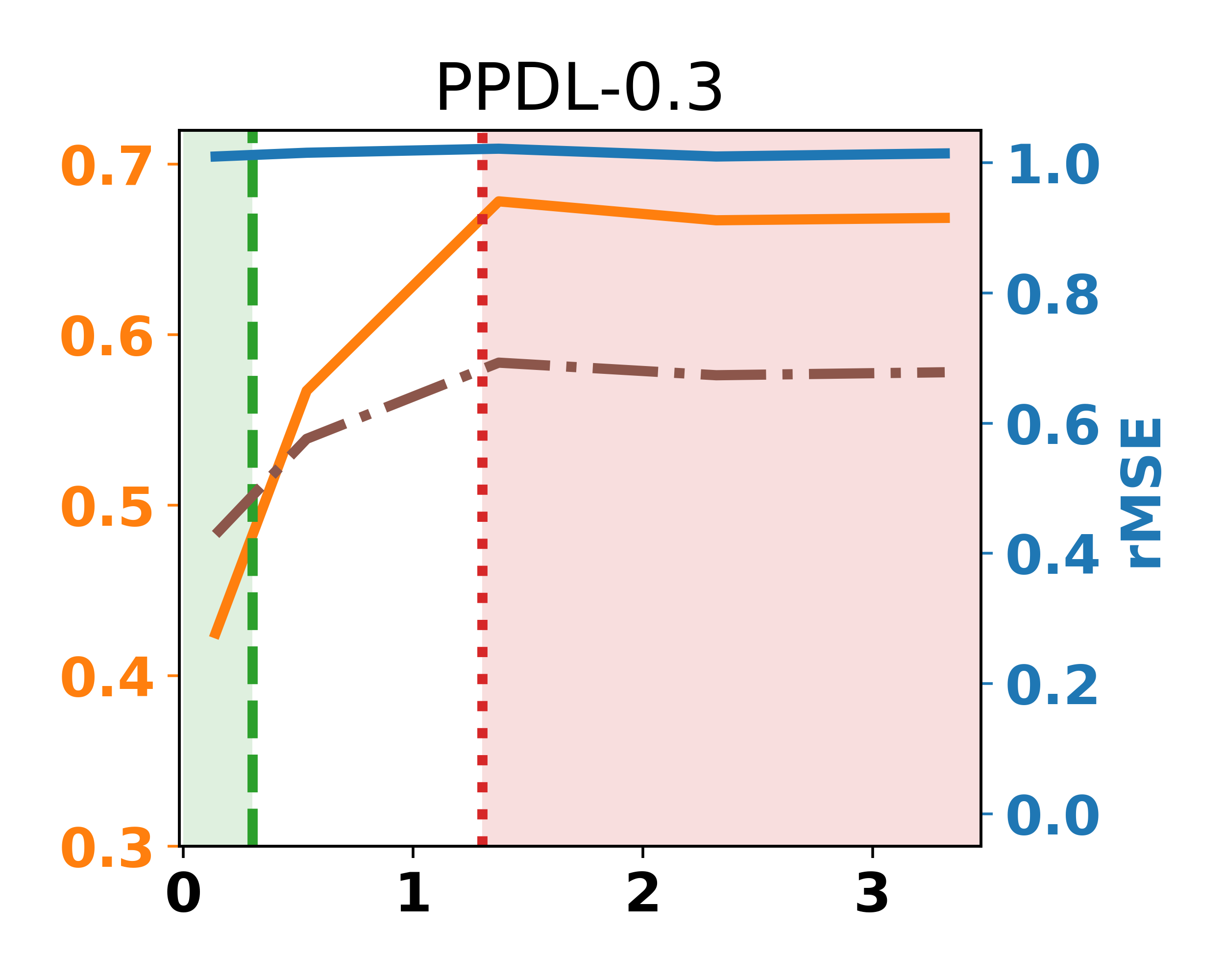}
			\includegraphics[width=0.24\linewidth]{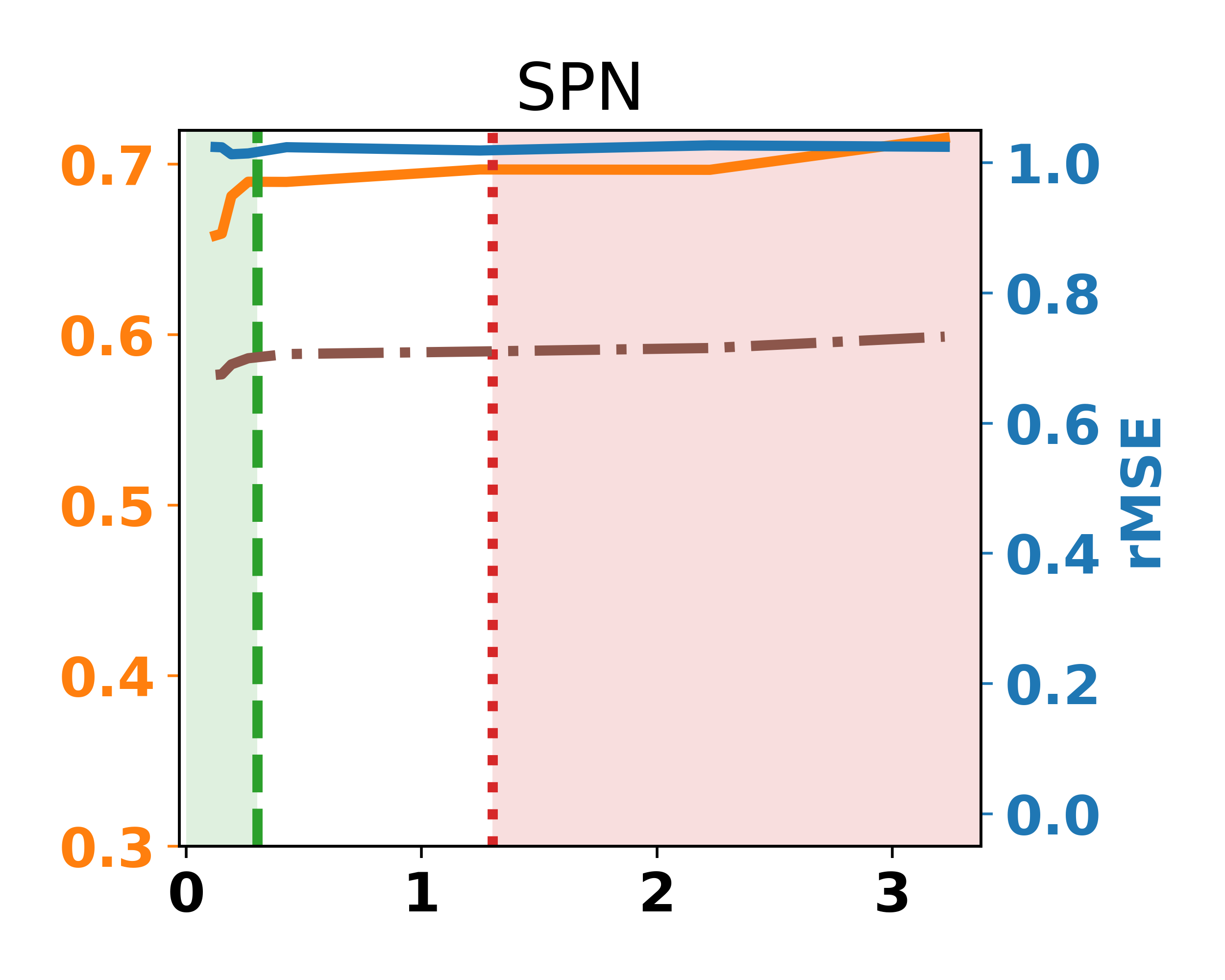}
			\caption{Reconstruction Attack}
		\end{subfigure}
		
		\begin{subfigure}{0.99\linewidth}
			\centering
			\includegraphics[width=0.24\linewidth]{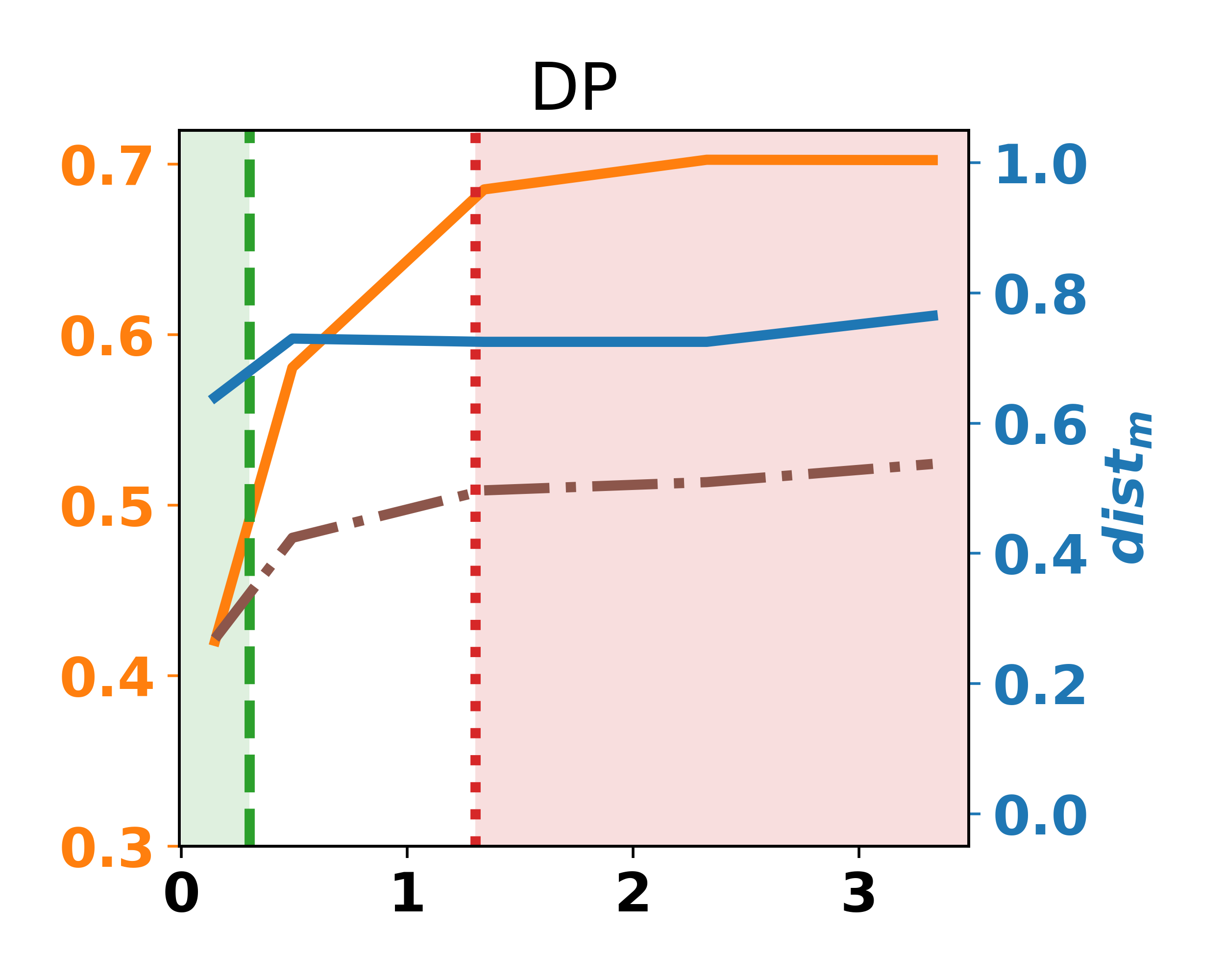}
			\includegraphics[width=0.24\linewidth]{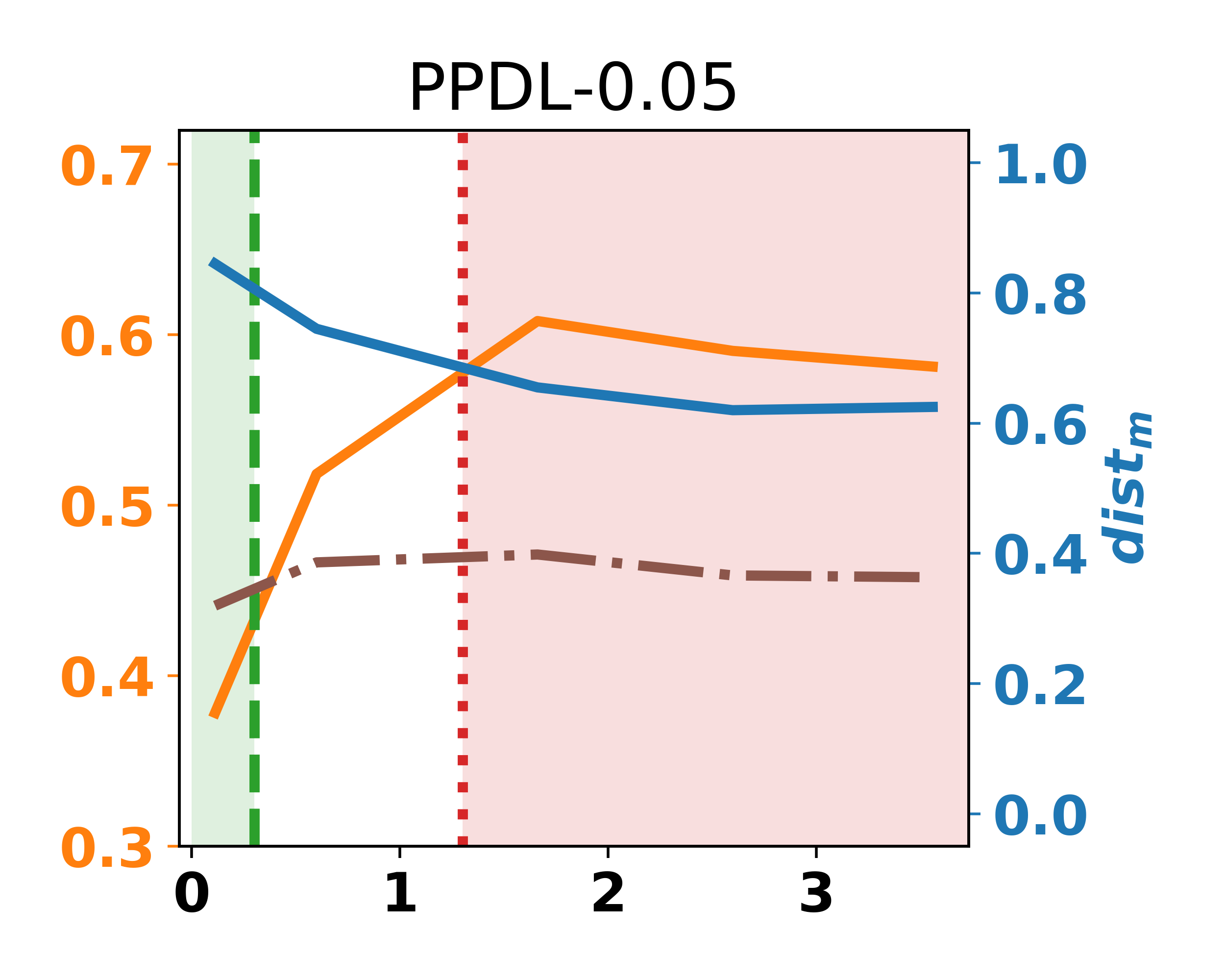}
			\includegraphics[width=0.24\linewidth]{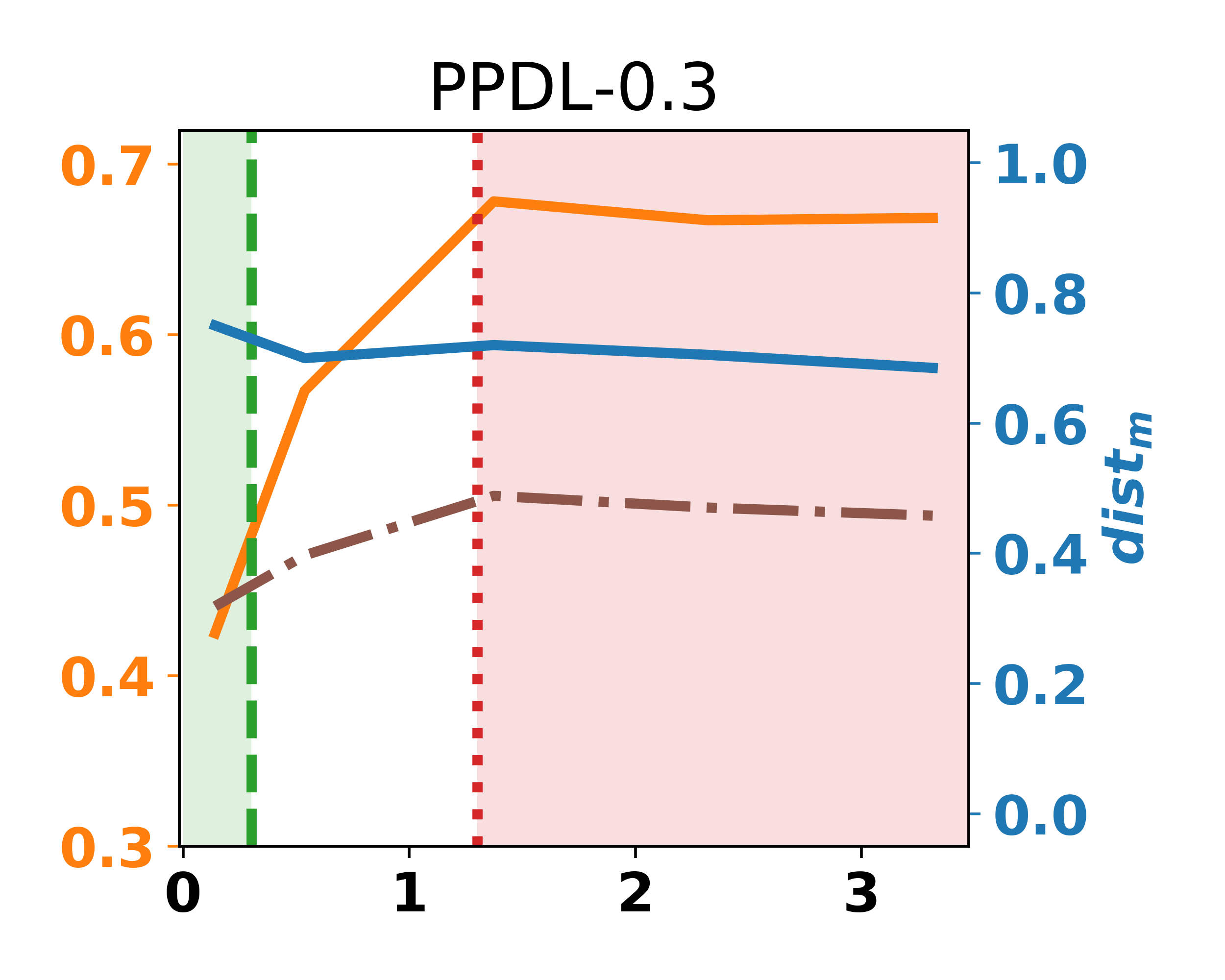}
			\includegraphics[width=0.24\linewidth]{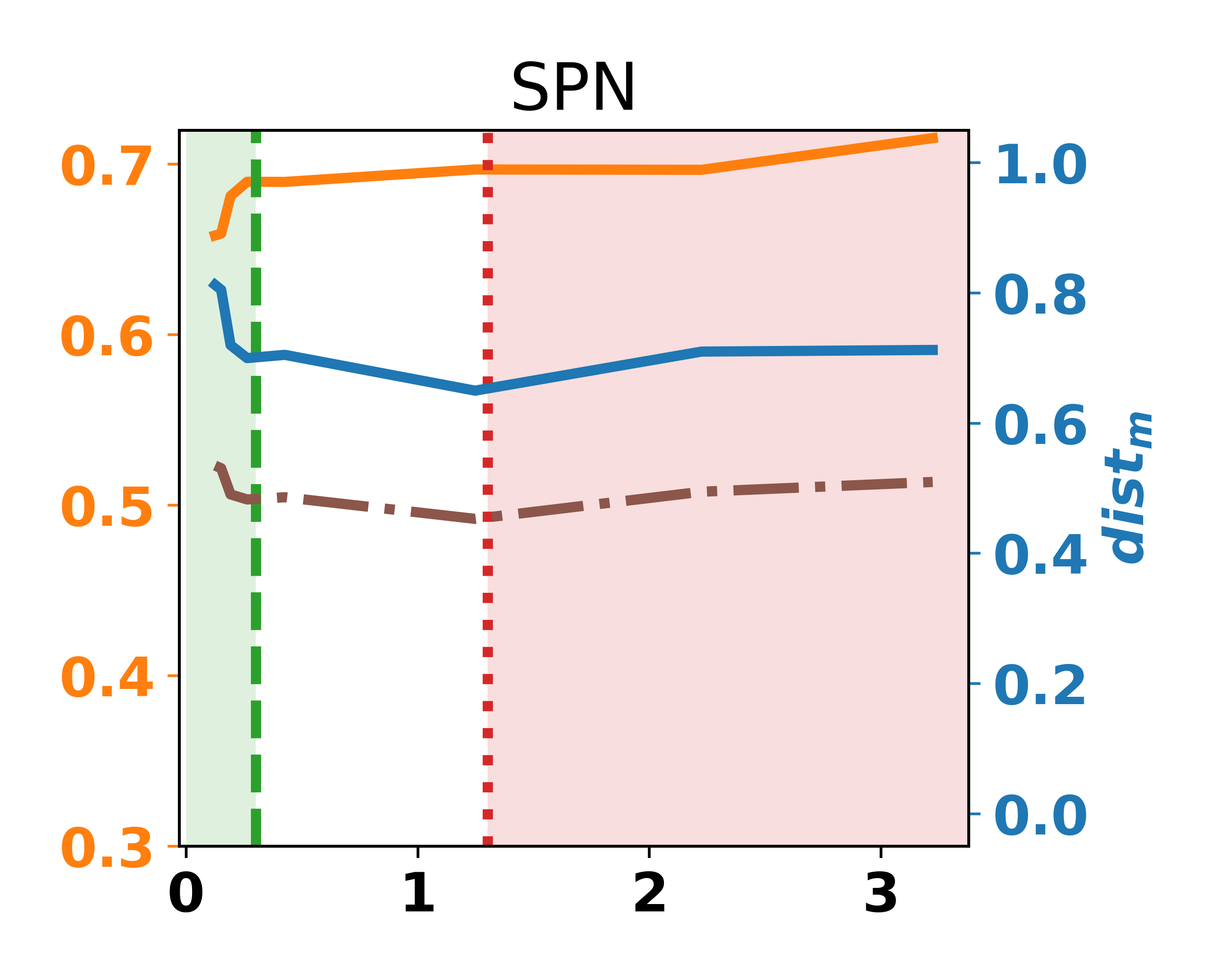}
			\caption{Membership Attack}
		\end{subfigure}
		
		\begin{subfigure}{0.99\linewidth}
			\centering
			\includegraphics[width=0.24\linewidth]{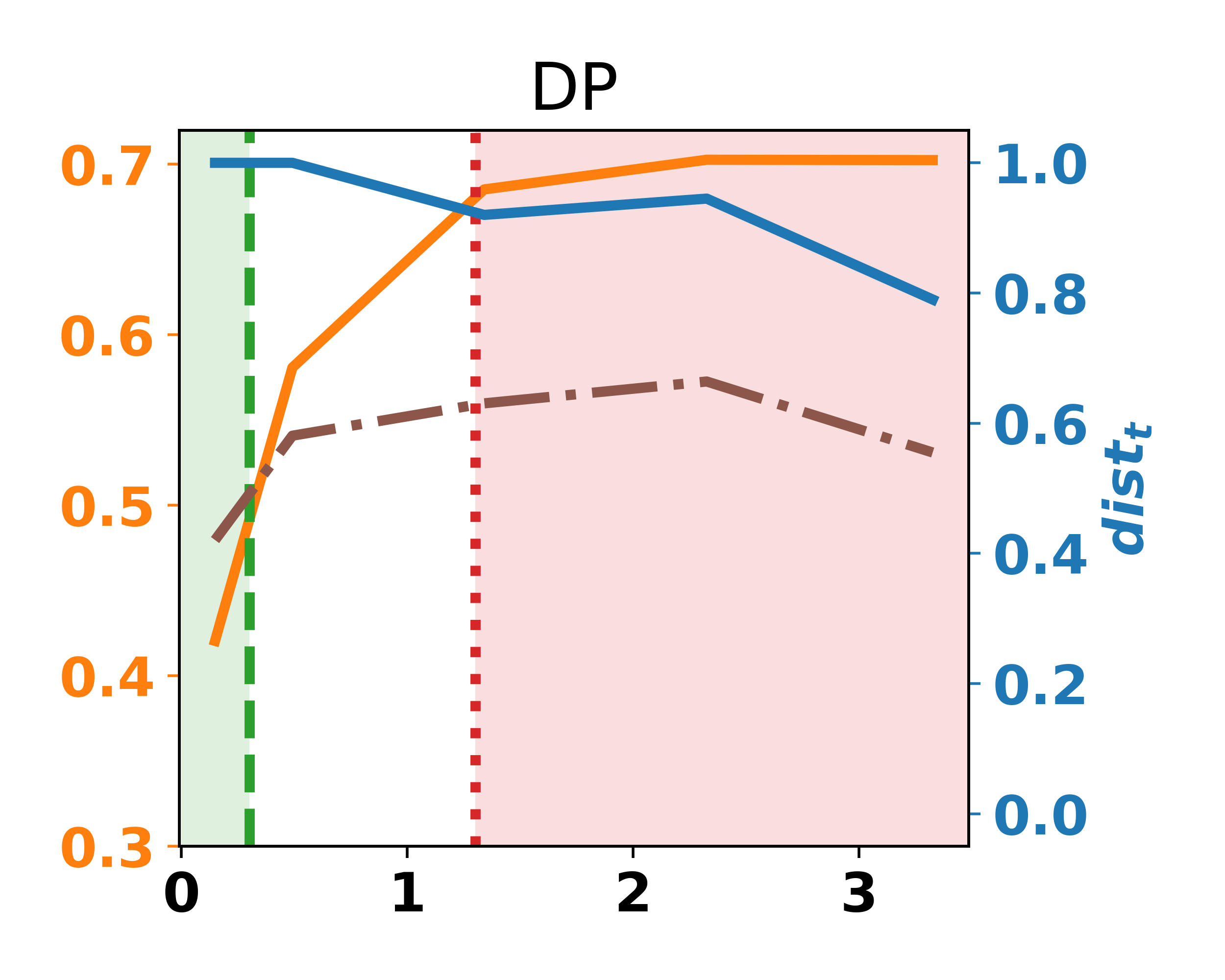}
			\includegraphics[width=0.24\linewidth]{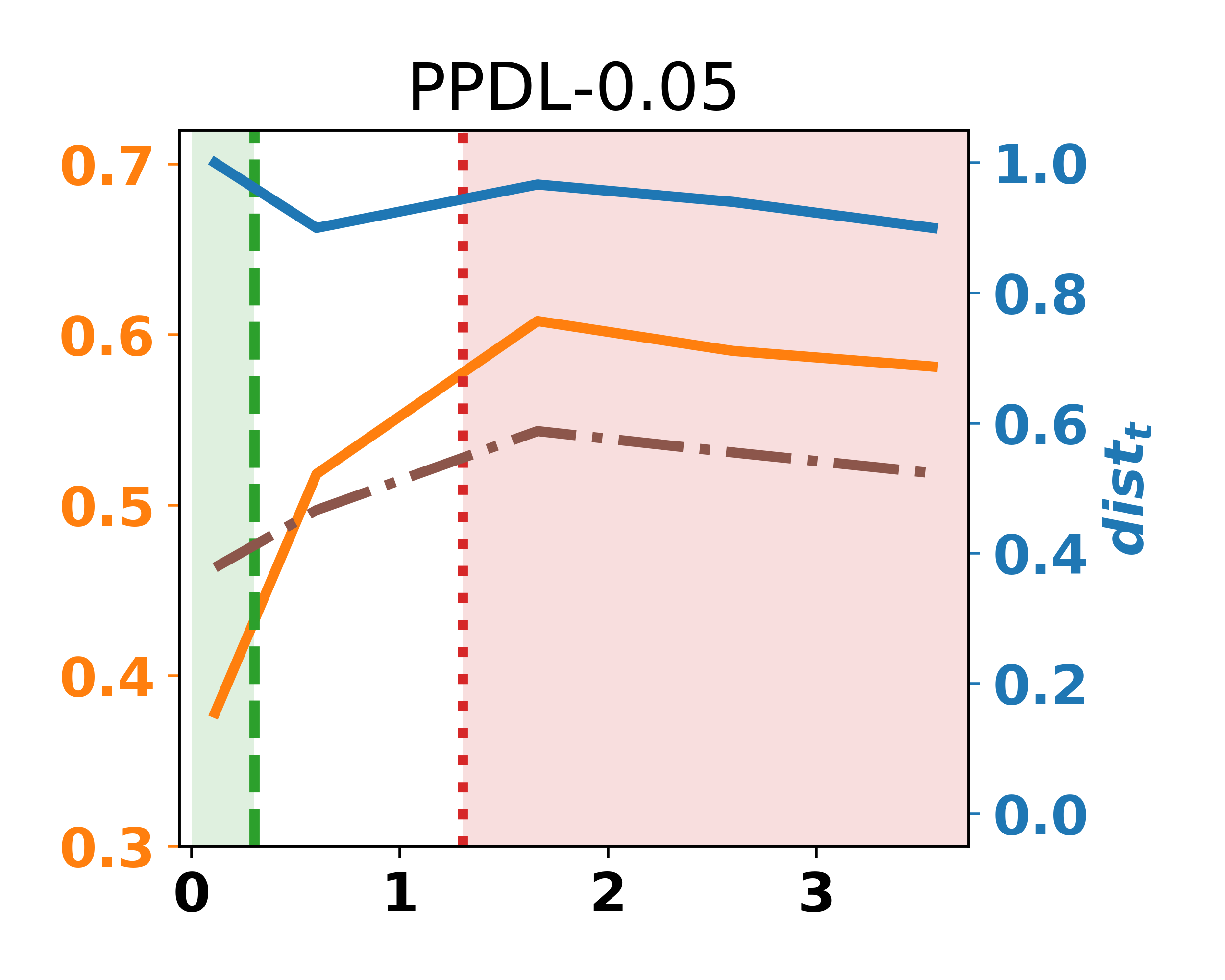}
			\includegraphics[width=0.24\linewidth]{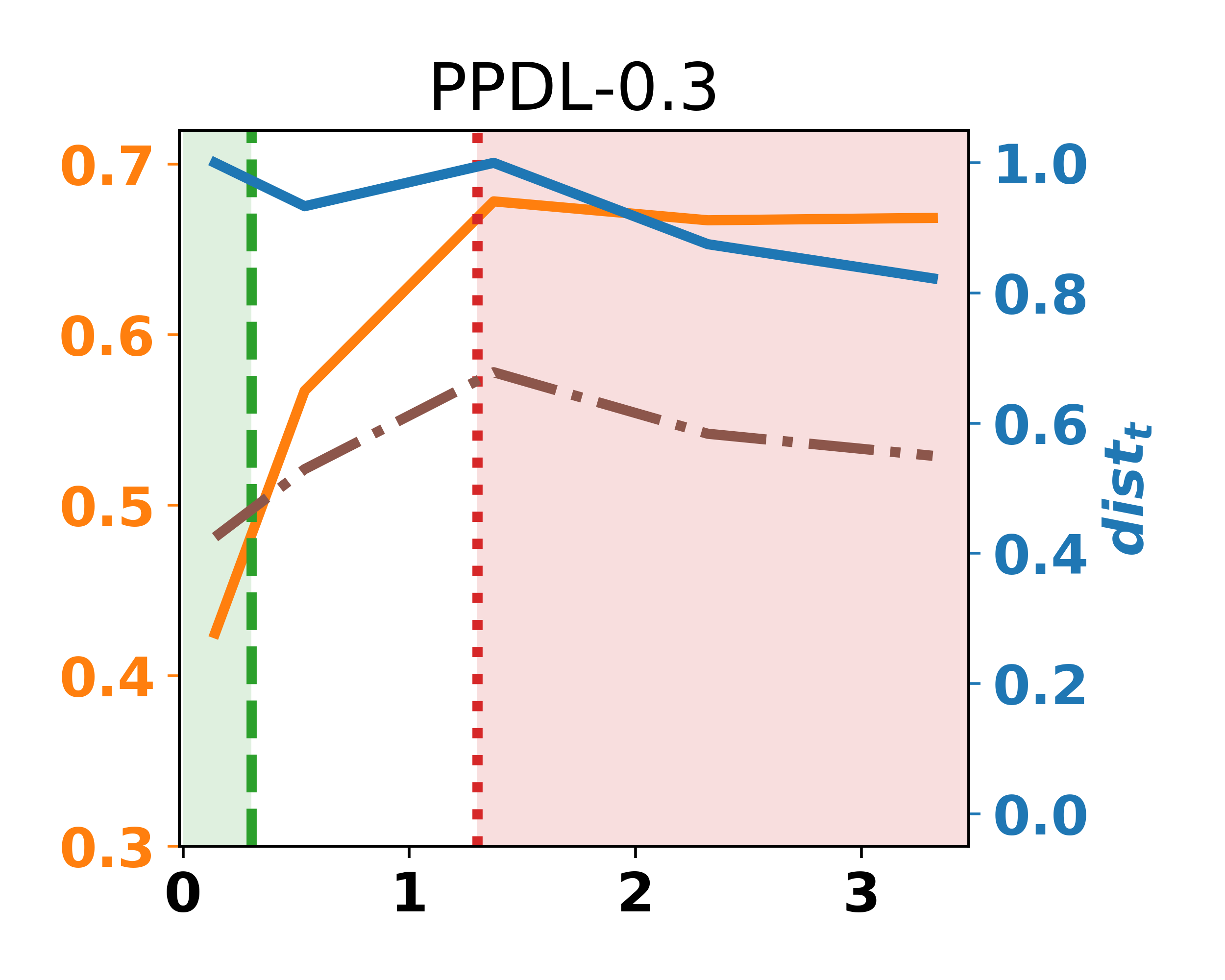}
			\includegraphics[width=0.24\linewidth]{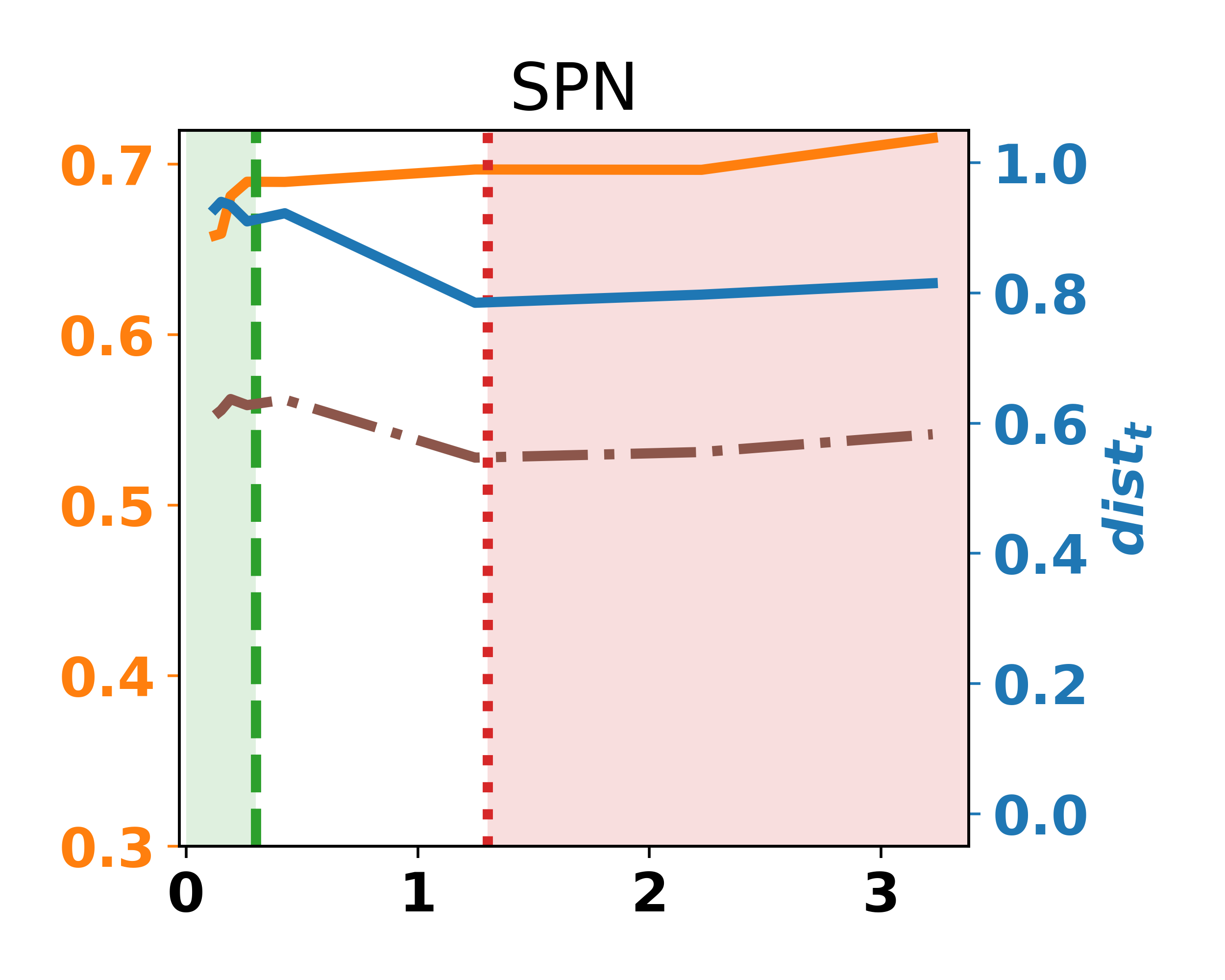}
			\caption{Tracing Attack}
		\end{subfigure}
		
		\caption{Attack with Batch Size 4}
		\label{fig:ppc-cifar10-bs4}
		%\vspace{-0.22cm}
	\end{figure}
	
	\begin{figure}[H]	
		%	\begin{subfigure}{0.95\linewidth}
		\centering
		\begin{subfigure}{0.99\linewidth}
			\centering
			\includegraphics[scale=0.7]{imgs/legends/legend_ppc_horizontal.png}
			\\
			\includegraphics[width=0.24\linewidth]{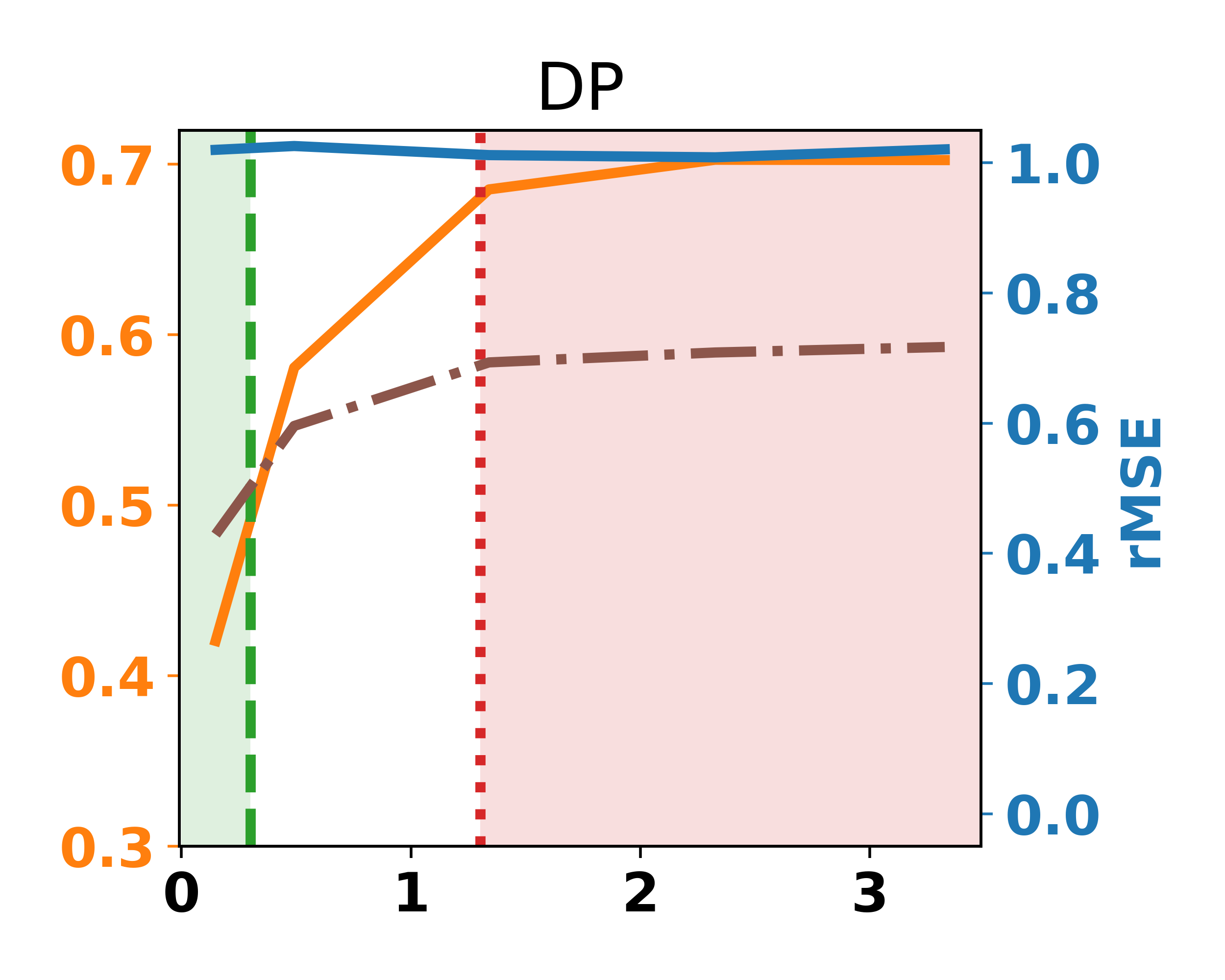}
			\includegraphics[width=0.24\linewidth]{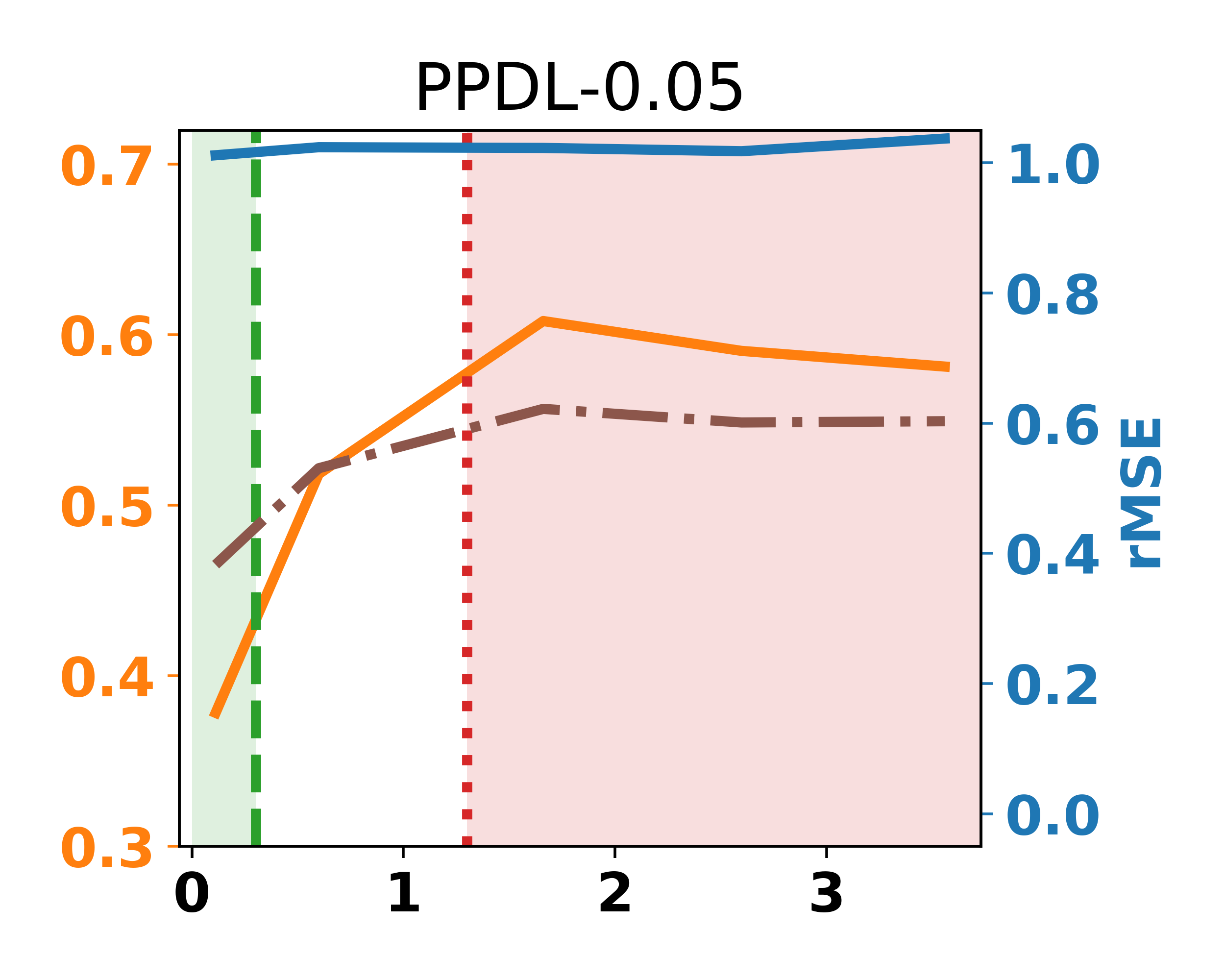}
			\includegraphics[width=0.24\linewidth]{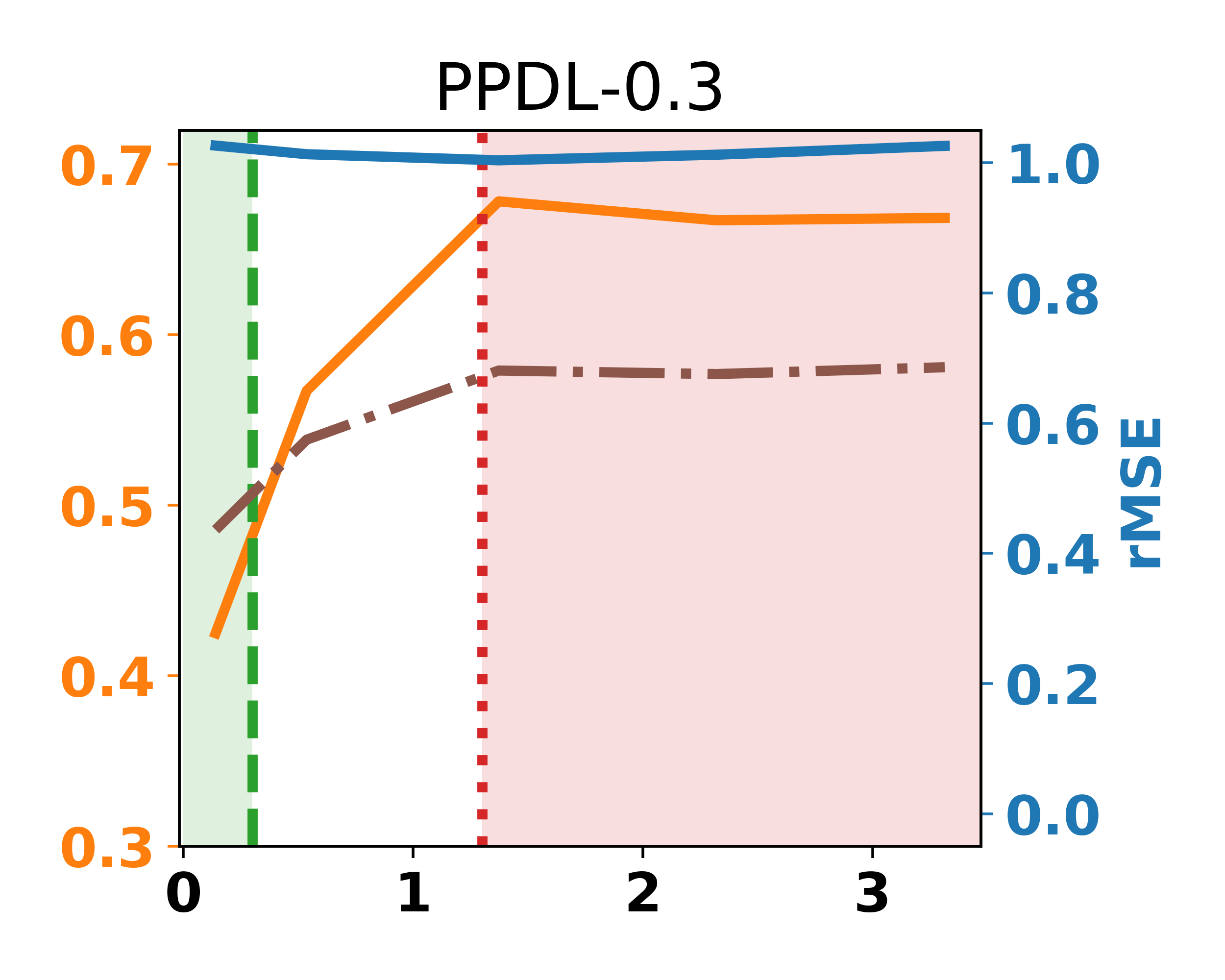}
			\includegraphics[width=0.24\linewidth]{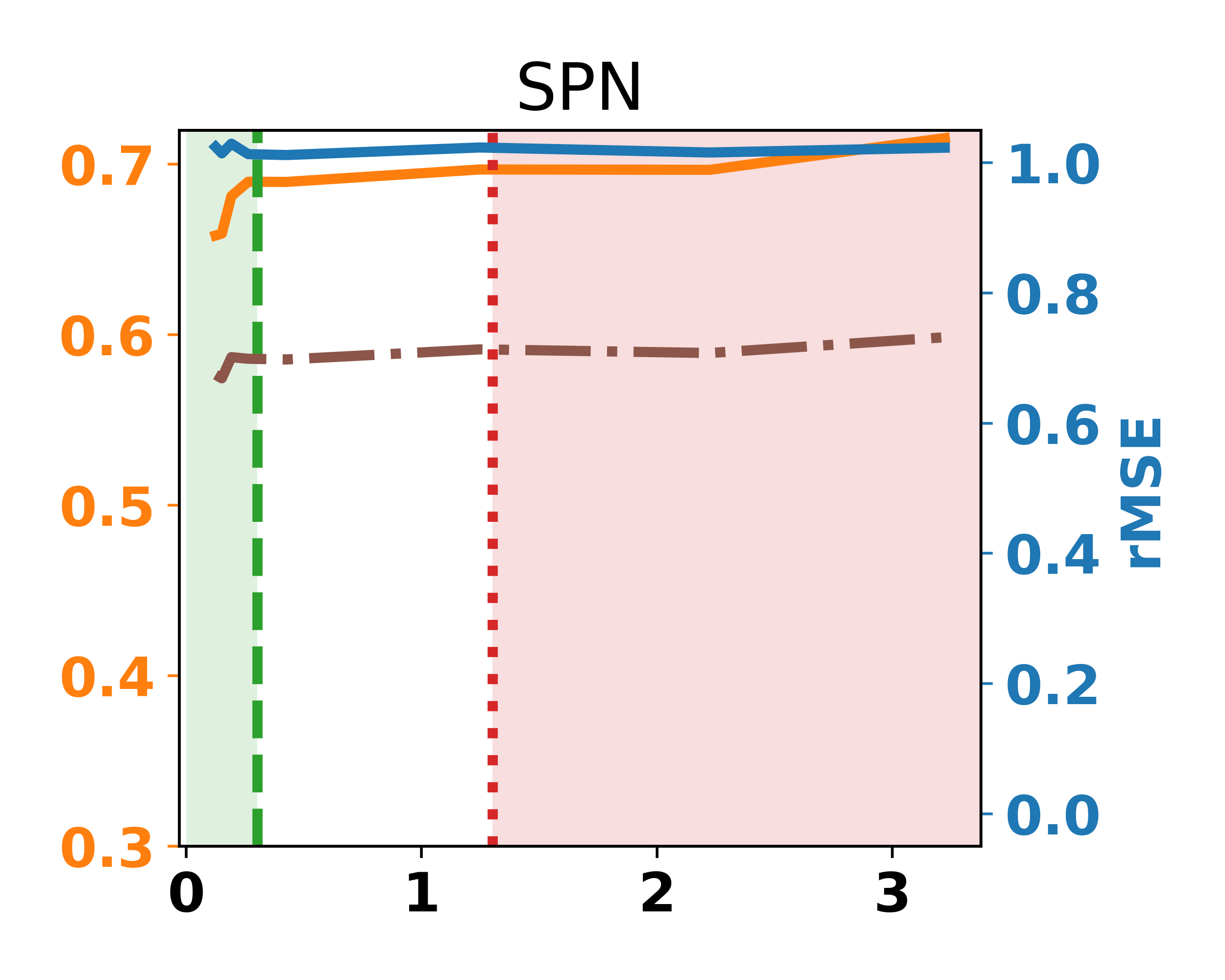}
			\caption{Reconstruction Attack}
		\end{subfigure}
		
		\begin{subfigure}{0.99\linewidth}
			\centering
			\includegraphics[width=0.24\linewidth]{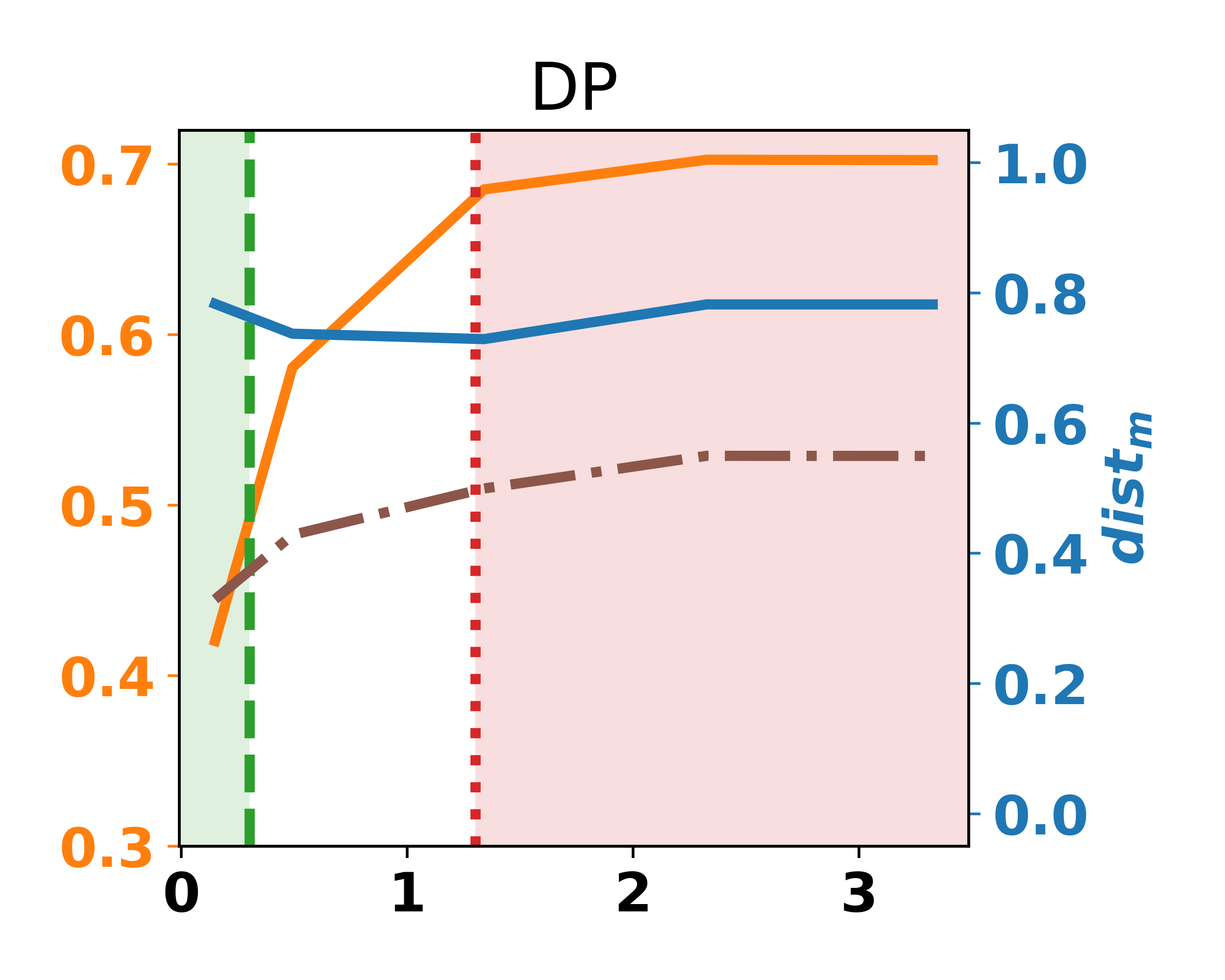}
			\includegraphics[width=0.24\linewidth]{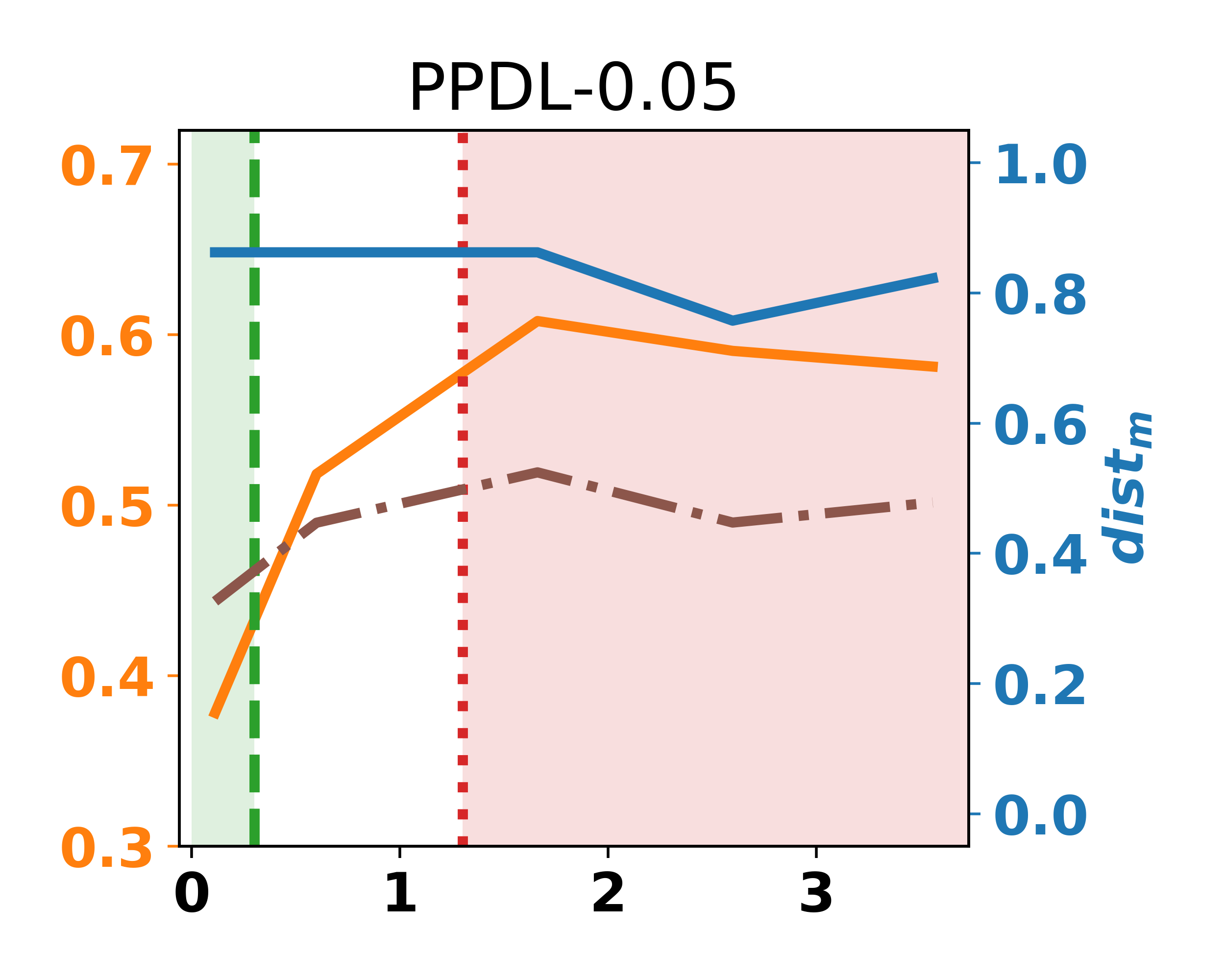}
			\includegraphics[width=0.24\linewidth]{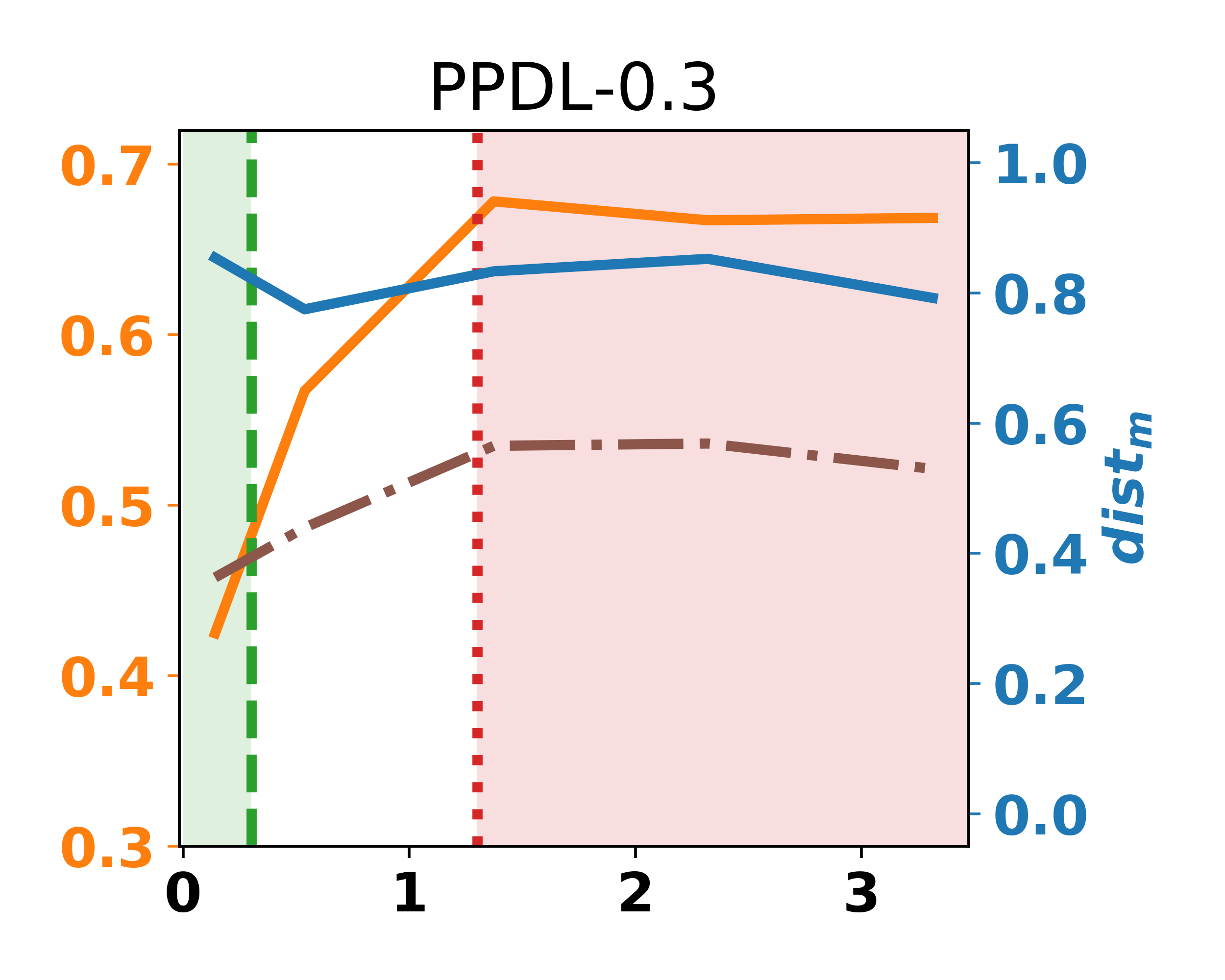}
			\includegraphics[width=0.24\linewidth]{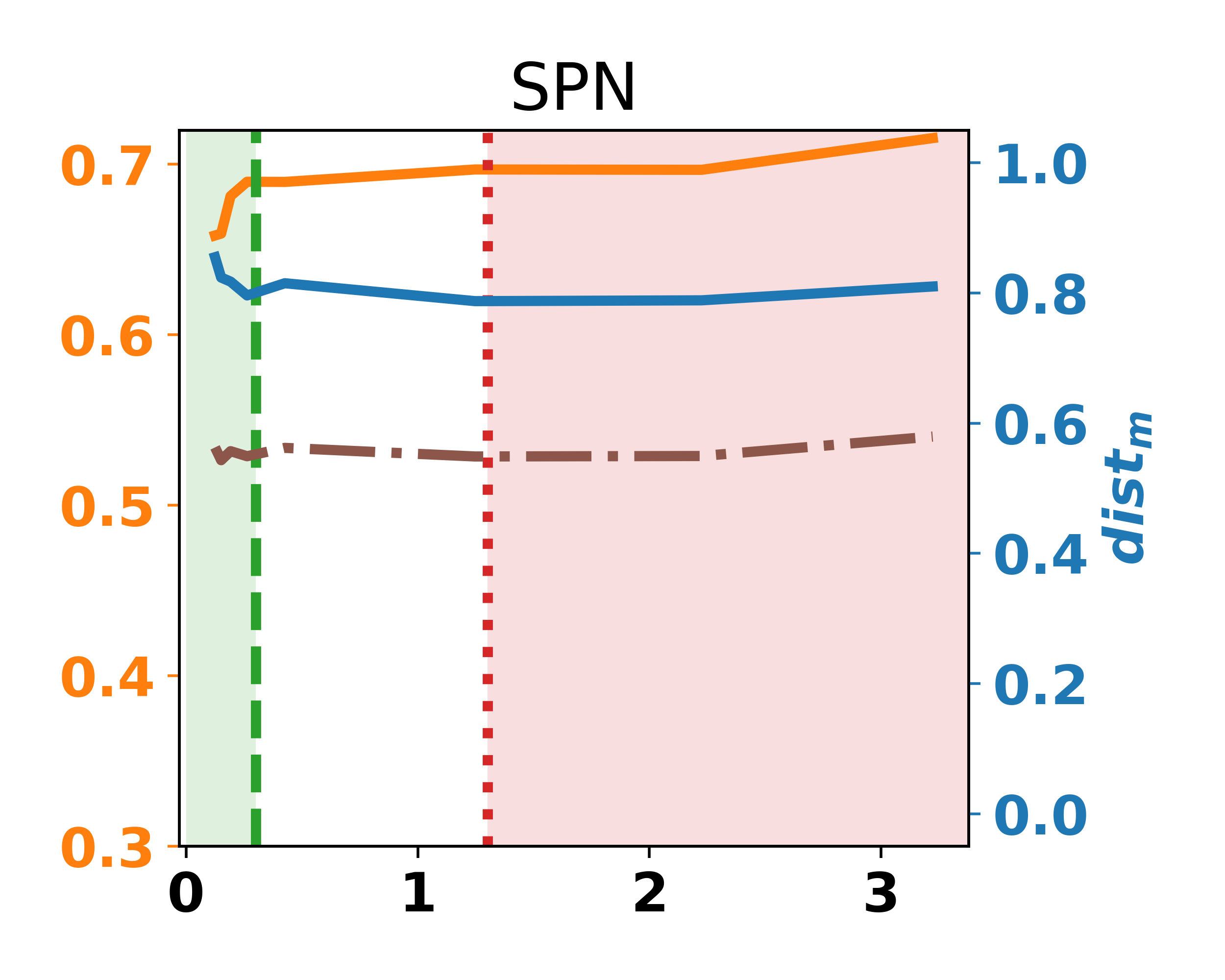}
			\caption{Membership Attack}
		\end{subfigure}
		
		\begin{subfigure}{0.99\linewidth}
			\centering
			\includegraphics[width=0.24\linewidth]{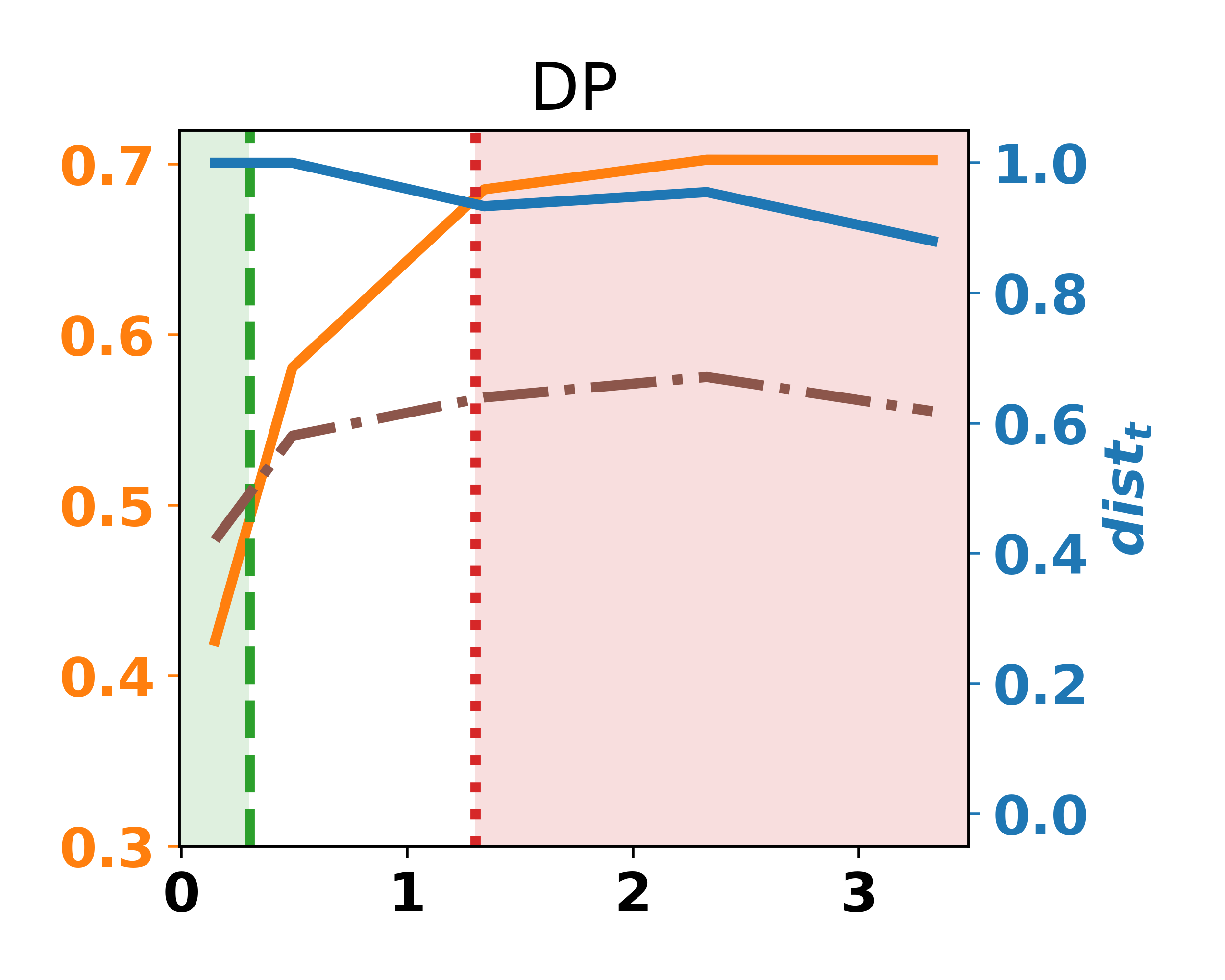}
			\includegraphics[width=0.24\linewidth]{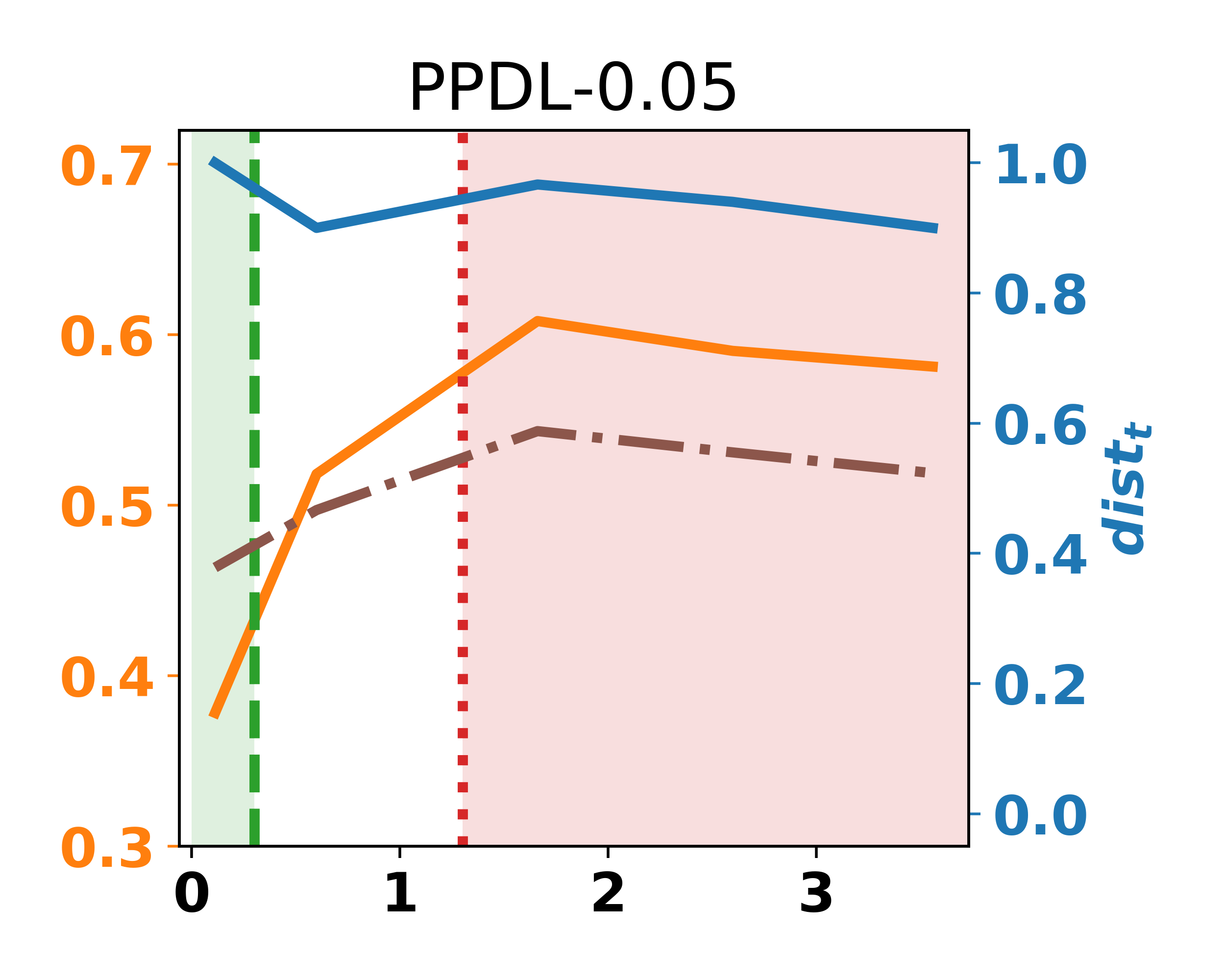}
			\includegraphics[width=0.24\linewidth]{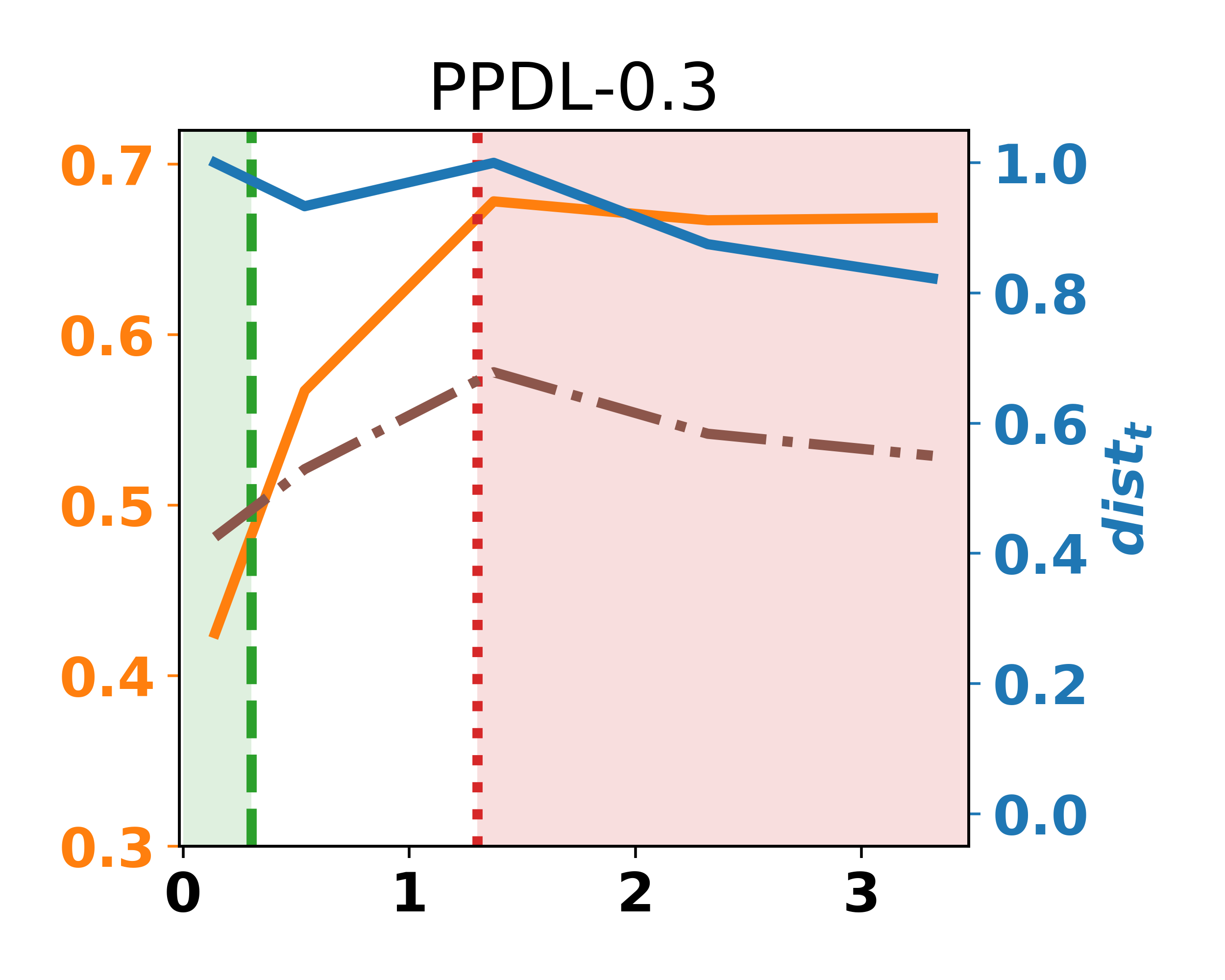}
			\includegraphics[width=0.24\linewidth]{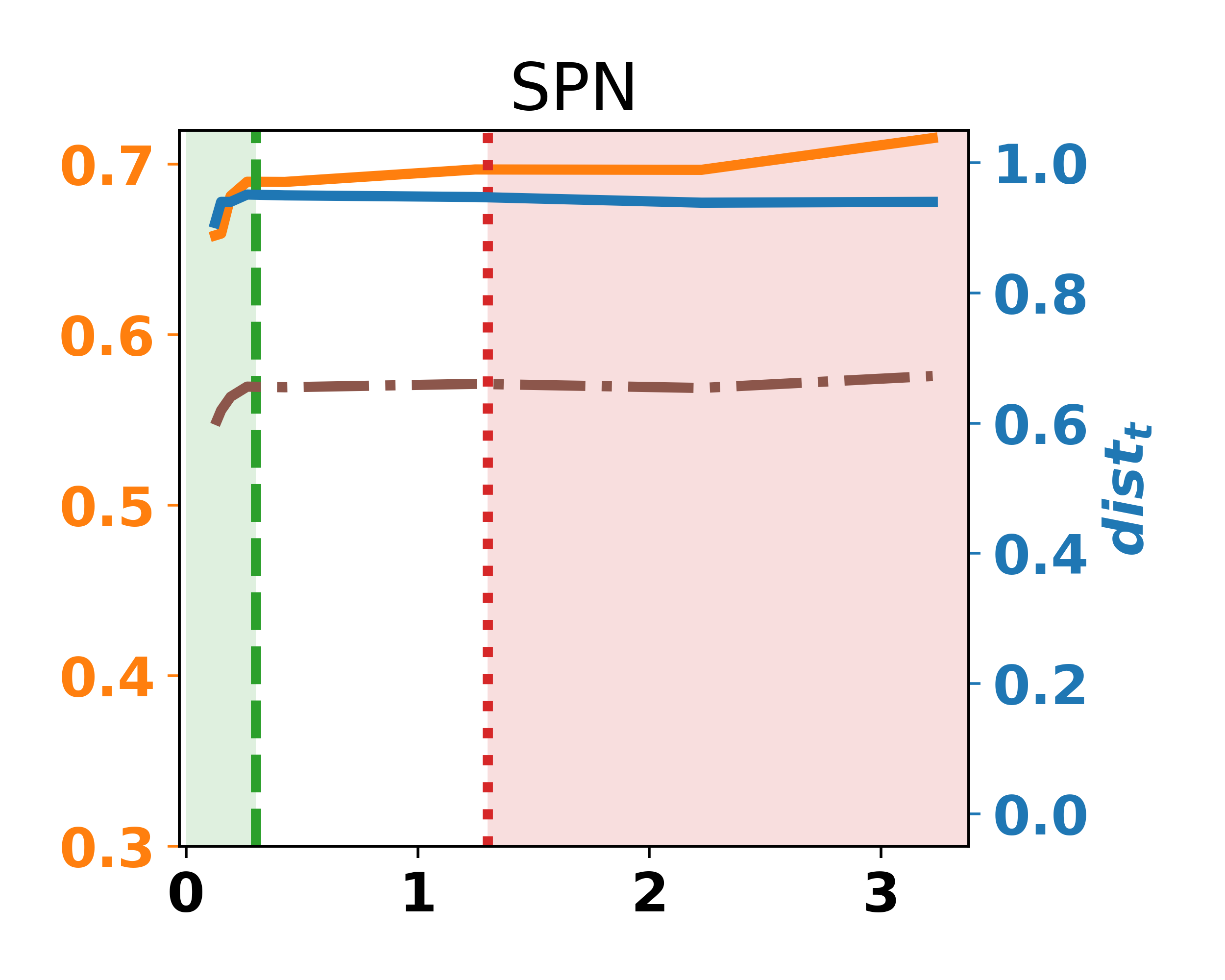}
			\caption{Tracing Attack}
		\end{subfigure}
		
		\caption{Attack with Batch Size 8}
		\label{fig:ppc-cifar10-bs8}
		%\vspace{-0.22cm}
	\end{figure}
	
	\subsubsection{Calibrated Averaged Performance (CAP)}
	
	\begin{table}[H]
		\centering
		\begin{tabular}{l|c|c|c|c|c|c|c|c|c}
			\toprule
			& \multicolumn{3}{c|}{Reconstruction} & \multicolumn{3}{c|}{Membership} & \multicolumn{3}{c}{Tracing} \\
			\midrule
			BS & 1 & 4 & 8 & 1 & 4 & 8 & 1 & 4 & 8 \\
			\midrule
			DP \cite{DLDP_Abadi16} & 0.57 & 0.63 & 0.63 & 0.00 & 0.45 & 0.47 & 0.42 & 0.57 & 0.58 \\
			PPDL-0.05 \cite{PPDL/shokri2015} & 0.55 & 0.55 & 0.55 & 0.00 & 0.37 & 0.44 & 0.50 & 0.50 & 0.50 \\
			PPDL-0.3 \cite{PPDL/shokri2015} & 0.57 & 0.61 & 0.61 & 0.00 & 0.43 & 0.49 & 0.54 & 0.54 & 0.54 \\
			SPN (ours) & \textbf{0.69} &\textbf{0.70} &\textbf{0.70} & \textbf{0.24} &\textbf{0.50} & \textbf{0.56}& \textbf{0.61} & \textbf{0.63} & \textbf{0.64}\\
			\bottomrule
		\end{tabular}
		\caption{CAP performance with different batch size on CIFAR10 for reconstruction, membership, and tracing attack. Higher better. BS = Attack Batch Size.}
		\label{tab:CAP-cifar10-supp}
	\end{table}
	
	\subsubsection{Reconstructed Images}
	
	\begin{figure}[H]
		\centering
		
		\begin{subfigure}{0.32\linewidth}
			\centering
			\includegraphics[scale=0.8]{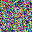}
			\hspace{1pt}
			\includegraphics[scale=0.8]{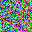} % + svhn
			\hspace{1pt}
			\includegraphics[scale=0.8]{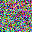}
			\hspace{1pt}
			\includegraphics[scale=0.8]{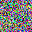}
			\\
			\vspace{2pt}
			\includegraphics[scale=0.8]{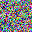}
			\hspace{1pt}
			\includegraphics[scale=0.8]{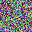} % + svhn
			\hspace{1pt}
			\includegraphics[scale=0.8]{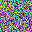}
			\hspace{1pt}
			\includegraphics[scale=0.8]{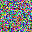}
			\\
			\vspace{2pt}
			\includegraphics[scale=0.8]{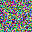}
			\hspace{1pt}
			\includegraphics[scale=0.8]{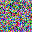} % + svhn
			\hspace{1pt}
			\includegraphics[scale=0.8]{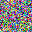}
			\hspace{1pt}
			\includegraphics[scale=0.8]{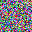}
			\\
			\vspace{2pt}
			\includegraphics[scale=0.8]{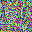}
			\hspace{1pt}
			\includegraphics[scale=0.8]{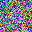} % + svhn
			\hspace{1pt}
			\includegraphics[scale=0.8]{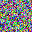}
			\hspace{1pt}
			\includegraphics[scale=0.8]{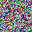}
			\caption{$1.05 (0.36)$ \label{fig:recon-images-green-cifar10-supp}}
		\end{subfigure}
		\hfill
		\begin{subfigure}{0.32\linewidth}
			\centering
			\includegraphics[scale=0.8]{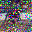}
			\hspace{1pt}
			\includegraphics[scale=0.8]{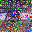} % + svhn
			\hspace{1pt}
			\includegraphics[scale=0.8]{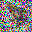}
			\hspace{1pt}
			\includegraphics[scale=0.8]{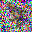}
			\\
			\vspace{2pt}
			\includegraphics[scale=0.8]{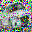}
			\hspace{1pt}
			\includegraphics[scale=0.8]{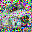} % + svhn
			\hspace{1pt}
			\includegraphics[scale=0.8]{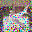}
			\hspace{1pt}
			\includegraphics[scale=0.8]{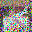}
			\\
			\vspace{2pt}
			\includegraphics[scale=0.8]{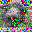}
			\hspace{1pt}
			\includegraphics[scale=0.8]{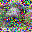} % + svhn
			\hspace{1pt}
			\includegraphics[scale=0.8]{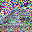}
			\hspace{1pt}
			\includegraphics[scale=0.8]{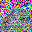}
			\\
			\vspace{2pt}
			\includegraphics[scale=0.8]{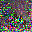}
			\hspace{1pt}
			\includegraphics[scale=0.8]{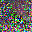} % + svhn
			\hspace{1pt}
			\includegraphics[scale=0.8]{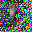}
			\hspace{1pt}
			\includegraphics[scale=0.8]{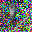}
			\caption{$1.05 (18.44)$ \label{fig:recon-images-white-cifar10-supp}}%; 1.10 (1.30)}
		\end{subfigure}
		\hfill
		\begin{subfigure}{0.32\linewidth}
			\centering
			\includegraphics[scale=0.8]{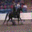}
			\hspace{1pt}
			\includegraphics[scale=0.8]{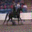} % + svhn
			\hspace{1pt}
			\includegraphics[scale=0.8]{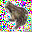}
			\hspace{1pt}
			\includegraphics[scale=0.8]{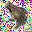}
			\\
			\vspace{2pt}
			\includegraphics[scale=0.8]{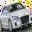}
			\hspace{1pt}
			\includegraphics[scale=0.8]{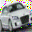} % + svhn
			\hspace{1pt}
			\includegraphics[scale=0.8]{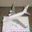}
			\hspace{1pt}
			\includegraphics[scale=0.8]{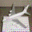}
			\\
			\vspace{2pt}
			\includegraphics[scale=0.8]{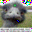}
			\hspace{1pt}
			\includegraphics[scale=0.8]{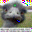} % + svhn
			\hspace{1pt}
			\includegraphics[scale=0.8]{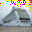}
			\hspace{1pt}
			\includegraphics[scale=0.8]{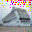}
			\\
			\vspace{2pt}
			\includegraphics[scale=0.8]{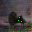}
			\hspace{1pt}
			\includegraphics[scale=0.8]{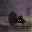} % + svhn
			\hspace{1pt}
			\includegraphics[scale=0.8]{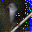}
			\hspace{1pt}
			\includegraphics[scale=0.8]{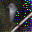}
			\caption{$0.48 (1106.17)$ \label{fig:recon-images-red-cifar10-supp}} %0.94 (1.59)}
		\end{subfigure}
		
		%\begin{subfigure}{0.2\linewidth}
		%	\centering
		%	\includegraphics[scale=0.8]{imgs/recimgs/0_0.0001_0.0151.png}
		%	\hspace{3pt}
		%	\includegraphics[scale=0.8]{imgs/recimgs/391_0.0001_0.0116.png}
		%	\caption{0.02 (3.34); 0.01 (3.28)}
		%\end{subfigure}
		\caption{Reconstructed images from different region mentioned in the main paper. \textbf{(a)} Green region \textbf{(b)} White region \textbf{(c)} Red region. 
			Values inside bracket are mean of $\frac{||B_I||}{||E_B||}$ and values outside are mean of rMSE of reconstructed w.r.t. original images. Noted that some images are having solid color pixels (e.g. the frog image), causing difficulty in reconstruction (e.g. noise pixel in the white solid color region), which is similar to MNIST dataset.
		}
		\label{fig:recon-images-cifar10-supp}
	\end{figure}
	
	\newpage
	\subsection{CIFAR100}
	
	\subsubsection{Privacy-Preserving Characteristics (PPC)}
	
	\begin{figure}[H]	
		%	\begin{subfigure}{0.95\linewidth}
		\centering
		\begin{subfigure}{0.99\linewidth}
			\centering
			\includegraphics[scale=0.7]{imgs/legends/legend_ppc_horizontal.png}
			\\
			\includegraphics[width=0.24\linewidth]{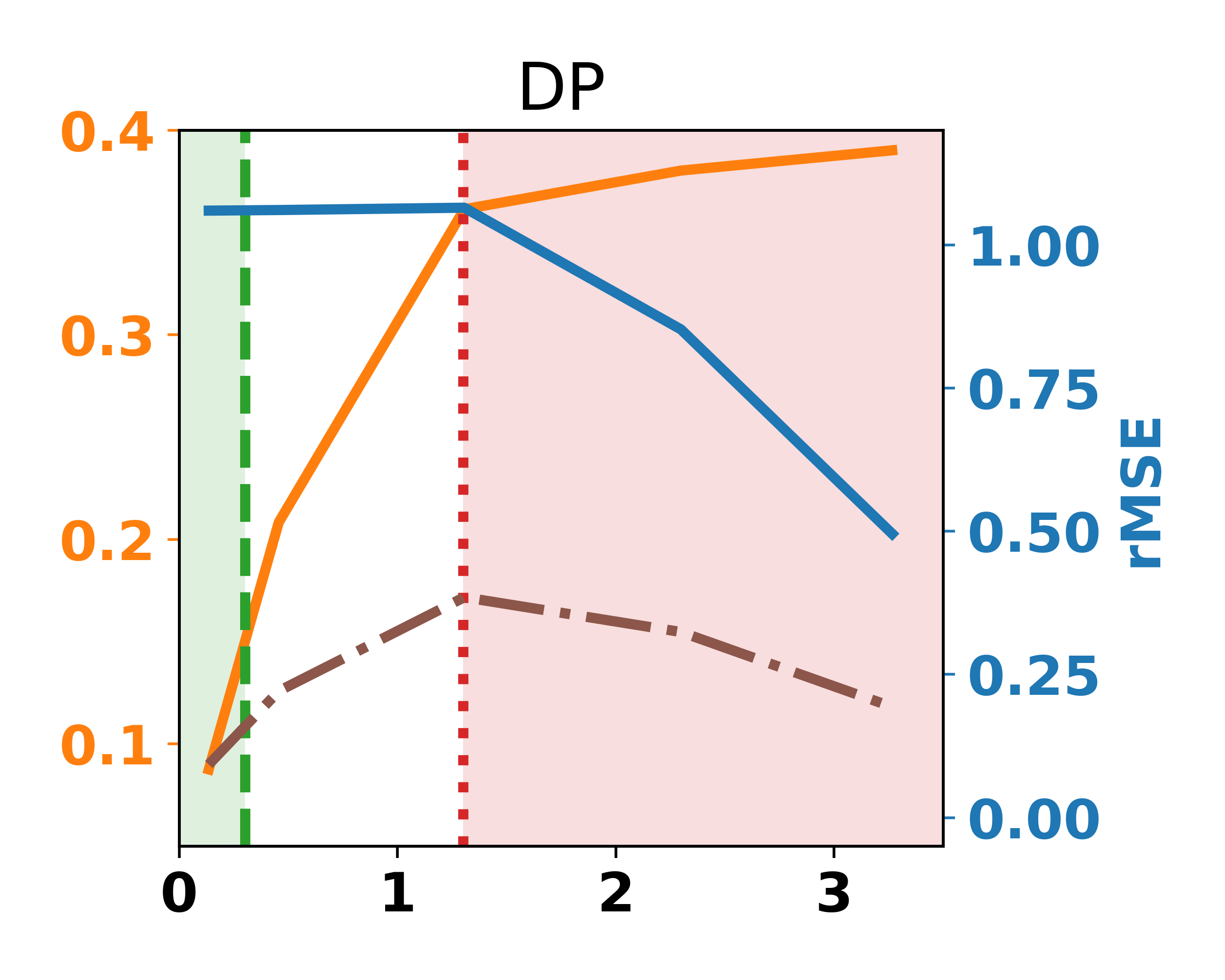}
			\includegraphics[width=0.24\linewidth]{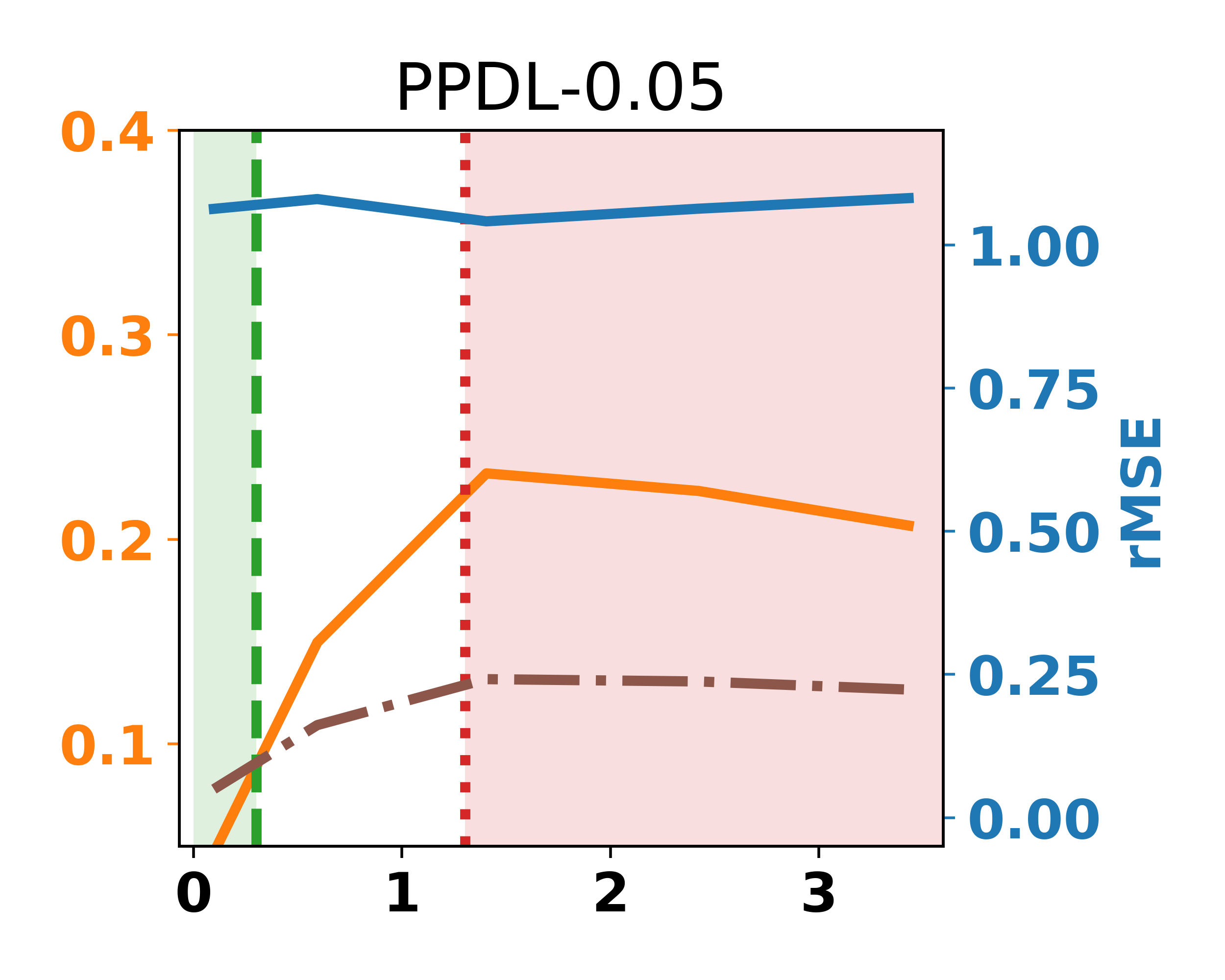}
			\includegraphics[width=0.24\linewidth]{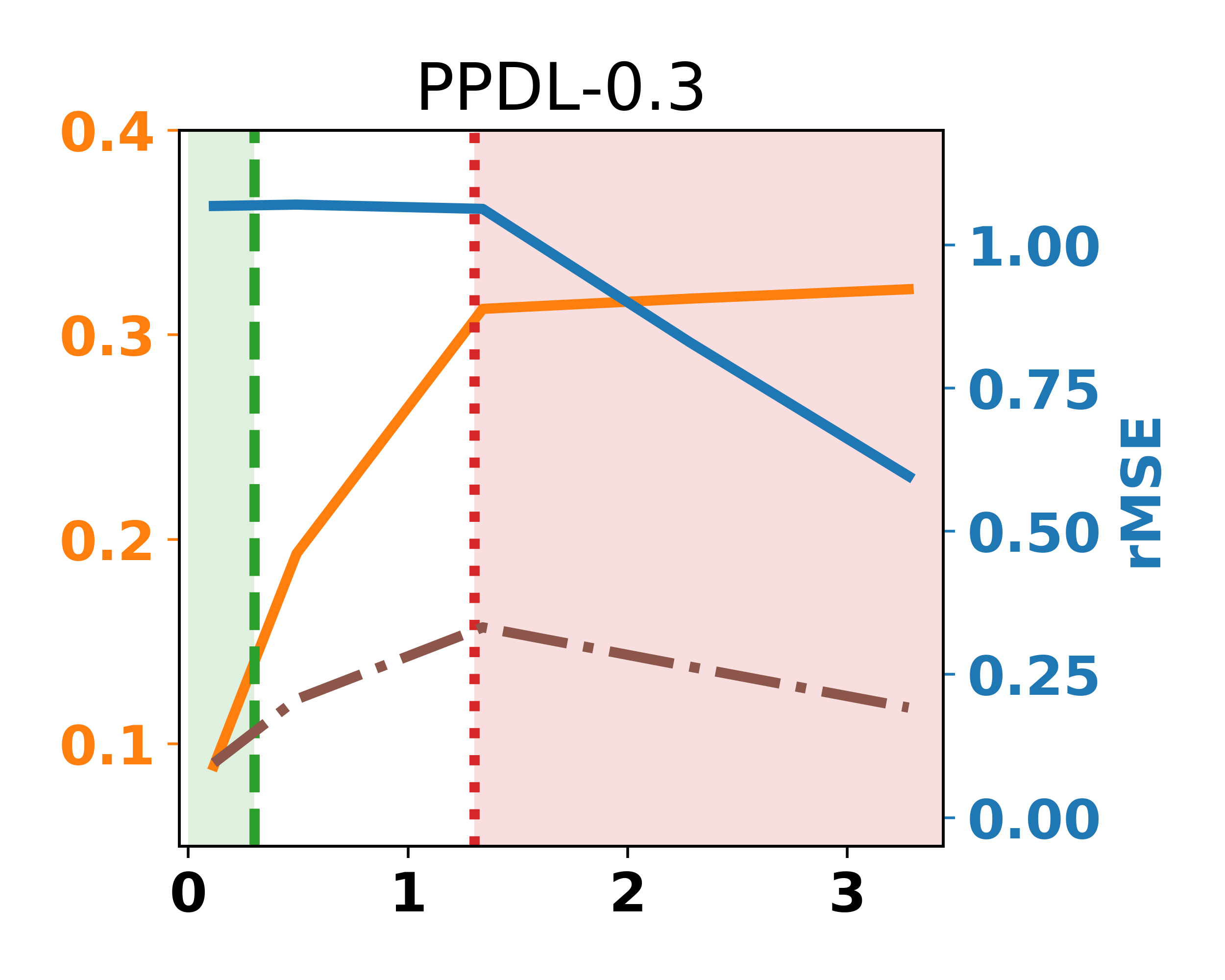}
			\includegraphics[width=0.24\linewidth]{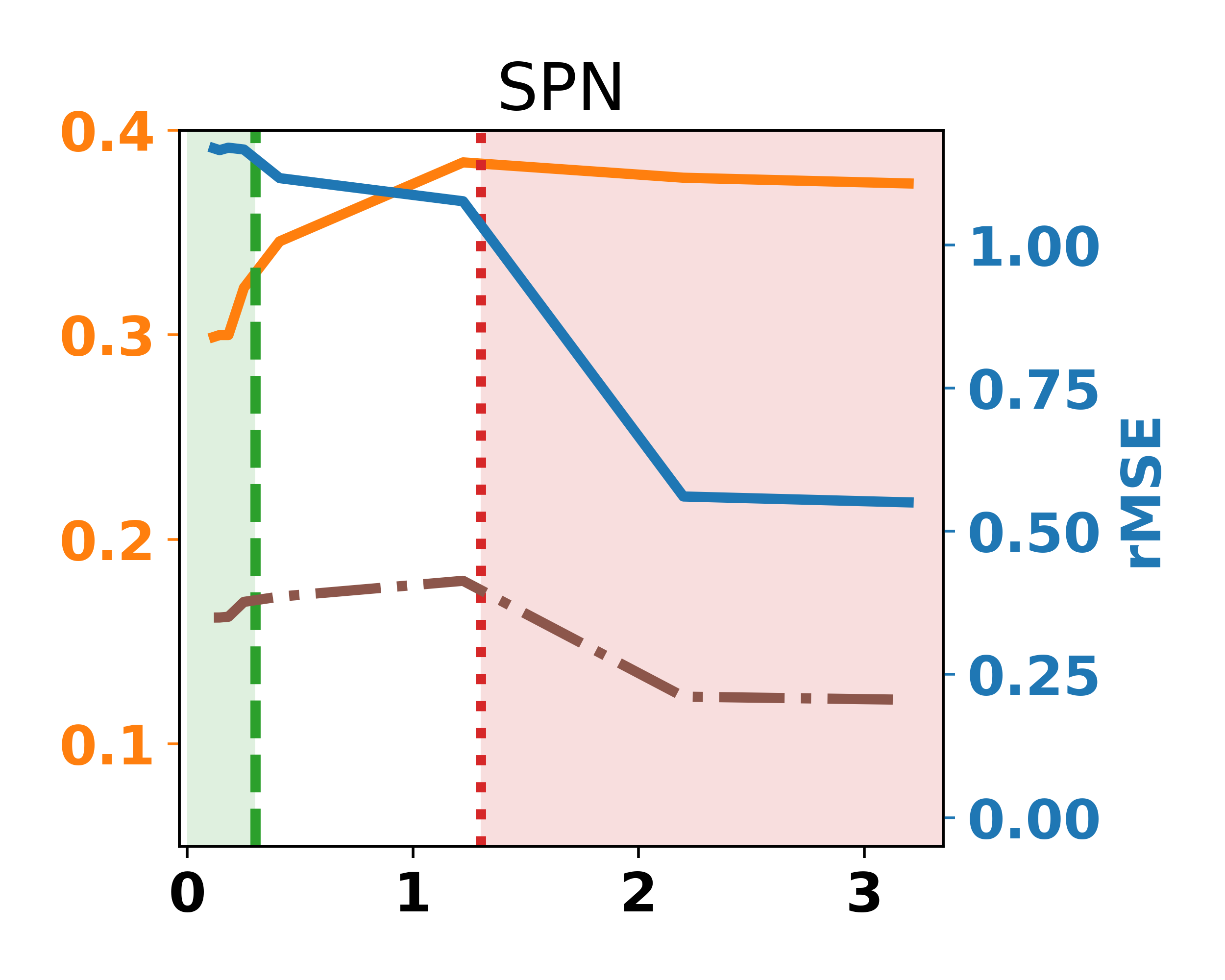}
			\caption{Reconstruction Attack}
		\end{subfigure}
		
		\begin{subfigure}{0.99\linewidth}
			\centering
			\includegraphics[width=0.24\linewidth]{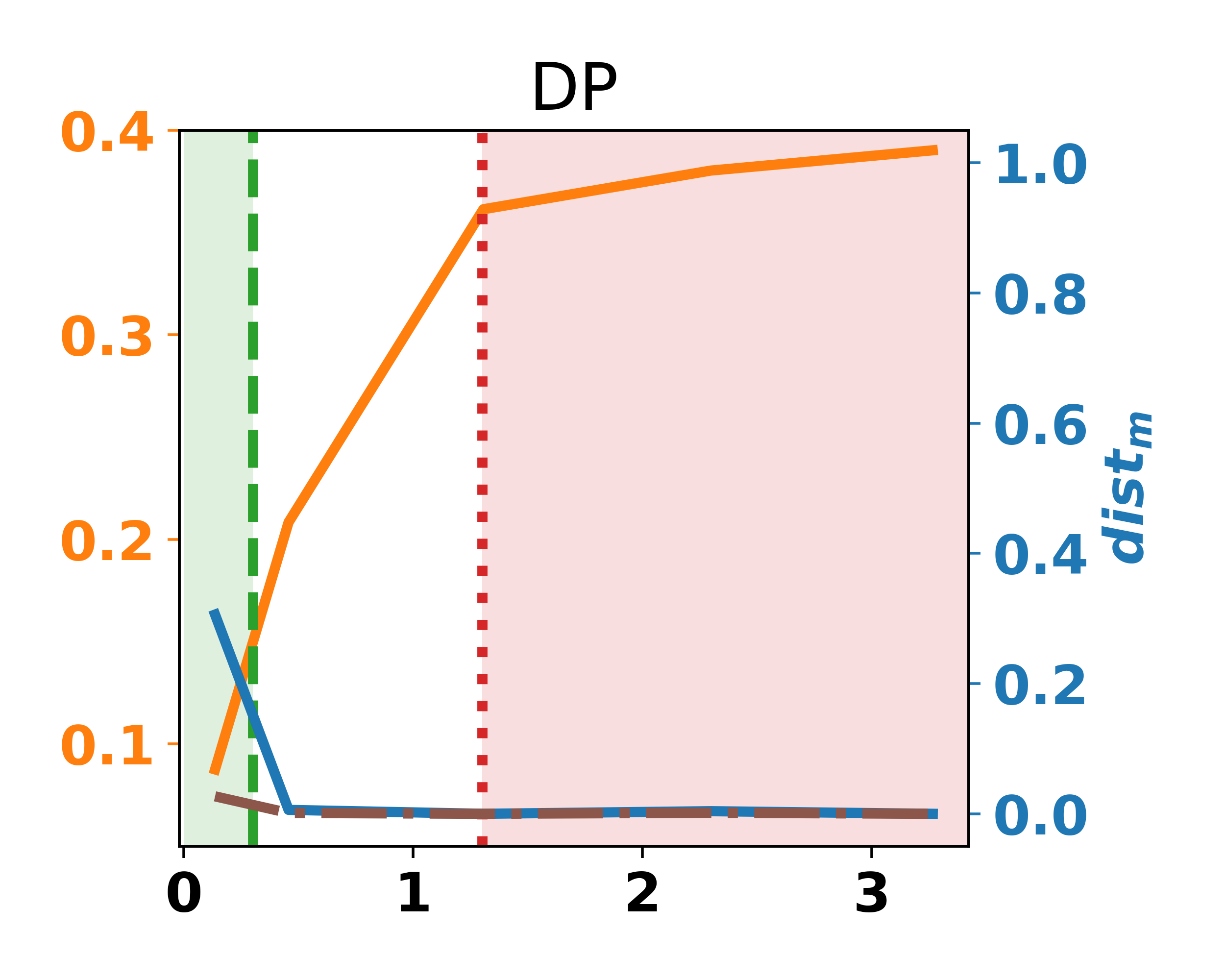}
			\includegraphics[width=0.24\linewidth]{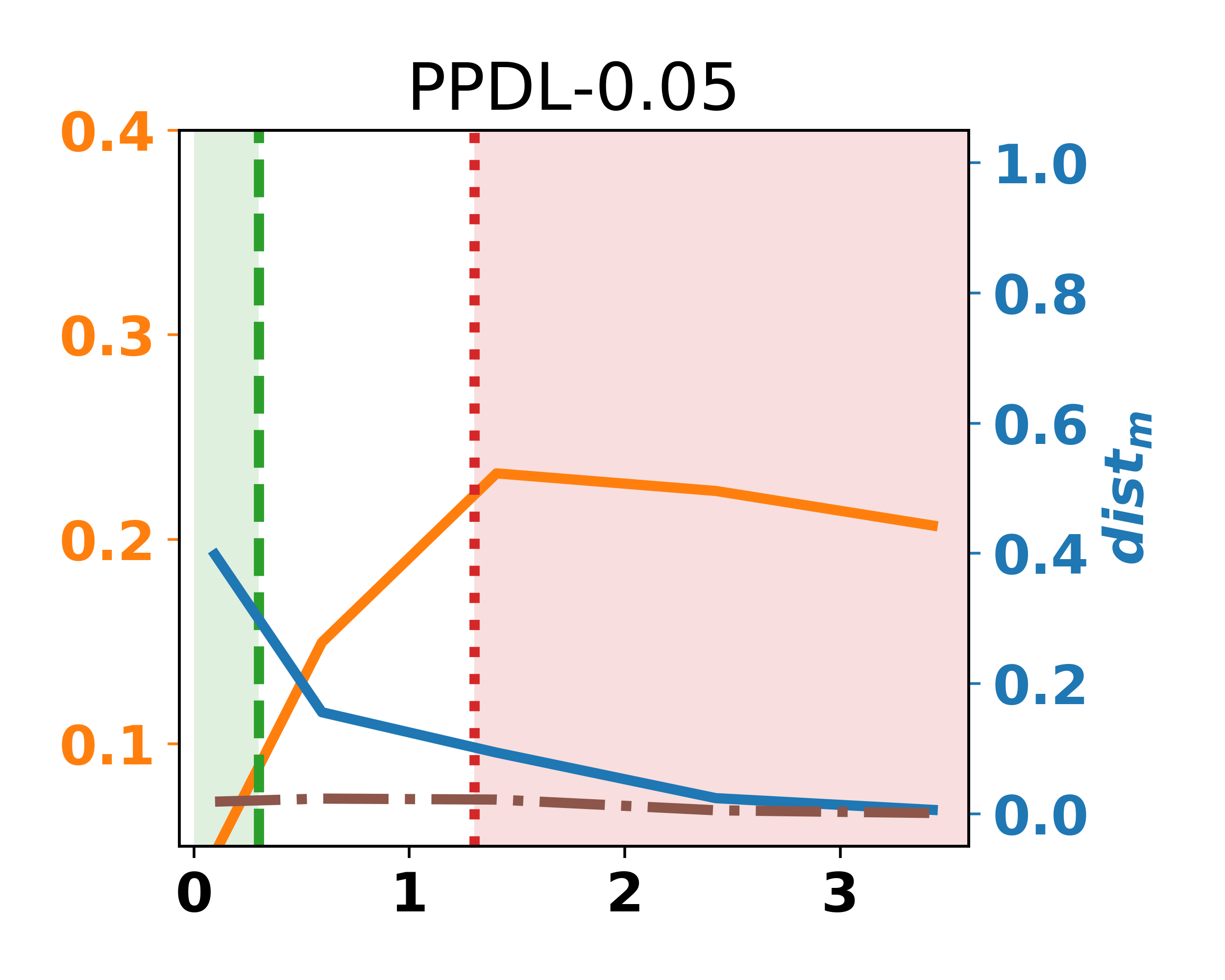}
			\includegraphics[width=0.24\linewidth]{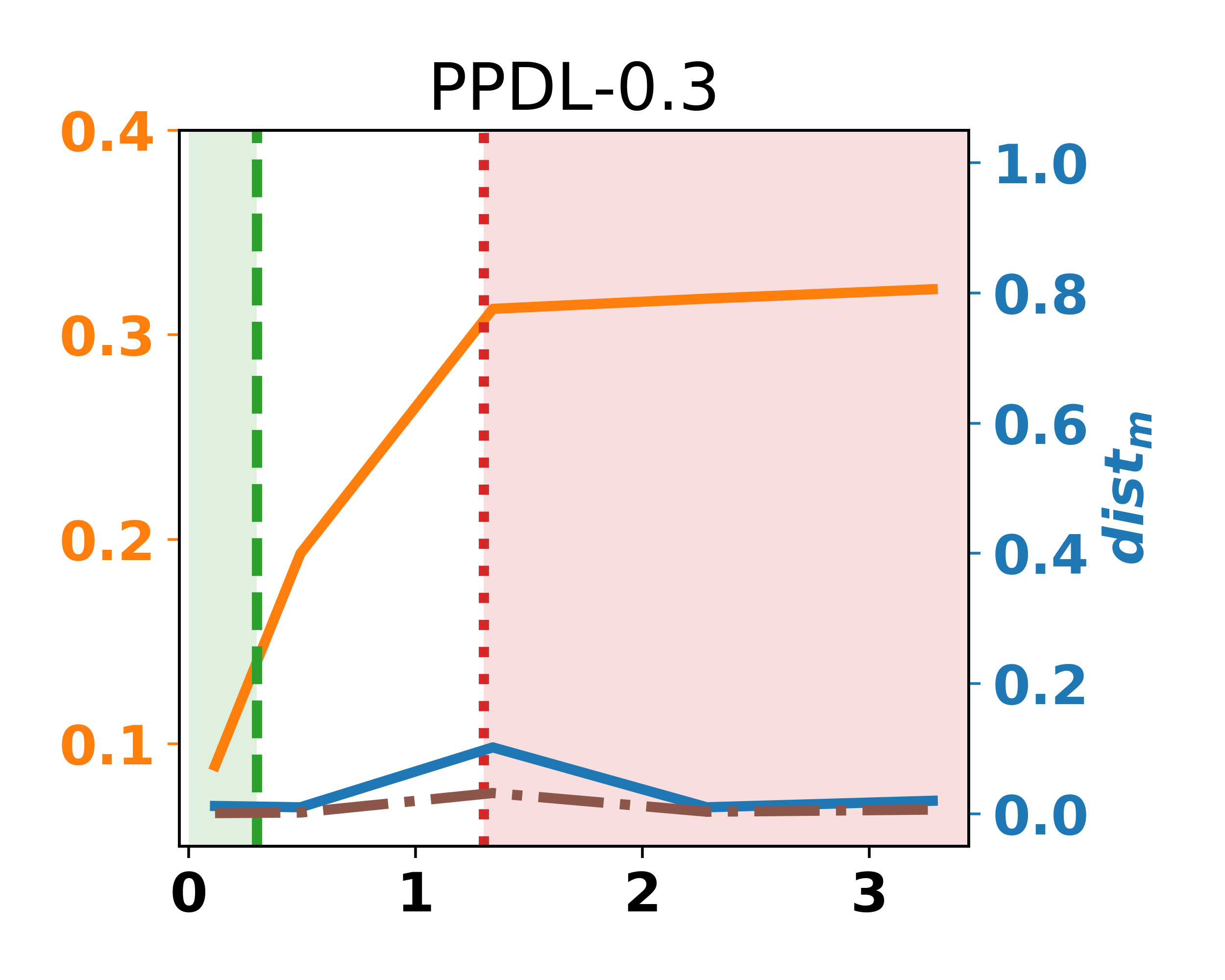}
			\includegraphics[width=0.24\linewidth]{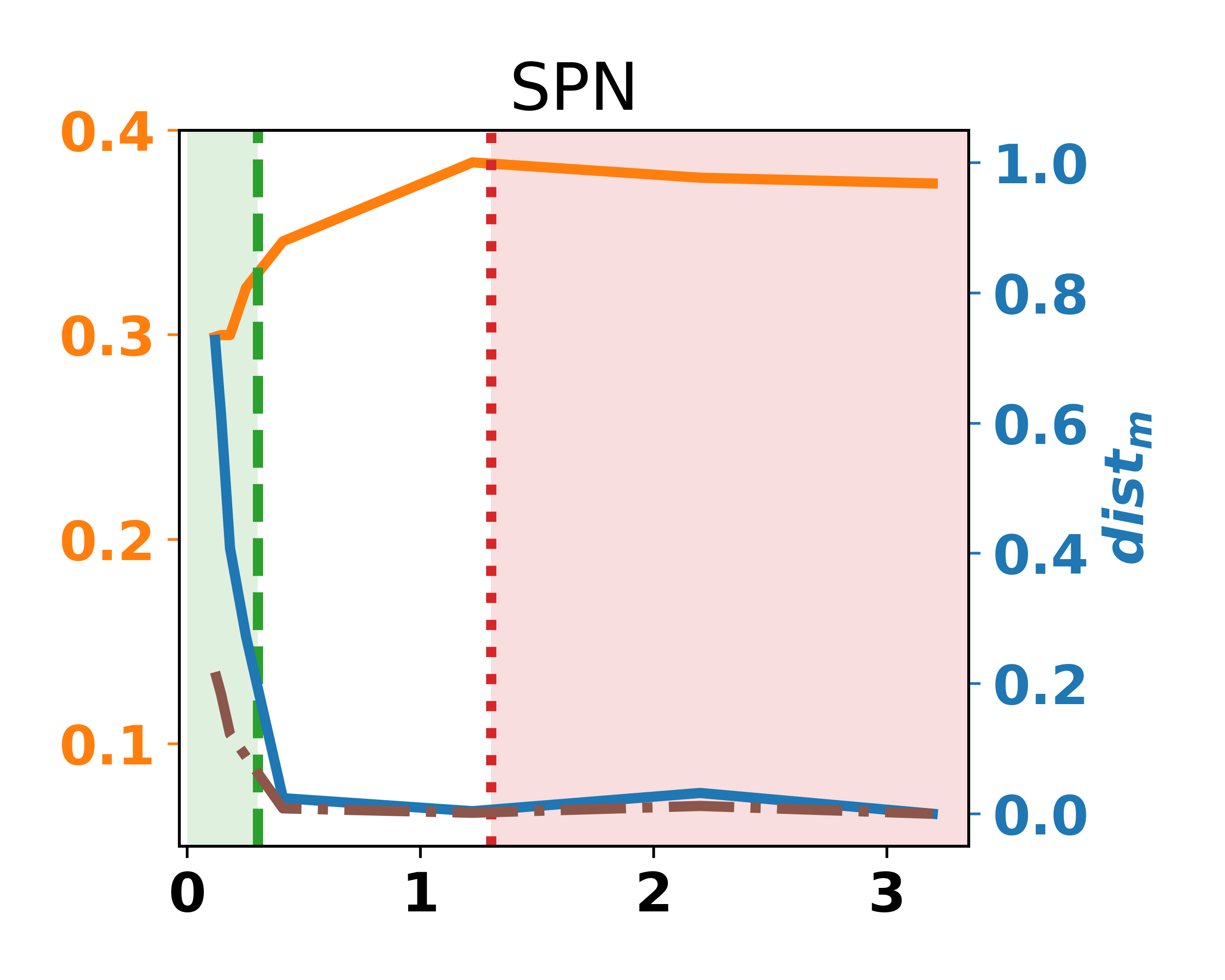}
			\caption{Membership Attack}
		\end{subfigure}
		
		\begin{subfigure}{0.99\linewidth}
			\centering
			\includegraphics[width=0.24\linewidth]{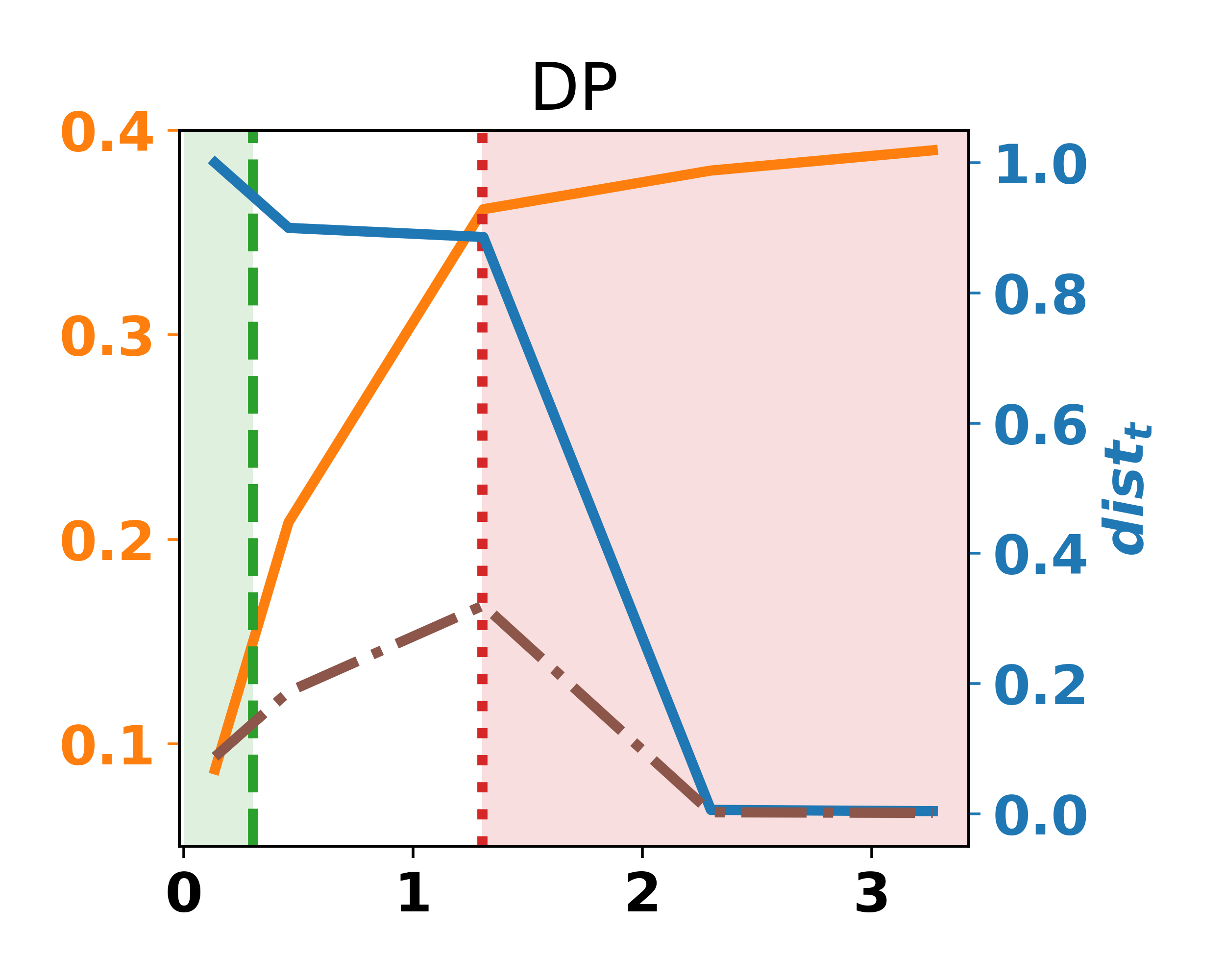}
			\includegraphics[width=0.24\linewidth]{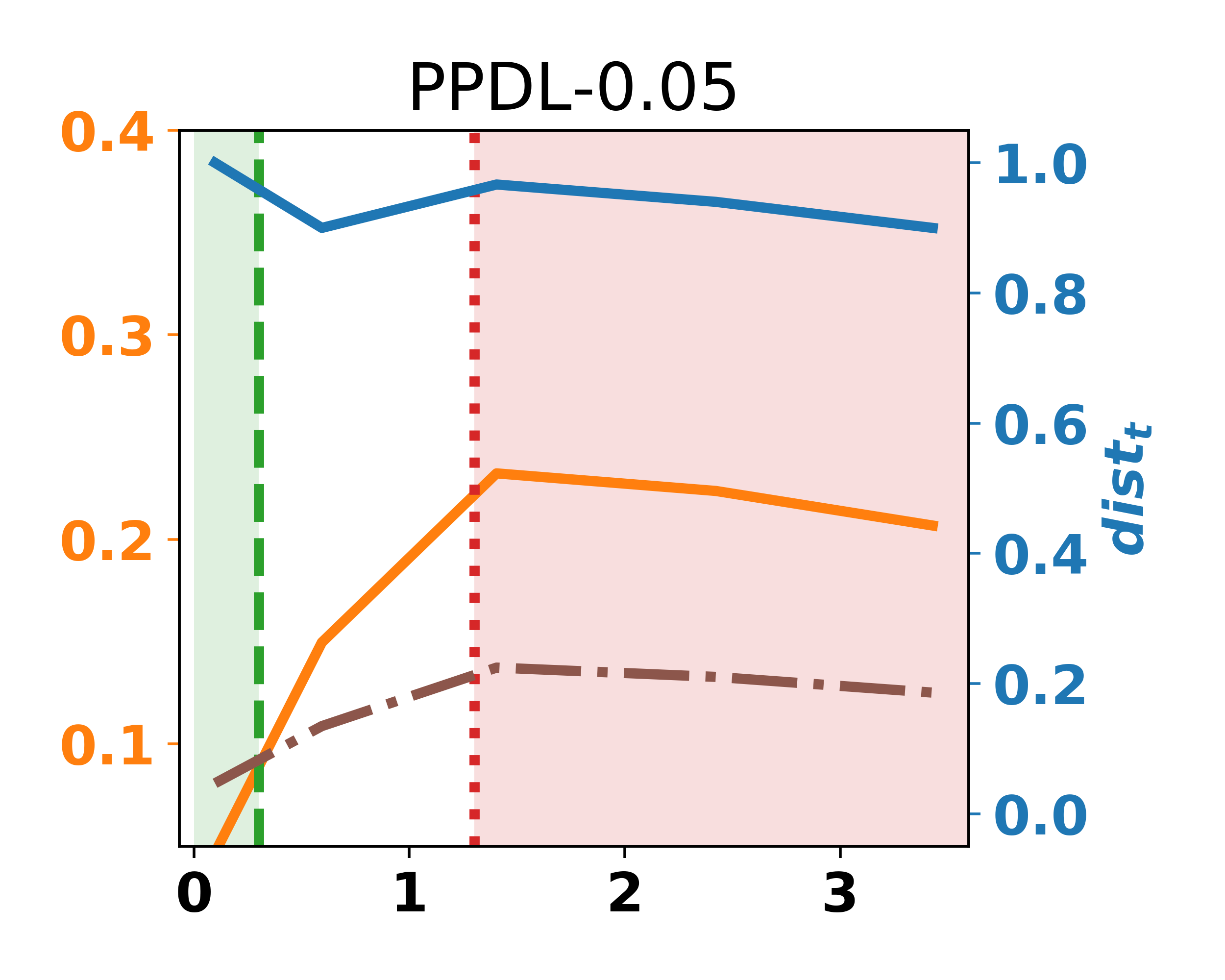}
			\includegraphics[width=0.24\linewidth]{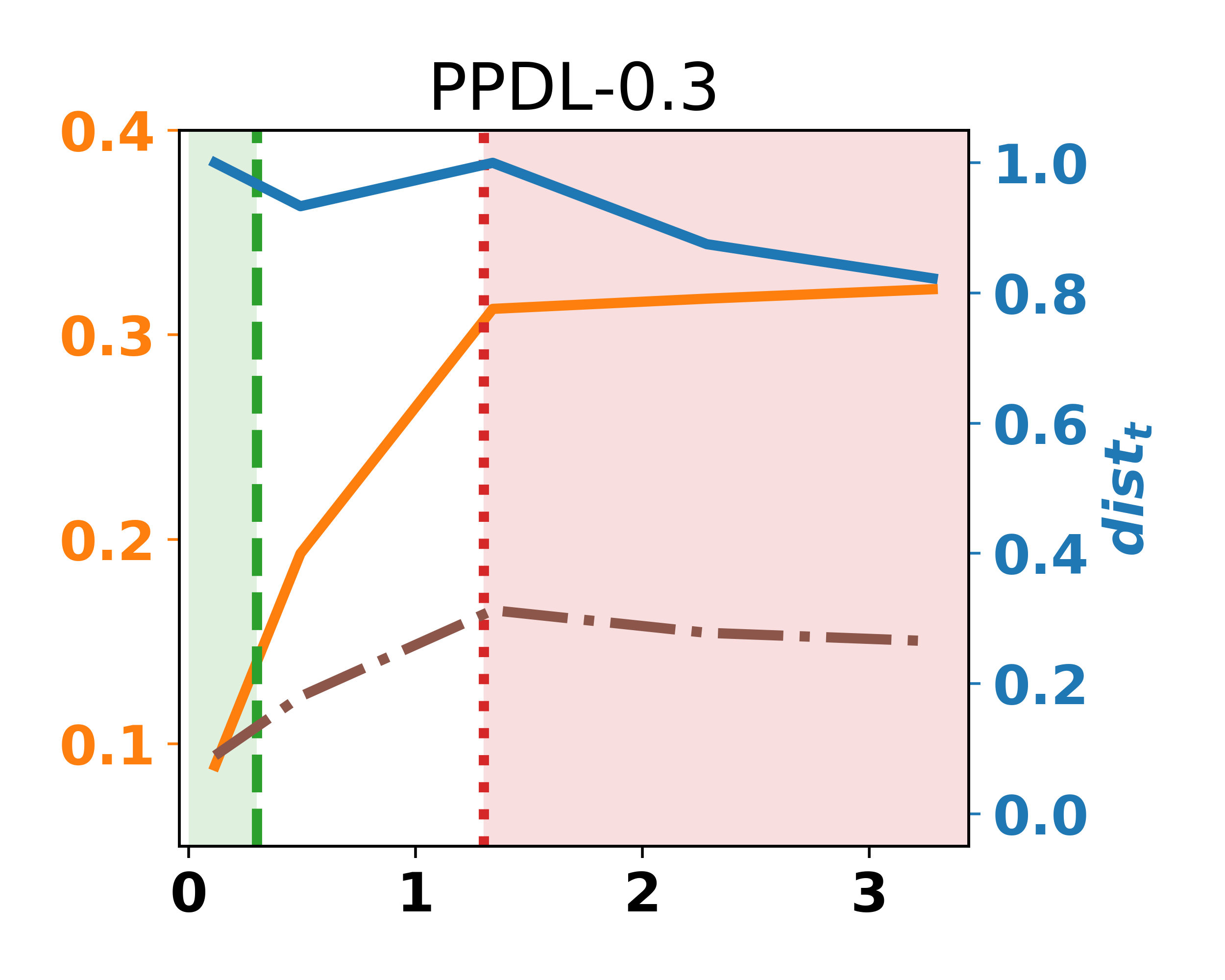}
			\includegraphics[width=0.24\linewidth]{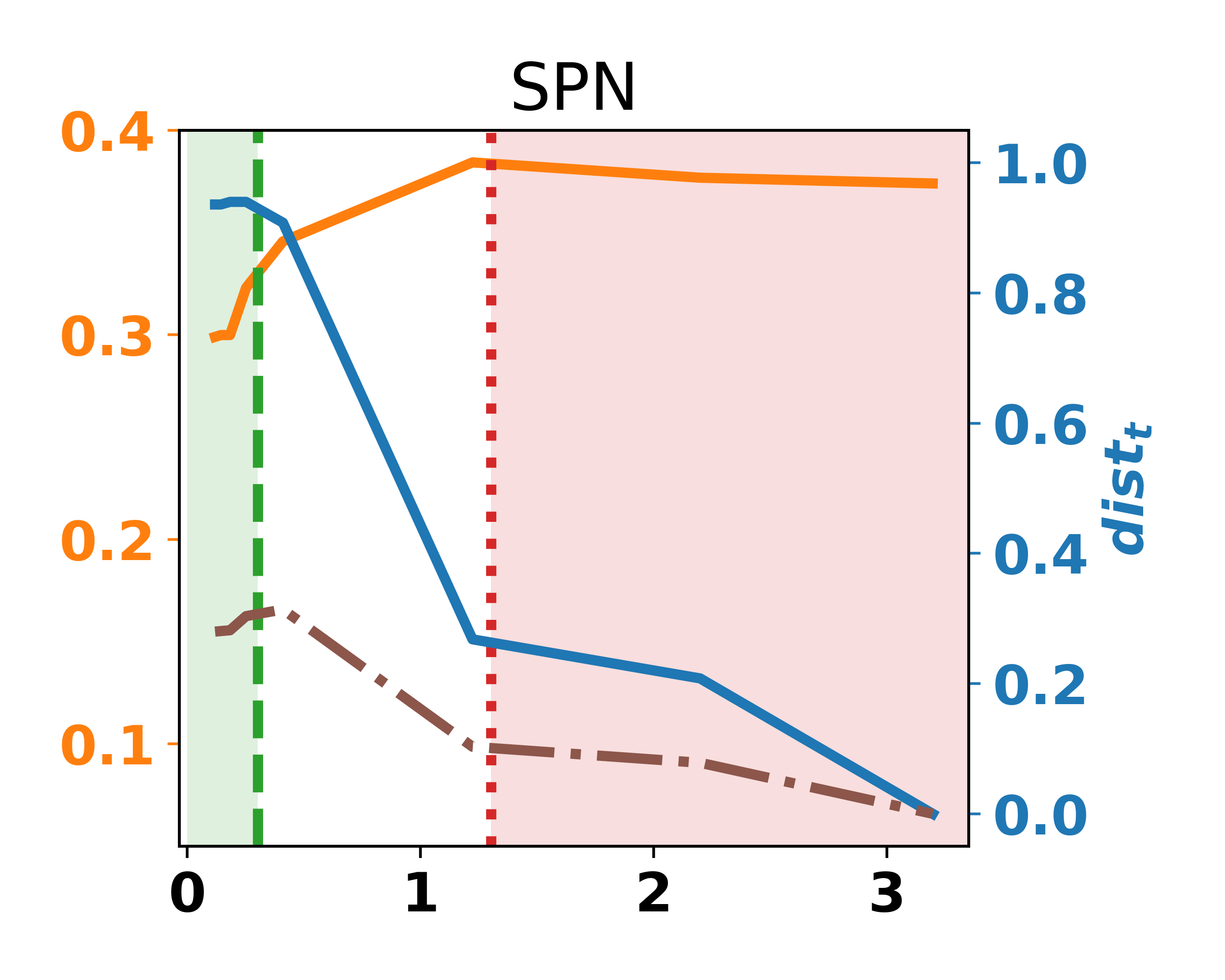}
			\caption{Tracing Attack}
		\end{subfigure}
		
		\caption{Attack Batch Size 1}
		\label{fig:ppc-cifar100-bs1}
		%\vspace{-0.22cm}
	\end{figure}

	\newpage
	
	\begin{figure}[H]	
		%	\begin{subfigure}{0.95\linewidth}
		\centering
		\begin{subfigure}{0.99\linewidth}
			\centering
			\includegraphics[scale=0.7]{imgs/legends/legend_ppc_horizontal.png}
			\\
			\includegraphics[width=0.24\linewidth]{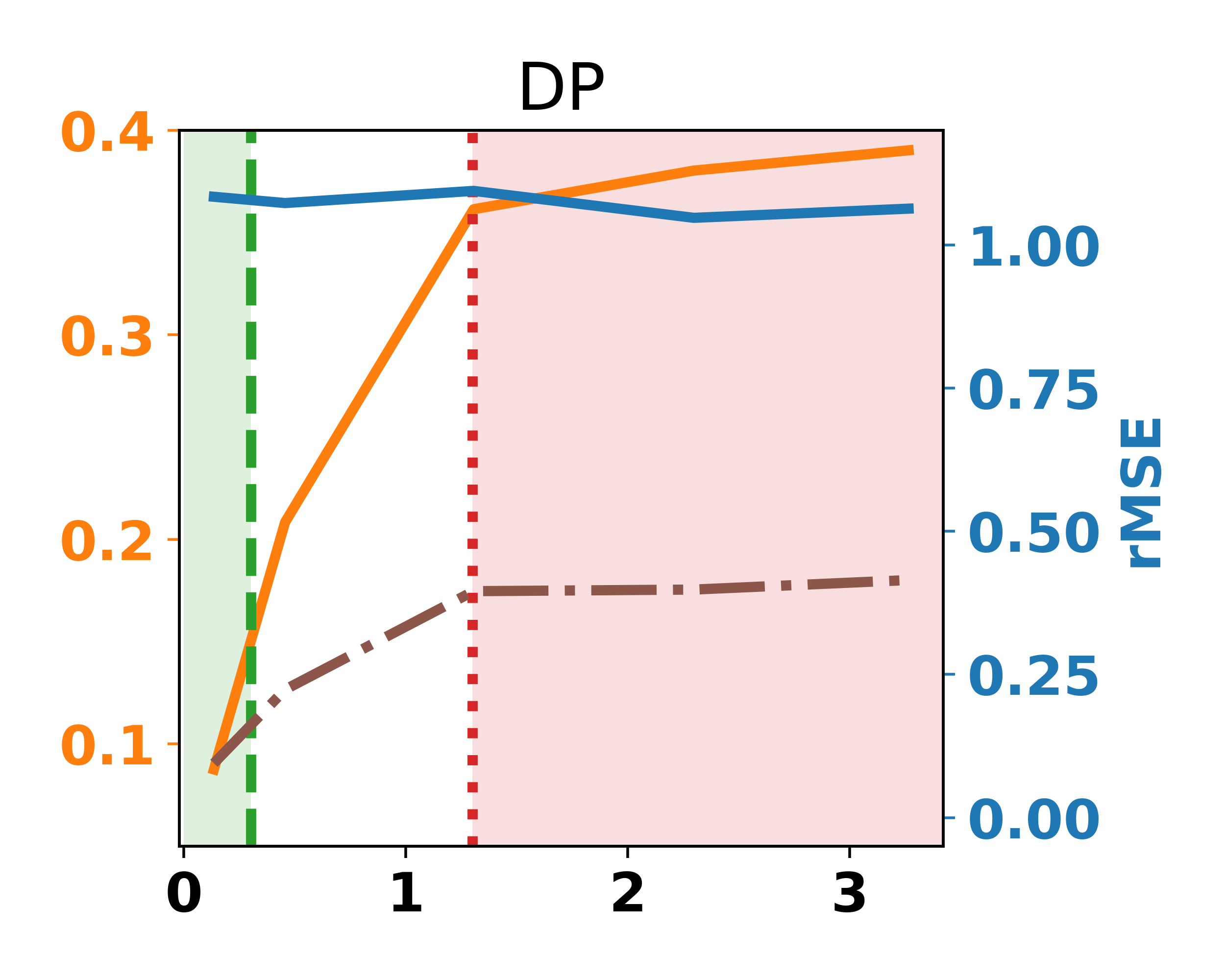}
			\includegraphics[width=0.24\linewidth]{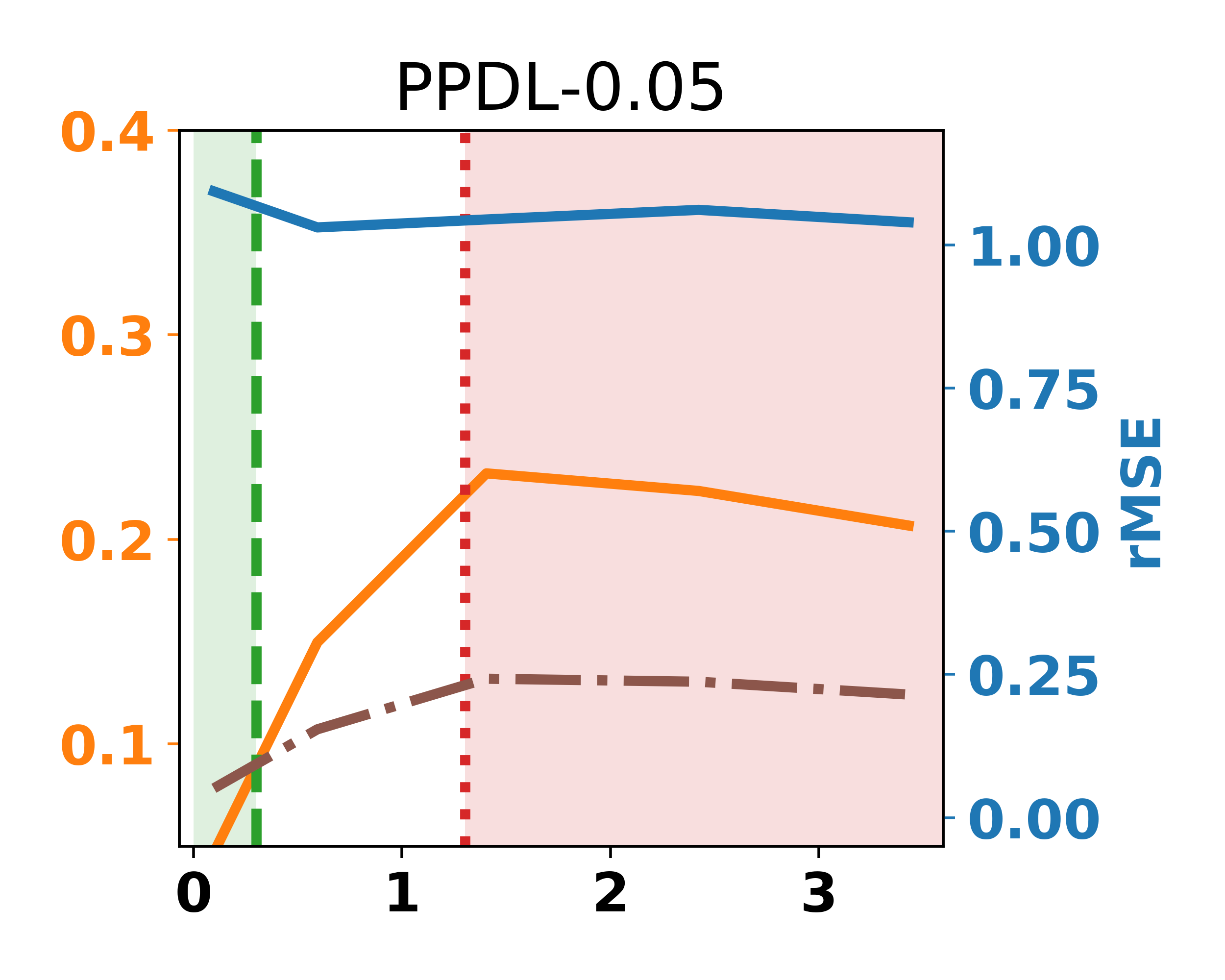}
			\includegraphics[width=0.24\linewidth]{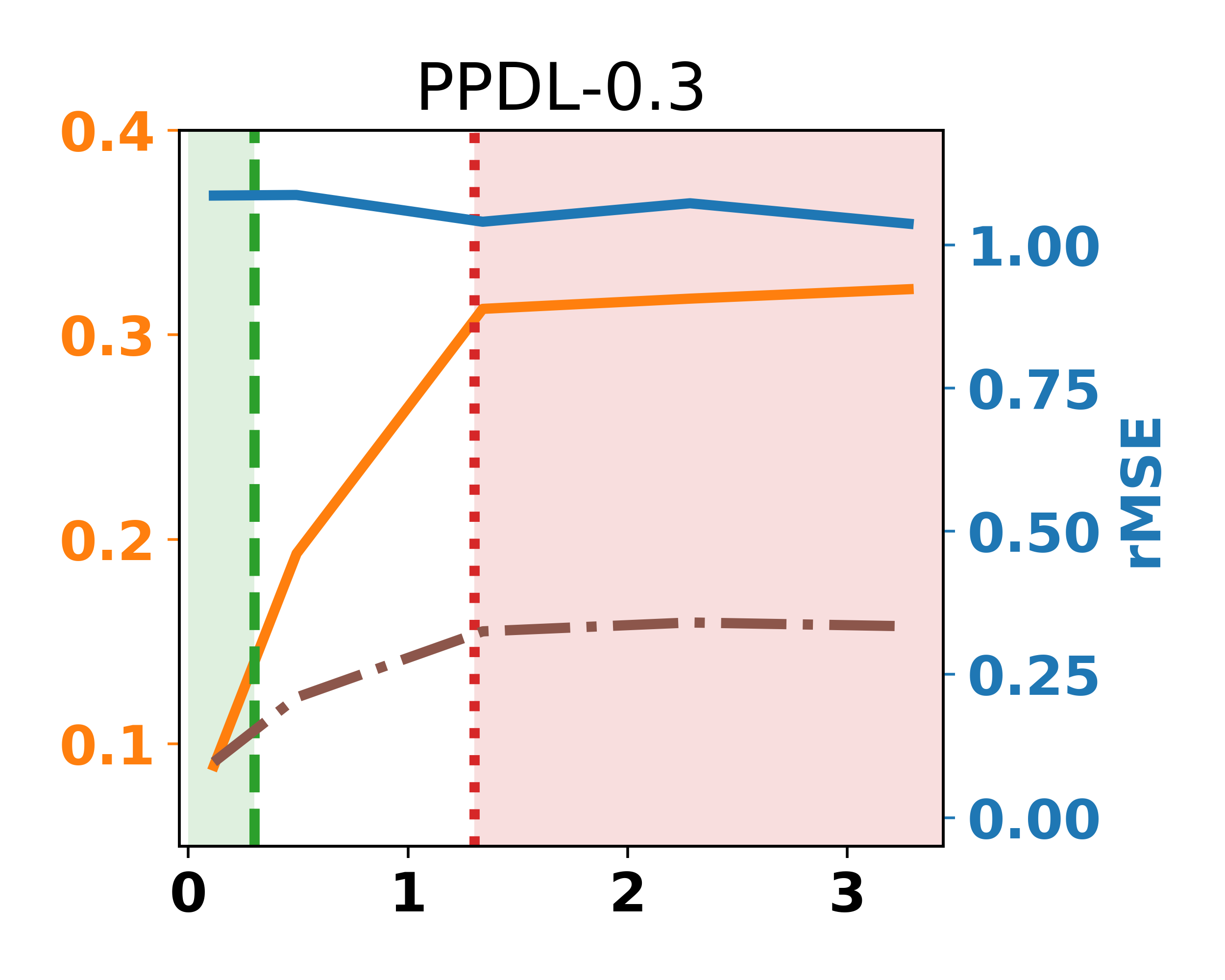}
			\includegraphics[width=0.24\linewidth]{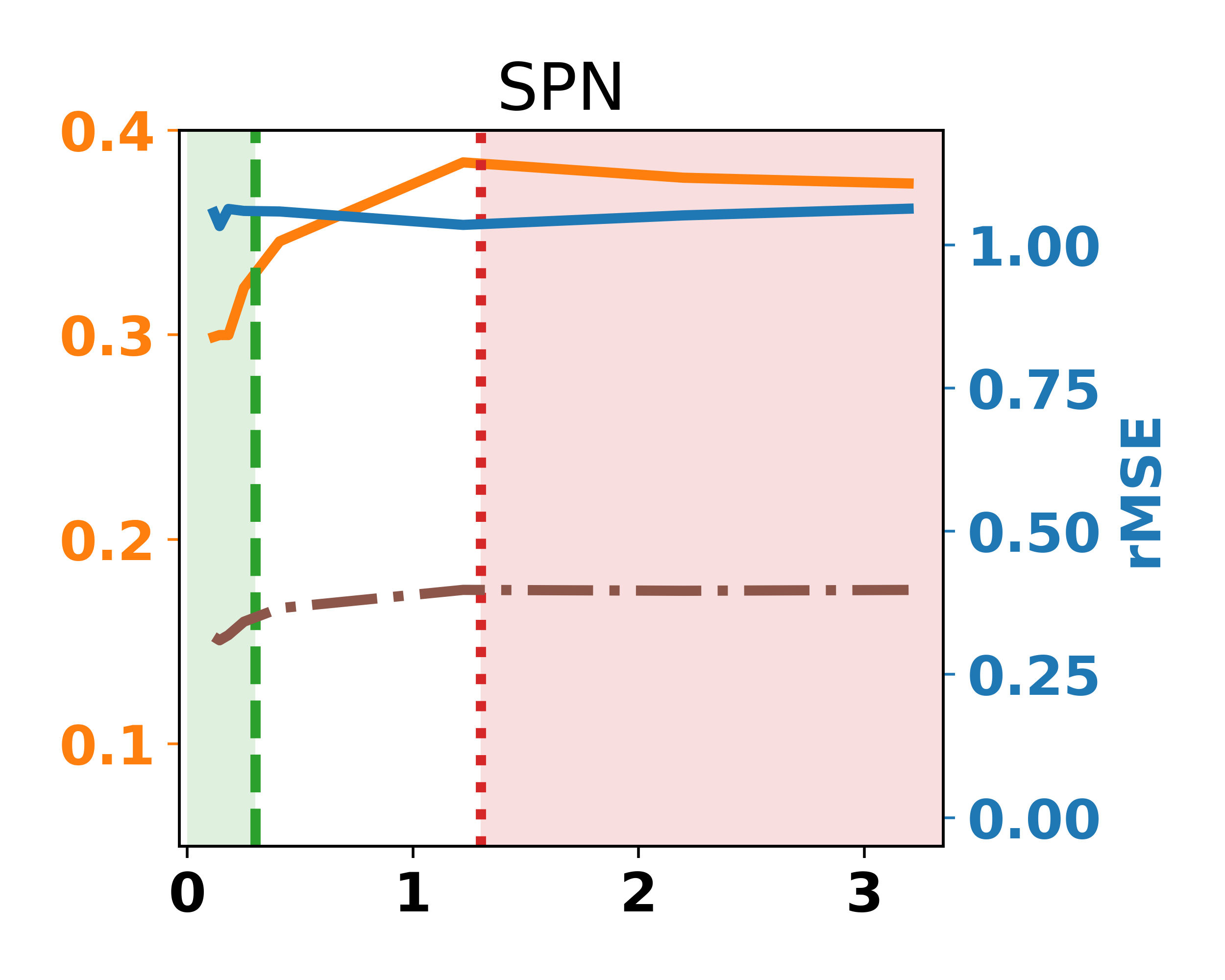}
			\caption{Reconstruction Attack}
		\end{subfigure}
		
		\begin{subfigure}{0.99\linewidth}
			\centering
			\includegraphics[width=0.24\linewidth]{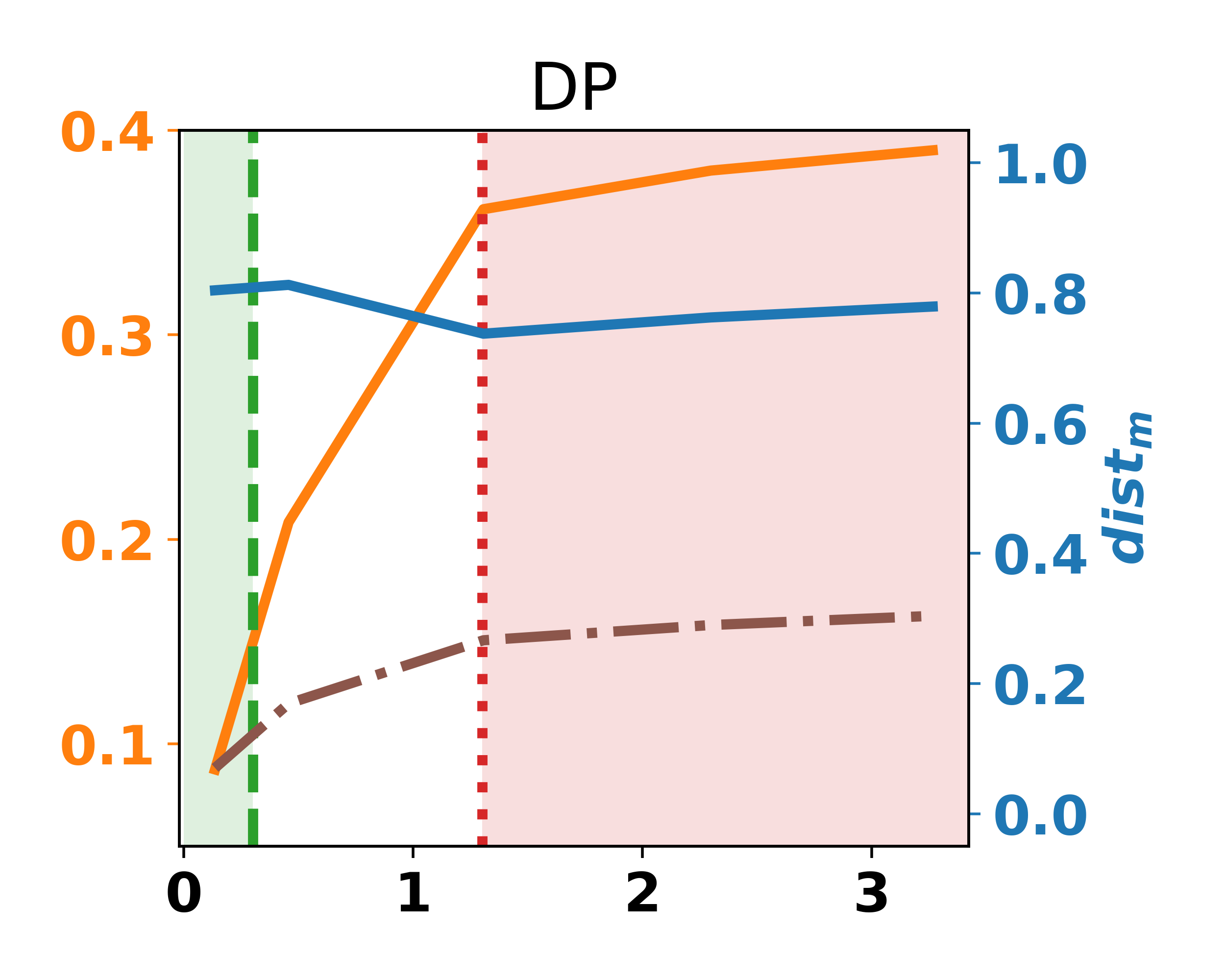}
			\includegraphics[width=0.24\linewidth]{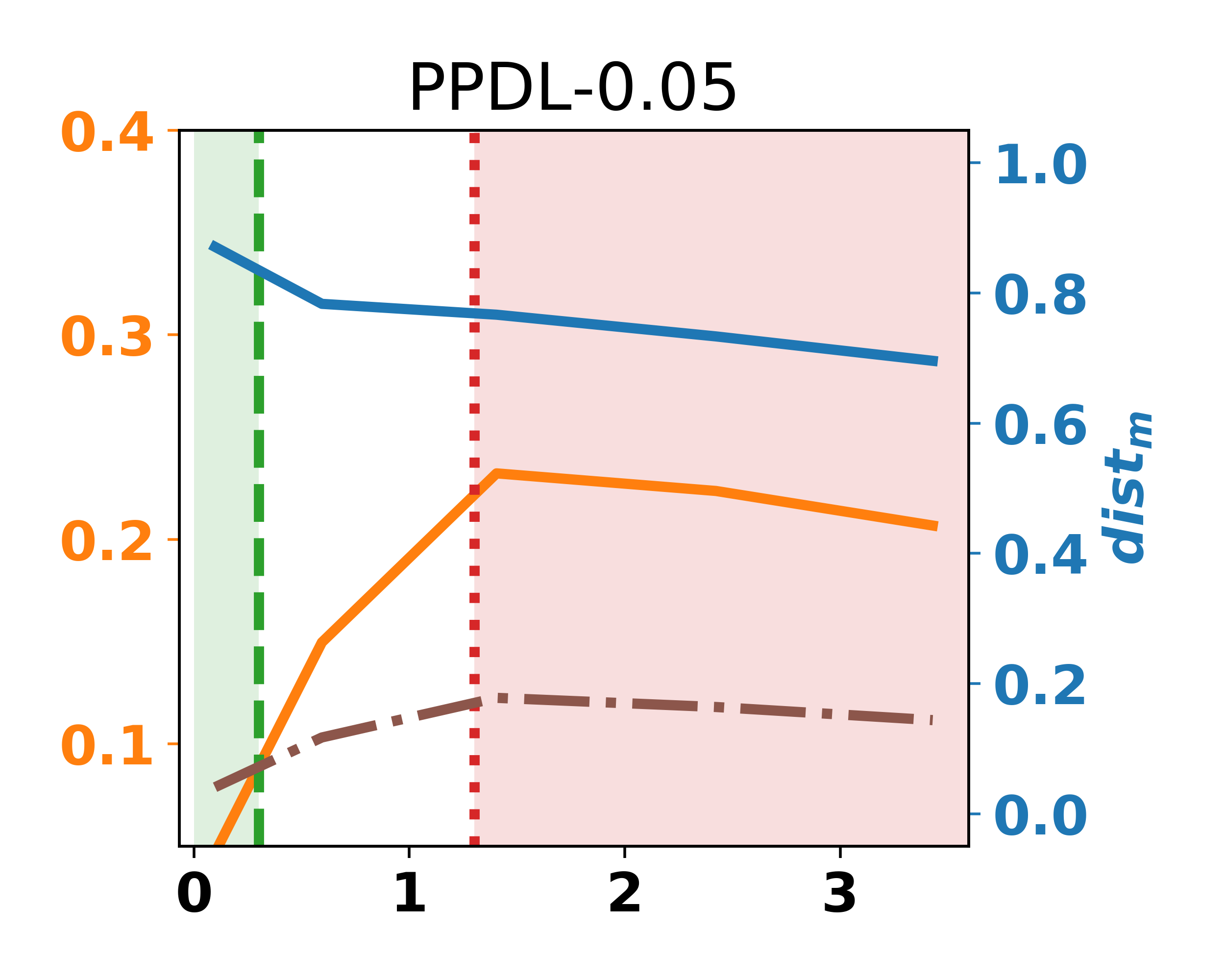}
			\includegraphics[width=0.24\linewidth]{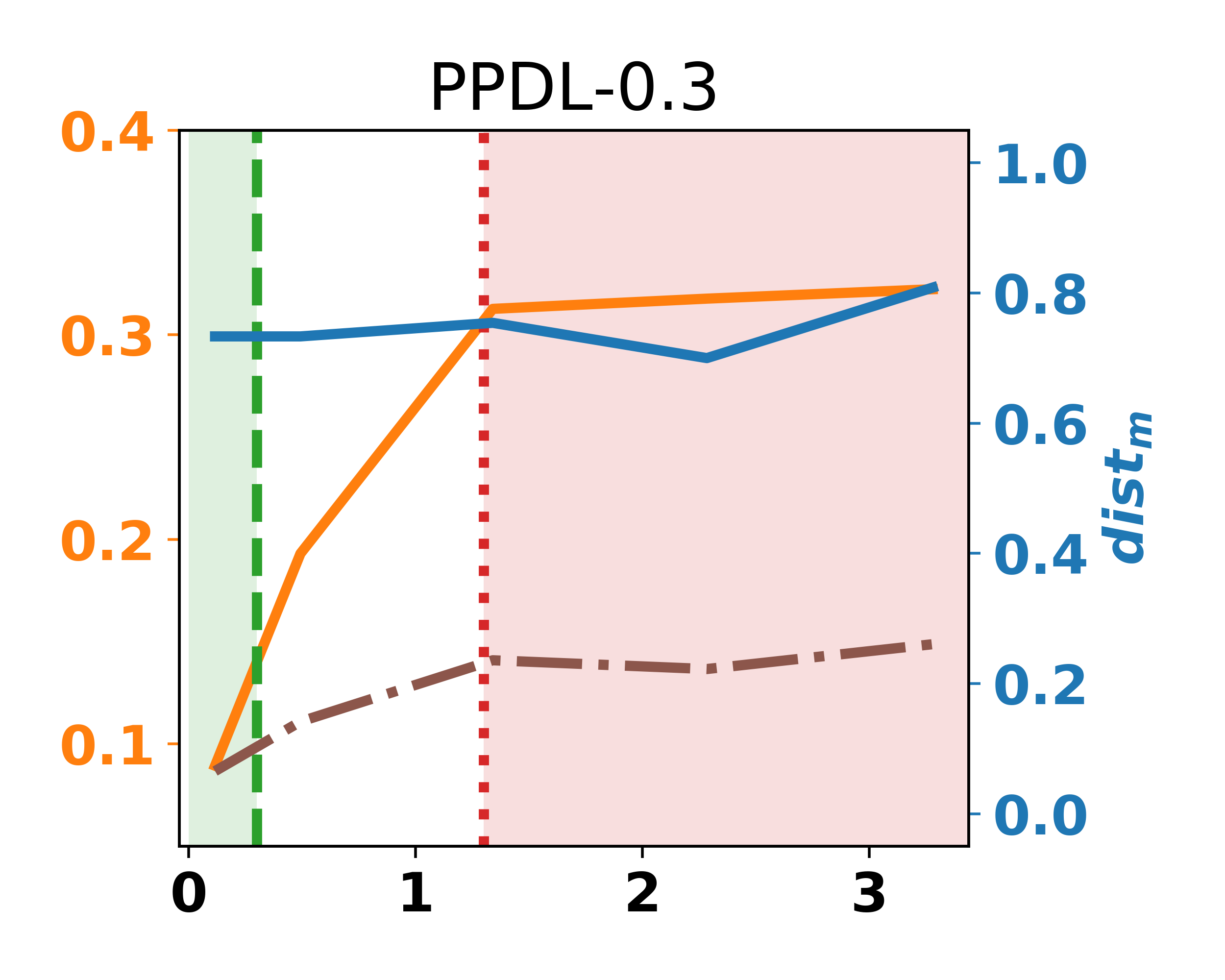}
			\includegraphics[width=0.24\linewidth]{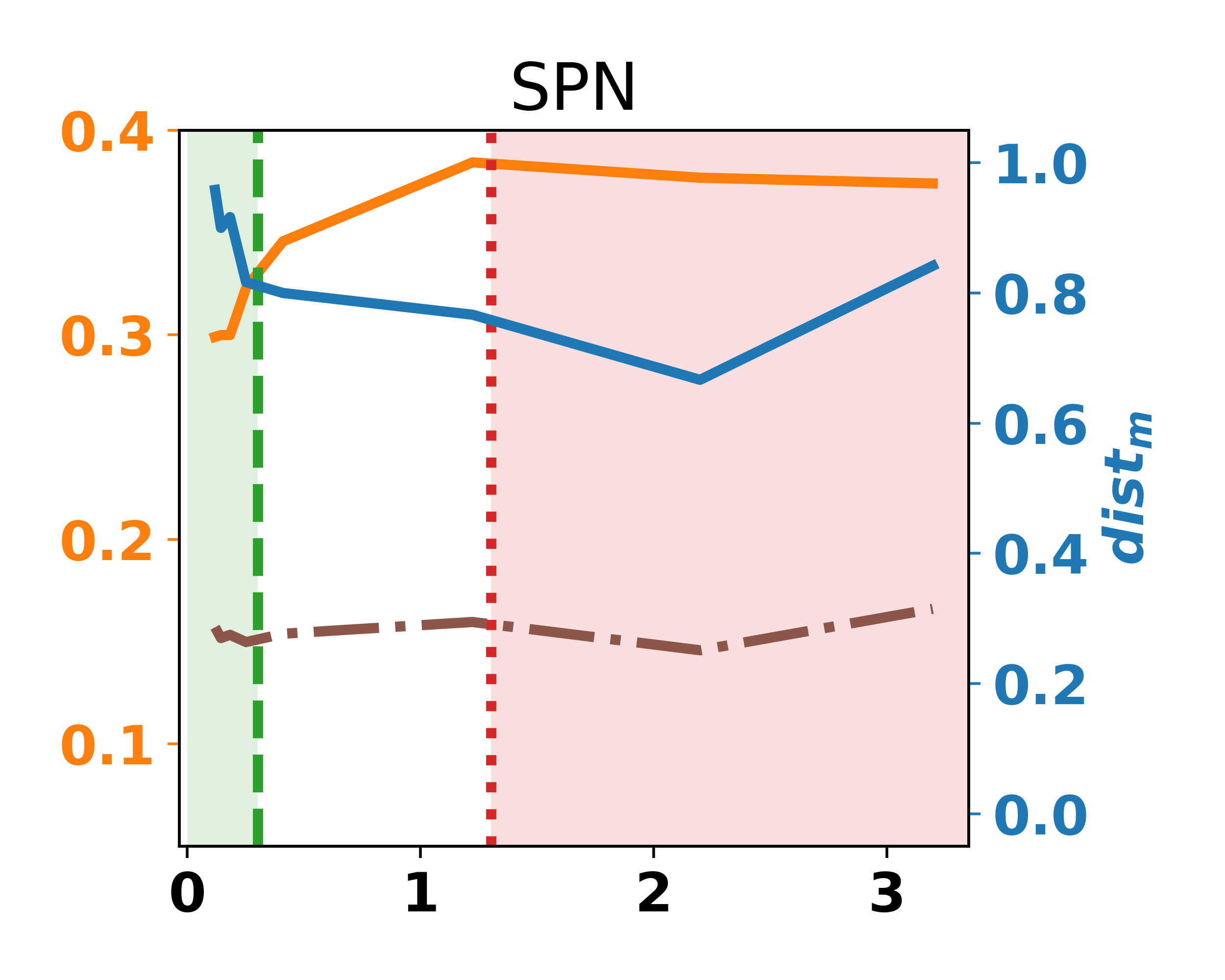}
			\caption{Membership Attack}
		\end{subfigure}
		
		\begin{subfigure}{0.99\linewidth}
			\centering
			\includegraphics[width=0.24\linewidth]{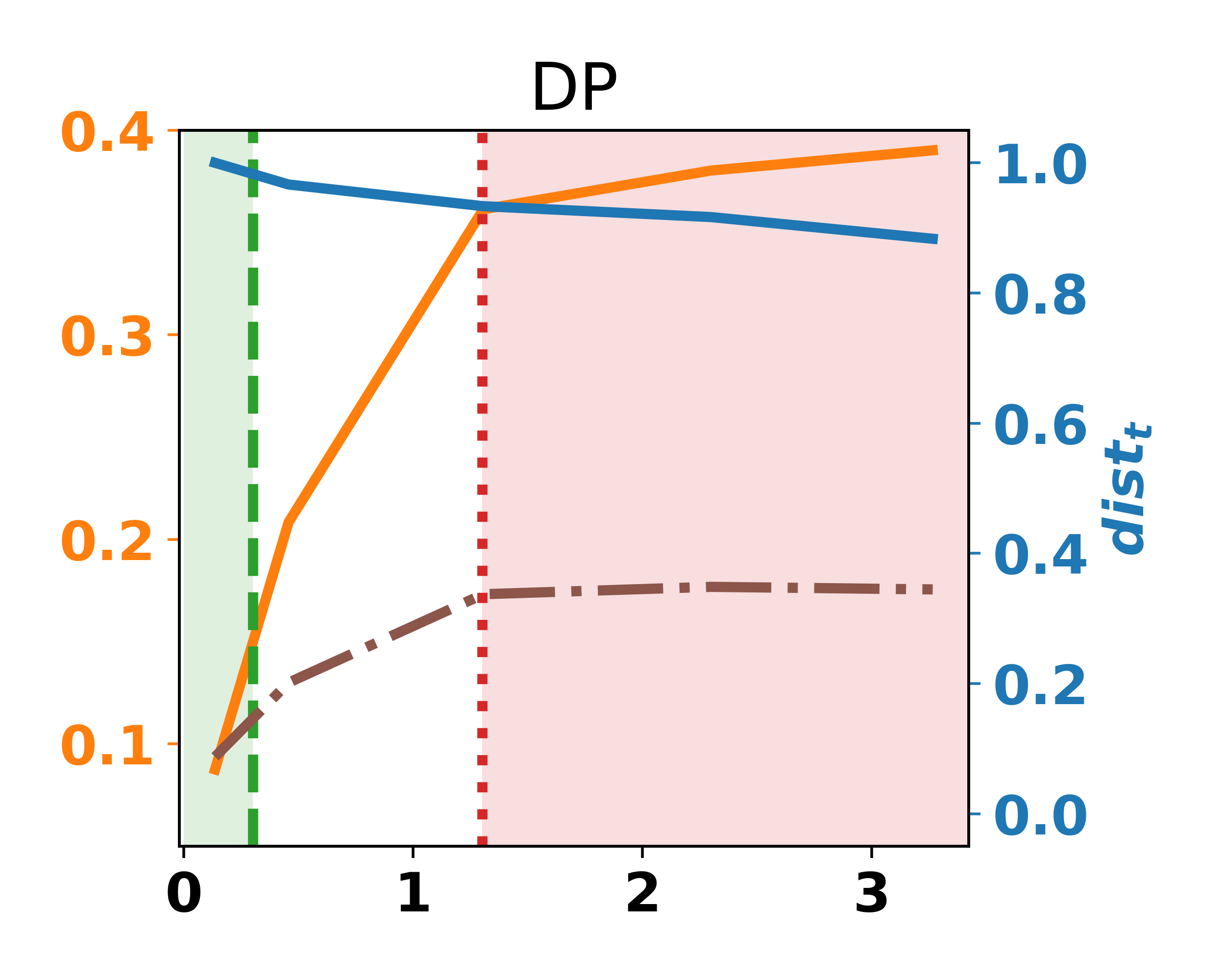}
			\includegraphics[width=0.24\linewidth]{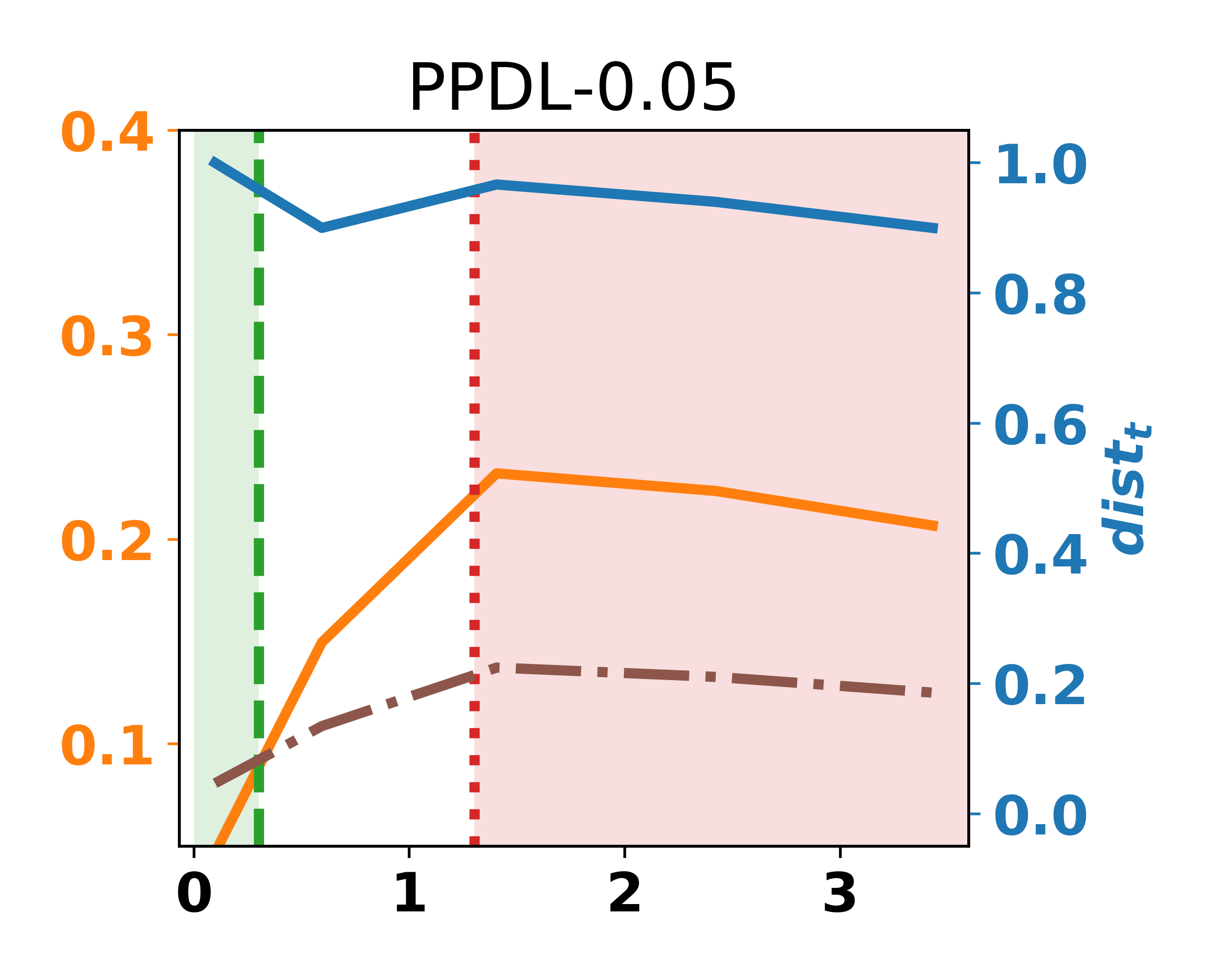}
			\includegraphics[width=0.24\linewidth]{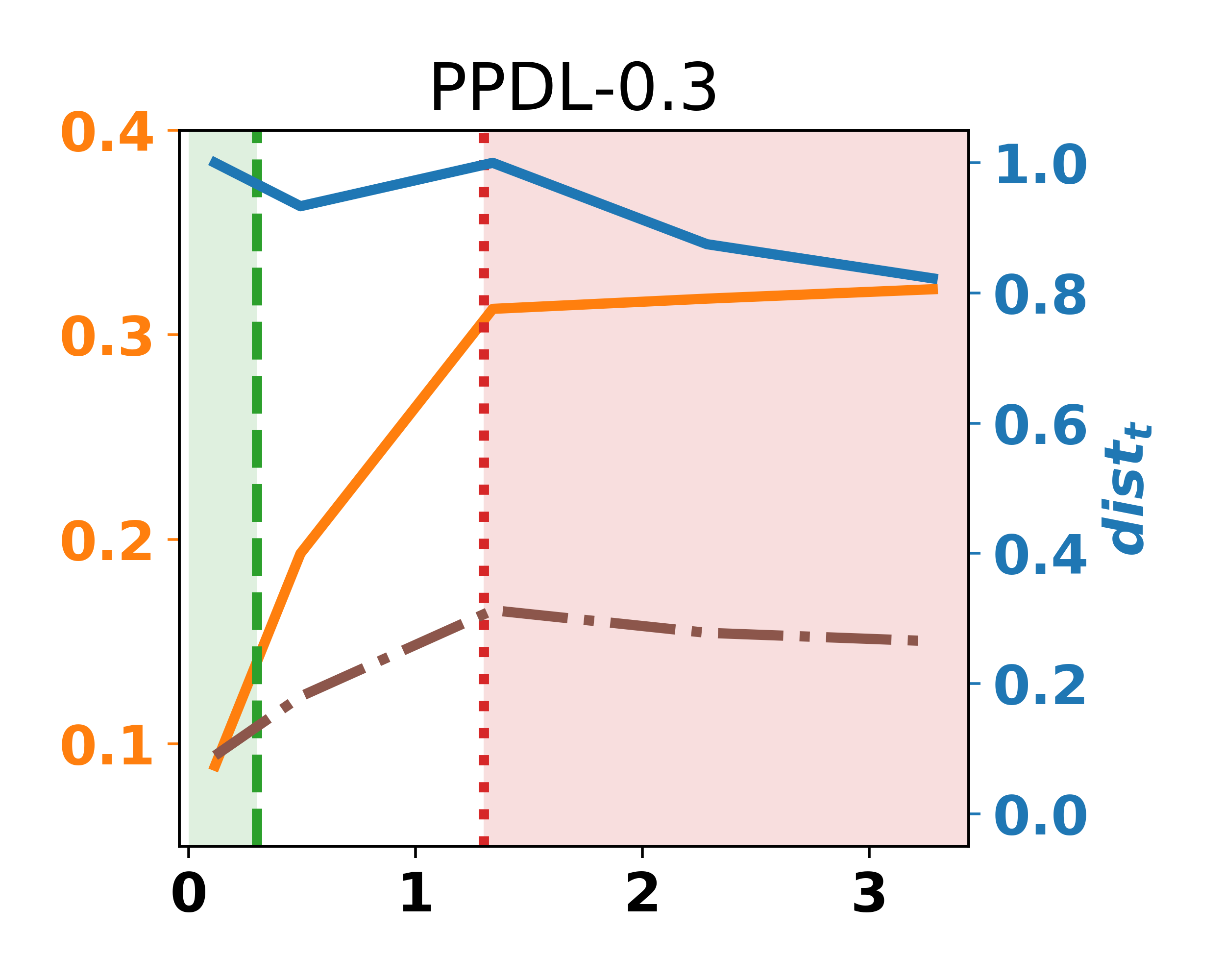}
			\includegraphics[width=0.24\linewidth]{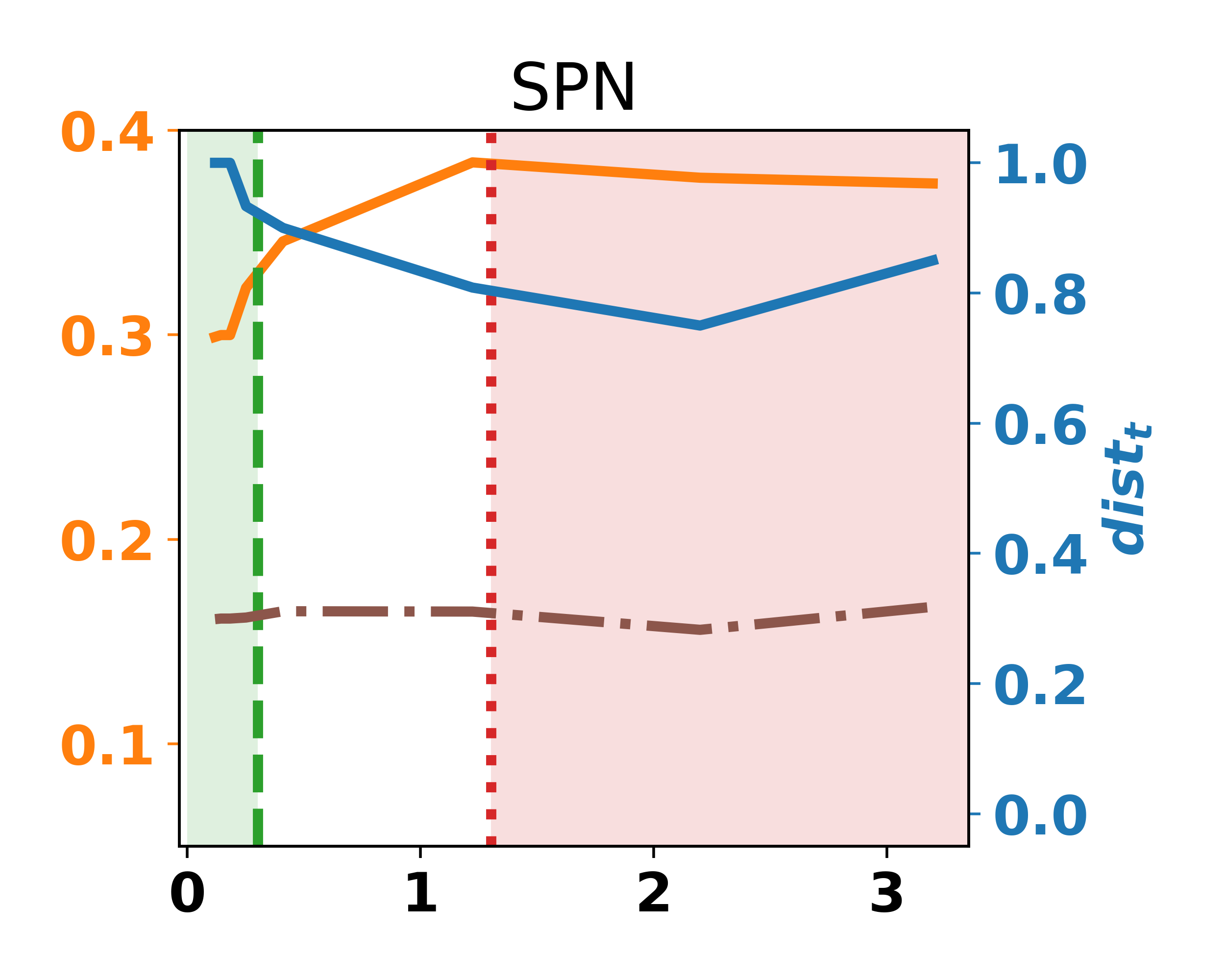}
			\caption{Tracing Attack}
		\end{subfigure}
		
		\caption{Attack with Batch Size 4}
		\label{fig:ppc-cifar100-bs4}
		%\vspace{-0.22cm}
	\end{figure}
	
	\begin{figure}[H]	
		%	\begin{subfigure}{0.95\linewidth}
		\centering
		\begin{subfigure}{0.99\linewidth}
			\centering
			\includegraphics[scale=0.7]{imgs/legends/legend_ppc_horizontal.png}
			\\
			\includegraphics[width=0.24\linewidth]{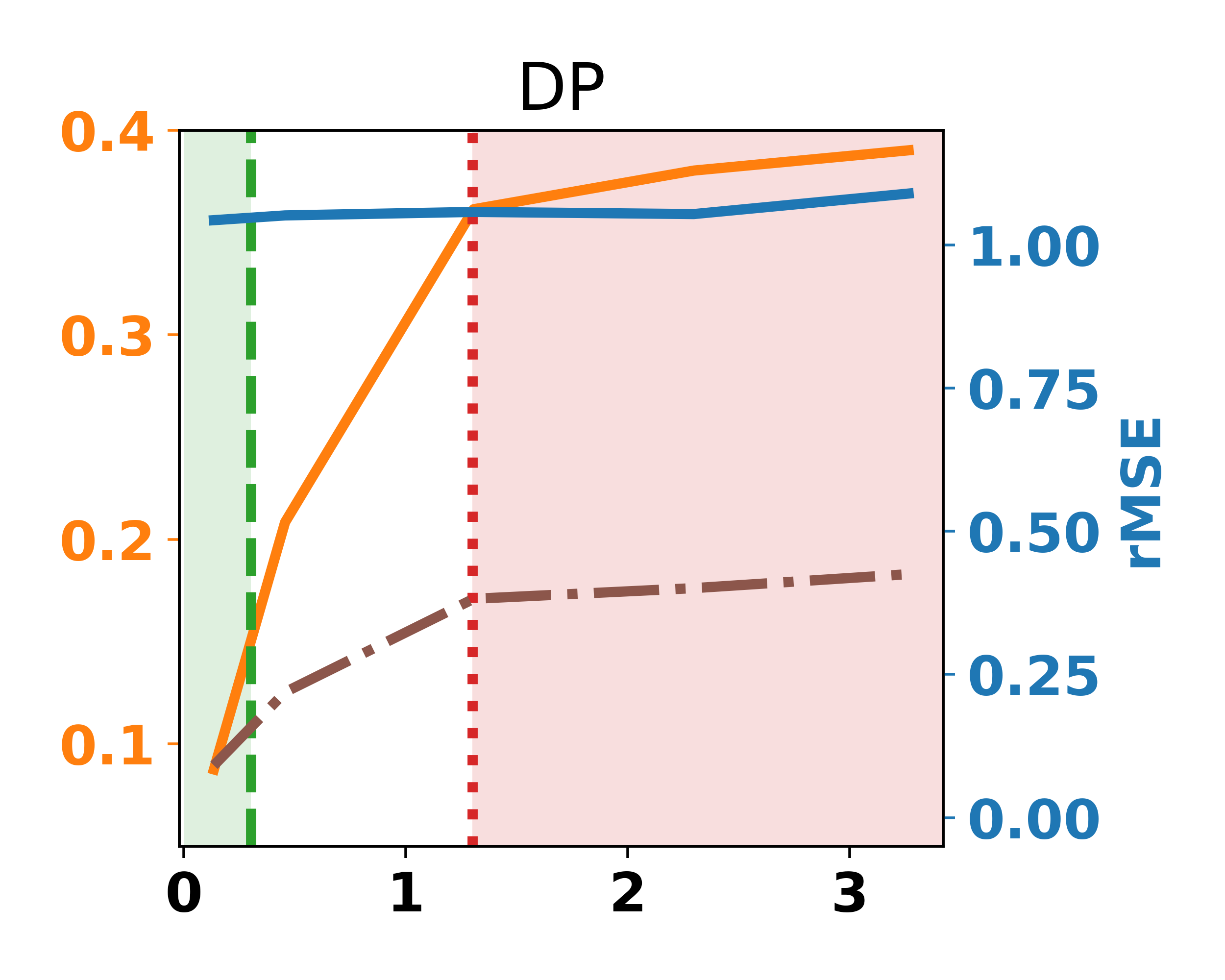}
			\includegraphics[width=0.24\linewidth]{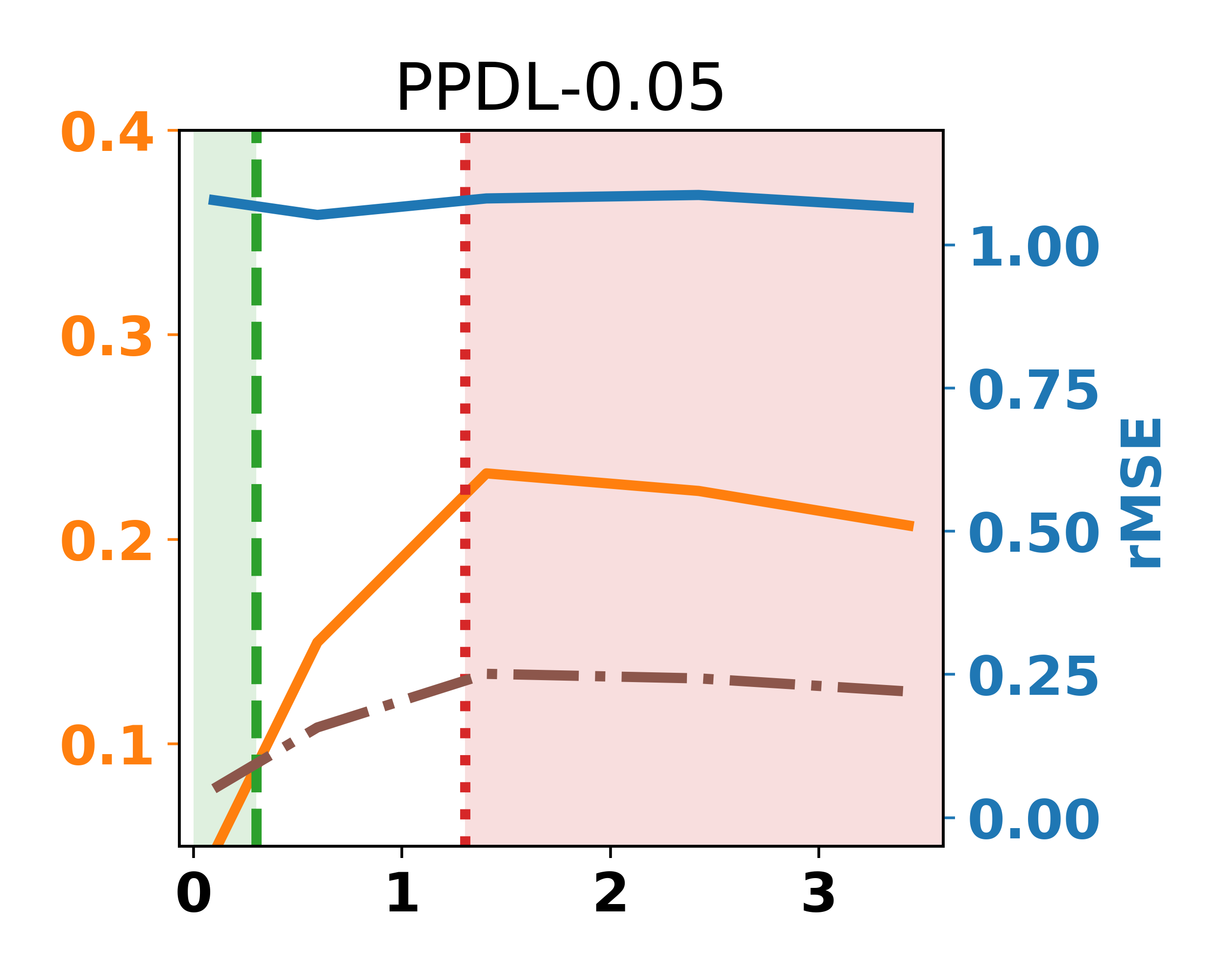}
			\includegraphics[width=0.24\linewidth]{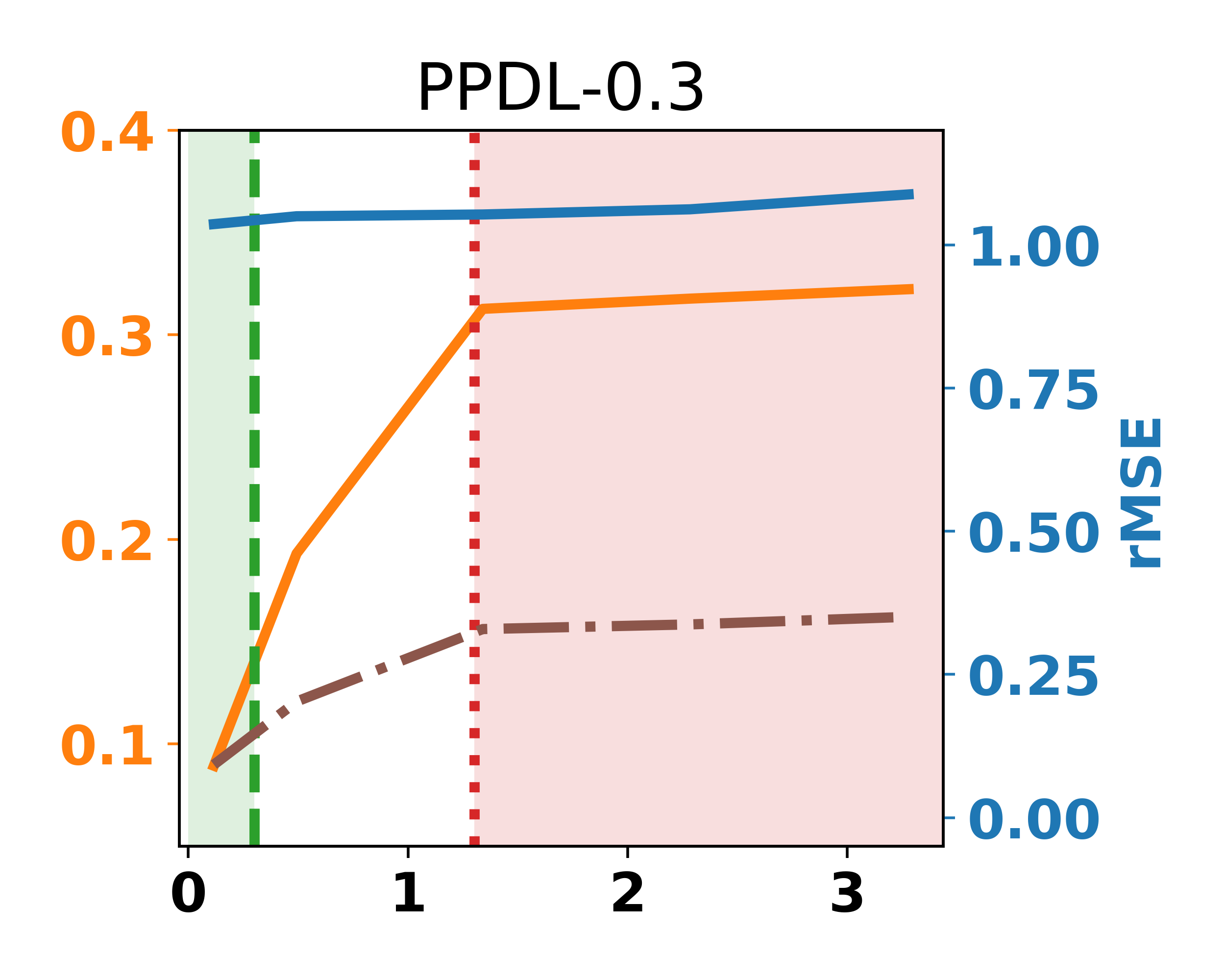}
			\includegraphics[width=0.24\linewidth]{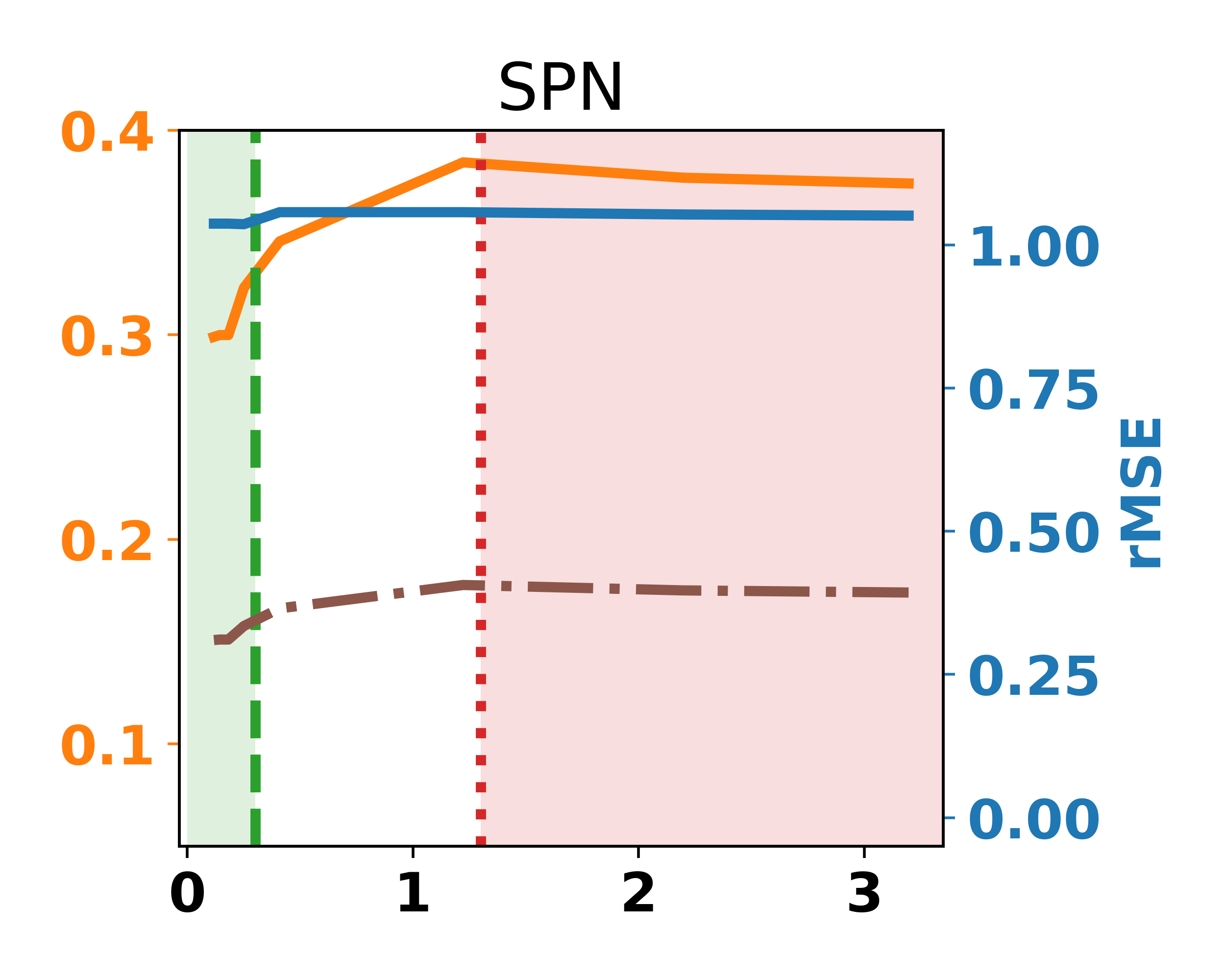}
			\caption{Reconstruction Attack}
		\end{subfigure}
		
		\begin{subfigure}{0.99\linewidth}
			\centering
			\includegraphics[width=0.24\linewidth]{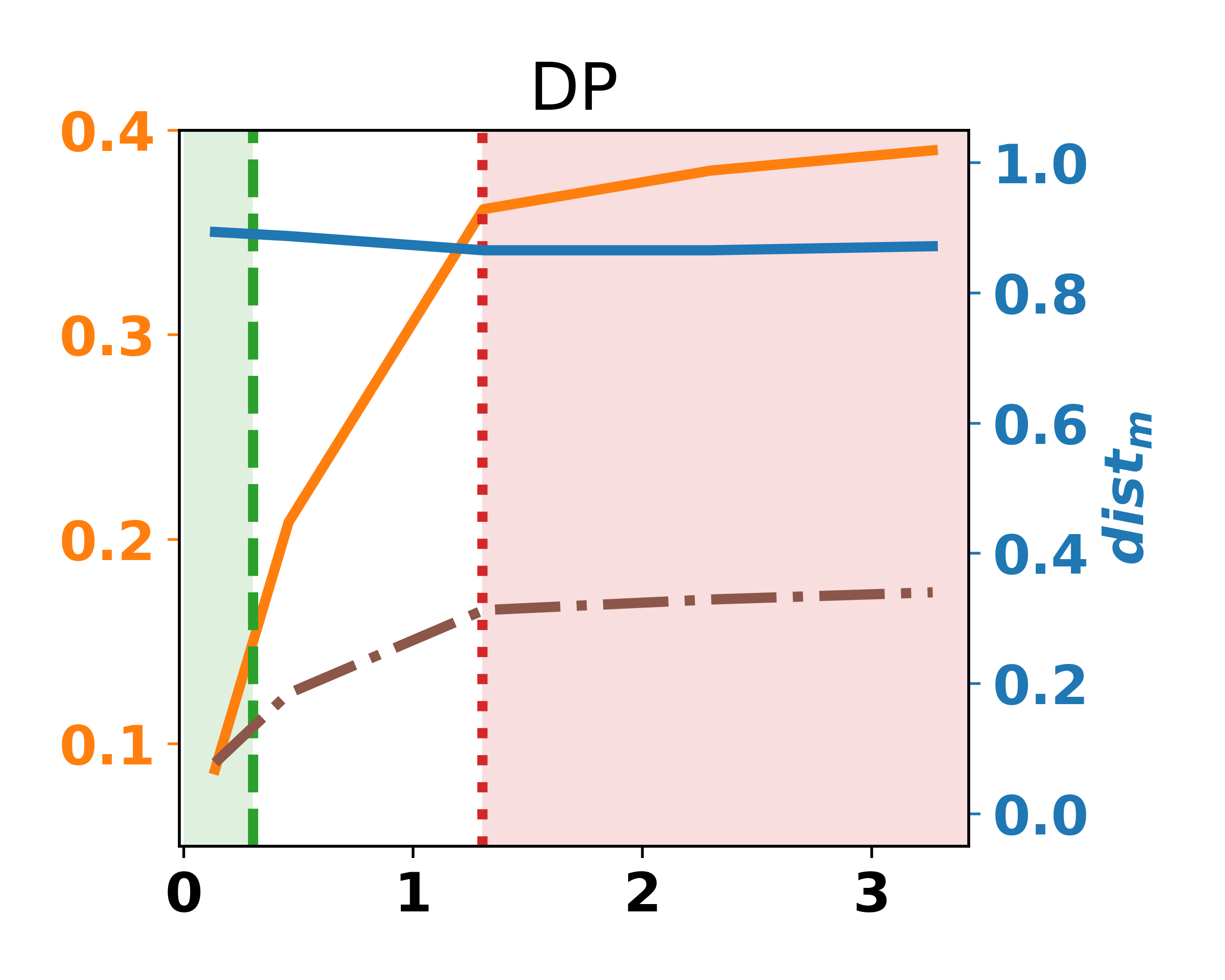}
			\includegraphics[width=0.24\linewidth]{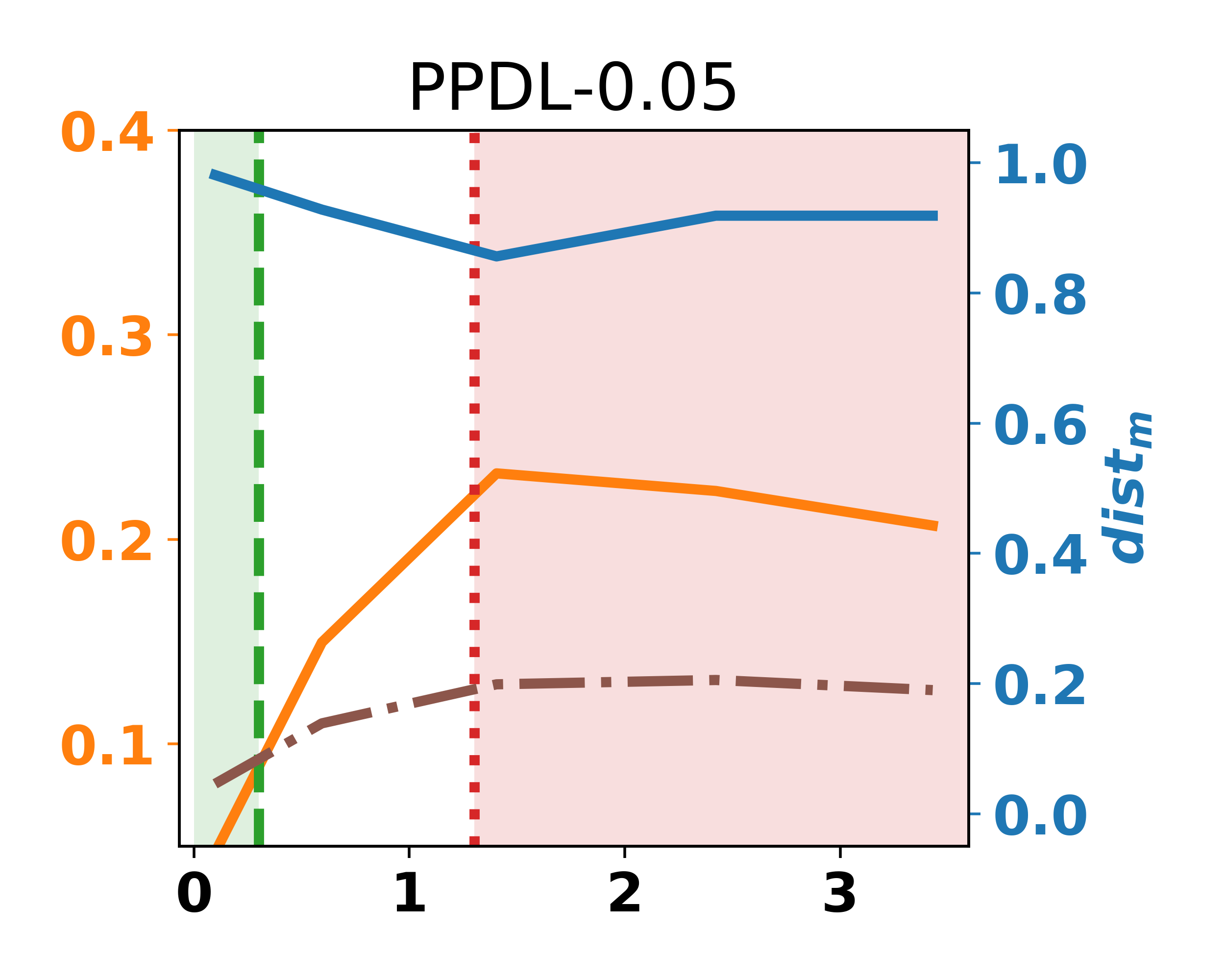}
			\includegraphics[width=0.24\linewidth]{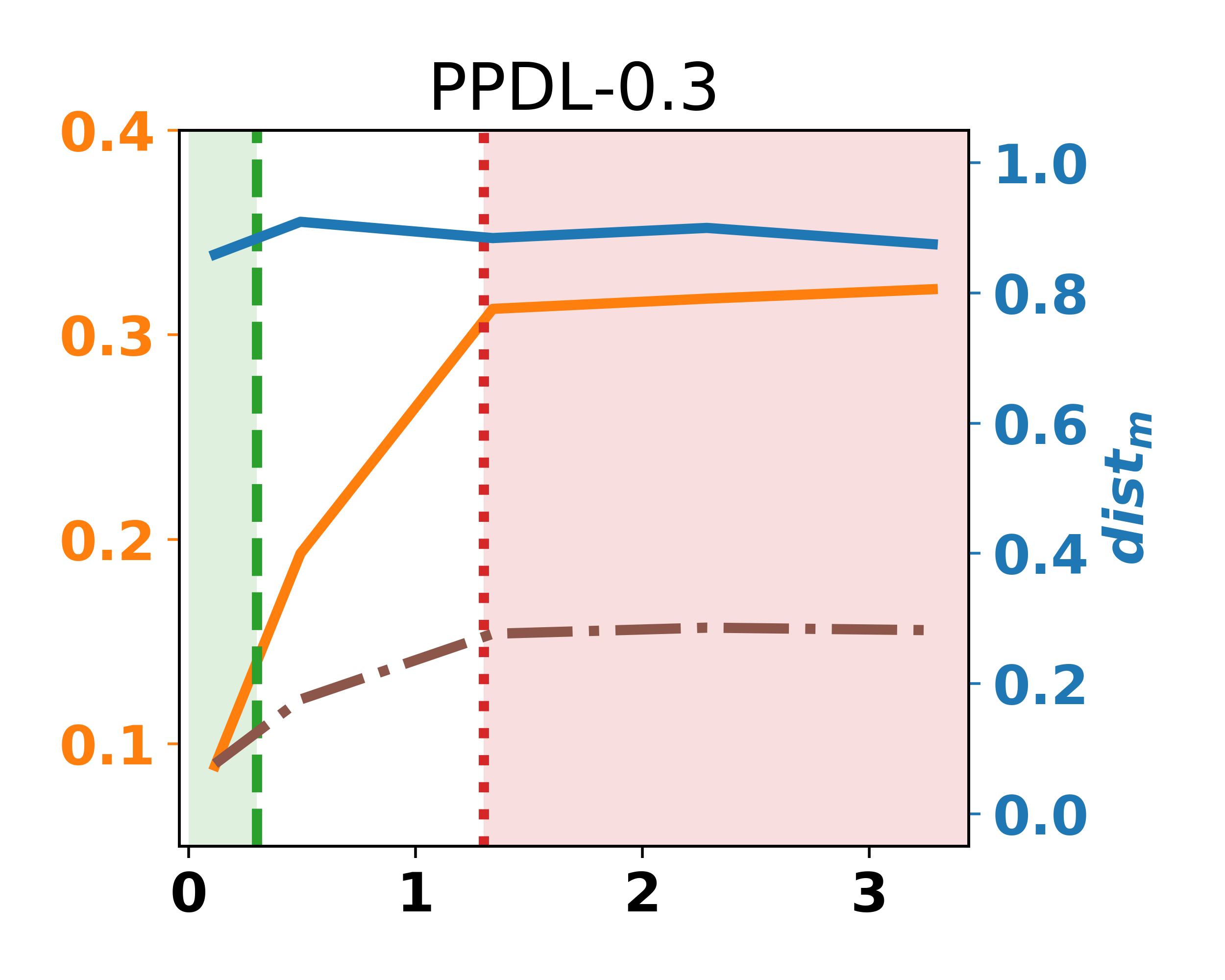}
			\includegraphics[width=0.24\linewidth]{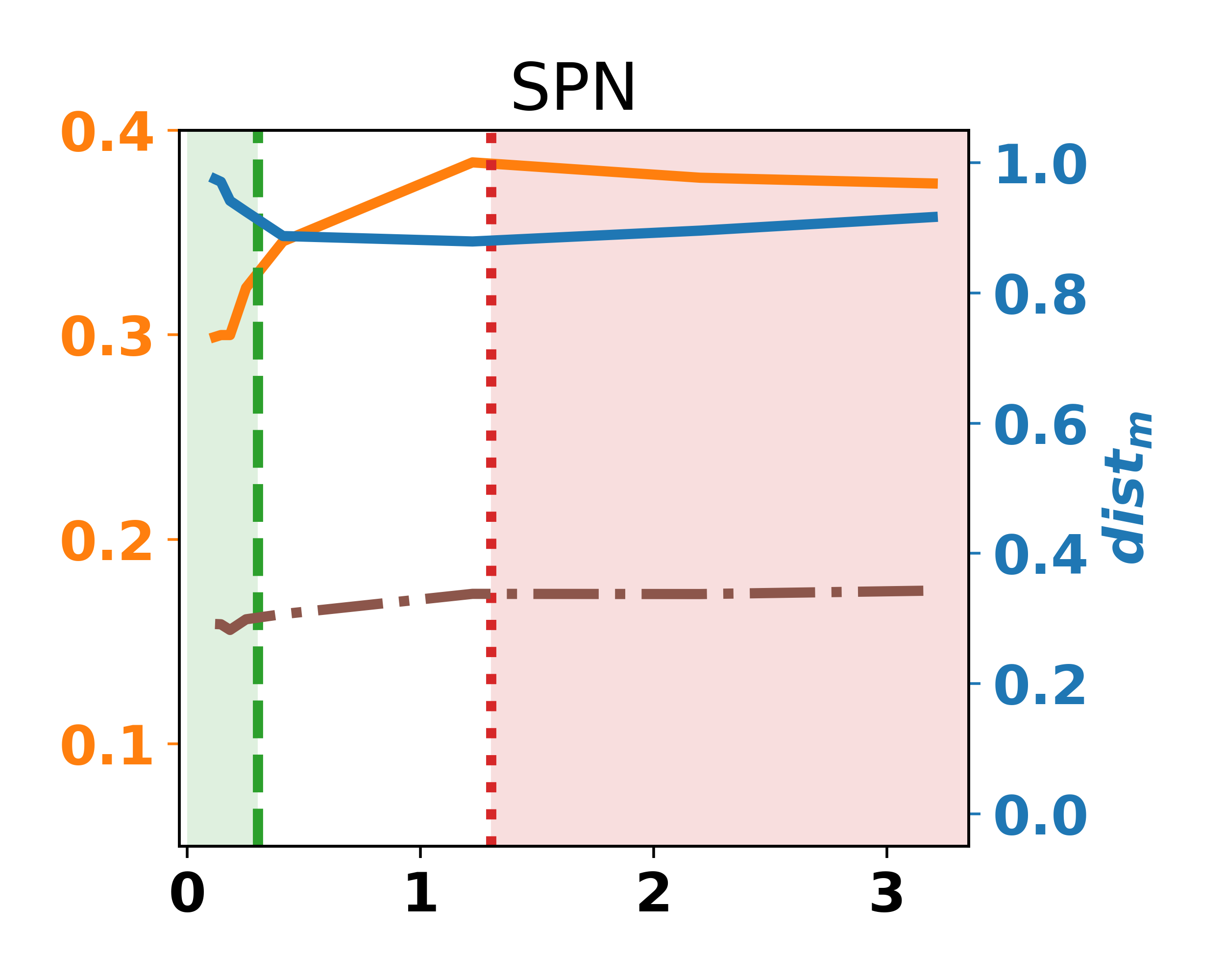}
			\caption{Membership Attack}
		\end{subfigure}
		
		\begin{subfigure}{0.99\linewidth}
			\centering
			\includegraphics[width=0.24\linewidth]{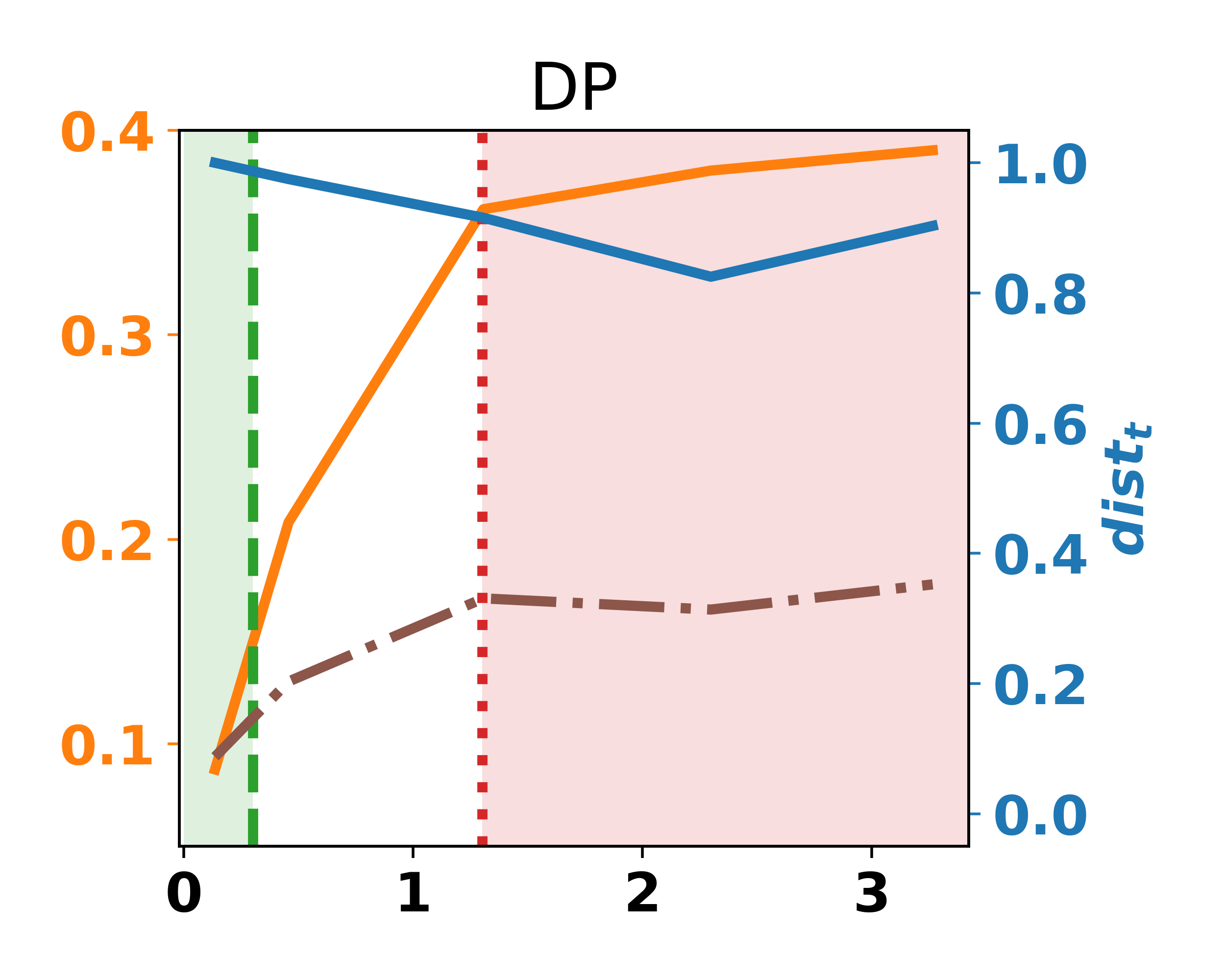}
			\includegraphics[width=0.24\linewidth]{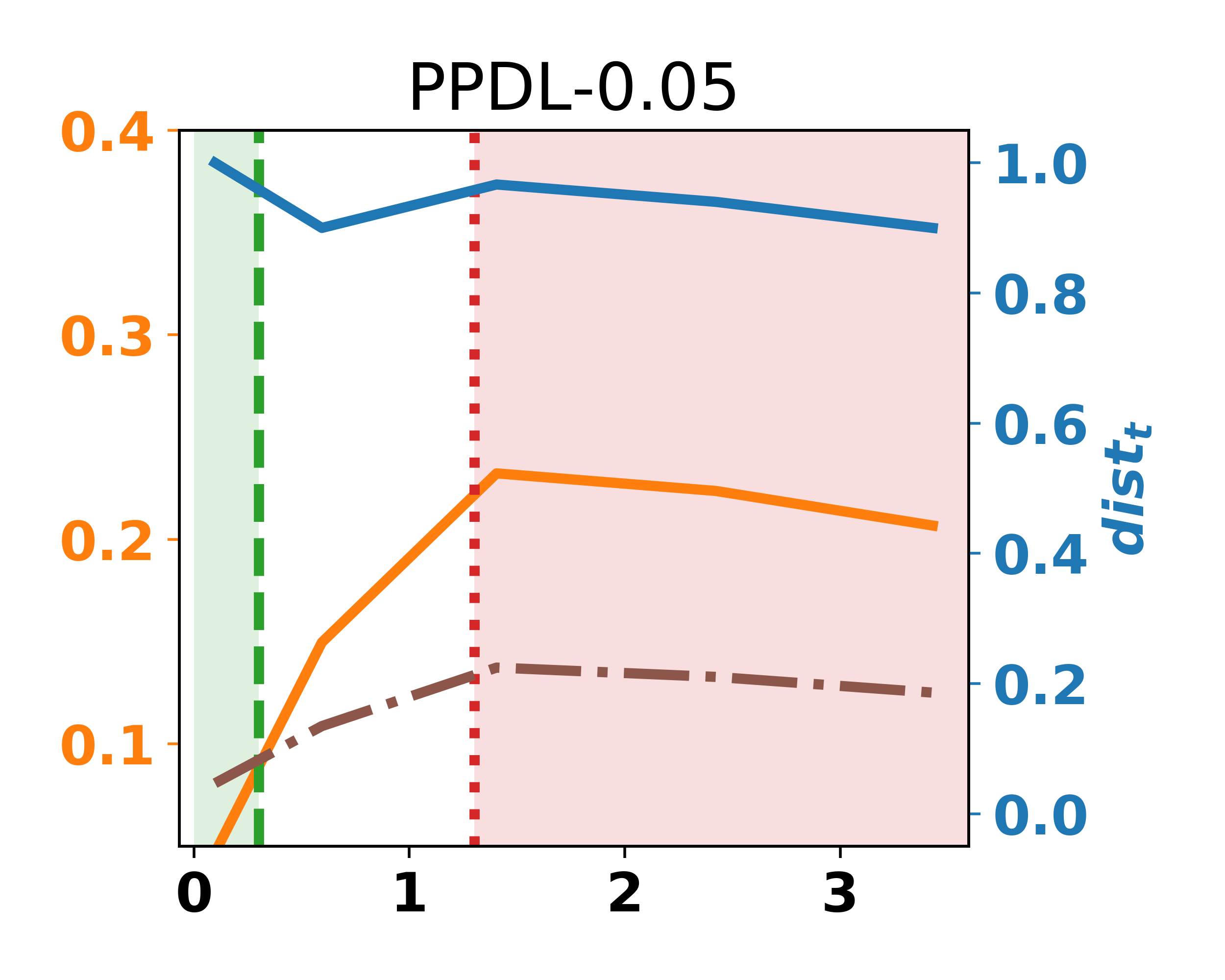}
			\includegraphics[width=0.24\linewidth]{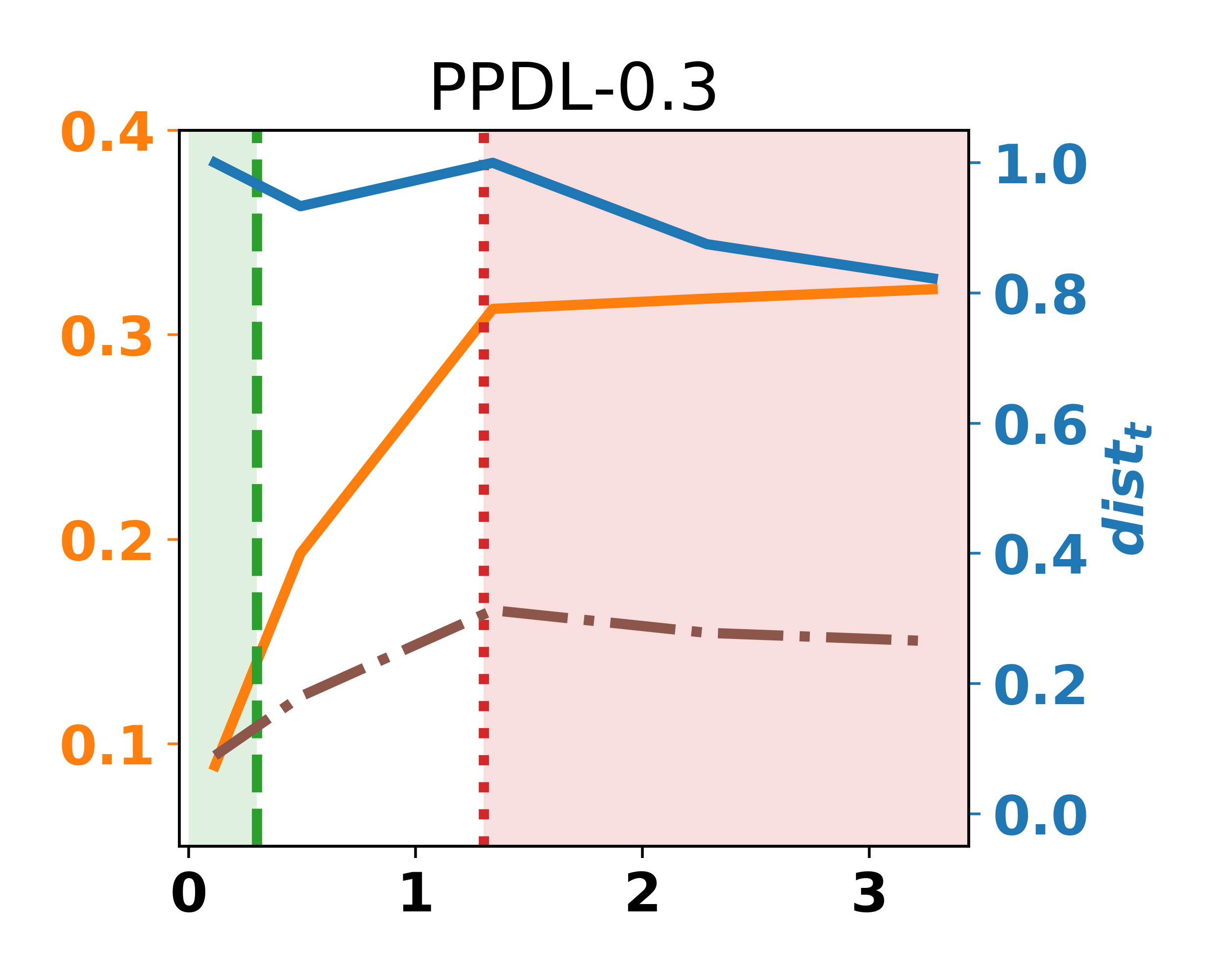}
			\includegraphics[width=0.24\linewidth]{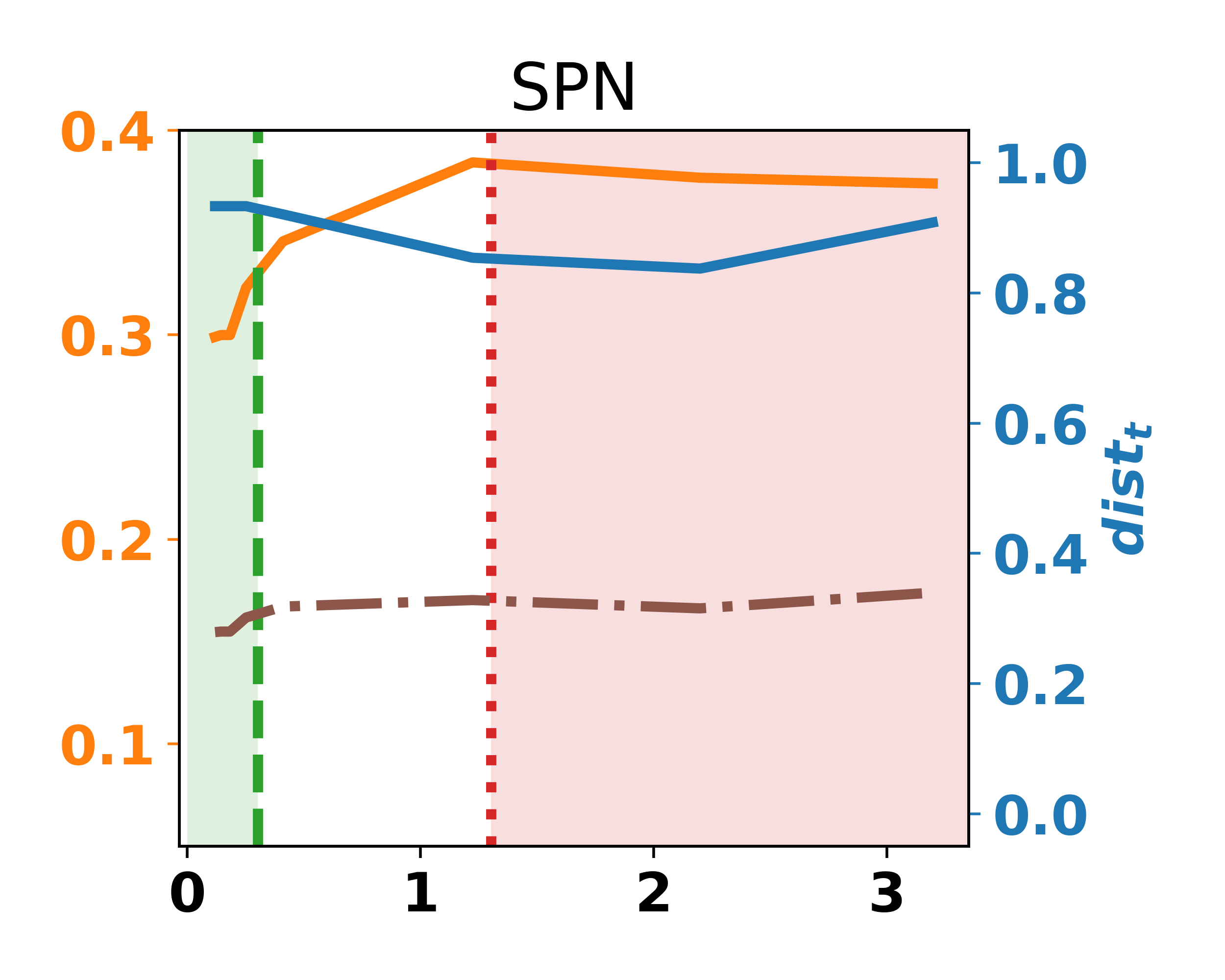}
			\caption{Tracing Attack}
		\end{subfigure}
		
		\caption{Attack with Batch Size 8}
		\label{fig:ppc-cifar100-bs8}
		%\vspace{-0.22cm}
	\end{figure}
	
	Figure \ref{fig:ppc-cifar10-bs1}, \ref{fig:ppc-cifar10-bs4}, \ref{fig:ppc-cifar10-bs8} are privacy-preserving characteristics (PPC) with different attacks on CIFAR100.
	
	\subsubsection{Calibrated Averaged Performance (CAP)}
	
	\begin{table}[H]
		\centering
		\adjustbox{max width=\textwidth}{
			\begin{tabular}{l|c|c|c|c|c|c|c|c|c}
				\toprule
				& \multicolumn{3}{c|}{Reconstruction} & \multicolumn{3}{c|}{Membership} & \multicolumn{3}{c|}{Tracing} \\
				\midrule
				BS & 1 & 4 & 8 & 1 & 4 & 8 & 1 & 4 & 8 \\
				\midrule
				DP \cite{DLDP_Abadi16} & 0.23 & 0.31 & 0.30 & 0.03 & 0.22 & 0.25 & 0.14 & 0.24 & 0.24 \\
				PPDL-0.05 \cite{PPDL/shokri2015} & 0.18 & 0.18 & 0.18 & 0.02 & 0.13 & 0.16 & 0.16 & 0.16 & 0.16 \\
				PPDL-0.3 \cite{PPDL/shokri2015} & 0.21 & 0.26 & 0.26 & 0.00 & 0.19 & 0.22 & 0.19 & 0.19 & 0.19 \\
				SPN (ours) & \textbf{0.37} & \textbf{0.36} & \textbf{0.35} & \textbf{0.17} & \textbf{0.28} & \textbf{0.31} & \textbf{0.29} & \textbf{0.30} & \textbf{0.30} \\
				\bottomrule
			\end{tabular}
		}
		\caption{CAP performance with different batch size on CIFAR100 for reconstruction, membership and tracing attack. Higher better. BS = Attack Batch Size.}
		\label{tab:CAP-cifar100-supp}
	\end{table}
	
	\subsubsection{Reconstructed Images}
	
	\begin{figure}[H]
		\centering
		
		\begin{subfigure}{0.32\linewidth}
			\centering
			\includegraphics[scale=0.8]{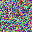}
			\hspace{1pt}
			\includegraphics[scale=0.8]{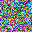} % + svhn
			\hspace{1pt}
			\includegraphics[scale=0.8]{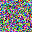}
			\hspace{1pt}
			\includegraphics[scale=0.8]{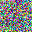}
			\\
			\vspace{2pt}
			\includegraphics[scale=0.8]{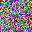}
			\hspace{1pt}
			\includegraphics[scale=0.8]{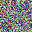} % + svhn
			\hspace{1pt}
			\includegraphics[scale=0.8]{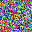}
			\hspace{1pt}
			\includegraphics[scale=0.8]{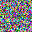}
			\\
			\vspace{2pt}
			\includegraphics[scale=0.8]{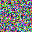}
			\hspace{1pt}
			\includegraphics[scale=0.8]{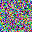} % + svhn
			\hspace{1pt}
			\includegraphics[scale=0.8]{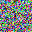}
			\hspace{1pt}
			\includegraphics[scale=0.8]{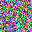}
			\\
			\vspace{2pt}
			\includegraphics[scale=0.8]{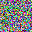}
			\hspace{1pt}
			\includegraphics[scale=0.8]{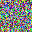} % + svhn
			\hspace{1pt}
			\includegraphics[scale=0.8]{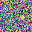}
			\hspace{1pt}
			\includegraphics[scale=0.8]{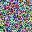}
			\caption{$1.12 (0.32)$ \label{fig:recon-images-green-cifar100-supp}}
		\end{subfigure}
		\hfill
		\begin{subfigure}{0.32\linewidth}
			\centering
			\includegraphics[scale=0.8]{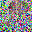}
			\hspace{1pt}
			\includegraphics[scale=0.8]{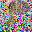} % + svhn
			\hspace{1pt}
			\includegraphics[scale=0.8]{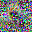}
			\hspace{1pt}
			\includegraphics[scale=0.8]{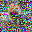}
			\\
			\vspace{2pt}
			\includegraphics[scale=0.8]{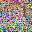}
			\hspace{1pt}
			\includegraphics[scale=0.8]{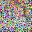} % + svhn
			\hspace{1pt}
			\includegraphics[scale=0.8]{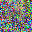}
			\hspace{1pt}
			\includegraphics[scale=0.8]{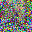}
			\\
			\vspace{2pt}
			\includegraphics[scale=0.8]{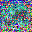}
			\hspace{1pt}
			\includegraphics[scale=0.8]{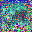} % + svhn
			\hspace{1pt}
			\includegraphics[scale=0.8]{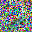}
			\hspace{1pt}
			\includegraphics[scale=0.8]{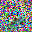}
			\\
			\vspace{2pt}
			\includegraphics[scale=0.8]{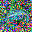}
			\hspace{1pt}
			\includegraphics[scale=0.8]{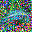} % + svhn
			\hspace{1pt}
			\includegraphics[scale=0.8]{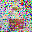}
			\hspace{1pt}
			\includegraphics[scale=0.8]{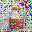}
			\caption{$1.11 (23.63)$ \label{fig:recon-images-white-cifar100-supp}}%; 1.10 (1.30)}
		\end{subfigure}
		\hfill
		\begin{subfigure}{0.32\linewidth}
			\centering
			\includegraphics[scale=0.8]{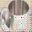}
			\hspace{1pt}
			\includegraphics[scale=0.8]{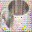} % + svhn
			\hspace{1pt}
			\includegraphics[scale=0.8]{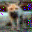}
			\hspace{1pt}
			\includegraphics[scale=0.8]{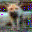}
			\\
			\vspace{2pt}
			\includegraphics[scale=0.8]{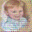}
			\hspace{1pt}
			\includegraphics[scale=0.8]{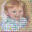} % + svhn
			\hspace{1pt}
			\includegraphics[scale=0.8]{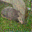}
			\hspace{1pt}
			\includegraphics[scale=0.8]{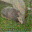}
			\\
			\vspace{2pt}
			\includegraphics[scale=0.8]{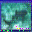}
			\hspace{1pt}
			\includegraphics[scale=0.8]{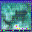} % + svhn
			\hspace{1pt}
			\includegraphics[scale=0.8]{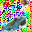}
			\hspace{1pt}
			\includegraphics[scale=0.8]{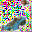}
			\\
			\vspace{2pt}
			\includegraphics[scale=0.8]{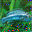}
			\hspace{1pt}
			\includegraphics[scale=0.8]{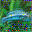} % + svhn
			\hspace{1pt}
			\includegraphics[scale=0.8]{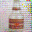}
			\hspace{1pt}
			\includegraphics[scale=0.8]{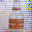}
			\caption{$0.77 (876.73)$ \label{fig:recon-images-red-cifar100-supp}} %0.94 (1.59)}
		\end{subfigure}
		
		%\begin{subfigure}{0.2\linewidth}
		%	\centering
		%	\includegraphics[scale=0.8]{imgs/recimgs/0_0.0001_0.0151.png}
		%	\hspace{3pt}
		%	\includegraphics[scale=0.8]{imgs/recimgs/391_0.0001_0.0116.png}
		%	\caption{0.02 (3.34); 0.01 (3.28)}
		%\end{subfigure}
		\caption{Reconstructed images from different region mentioned in the main paper. \textbf{(a)} Green region \textbf{(b)} White region \textbf{(c)} Red region. 
			Values inside bracket are mean of $\frac{||B_I||}{||E_B||}$ and values outside are mean of rMSE of reconstructed w.r.t. original images.
		}
		\label{fig:recon-images-cifar100-supp}
	\end{figure}
	
	\newpage
	\subsection{SVHN}
	
	\subsubsection{Privacy-Preserving Characteristics (PPC)}
	
	\begin{figure}[H]	
		%	\begin{subfigure}{0.95\linewidth}
		\centering
		\begin{subfigure}{0.99\linewidth}
			\centering
			\includegraphics[scale=0.7]{imgs/legends/legend_ppc_horizontal.png}
			\\
			\includegraphics[width=0.24\linewidth]{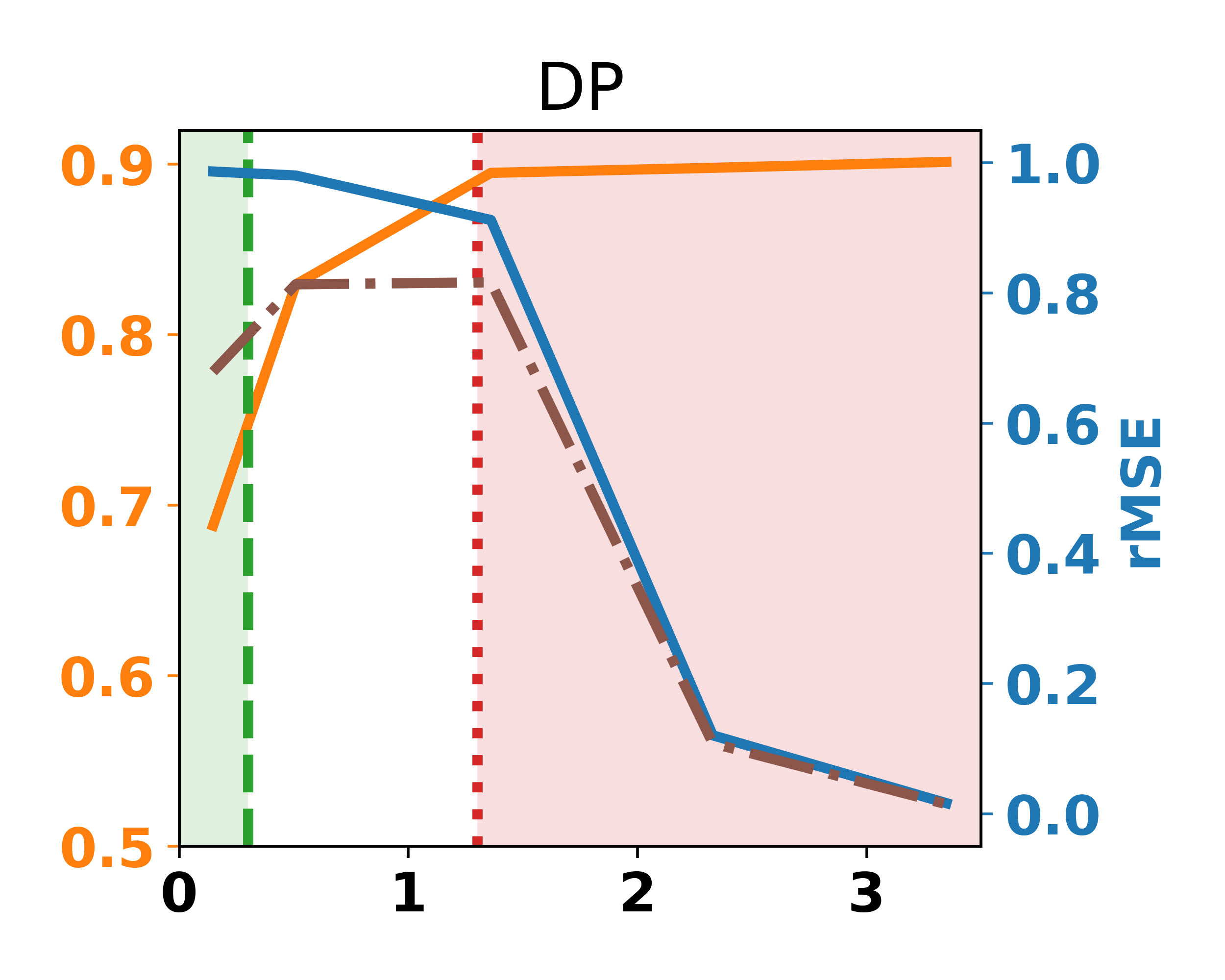}
			\includegraphics[width=0.24\linewidth]{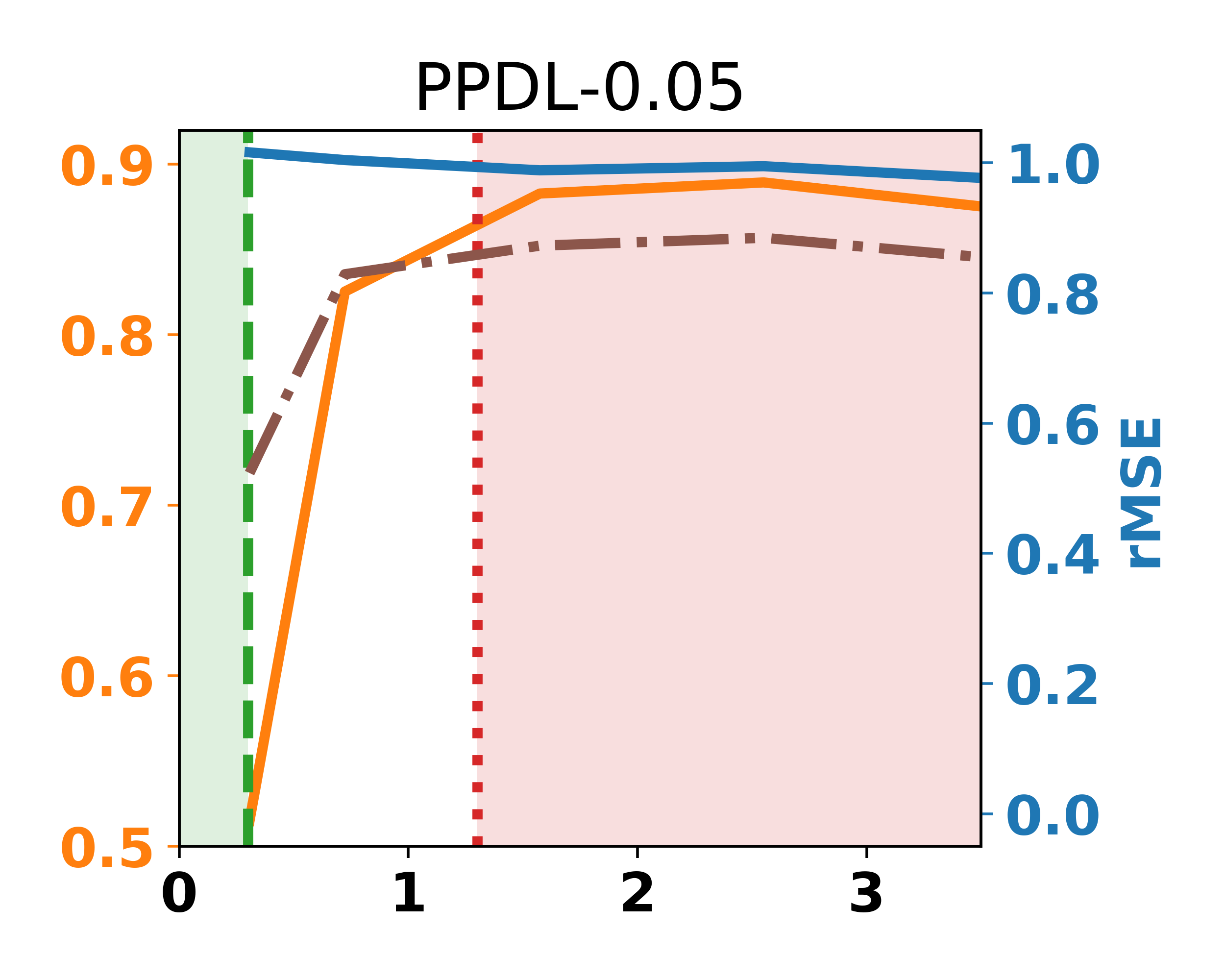}
			\includegraphics[width=0.24\linewidth]{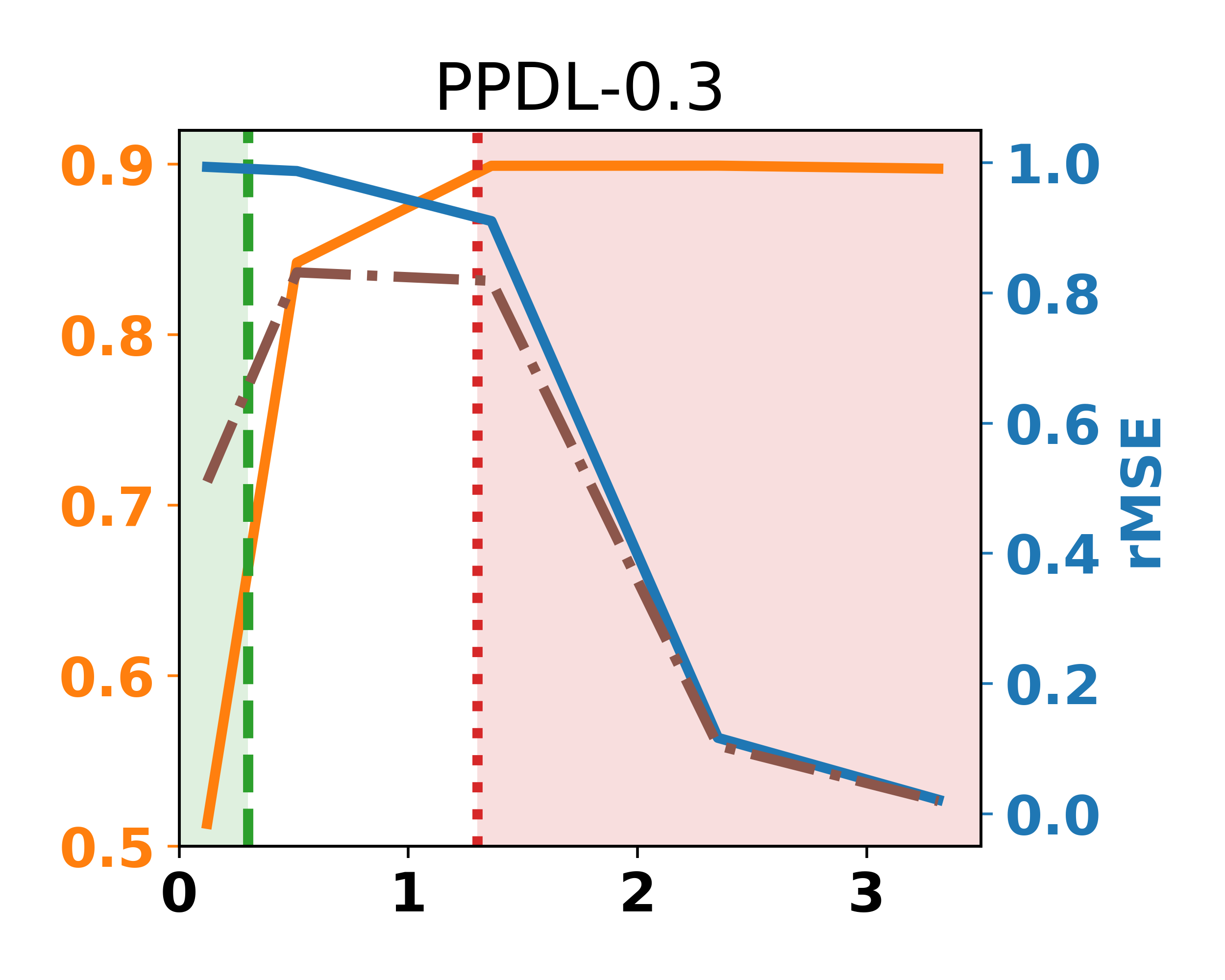}
			\includegraphics[width=0.24\linewidth]{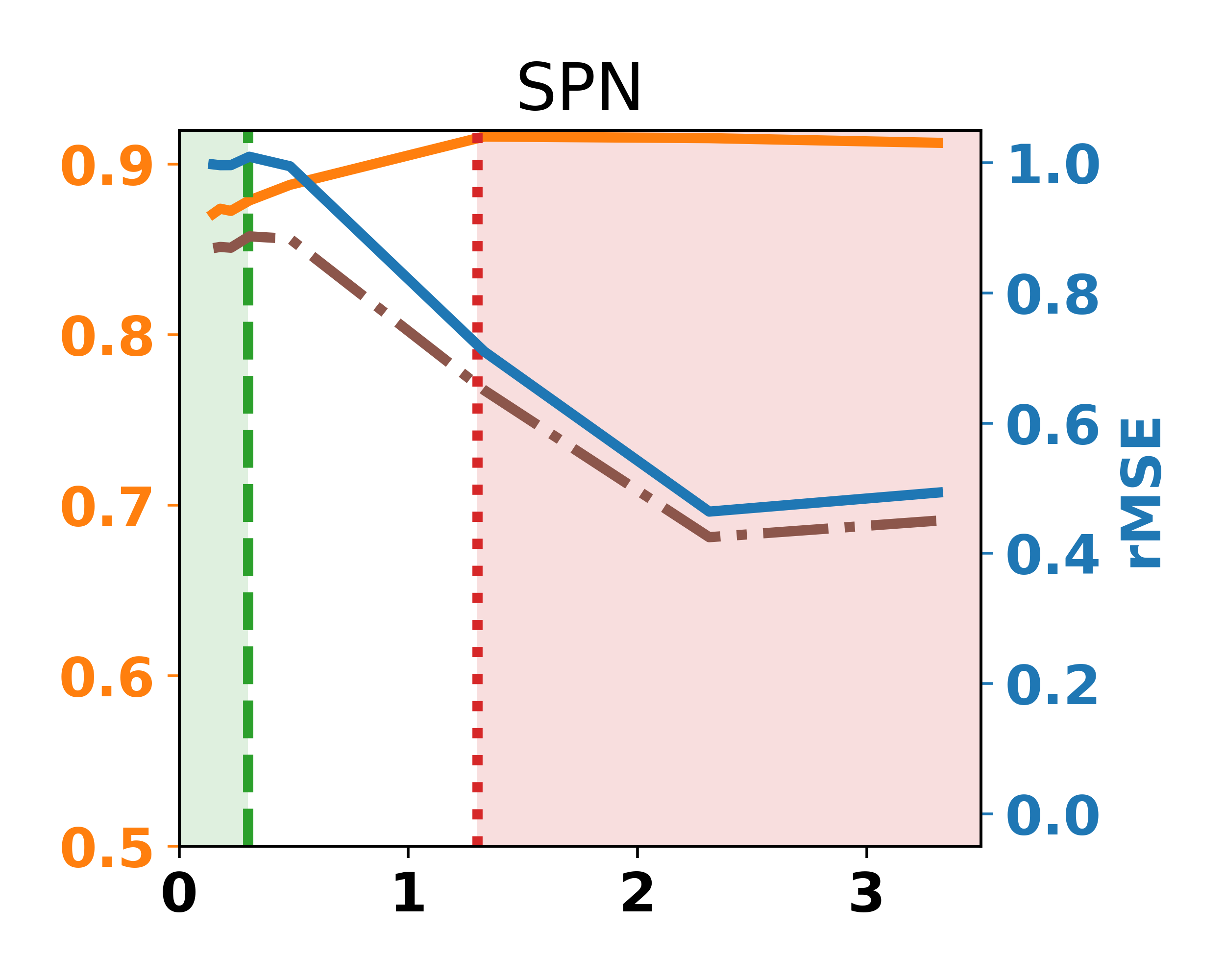}
			\caption{Reconstruction Attack}
		\end{subfigure}
		
		\begin{subfigure}{0.99\linewidth}
			\centering
			\includegraphics[width=0.24\linewidth]{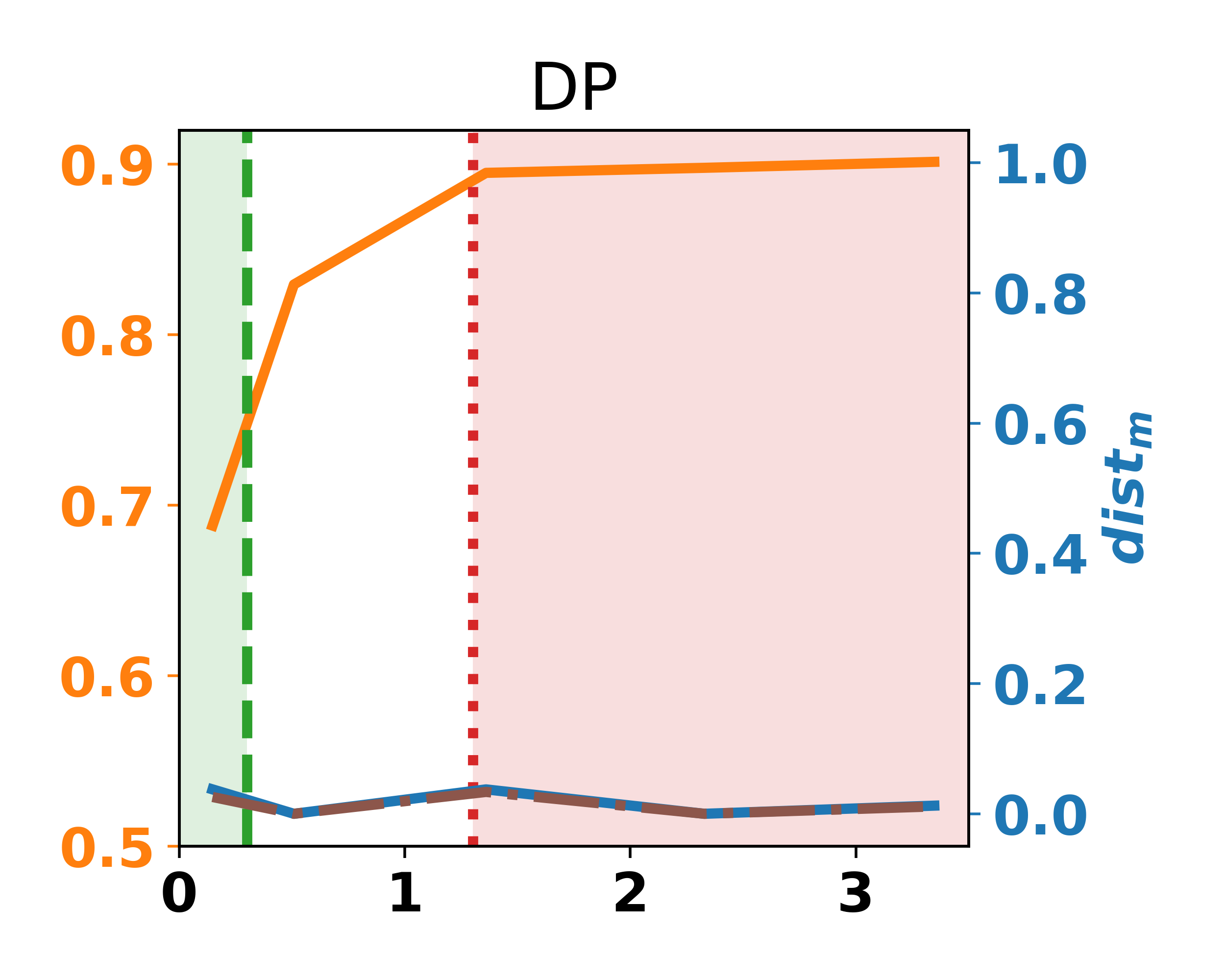}
			\includegraphics[width=0.24\linewidth]{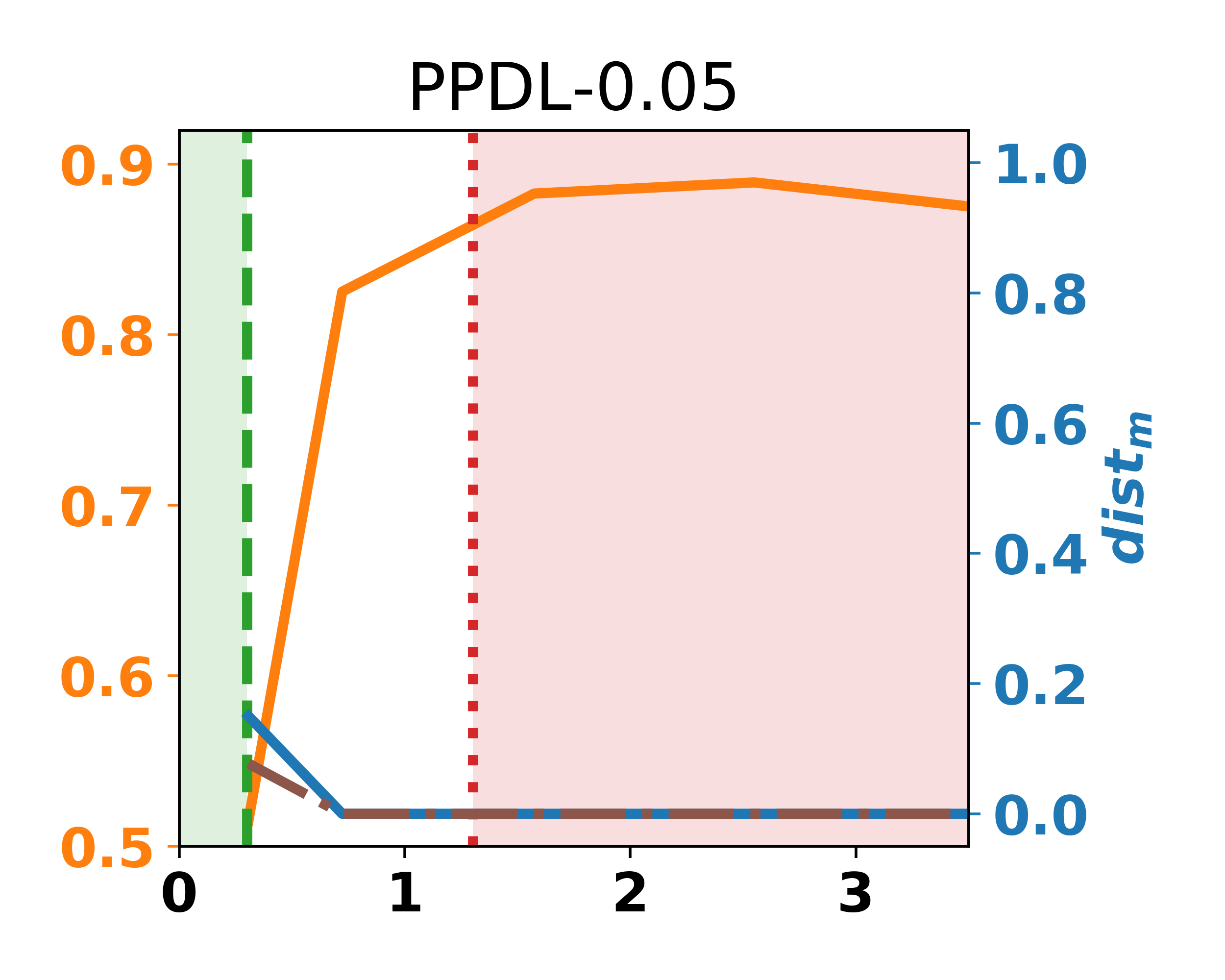}
			\includegraphics[width=0.24\linewidth]{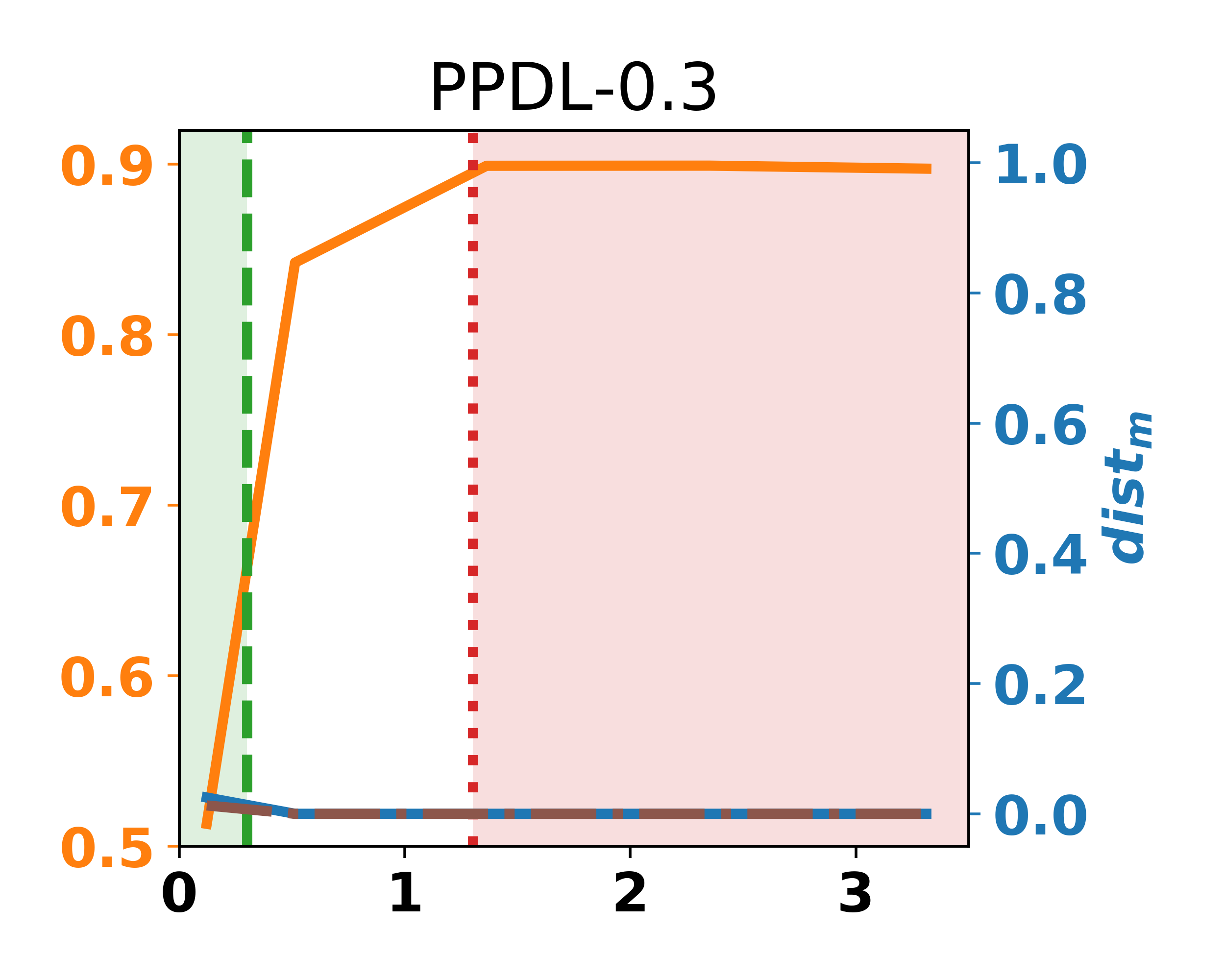}
			\includegraphics[width=0.24\linewidth]{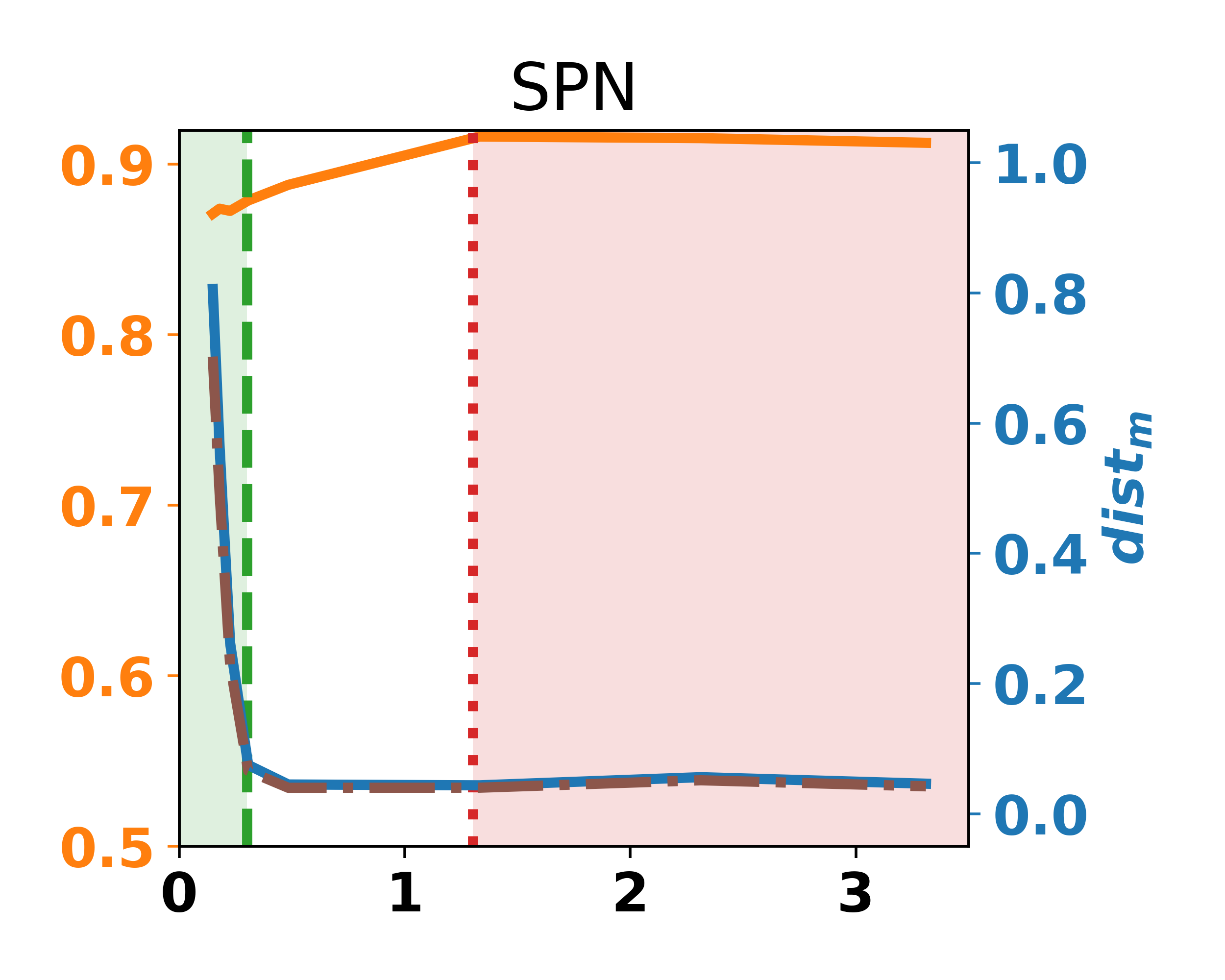}
			\caption{Membership Attack}
		\end{subfigure}
		
		\begin{subfigure}{0.99\linewidth}
			\centering
			\includegraphics[width=0.24\linewidth]{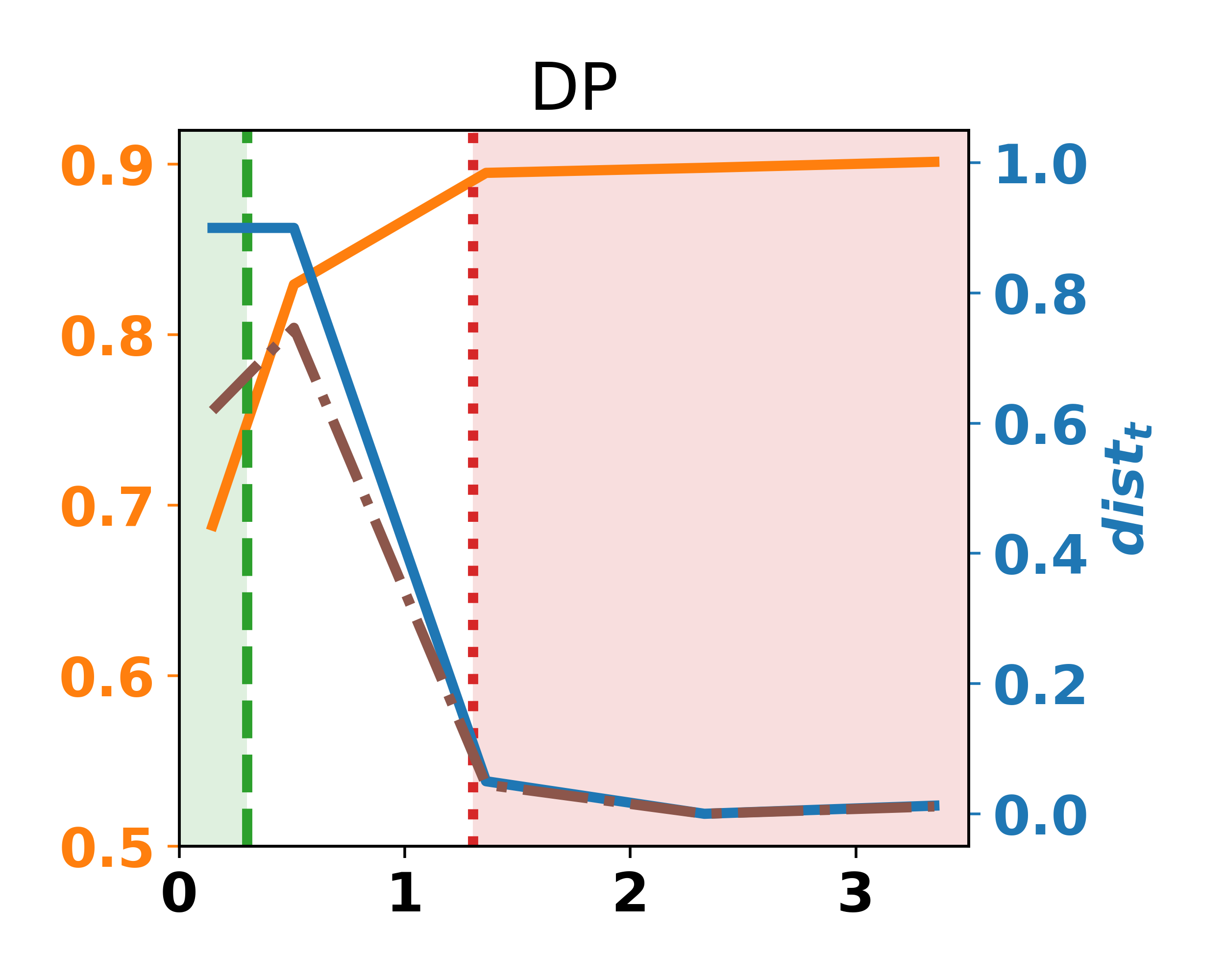}
			\includegraphics[width=0.24\linewidth]{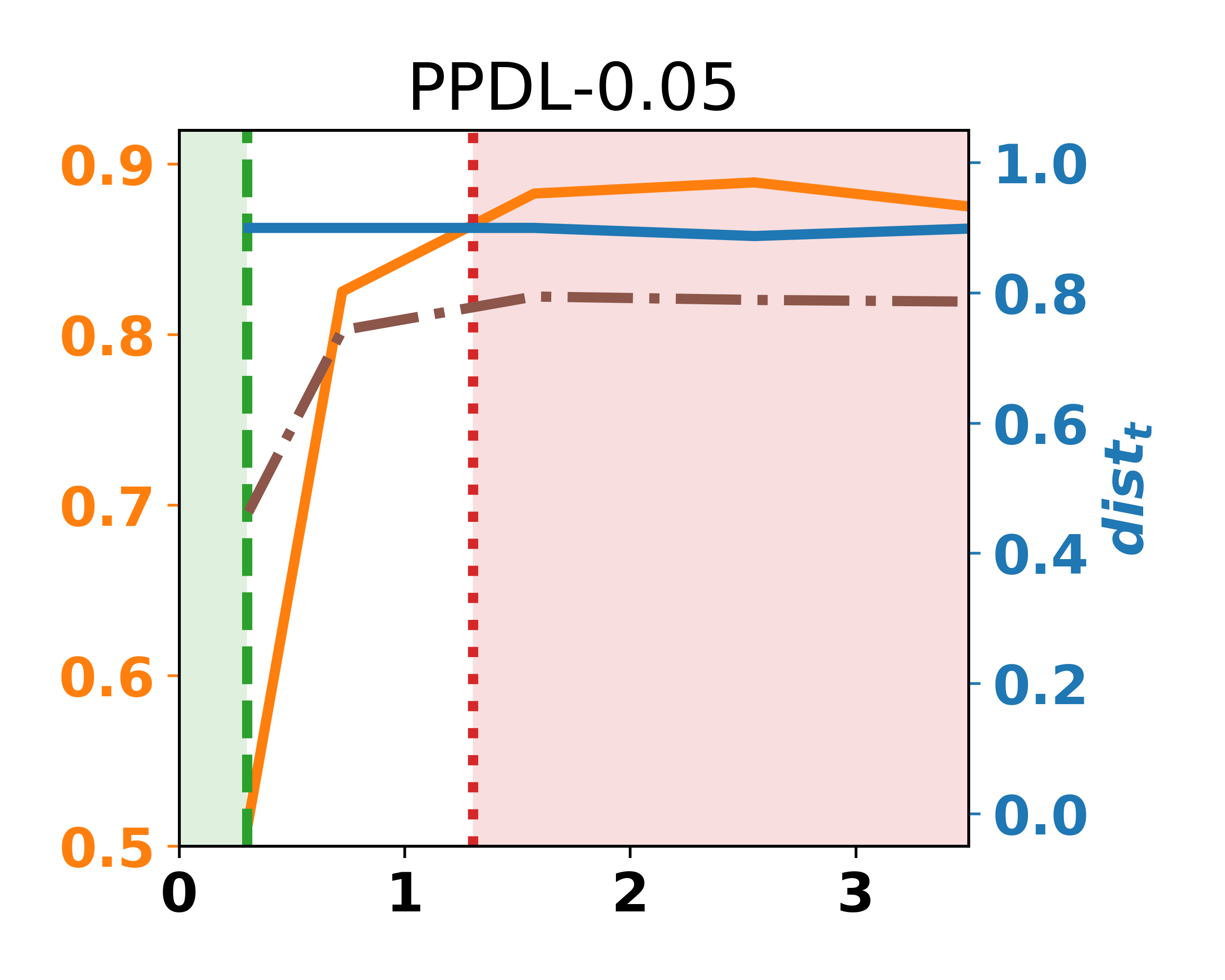}
			\includegraphics[width=0.24\linewidth]{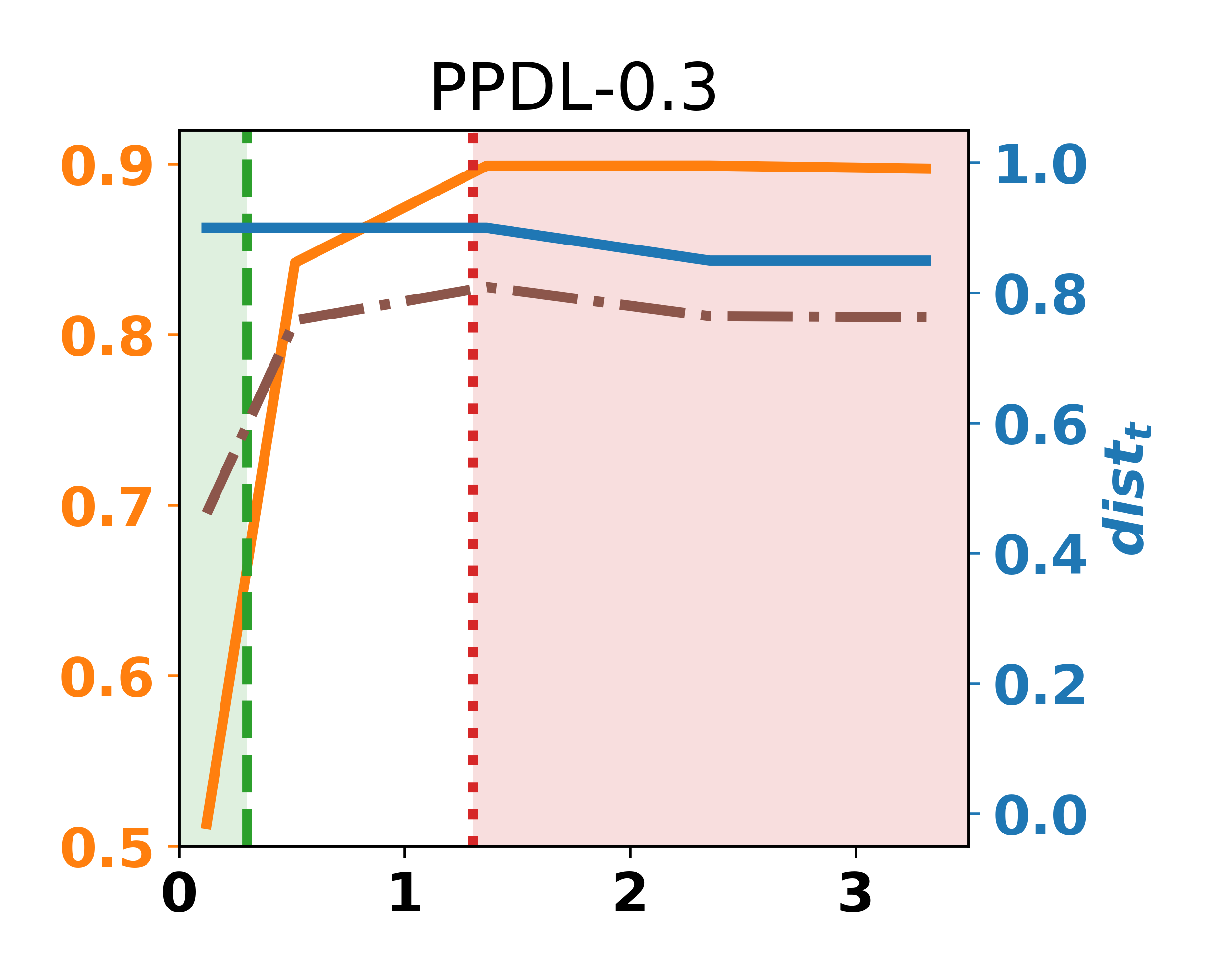}
			\includegraphics[width=0.24\linewidth]{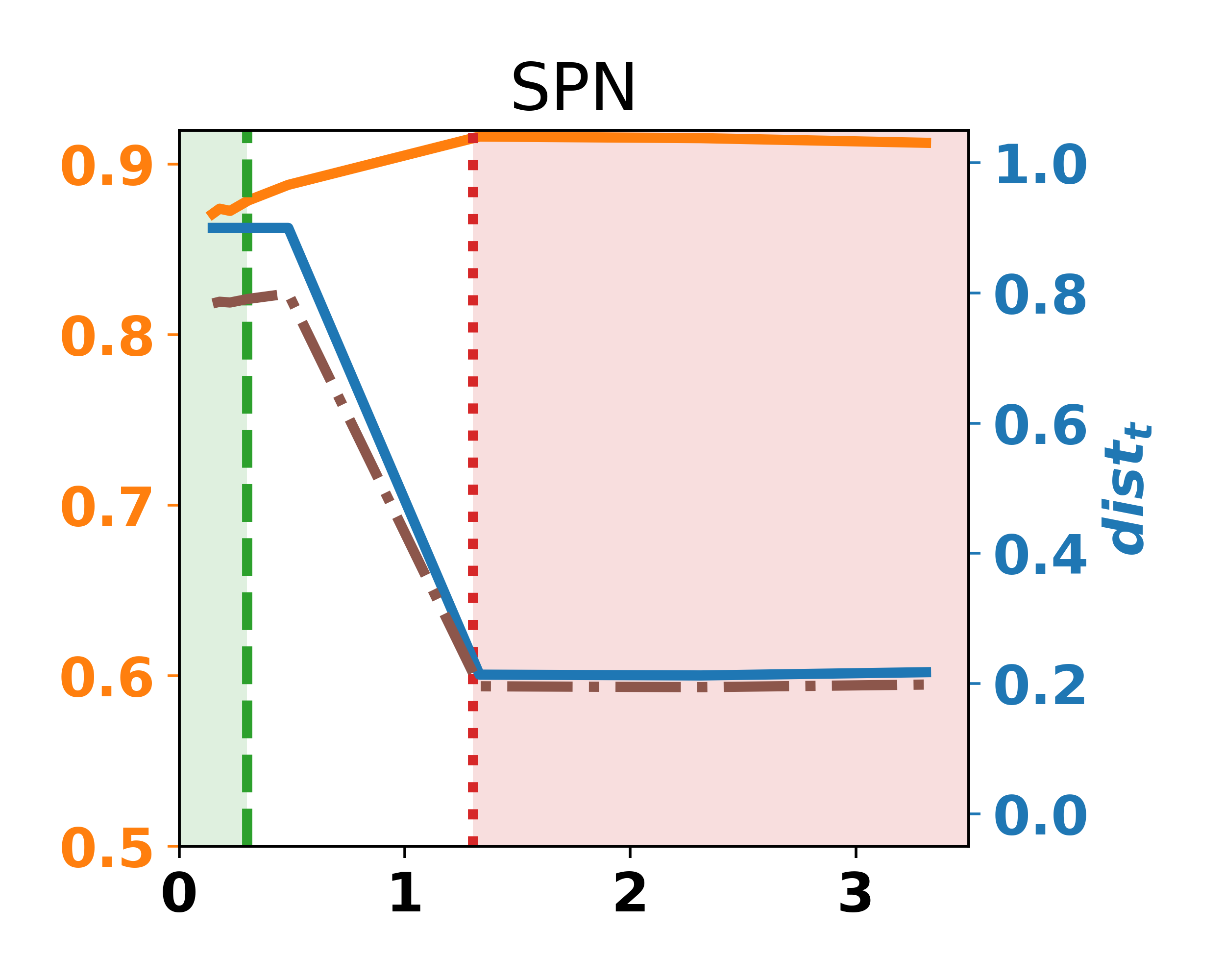}
			\caption{Tracing Attack}
		\end{subfigure}
		
		\caption{Attack Batch Size 1}
		\label{fig:ppc-svhn-bs1}
		%\vspace{-0.22cm}
	\end{figure}

	\newpage
	
	\begin{figure}[H]	
		%	\begin{subfigure}{0.95\linewidth}
		\centering
		\begin{subfigure}{0.99\linewidth}
			\centering
			\includegraphics[scale=0.7]{imgs/legends/legend_ppc_horizontal.png}
			\\
			\includegraphics[width=0.24\linewidth]{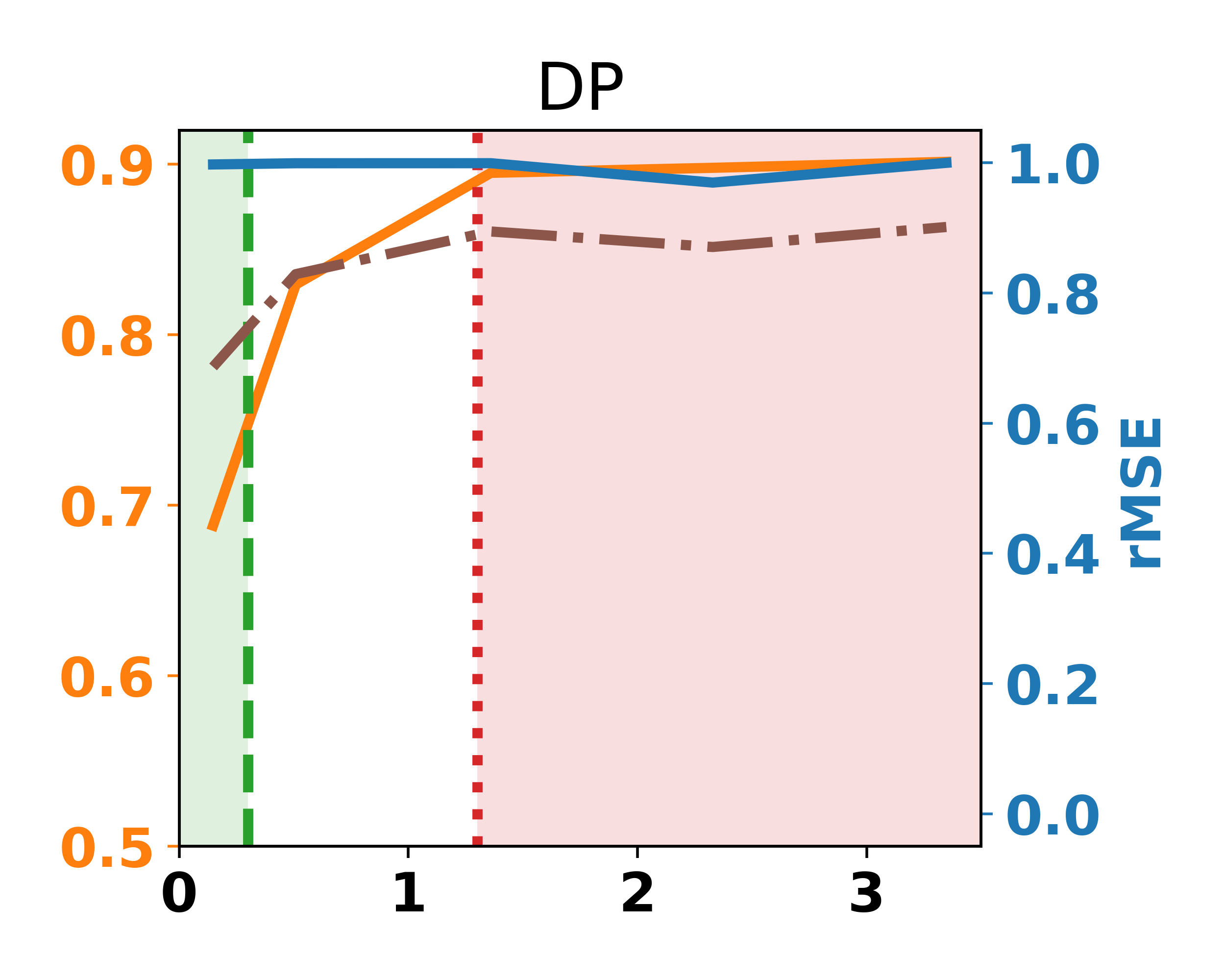}
			\includegraphics[width=0.24\linewidth]{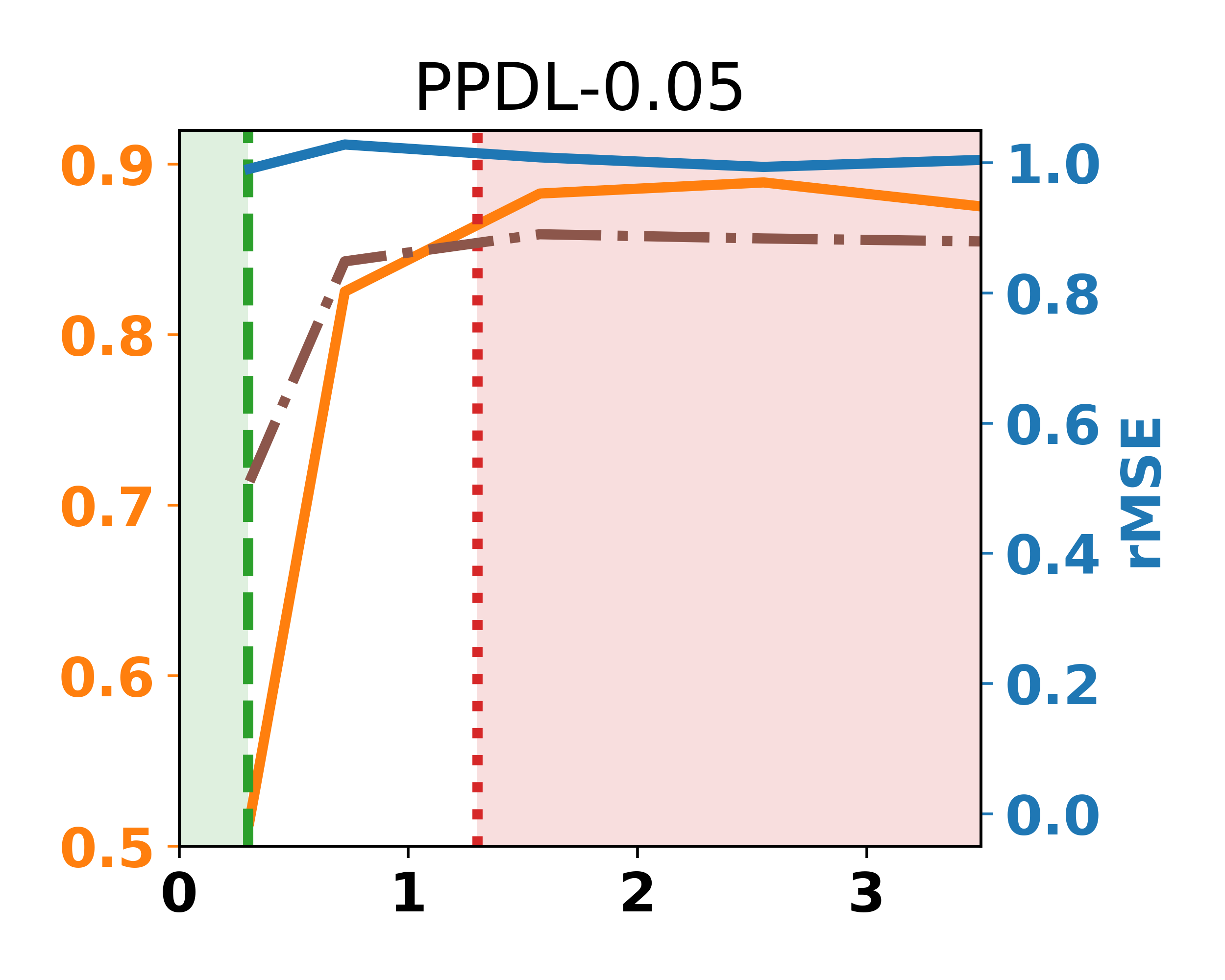}
			\includegraphics[width=0.24\linewidth]{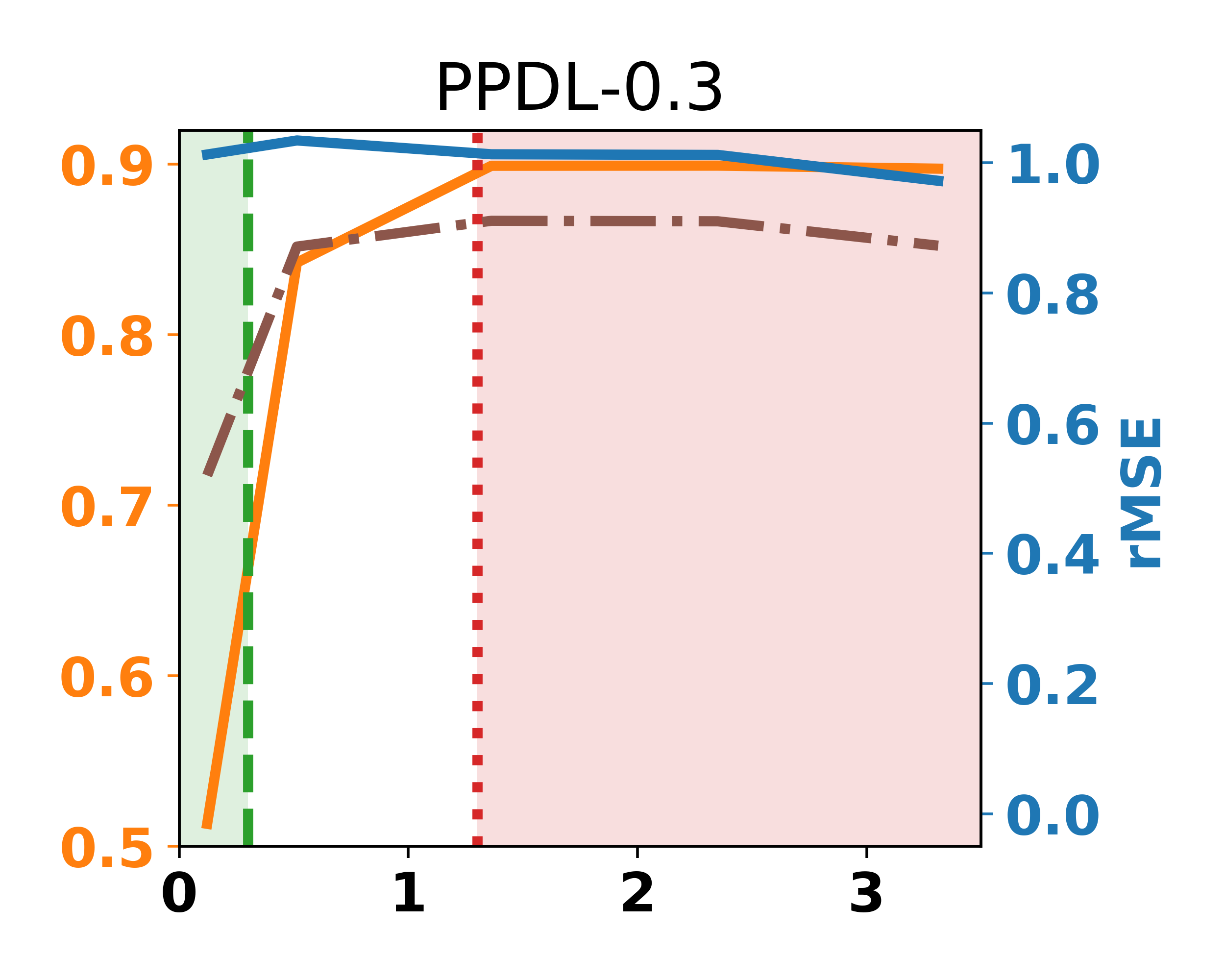}
			\includegraphics[width=0.24\linewidth]{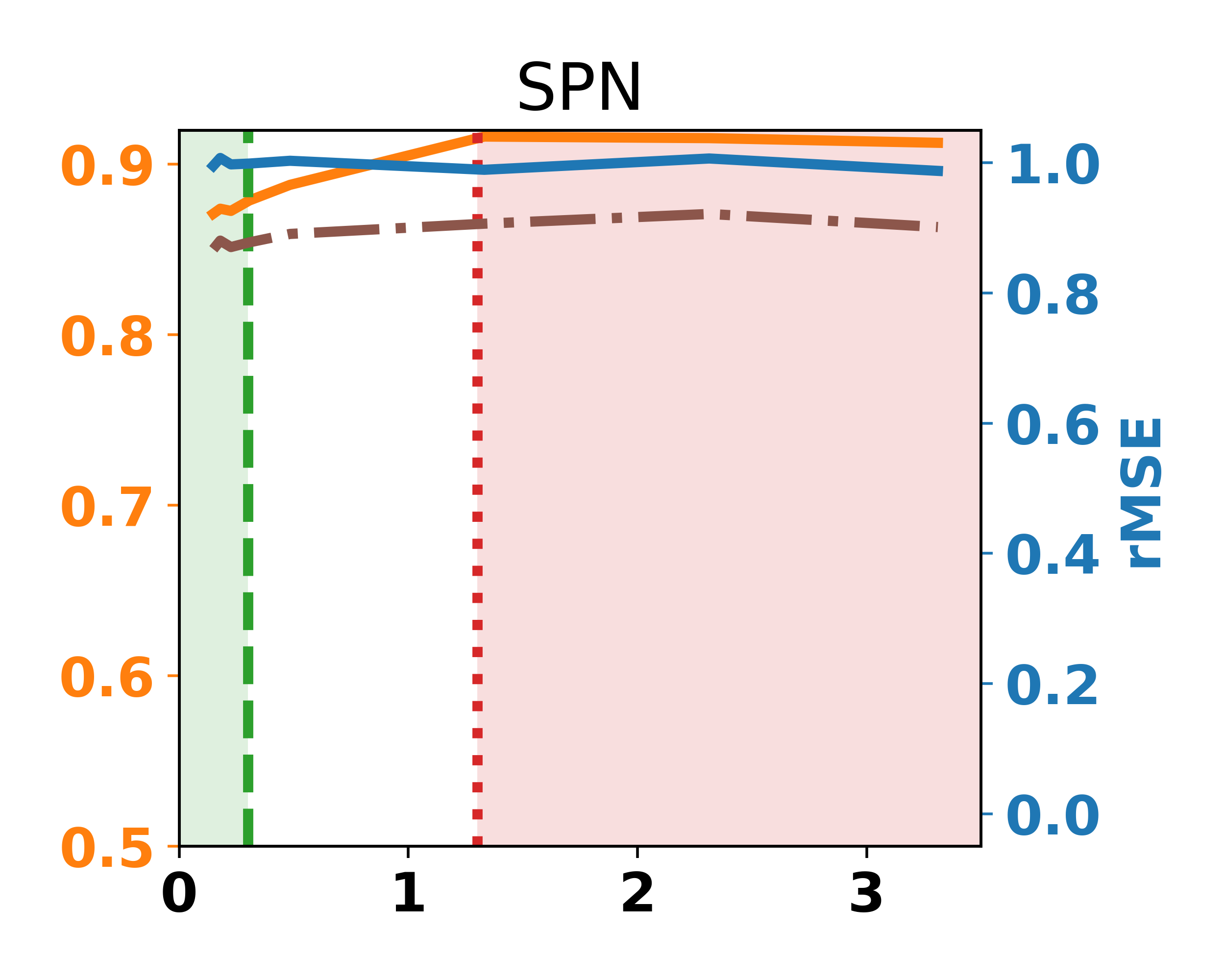}
			\caption{Reconstruction Attack}
		\end{subfigure}
		
		\begin{subfigure}{0.99\linewidth}
			\centering
			\includegraphics[width=0.24\linewidth]{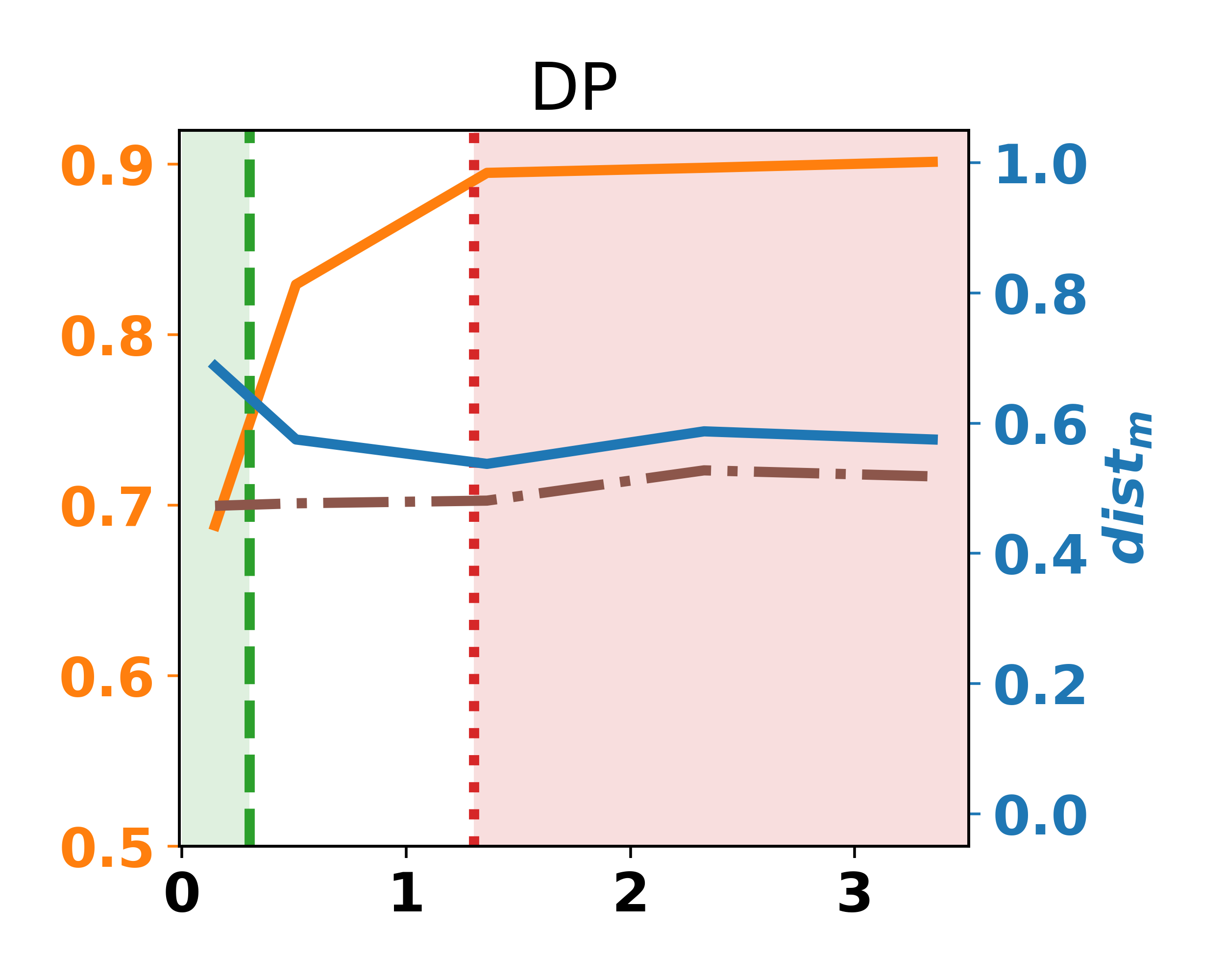}
			\includegraphics[width=0.24\linewidth]{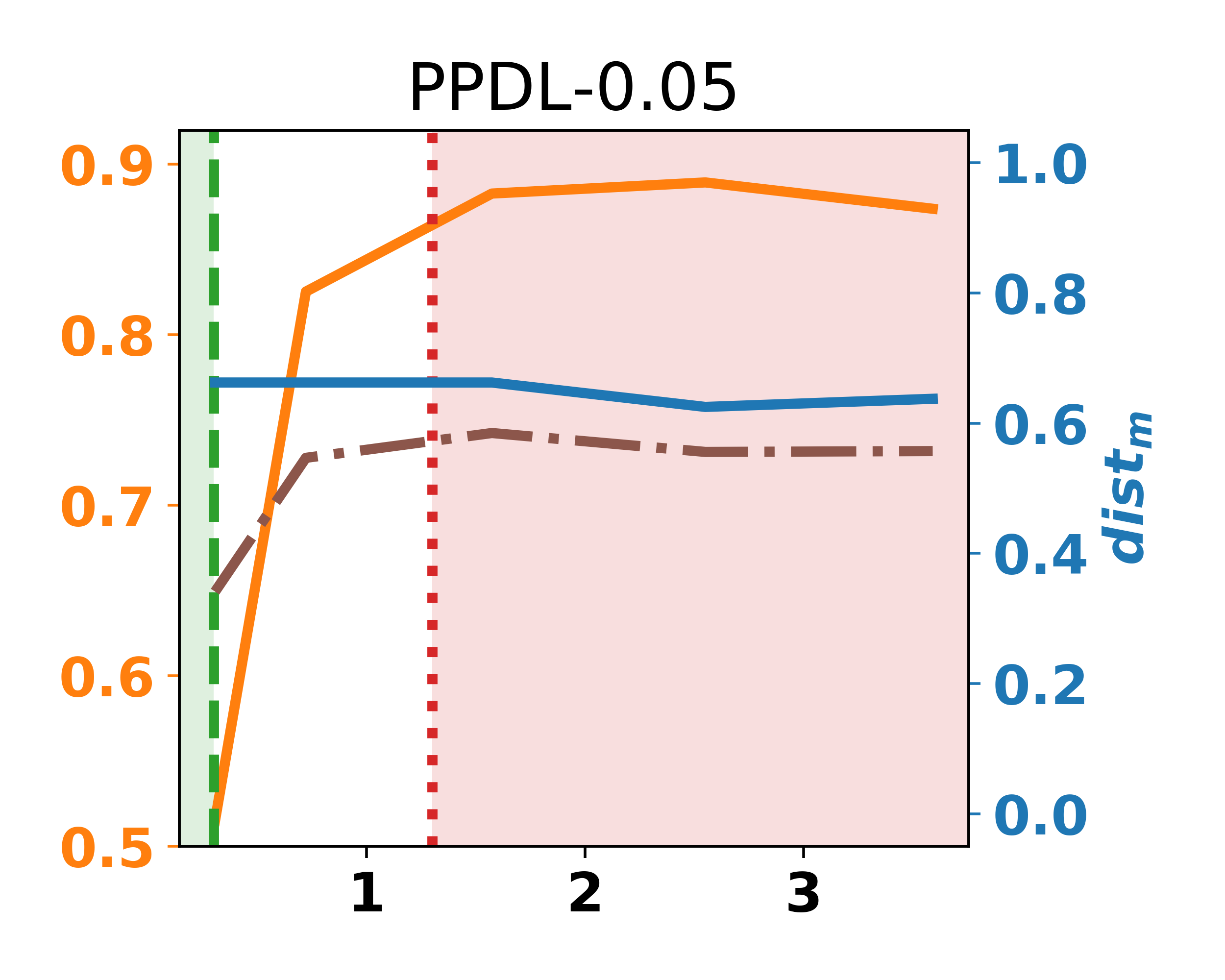}
			\includegraphics[width=0.24\linewidth]{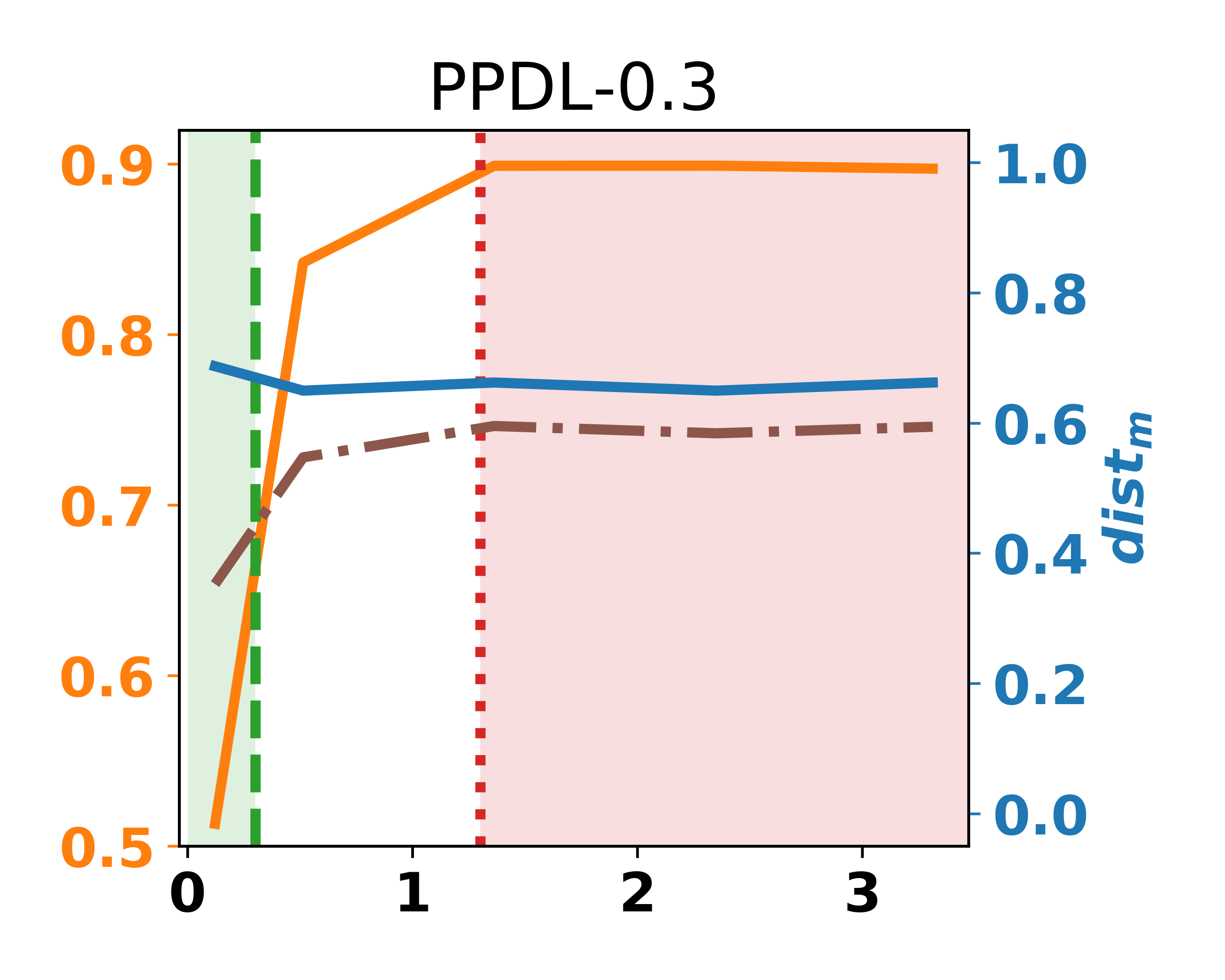}
			\includegraphics[width=0.24\linewidth]{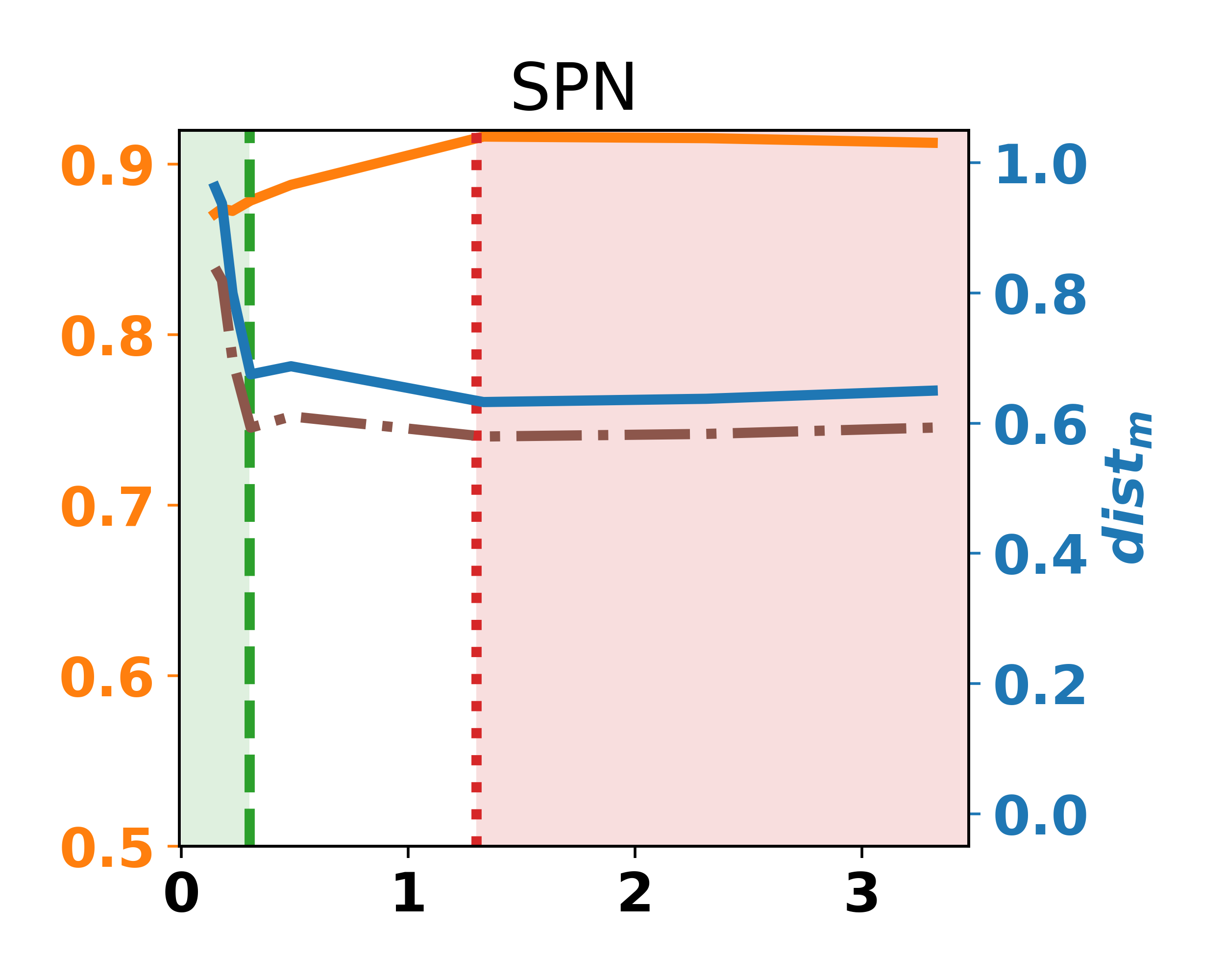}
			\caption{Membership Attack}
		\end{subfigure}
		
		\begin{subfigure}{0.99\linewidth}
			\centering
			\includegraphics[width=0.24\linewidth]{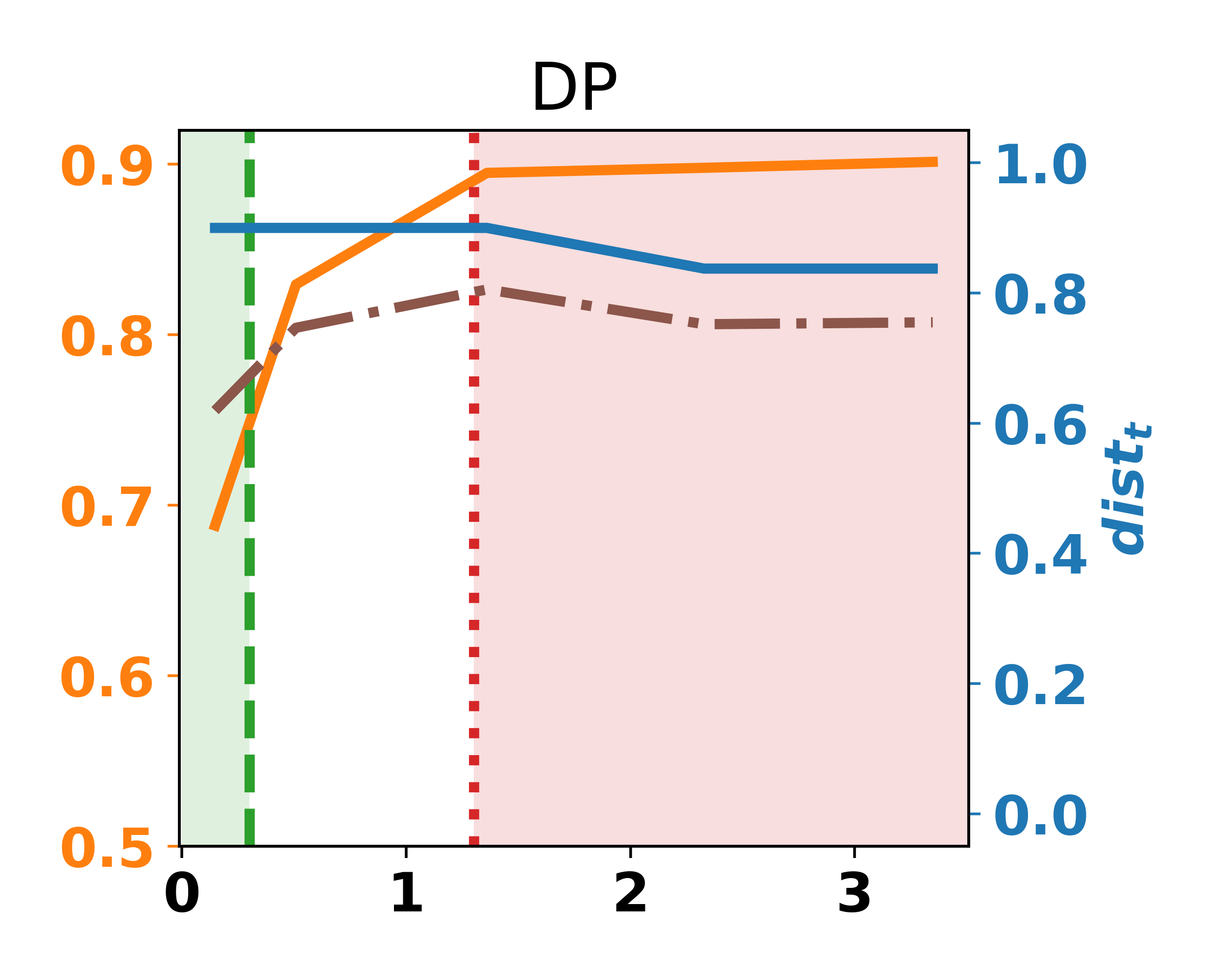}
			\includegraphics[width=0.24\linewidth]{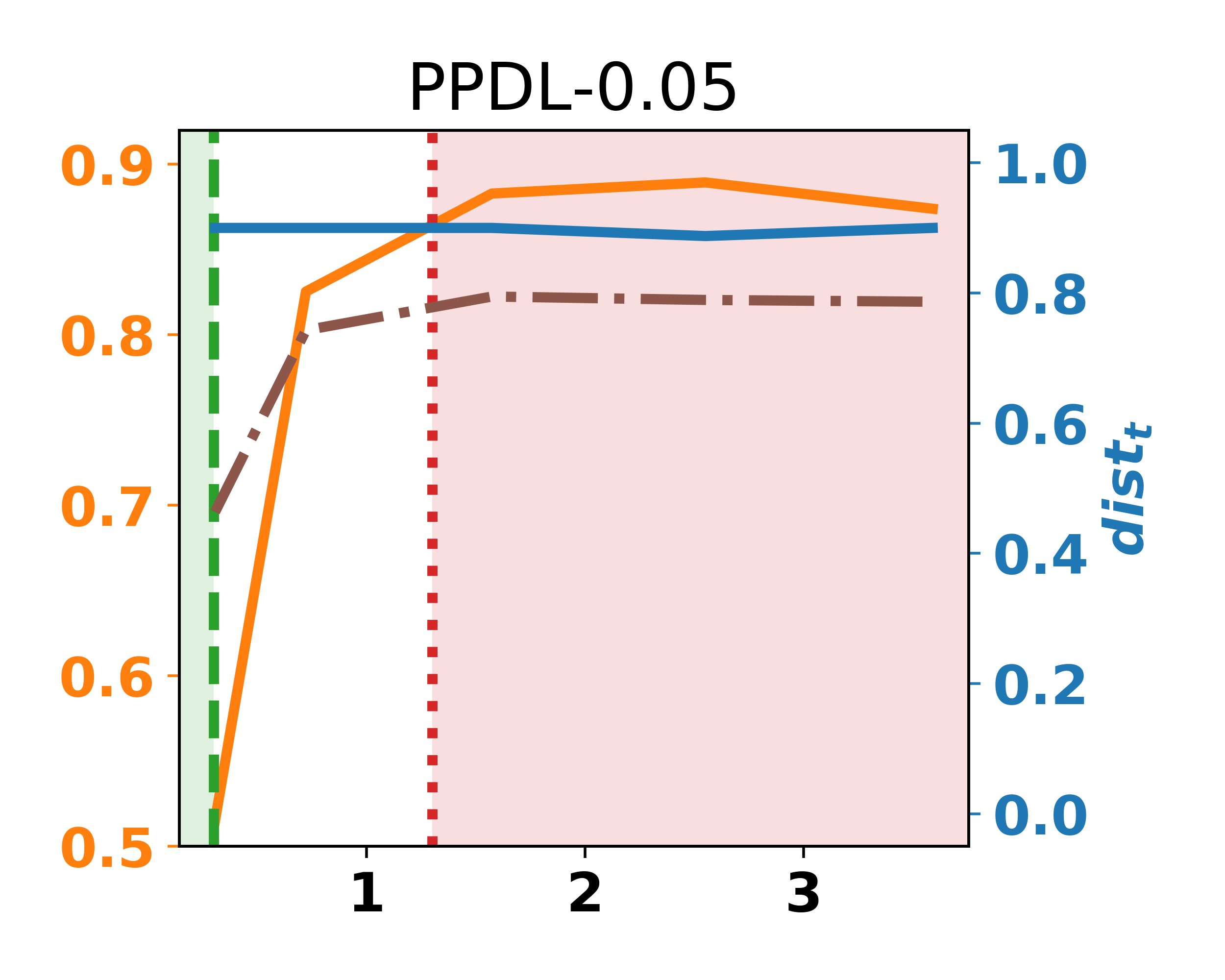}
			\includegraphics[width=0.24\linewidth]{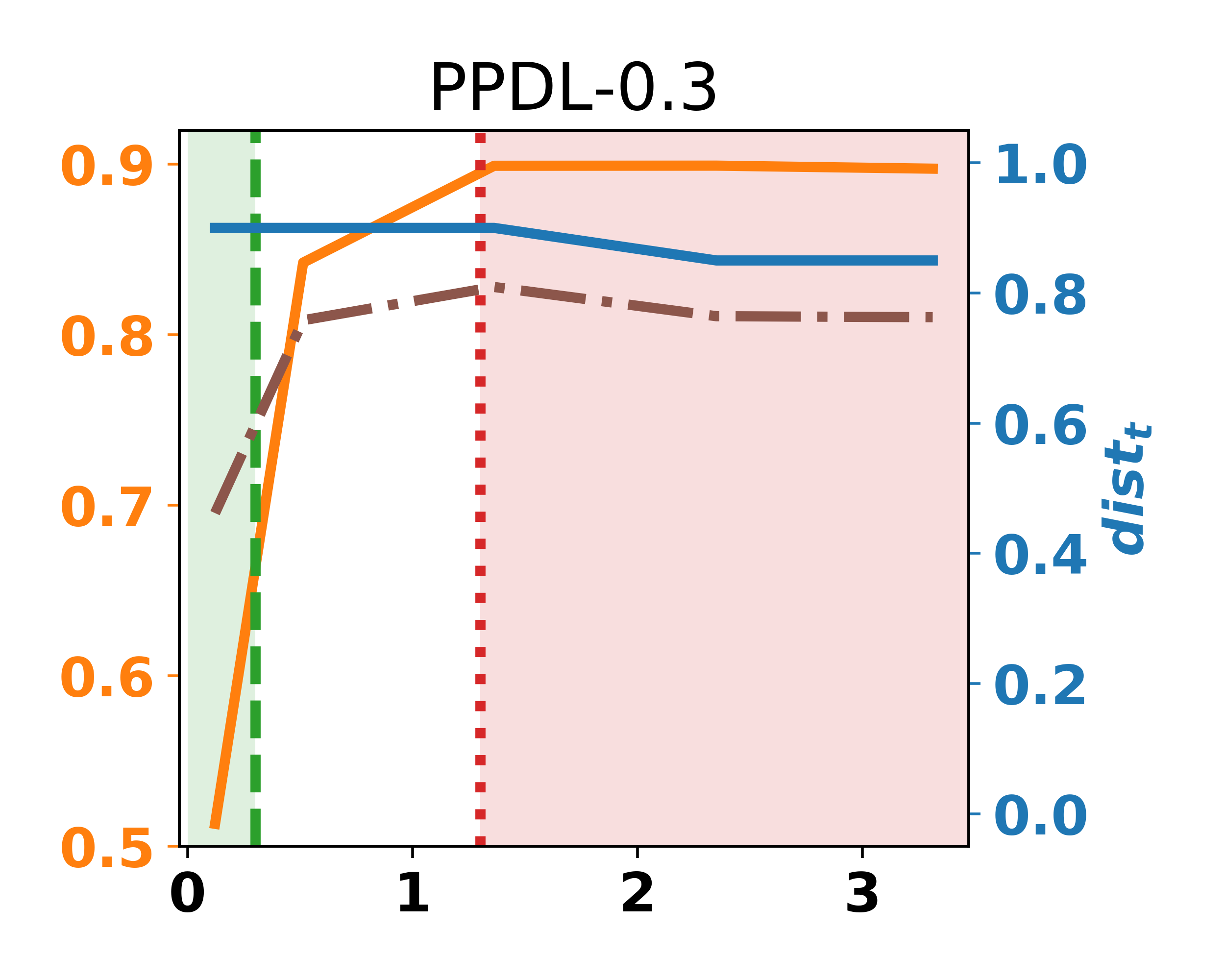}
			\includegraphics[width=0.24\linewidth]{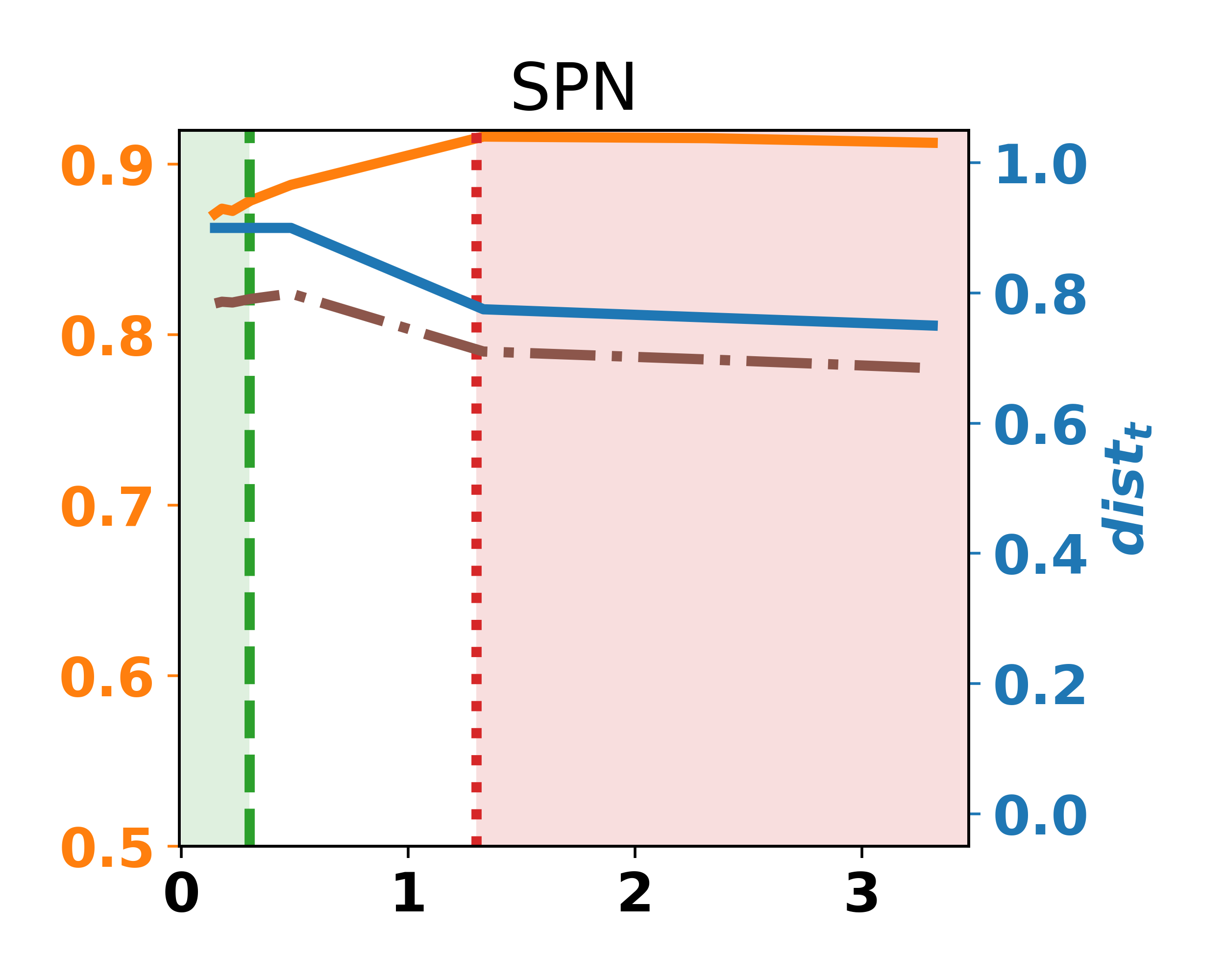}
			\caption{Tracing Attack}
		\end{subfigure}
		
		\caption{Attack with Batch Size 4}
		\label{fig:ppc-svhn-bs4}
		%\vspace{-0.22cm}
	\end{figure}
	
	\begin{figure}[H]	
		%	\begin{subfigure}{0.95\linewidth}
		\centering
		\begin{subfigure}{0.99\linewidth}
			\centering
			\includegraphics[scale=0.7]{imgs/legends/legend_ppc_horizontal.png}
			\\
			\includegraphics[width=0.24\linewidth]{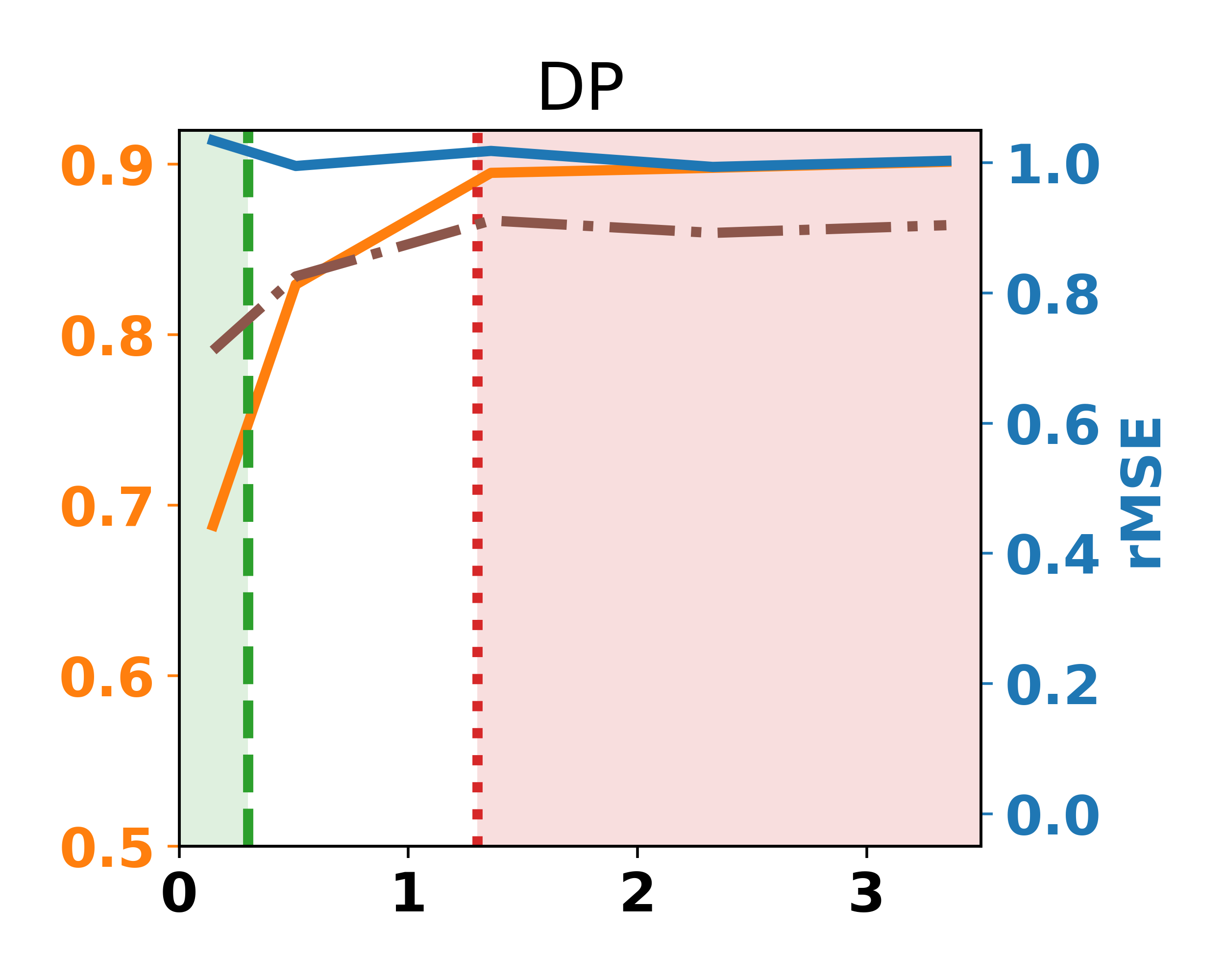}
			\includegraphics[width=0.24\linewidth]{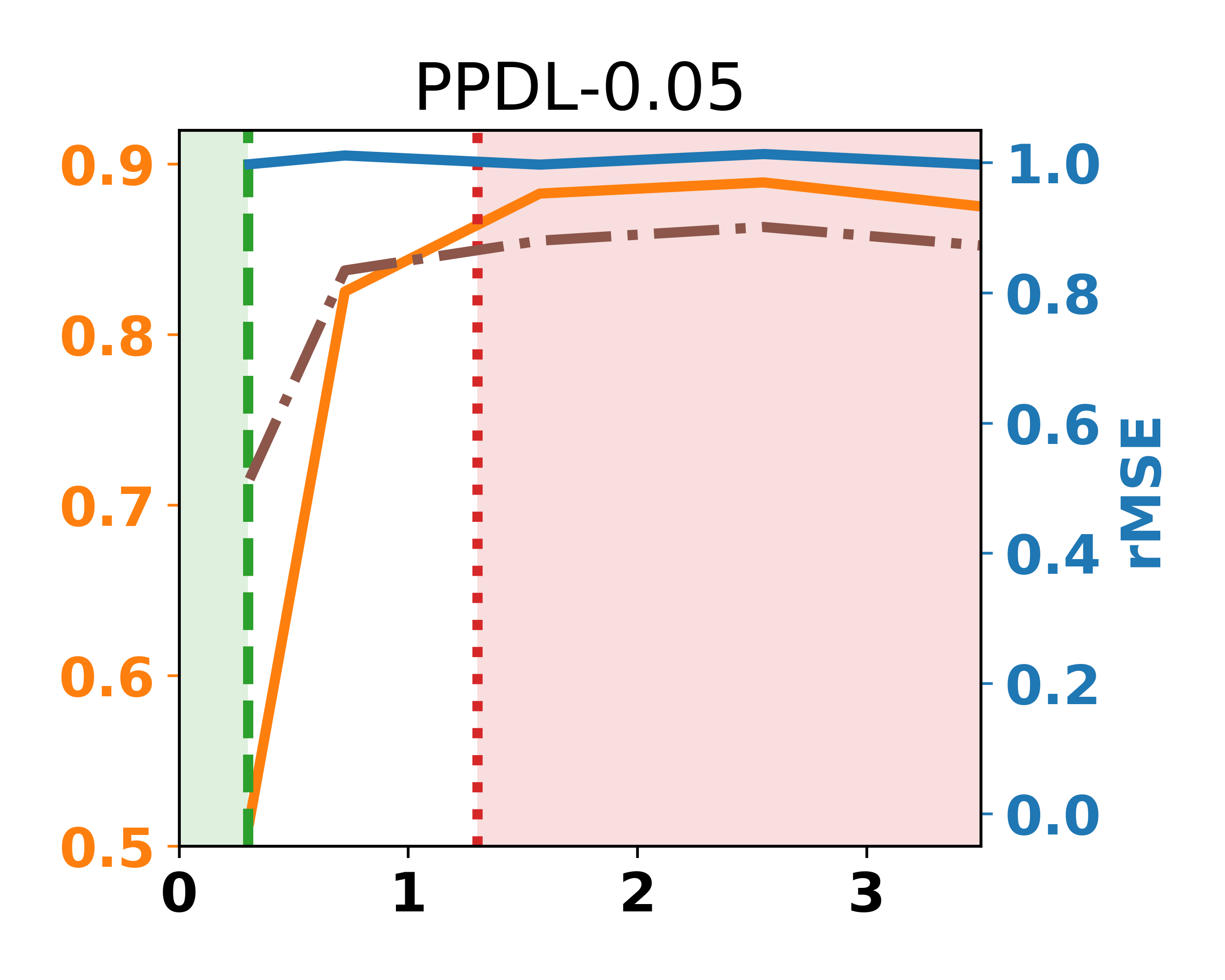}
			\includegraphics[width=0.24\linewidth]{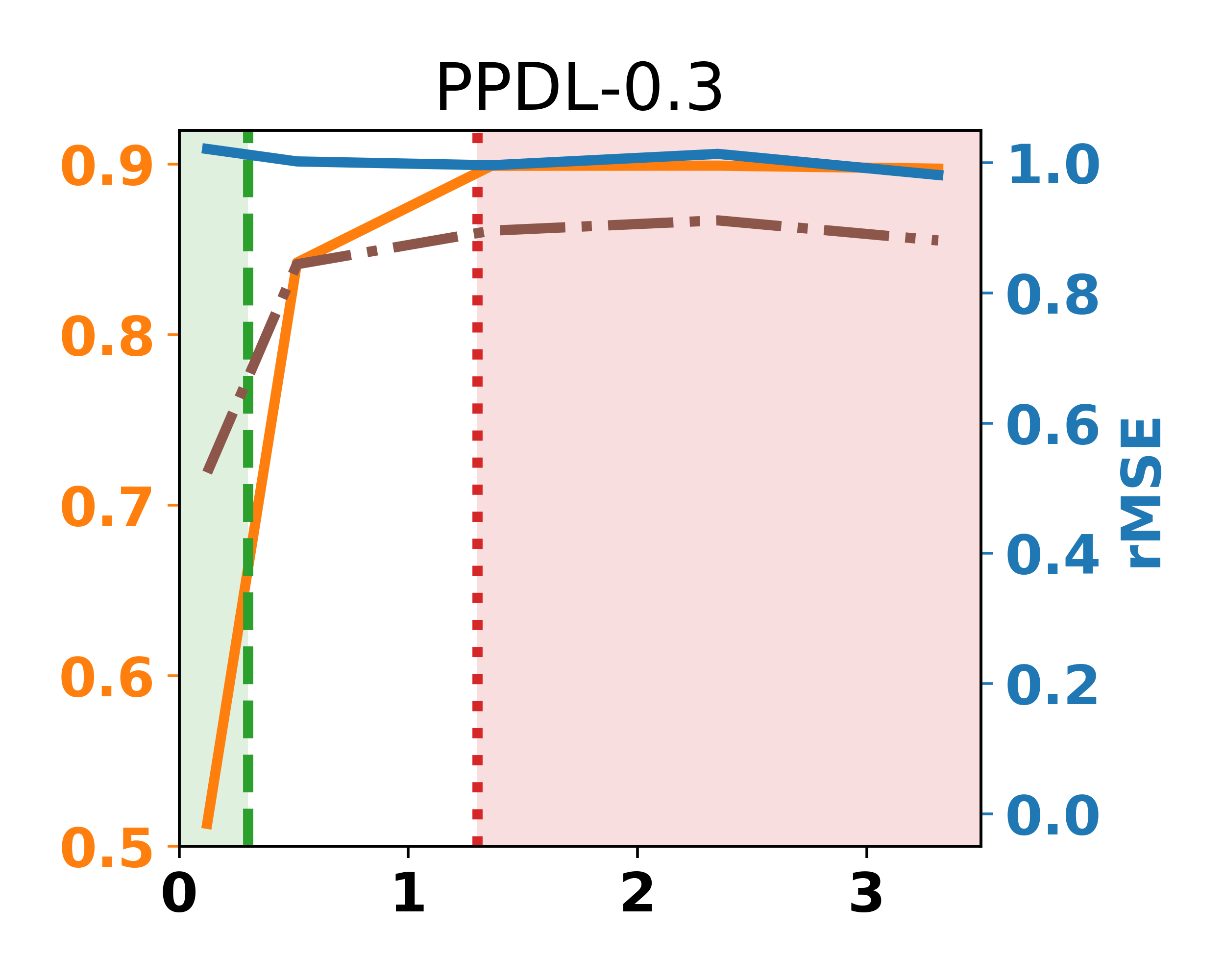}
			\includegraphics[width=0.24\linewidth]{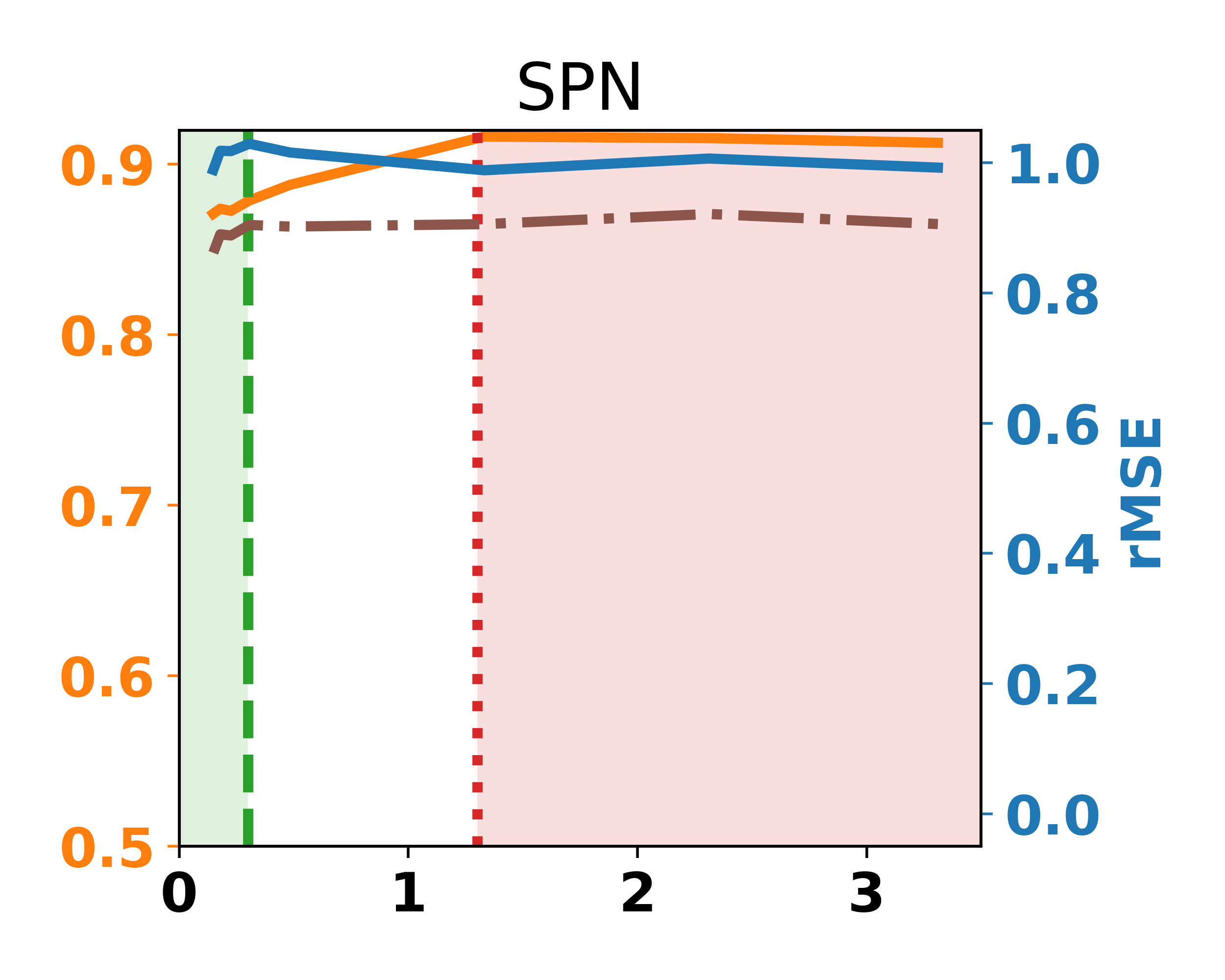}
			\caption{Reconstruction Attack}
		\end{subfigure}
		
		\begin{subfigure}{0.99\linewidth}
			\centering
			\includegraphics[width=0.24\linewidth]{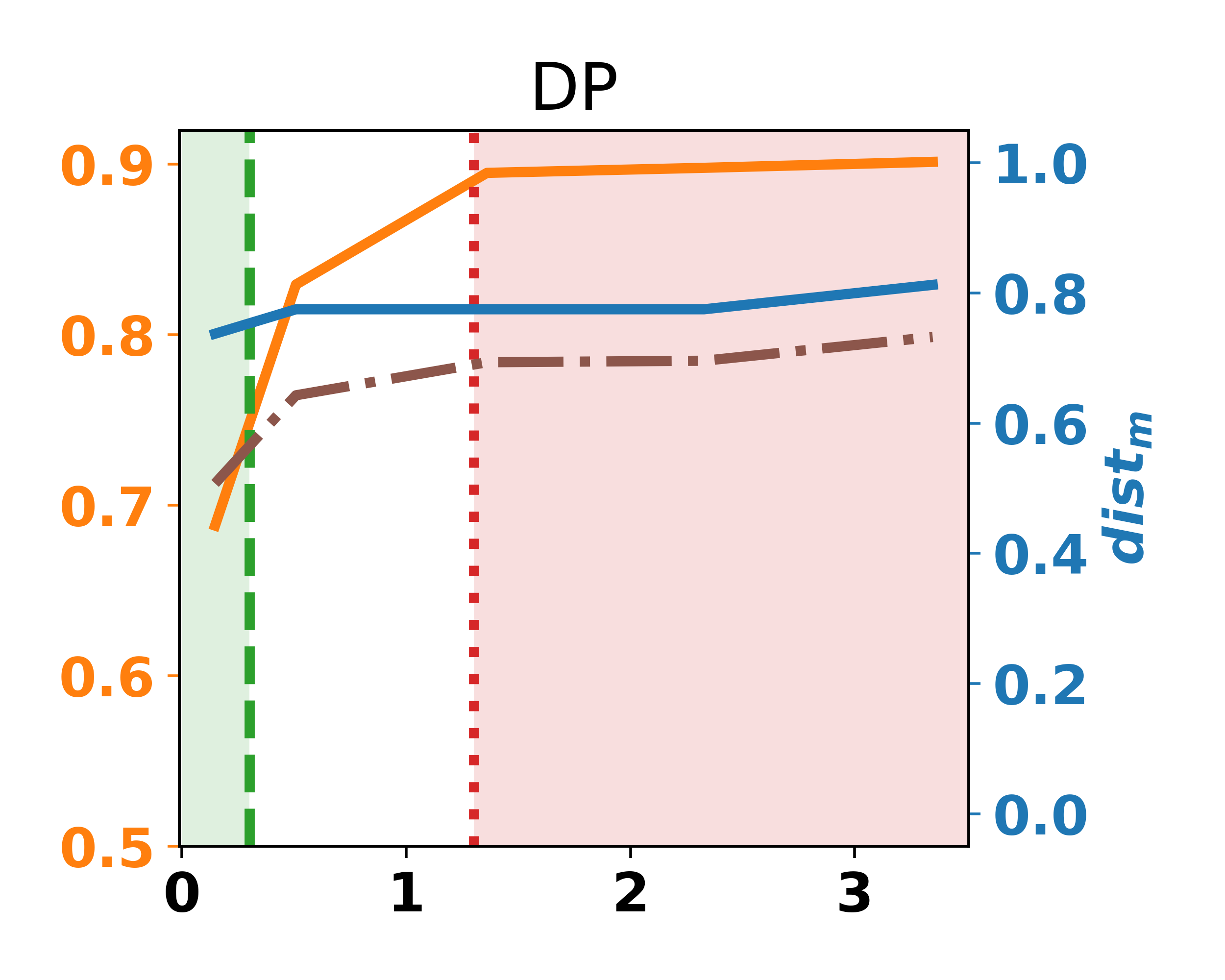}
			\includegraphics[width=0.24\linewidth]{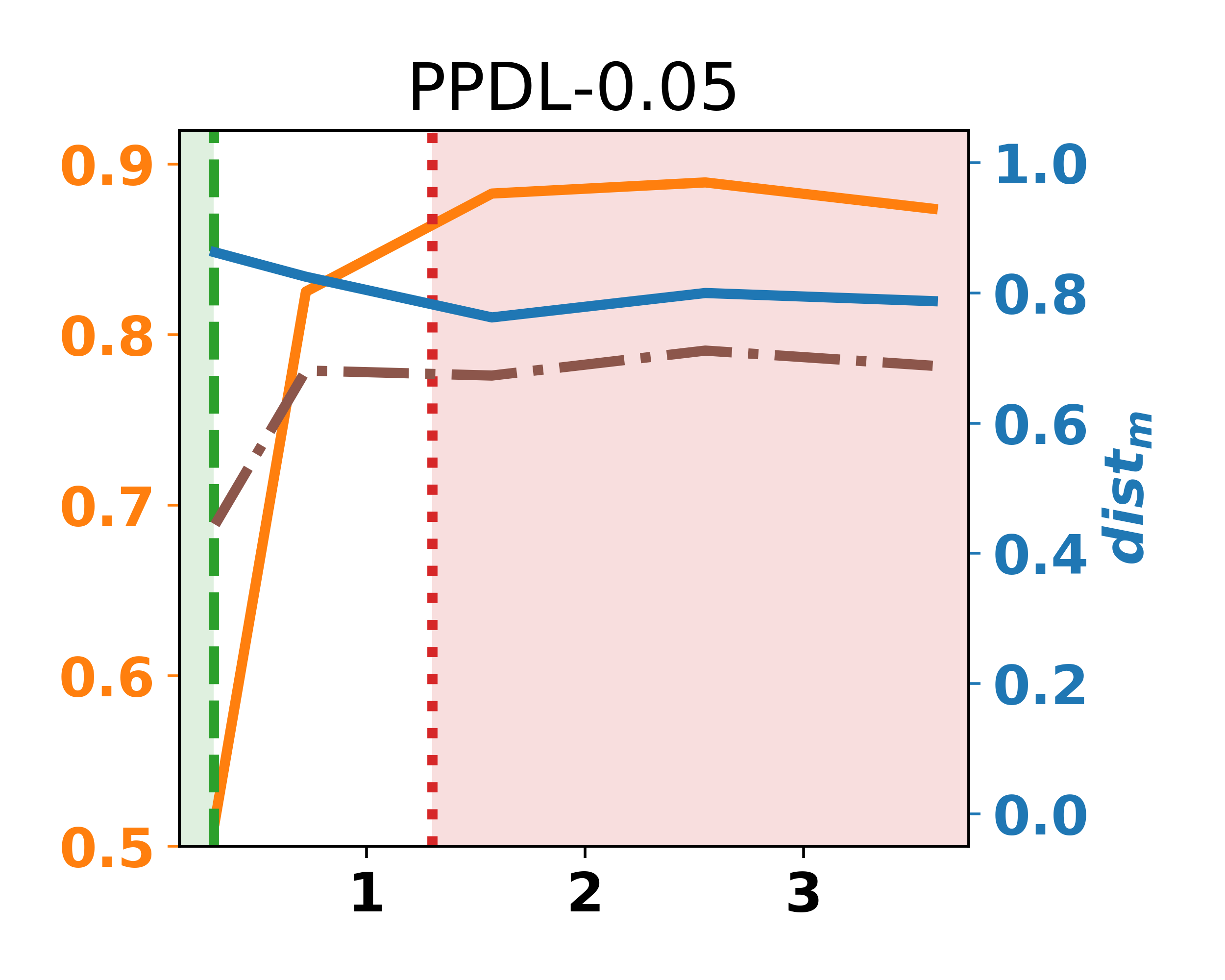}
			\includegraphics[width=0.24\linewidth]{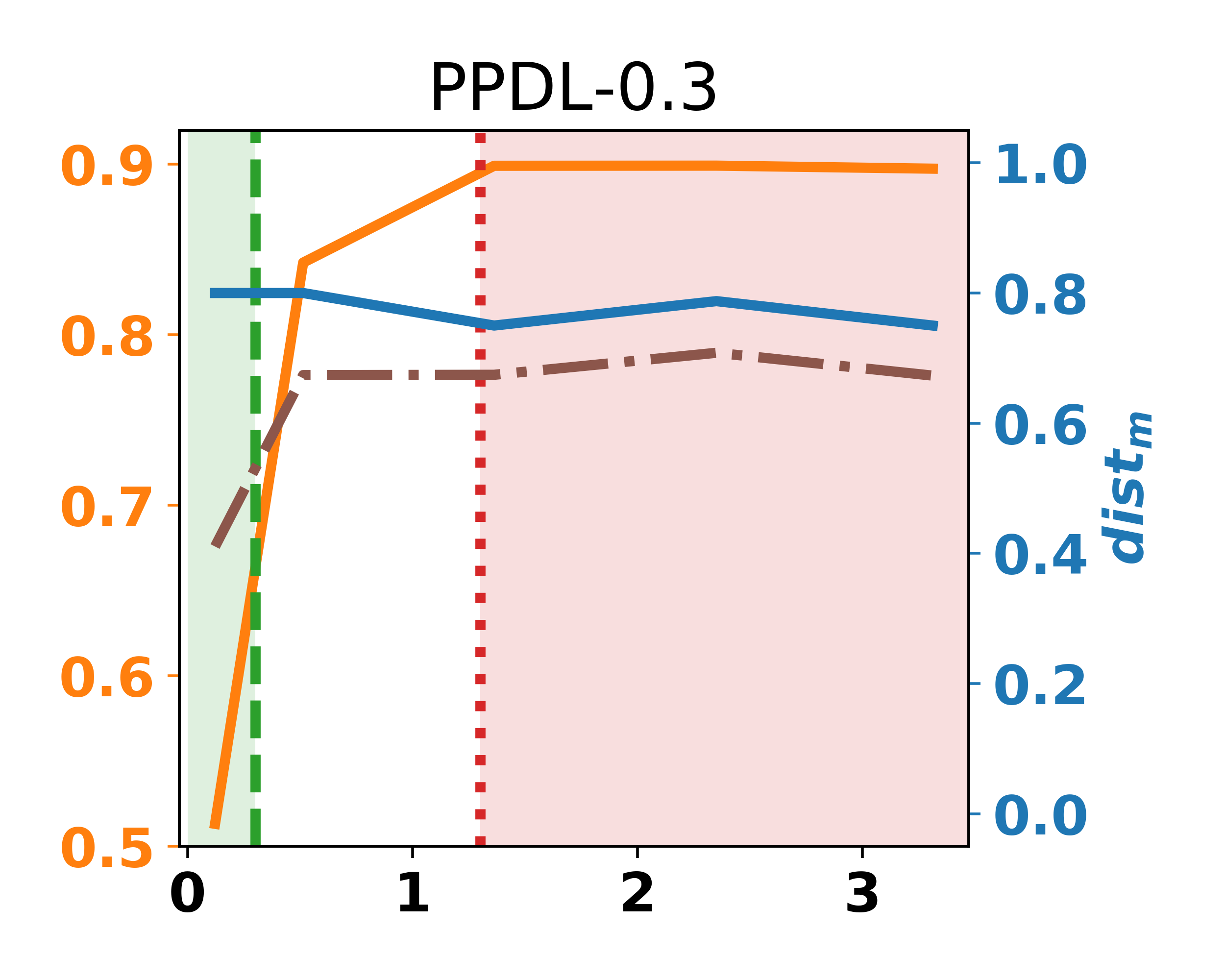}
			\includegraphics[width=0.24\linewidth]{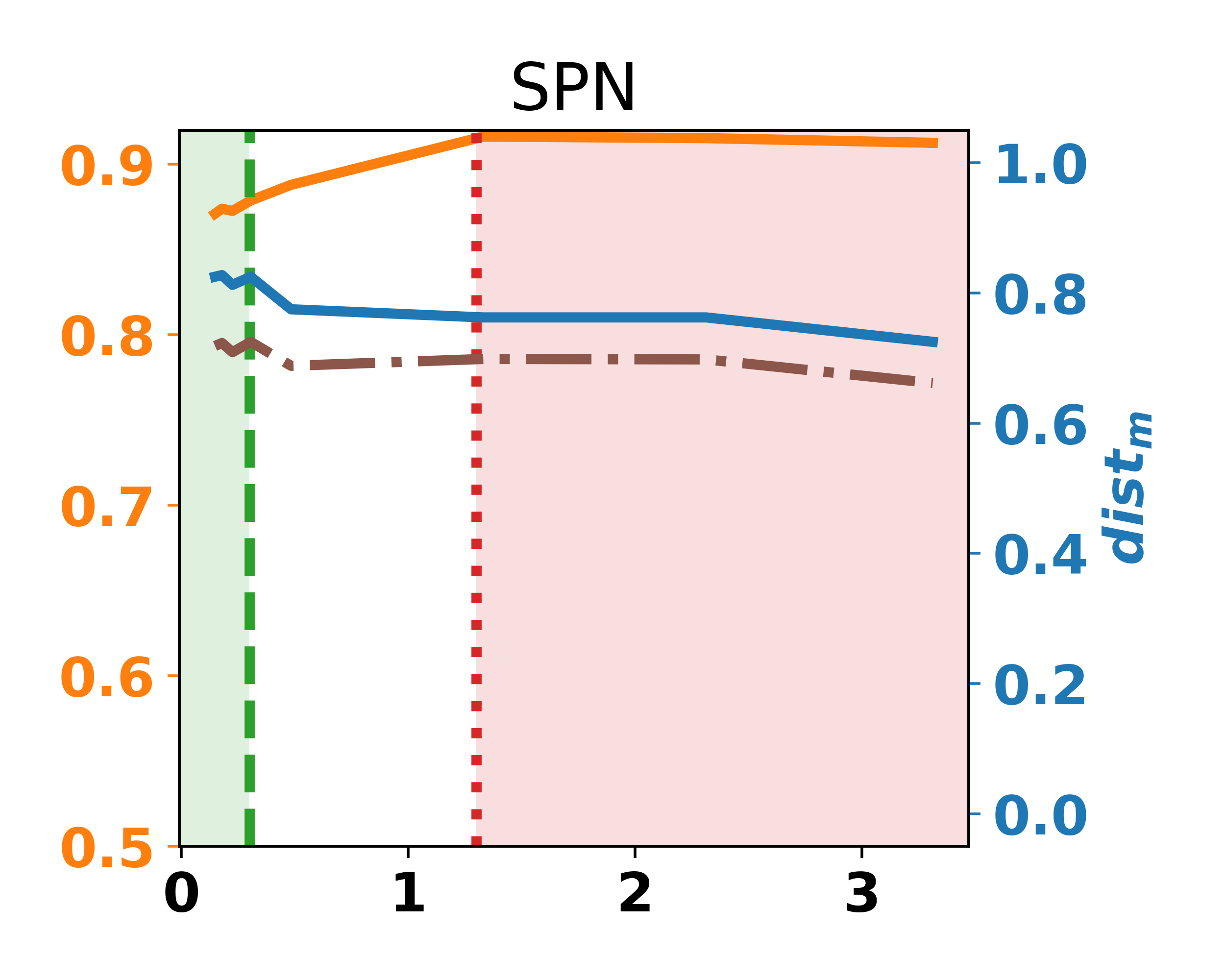}
			\caption{Membership Attack}
		\end{subfigure}
		
		\begin{subfigure}{0.99\linewidth}
			\centering
			\includegraphics[width=0.24\linewidth]{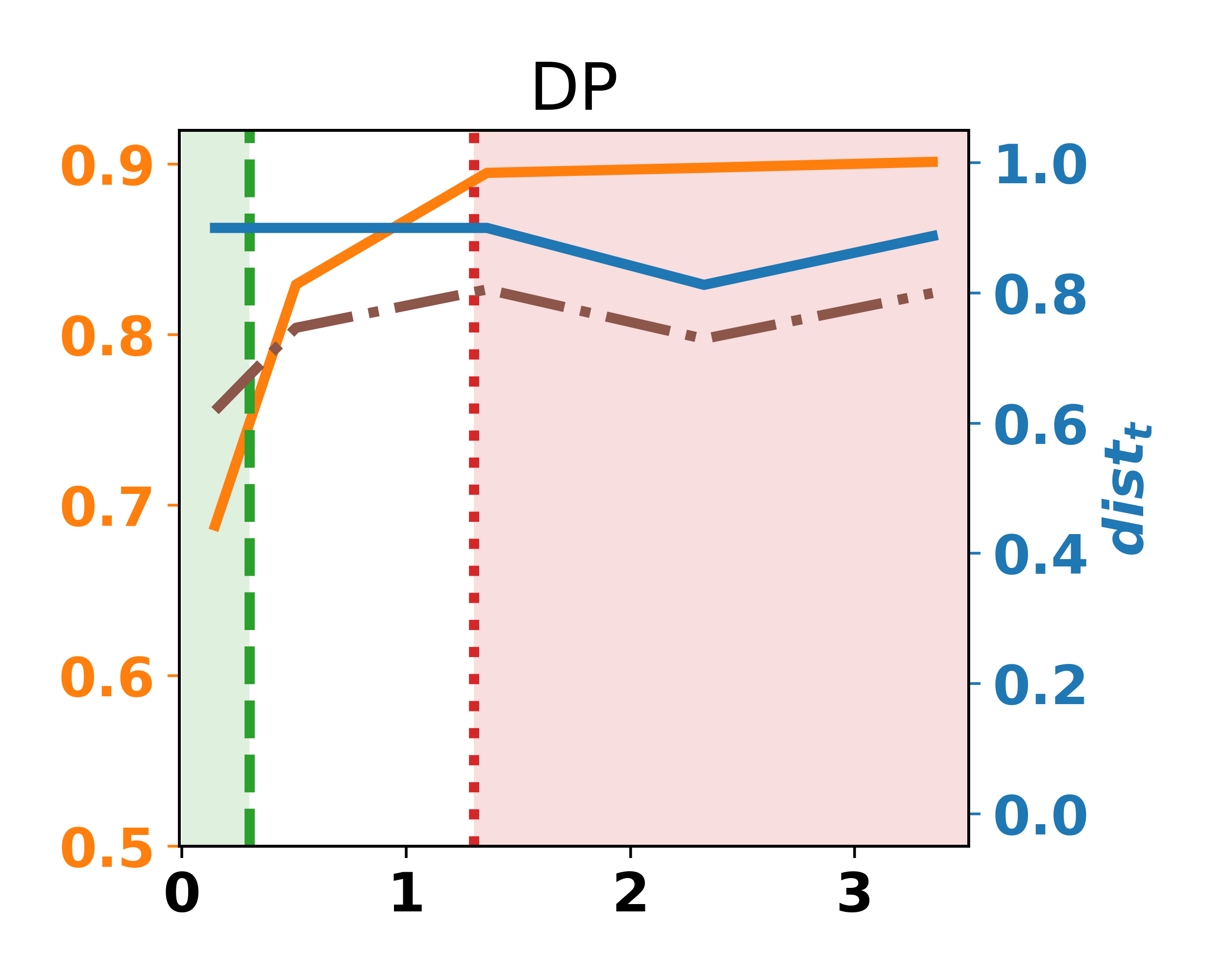}
			\includegraphics[width=0.24\linewidth]{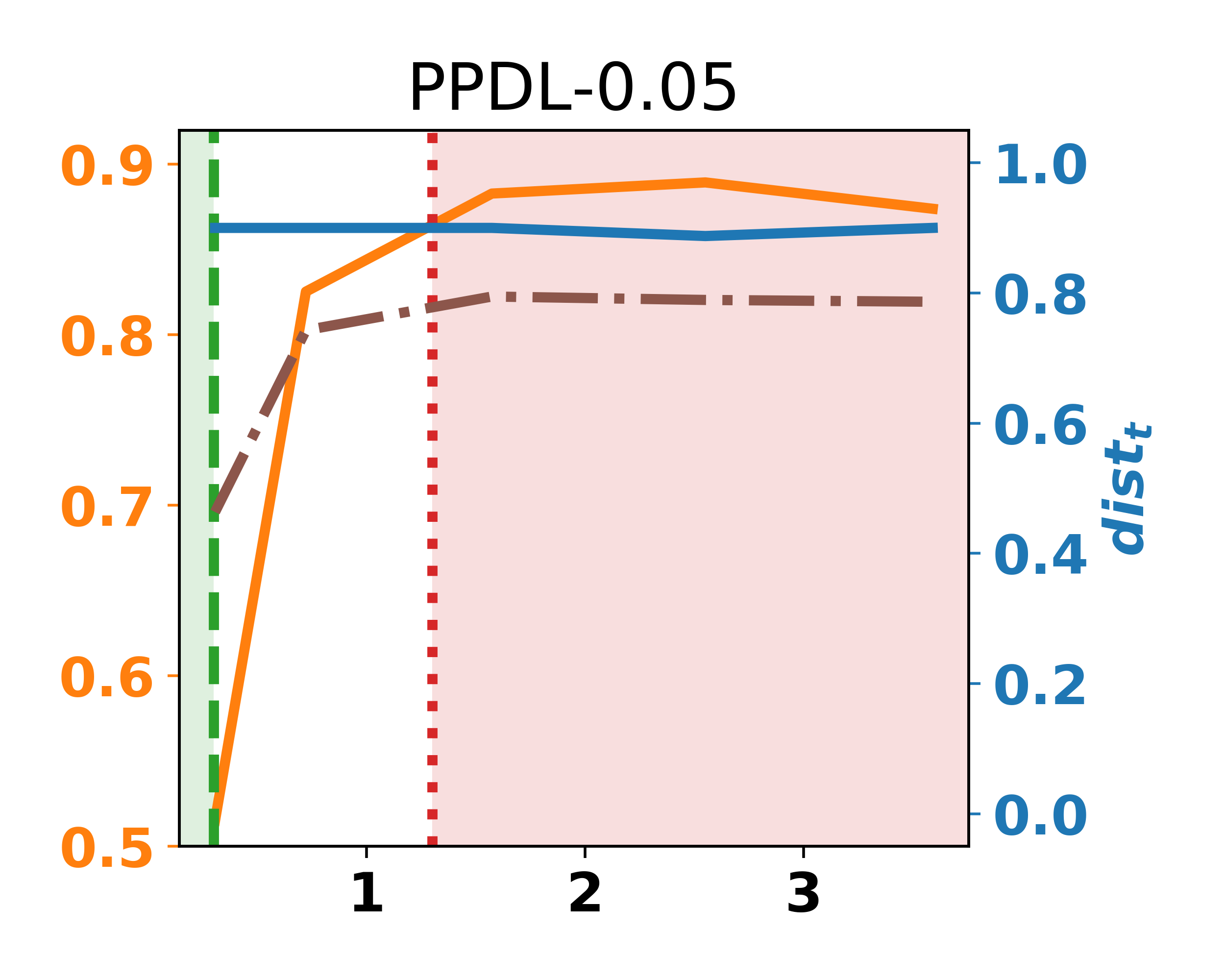}
			\includegraphics[width=0.24\linewidth]{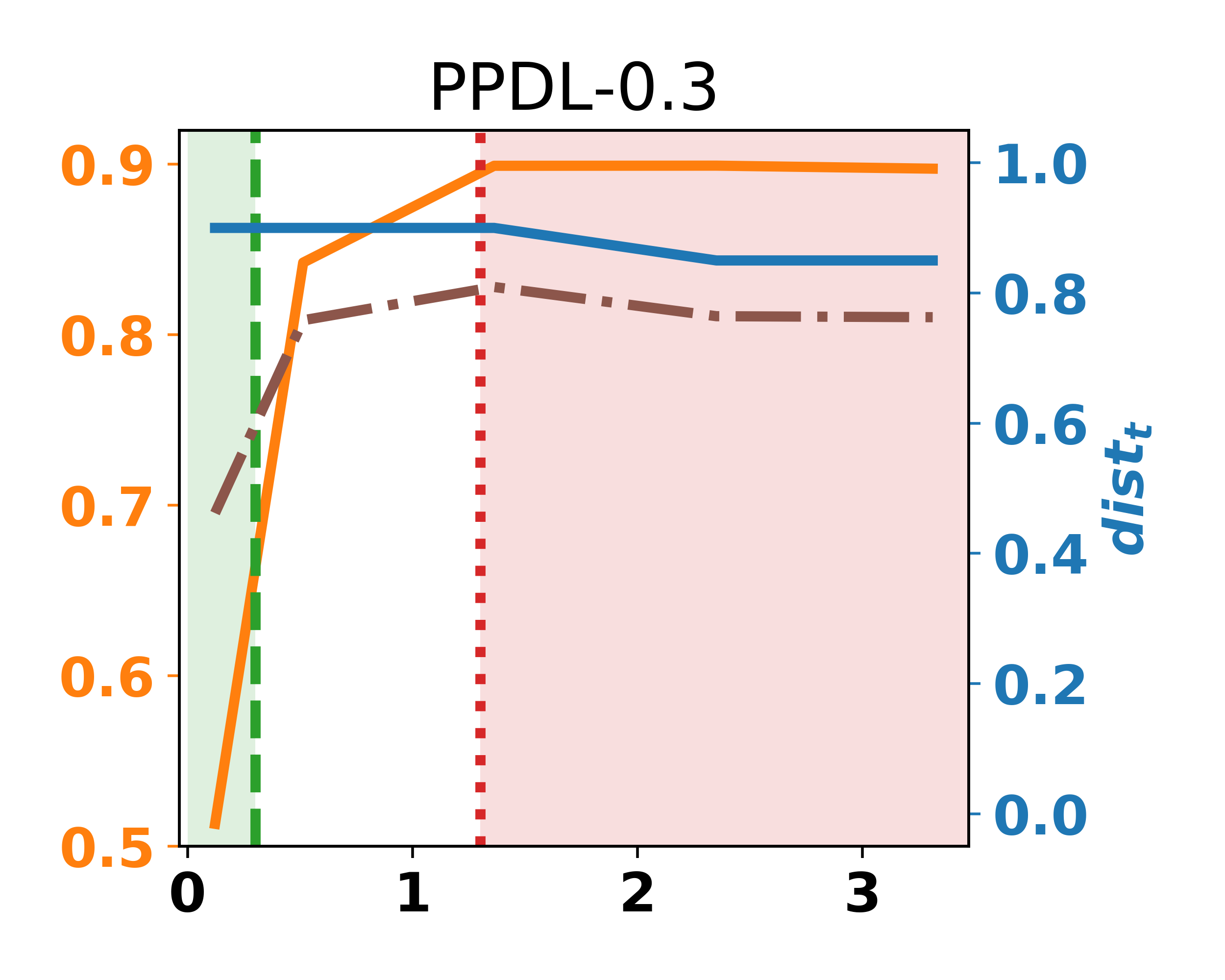}
			\includegraphics[width=0.24\linewidth]{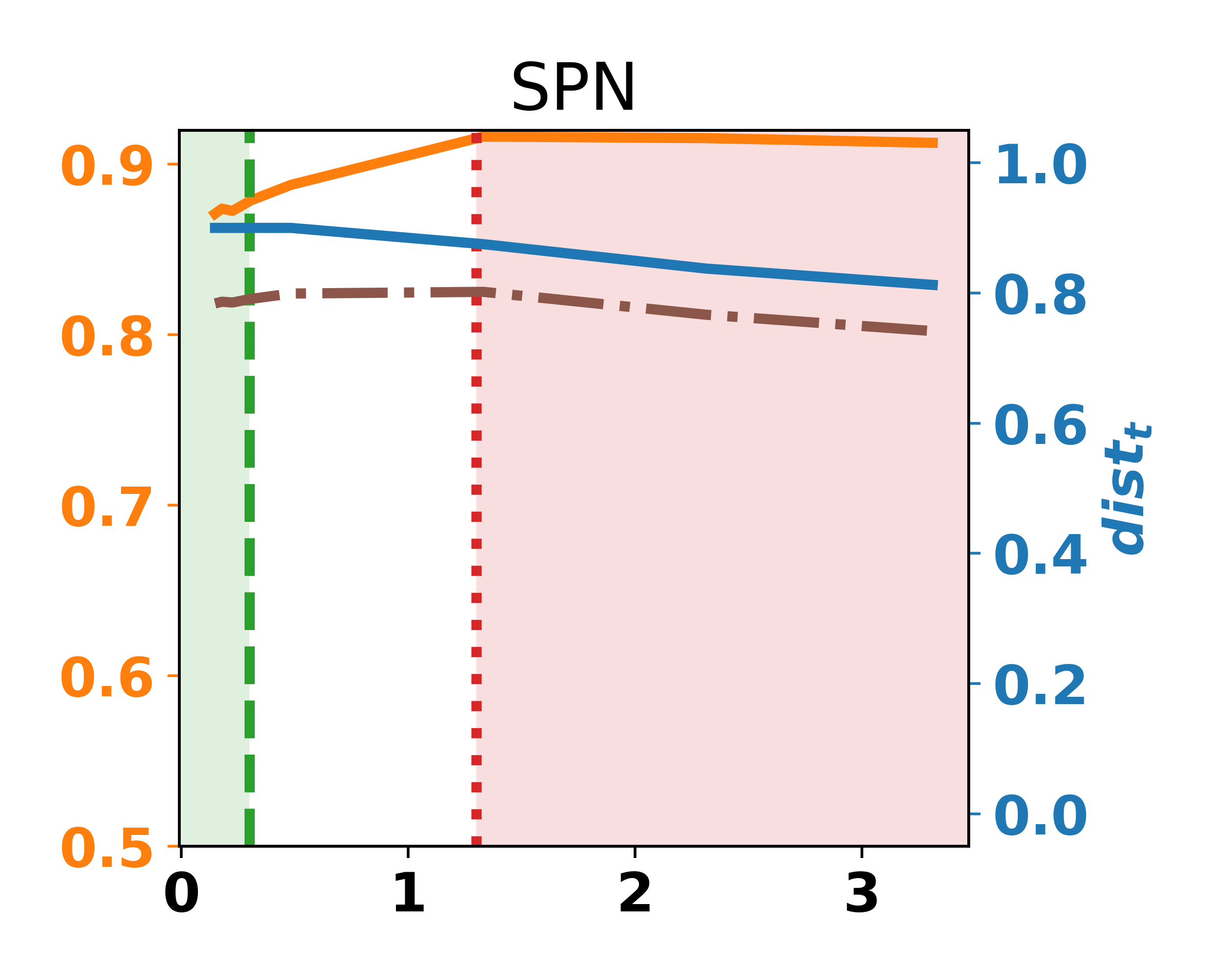}
			\caption{Tracing Attack}
		\end{subfigure}
		
		\caption{Attack with Batch Size 8}
		\label{fig:ppc-svhn-bs8}
		%\vspace{-0.22cm}
	\end{figure}

	\subsubsection{Calibrated Averaged Performance (CAP)}
	
	\begin{table}[H]
		\centering
		\adjustbox{max width=\textwidth}{
			\begin{tabular}{l|c|c|c|c|c|c|c|c|c}
				\toprule
				& \multicolumn{3}{c|}{Reconstruction} & \multicolumn{3}{c|}{Membership} & \multicolumn{3}{c}{Tracing} \\
				\midrule
				BS & 1 & 4 & 8 & 1 & 4 & 8 & 1 & 4 & 8 \\
				\midrule
				DP \cite{DLDP_Abadi16} & 0.77 & 0.84 & 0.85 & 0.00 & 0.50 & 0.65 & 0.68 & 0.72 & 0.72 \\
				PPDL-0.05 \cite{PPDL/shokri2015} & 0.79 & 0.80 & 0.80 & 0.00 & 0.52 & 0.64 &\textbf{0.70} & 0.70 & 0.70 \\
				PPDL-0.3 \cite{PPDL/shokri2015} & 0.72 & 0.82 & 0.81 & 0.00 & 0.54 & 0.63 & 0.68 & 0.68 & 0.68 \\
				SPN (ours) & \textbf{0.88} & \textbf{0.89} & \textbf{0.90} &\textbf{0.60} & \textbf{0.66} & \textbf{0.70} & 0.66 & \textbf{0.79} & \textbf{0.79} \\
				\bottomrule
			\end{tabular}
		}
		\caption{CAP performance with different batch size on SVHN for reconstruction, membership and tracing attack. Higher better. BS = Attack Batch Size.}
		\label{tab:CAP-svhn-supp}
	\end{table}
	
	\subsubsection{Reconstructed Images}
	
	\begin{figure}[H]
		\centering
		
		\begin{subfigure}{0.32\linewidth}
			\centering
			\includegraphics[scale=0.8]{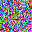}
			\hspace{1pt}
			\includegraphics[scale=0.8]{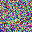} % + svhn
			\hspace{1pt}
			\includegraphics[scale=0.8]{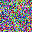}
			\hspace{1pt}
			\includegraphics[scale=0.8]{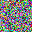}
			\\
			\vspace{2pt}
			\includegraphics[scale=0.8]{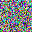}
			\hspace{1pt}
			\includegraphics[scale=0.8]{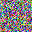} % + svhn
			\hspace{1pt}
			\includegraphics[scale=0.8]{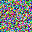}
			\hspace{1pt}
			\includegraphics[scale=0.8]{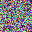}
			\\
			\vspace{2pt}
			\includegraphics[scale=0.8]{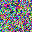}
			\hspace{1pt}
			\includegraphics[scale=0.8]{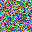} % + svhn
			\hspace{1pt}
			\includegraphics[scale=0.8]{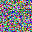}
			\hspace{1pt}
			\includegraphics[scale=0.8]{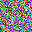}
			\\
			\vspace{2pt}
			\includegraphics[scale=0.8]{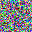}
			\hspace{1pt}
			\includegraphics[scale=0.8]{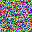} % + svhn
			\hspace{1pt}
			\includegraphics[scale=0.8]{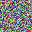}
			\hspace{1pt}
			\includegraphics[scale=0.8]{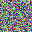}
			\caption{$0.99 (0.39)$ \label{fig:recon-images-green-svhn-supp}}
		\end{subfigure}
		\hfill
		\begin{subfigure}{0.32\linewidth}
			\centering
			\includegraphics[scale=0.8]{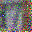}
			\hspace{1pt}
			\includegraphics[scale=0.8]{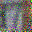} % + svhn
			\hspace{1pt}
			\includegraphics[scale=0.8]{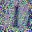}
			\hspace{1pt}
			\includegraphics[scale=0.8]{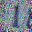}
			\\
			\vspace{2pt}
			\includegraphics[scale=0.8]{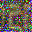}
			\hspace{1pt}
			\includegraphics[scale=0.8]{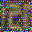} % + svhn
			\hspace{1pt}
			\includegraphics[scale=0.8]{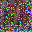}
			\hspace{1pt}
			\includegraphics[scale=0.8]{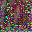}
			\\
			\vspace{2pt}
			\includegraphics[scale=0.8]{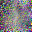}
			\hspace{1pt}
			\includegraphics[scale=0.8]{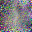} % + svhn
			\hspace{1pt}
			\includegraphics[scale=0.8]{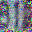}
			\hspace{1pt}
			\includegraphics[scale=0.8]{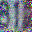}
			\\
			\vspace{2pt}
			\includegraphics[scale=0.8]{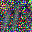}
			\hspace{1pt}
			\includegraphics[scale=0.8]{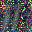} % + svhn
			\hspace{1pt}
			\includegraphics[scale=0.8]{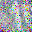}
			\hspace{1pt}
			\includegraphics[scale=0.8]{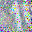}
			\caption{$0.93 (20.86)$ \label{fig:recon-images-white-svhn-supp}}%; 1.10 (1.30)}
		\end{subfigure}
		\hfill
		\begin{subfigure}{0.32\linewidth}
			\centering
			\includegraphics[scale=0.8]{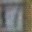}
			\hspace{1pt}
			\includegraphics[scale=0.8]{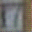} % + svhn
			\hspace{1pt}
			\includegraphics[scale=0.8]{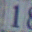}
			\hspace{1pt}
			\includegraphics[scale=0.8]{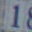}
			\\
			\vspace{2pt}
			\includegraphics[scale=0.8]{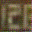}
			\hspace{1pt}
			\includegraphics[scale=0.8]{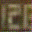} % + svhn
			\hspace{1pt}
			\includegraphics[scale=0.8]{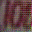}
			\hspace{1pt}
			\includegraphics[scale=0.8]{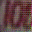}
			\\
			\vspace{2pt}
			\includegraphics[scale=0.8]{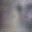}
			\hspace{1pt}
			\includegraphics[scale=0.8]{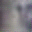} % + svhn
			\hspace{1pt}
			\includegraphics[scale=0.8]{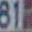}
			\hspace{1pt}
			\includegraphics[scale=0.8]{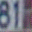}
			\\
			\vspace{2pt}
			\includegraphics[scale=0.8]{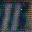}
			\hspace{1pt}
			\includegraphics[scale=0.8]{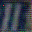} % + svhn
			\hspace{1pt}
			\includegraphics[scale=0.8]{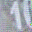}
			\hspace{1pt}
			\includegraphics[scale=0.8]{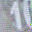}
			\caption{$0.35 (1251.48)$ \label{fig:recon-images-red-svhn-supp}} %0.94 (1.59)}
		\end{subfigure}
		
		%\begin{subfigure}{0.2\linewidth}
		%	\centering
		%	\includegraphics[scale=0.8]{imgs/recimgs/0_0.0001_0.0151.png}
		%	\hspace{3pt}
		%	\includegraphics[scale=0.8]{imgs/recimgs/391_0.0001_0.0116.png}
		%	\caption{0.02 (3.34); 0.01 (3.28)}
		%\end{subfigure}
		\caption{Reconstructed images from different region mentioned in the main paper. \textbf{(a)} Green region \textbf{(b)} White region \textbf{(c)} Red region. 
			Values inside bracket are mean of $\frac{||B_I||}{||E_B||}$ and values outside are mean of rMSE of reconstructed w.r.t. original images.
		}
		\label{fig:recon-images-svhn-supp}
	\end{figure}
	
	\newpage
	\subsection{Summary of Calibrated Averaged Performance}
		
	\begin{figure}[H]
		%\centering
		%\begin{minipage}[t]{0.25\textwidth}
		\centering
		\includegraphics[scale=0.48]{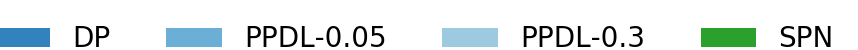}	
		\\
		\begin{subfigure}{0.99\linewidth}
			\includegraphics[width=0.24\linewidth]{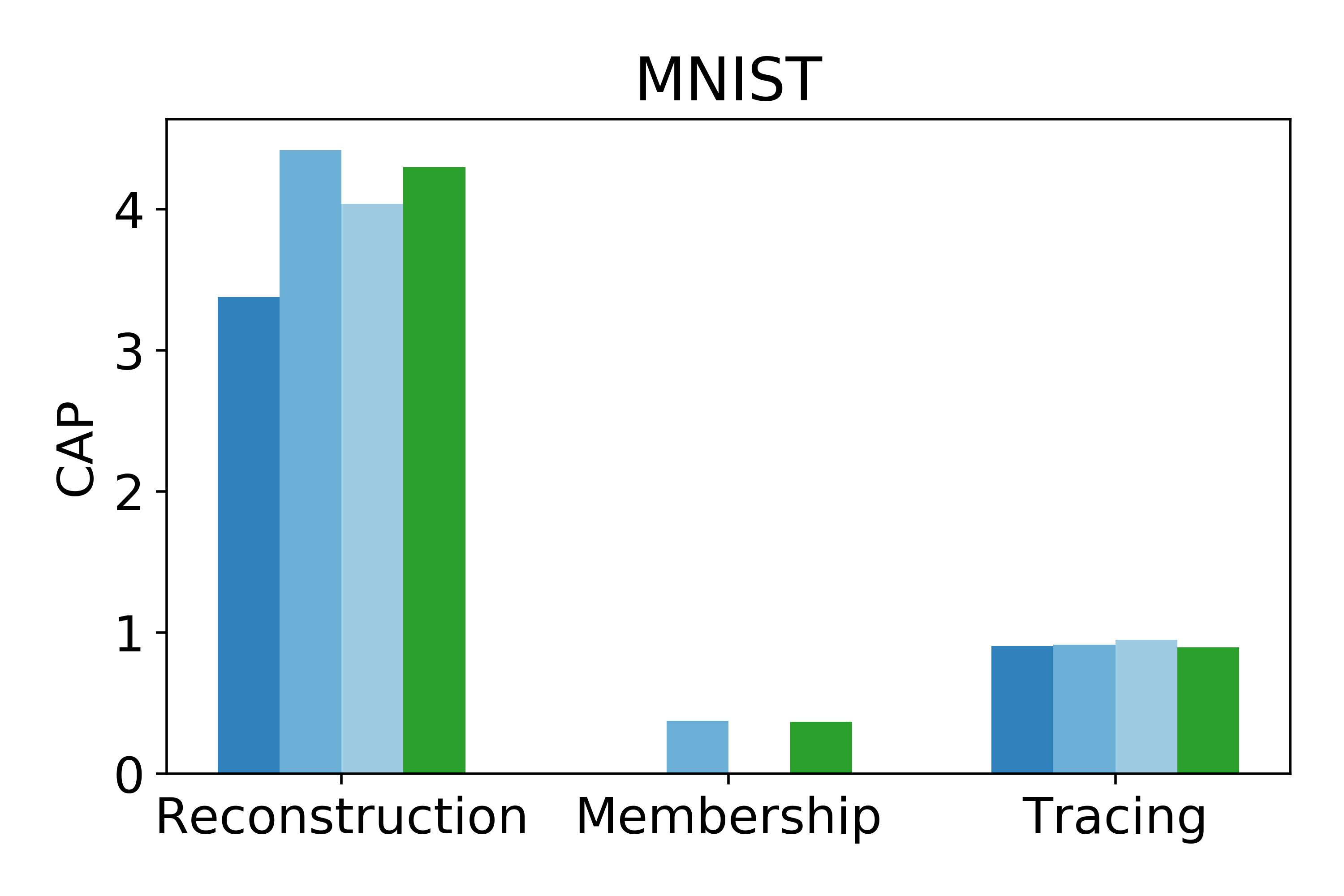}
			\includegraphics[width=0.24\linewidth]{imgs/capbar/cifar10_capsummary_sbs1.png}
			\includegraphics[width=0.24\linewidth]{imgs/capbar/cifar100_capsummary_sbs1.png}
			\includegraphics[width=0.24\linewidth]{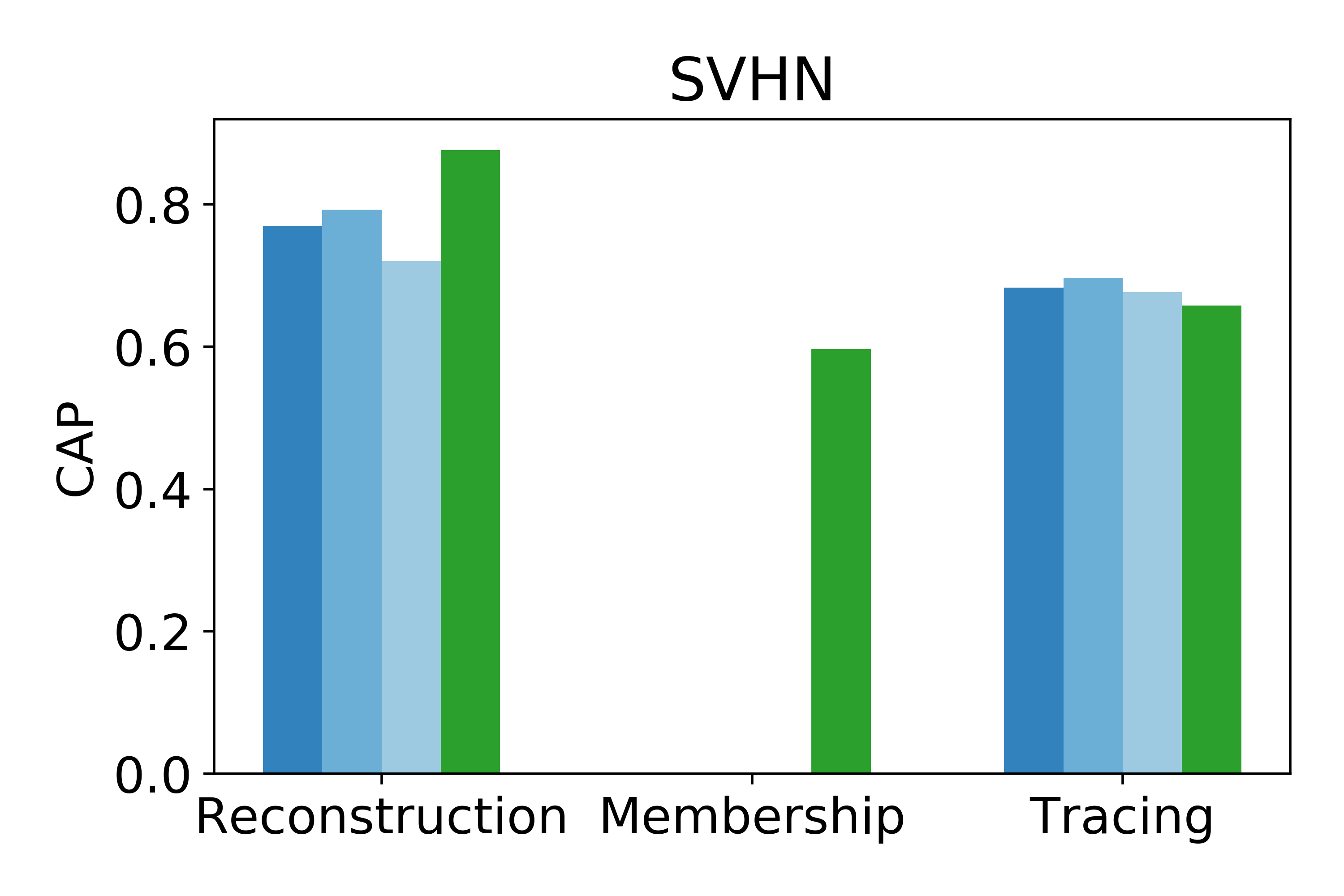}
			\caption{Attack Batch Size 1. Left to Right: MNIST, CIFAR10, CIFAR100 and SVHN}
		\end{subfigure}
		
		\begin{subfigure}{0.99\linewidth}
			\includegraphics[width=0.24\linewidth]{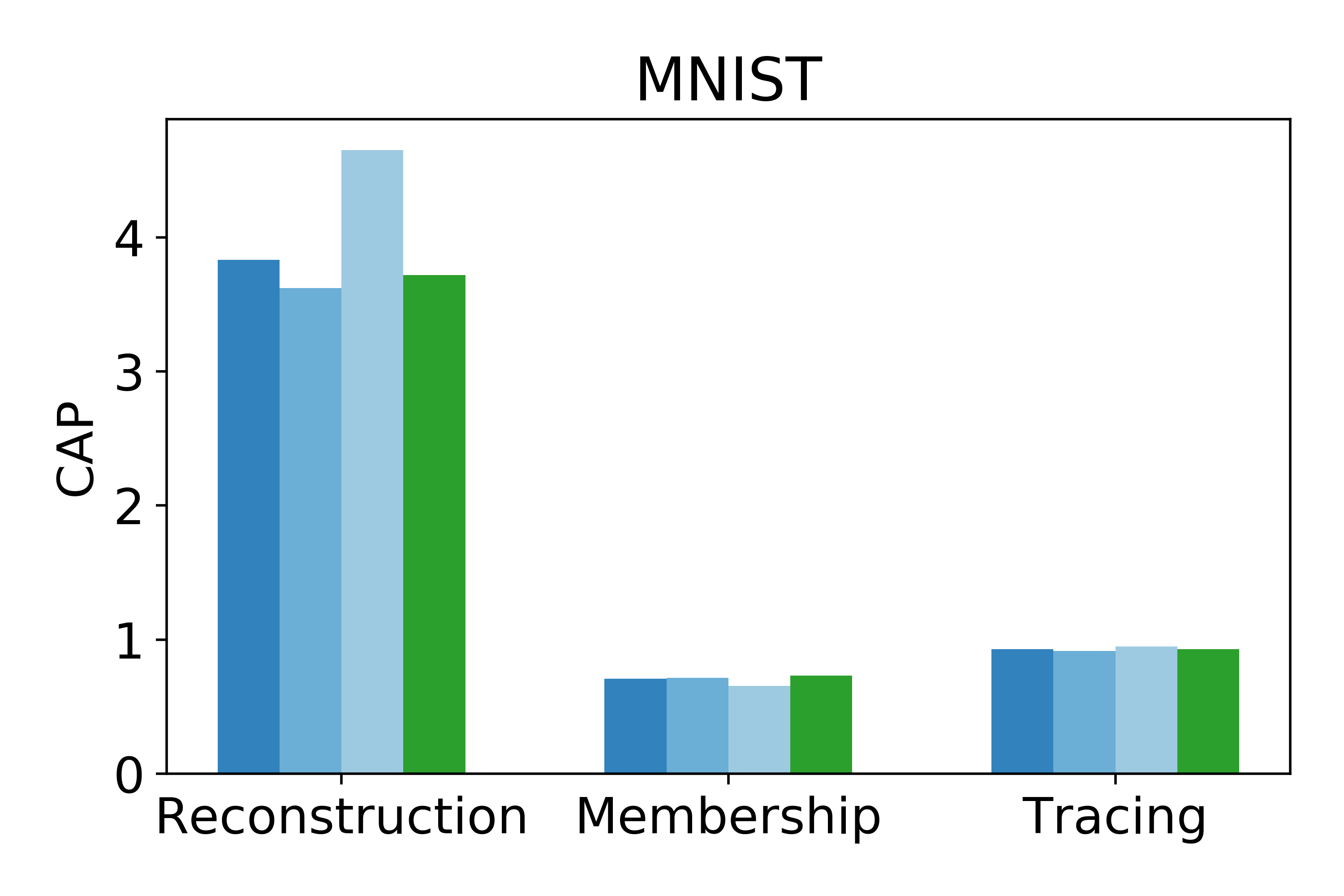}
			\includegraphics[width=0.24\linewidth]{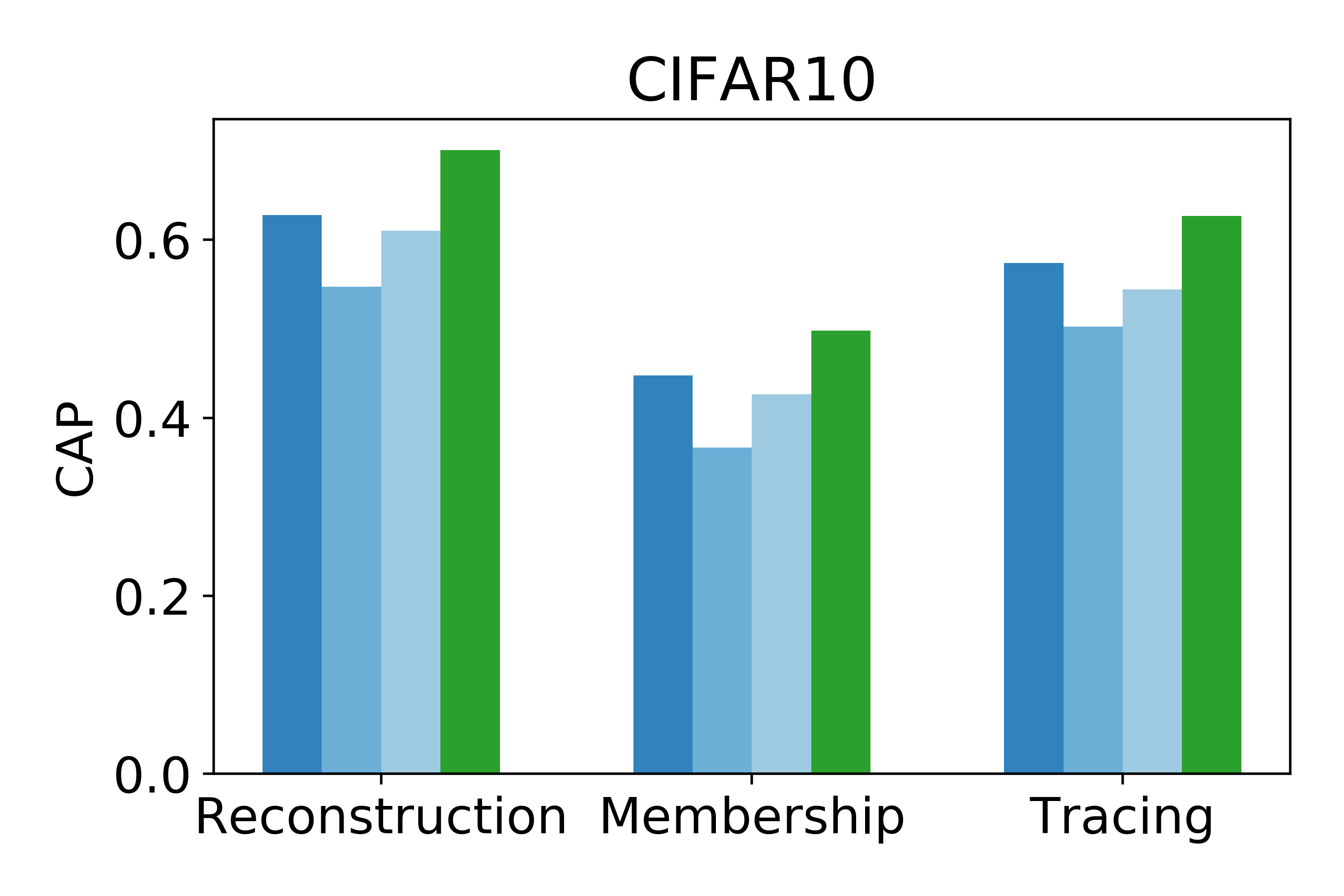}
			\includegraphics[width=0.24\linewidth]{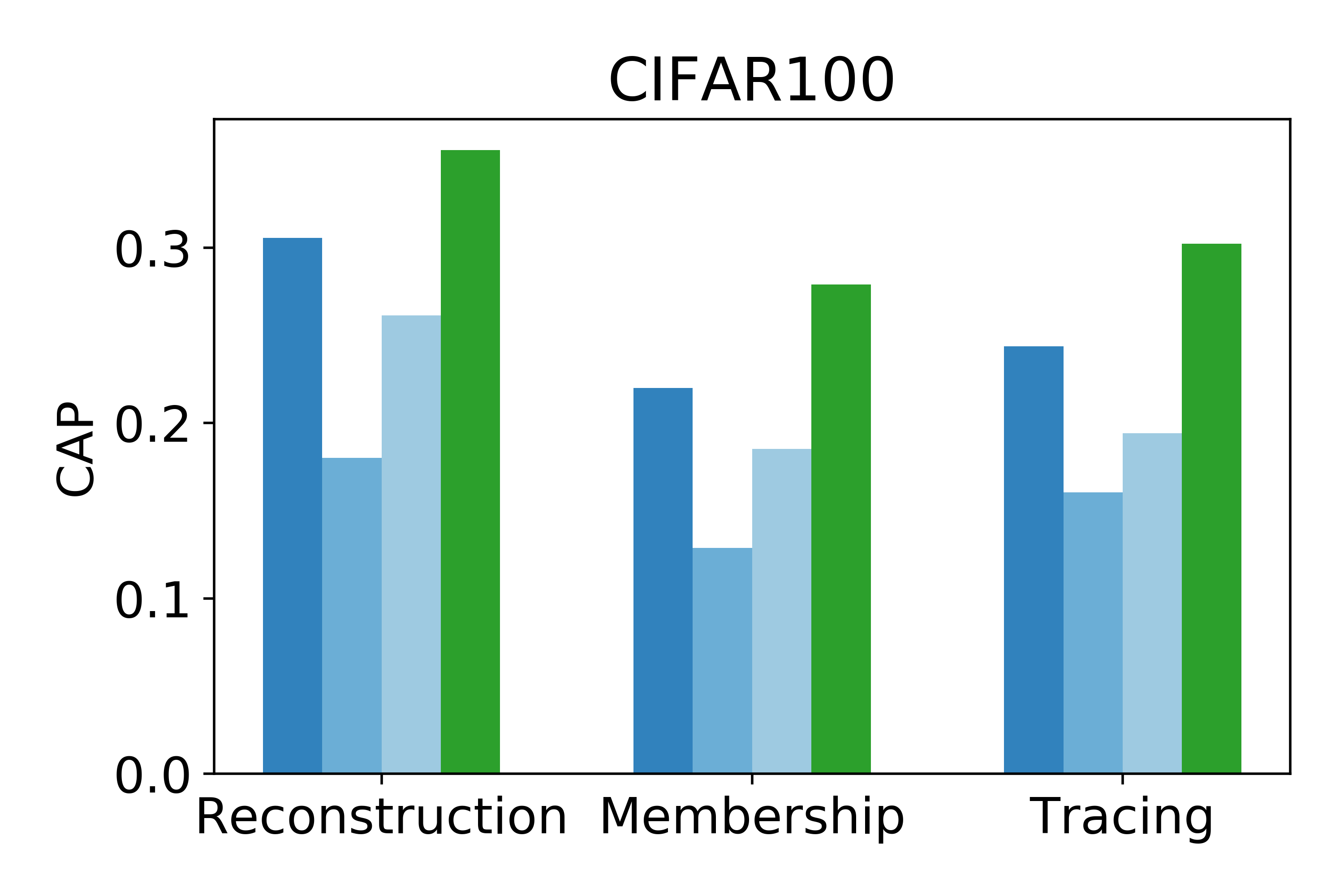}
			\includegraphics[width=0.24\linewidth]{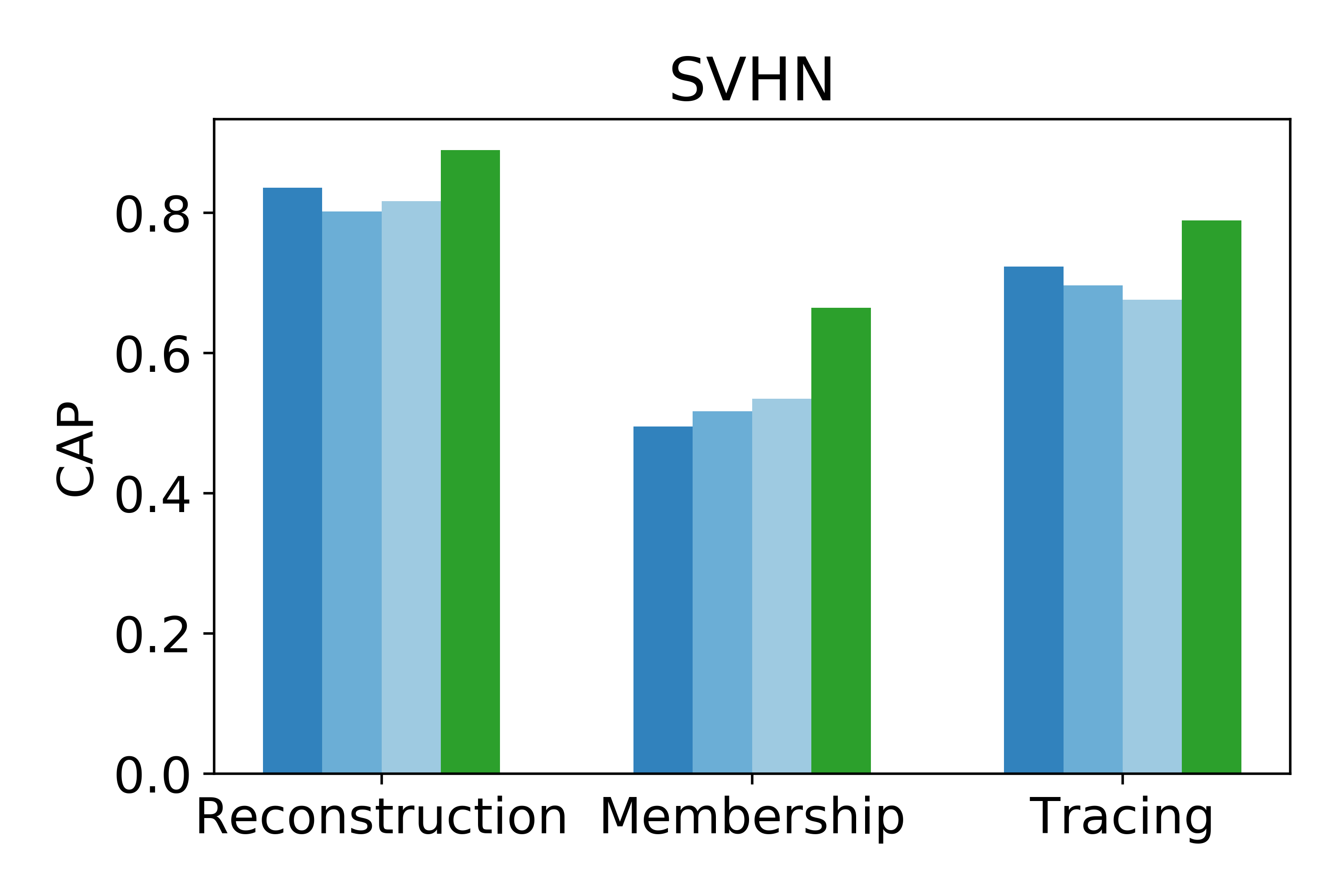}
			\caption{Attack Batch Size 4. Left to Right: MNIST, CIFAR10, CIFAR100 and SVHN}
		\end{subfigure}
		
		\begin{subfigure}{0.99\linewidth}
			\includegraphics[width=0.24\linewidth]{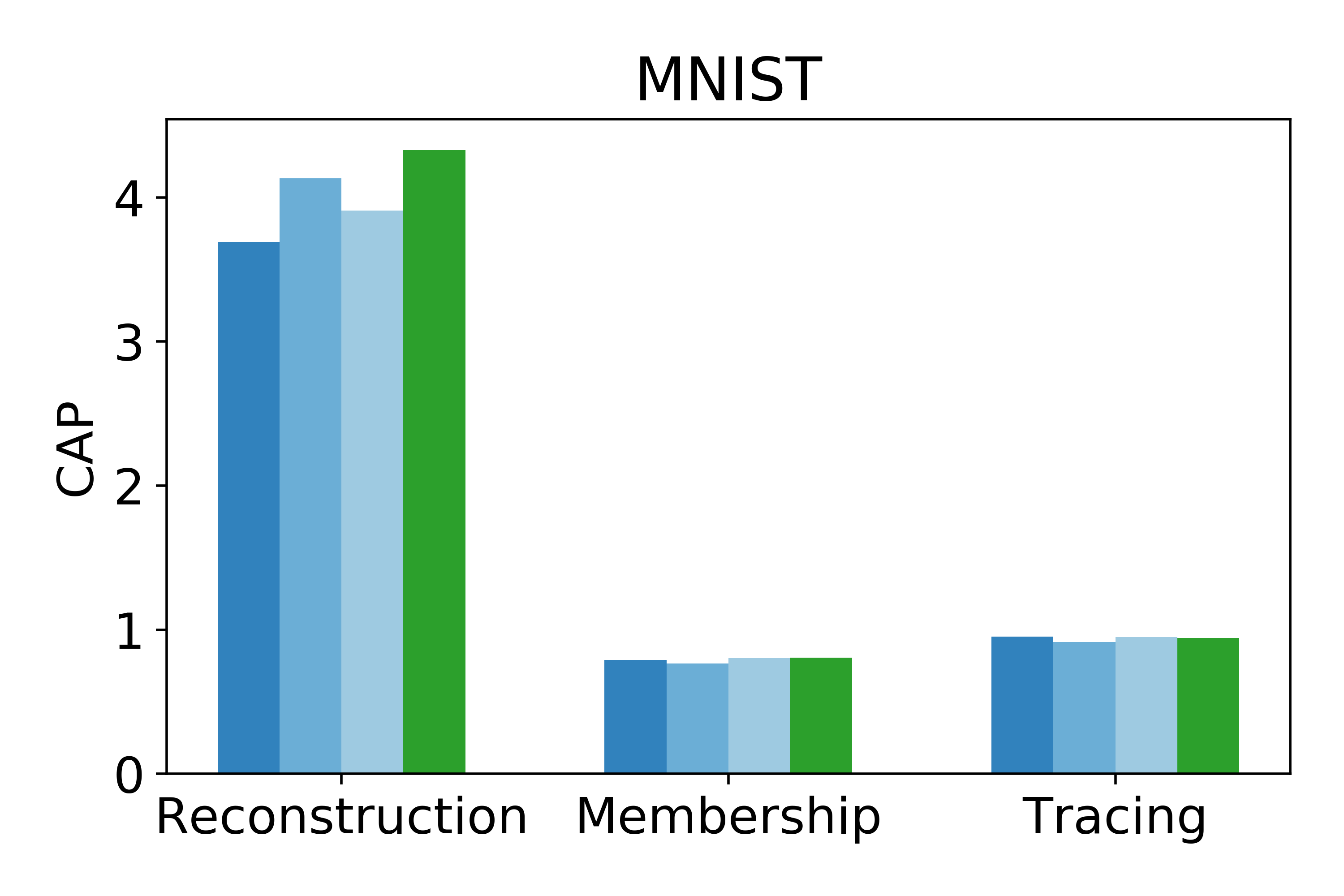}
			\includegraphics[width=0.24\linewidth]{imgs/capbar/cifar10_capsummary_sbs8.png}
			\includegraphics[width=0.24\linewidth]{imgs/capbar/cifar100_capsummary_sbs8.png}
			\includegraphics[width=0.24\linewidth]{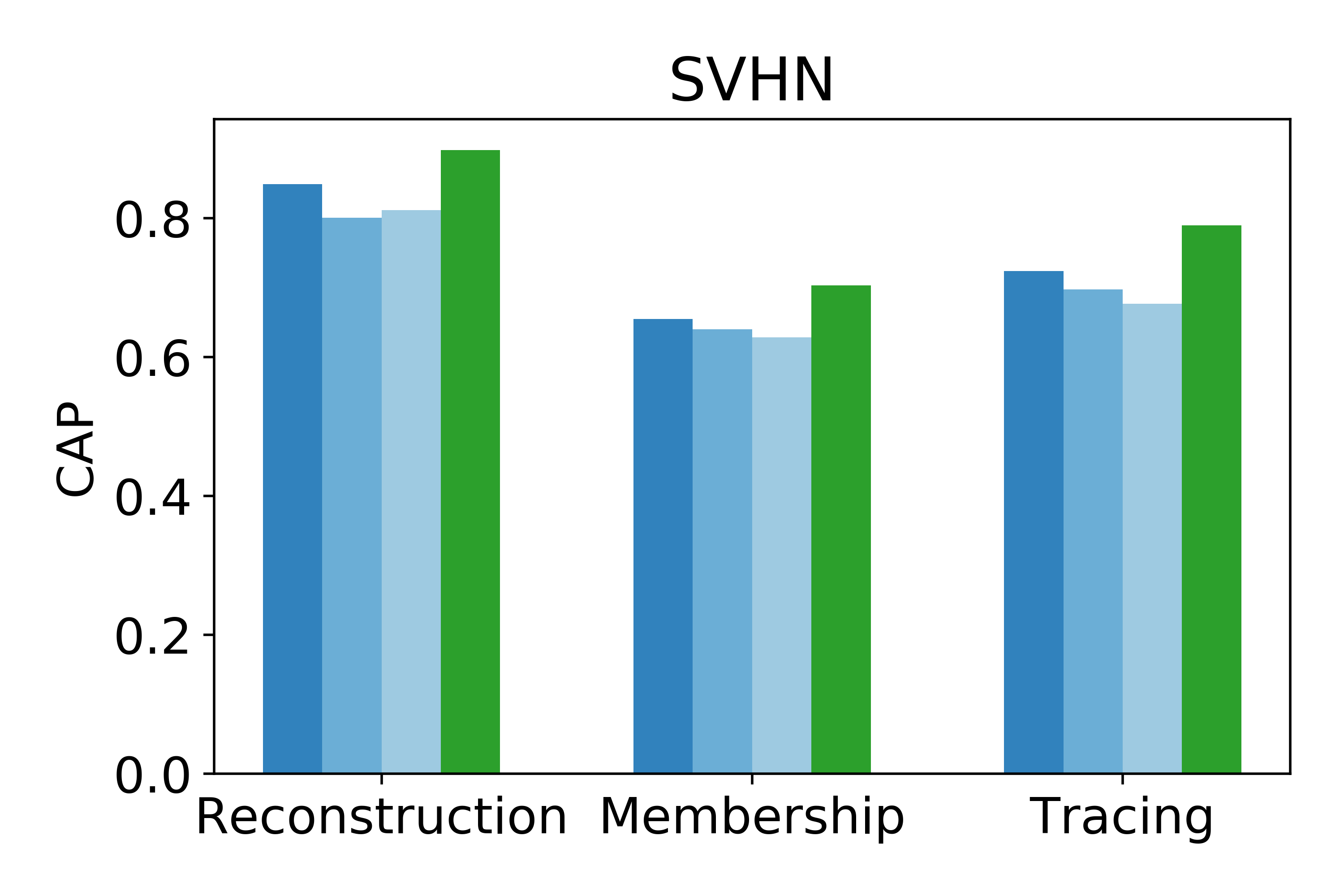}
			\caption{Attack Batch Size 8. Left to Right: MNIST, CIFAR10, CIFAR100 and SVHN.}
		\end{subfigure}
		
		%\end{minipage}
		\caption{Comparison of Calibrated %Maximum/
			Averaged Performances (%CMPs/
			CAPs) for the proposed SPN, PPDL \cite{PPDL/shokri2015} and DP \cite{DLDP_Abadi16} methods, against reconstruction, membership and tracing attacks (%CMP/
			CAP the higher the better, see threat model and evaluation protocol in main paper).}
		%\caption{Comparison of SPN, PPDL and DP privacy-preserving mechanisms. \label{fig:barchart-summary}}
	\end{figure}
	
	\newpage
	\section*{Appendix E: Ablation Studies}

	\subsection{Replace Gaussian Noise with Laplacian Noise}
	
	In this section, we replace Gaussian Noise with Laplacian Noise. For Laplcian noise and Gaussian noise, the scales we used are \{$0.5, 0.1, 0.01, 0.001, 0.0001$\}. 
	
	In Figure \ref{fig:laplacian}, Laplacian noise and Gaussian noise with the same scale are having almost identical $\frac{||B_I||}{||E_B||}$ and protection strength (i.e. rMSEs at different $\frac{||B_I||}{||E_B||}$ are almost the same).
	
	\begin{figure}[H]
		\centering
		\includegraphics[width=0.3\linewidth,height=3cm]{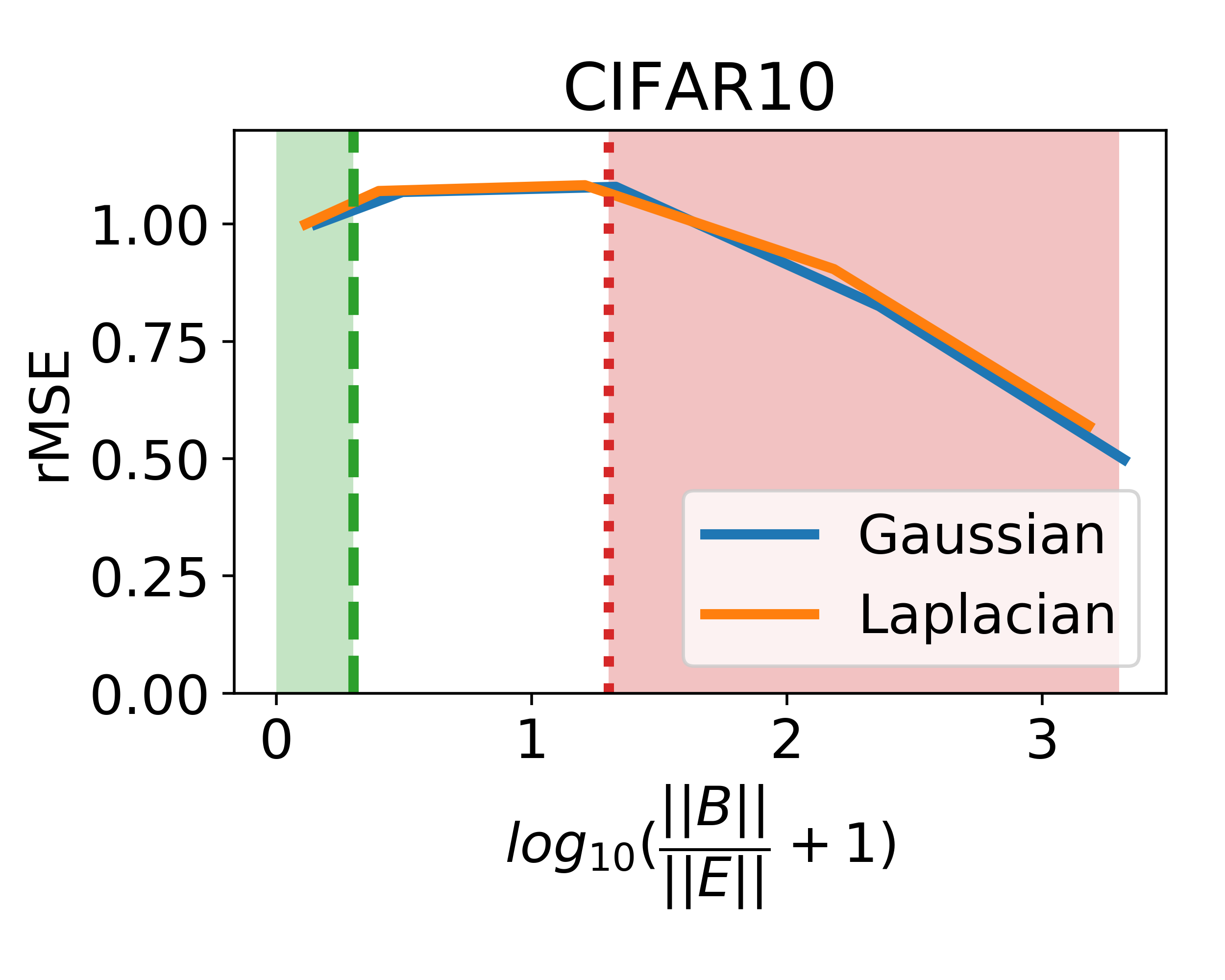}
		\includegraphics[width=0.3\linewidth,height=3cm]{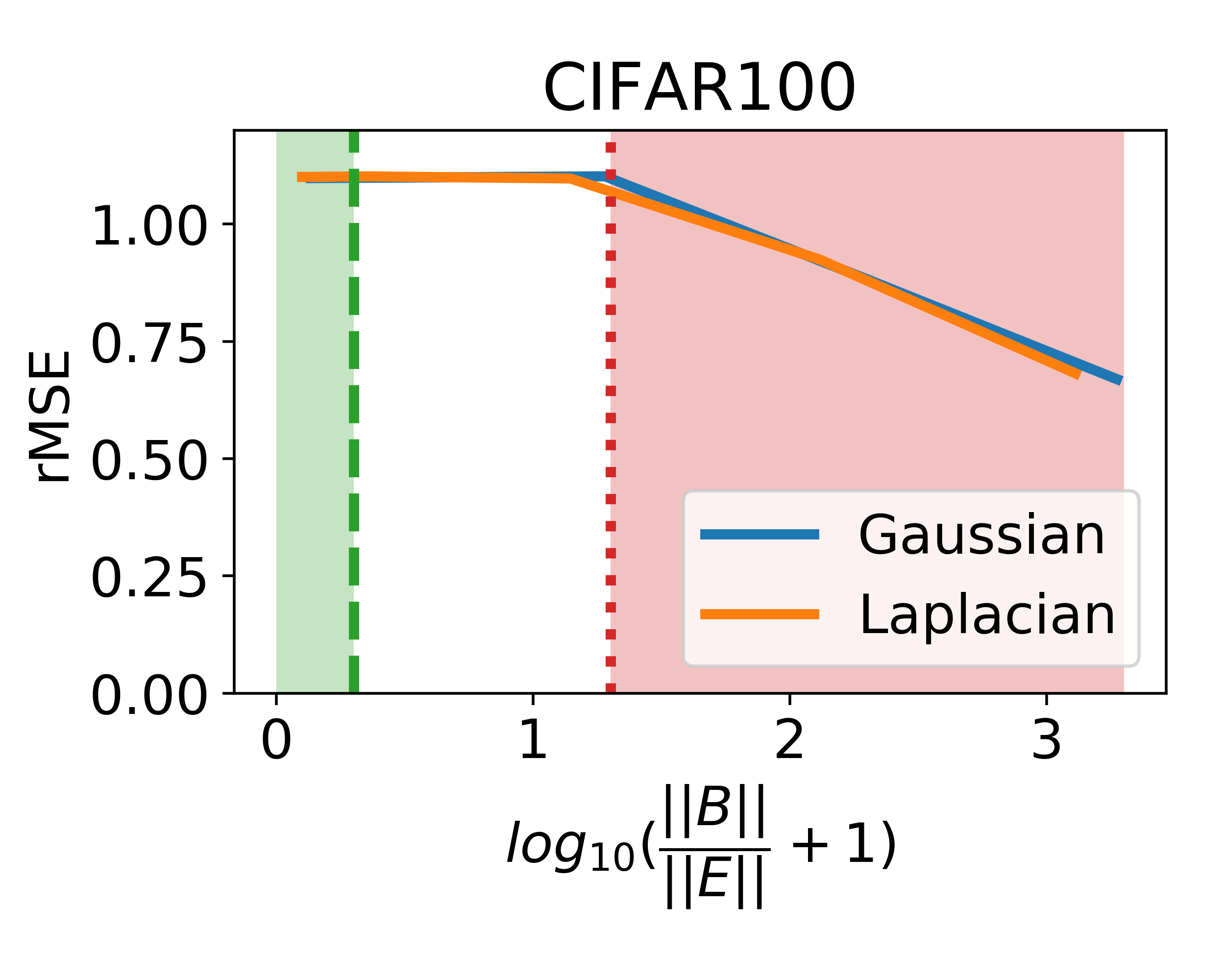}
		\caption{\textbf{Left}: CIFAR10; \textbf{Right}: CIFAR100.}
		\label{fig:laplacian}
	\end{figure}
	
	\subsection{Effect of Number of Bits in SPN}
	
	In this section, we shows that with number of bits in SPN will affect $\frac{||B_I||}{||E_B||}$ and hence improves the protection against reconstruction attack.
	
	From 32-bit to 128-bit, $\frac{||B_I||}{||E_B||}$ increased as shown in Figure \ref{fig:numbits}, protection strength (i.e. rMSE) is also increased.
	
	\begin{figure}[H]
		\centering
		\includegraphics[width=0.3\linewidth,height=3cm]{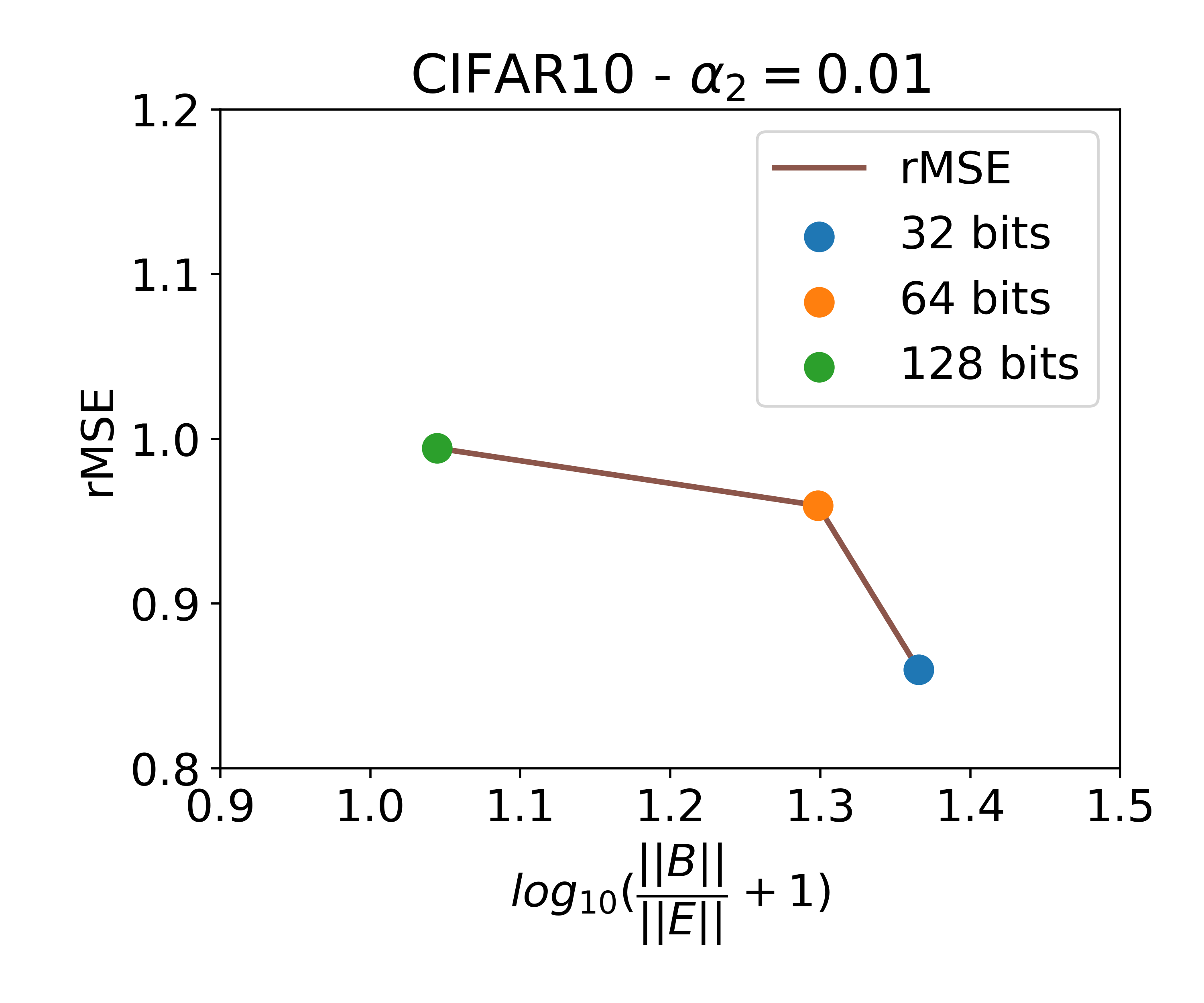}
		\includegraphics[width=0.3\linewidth,height=3cm]{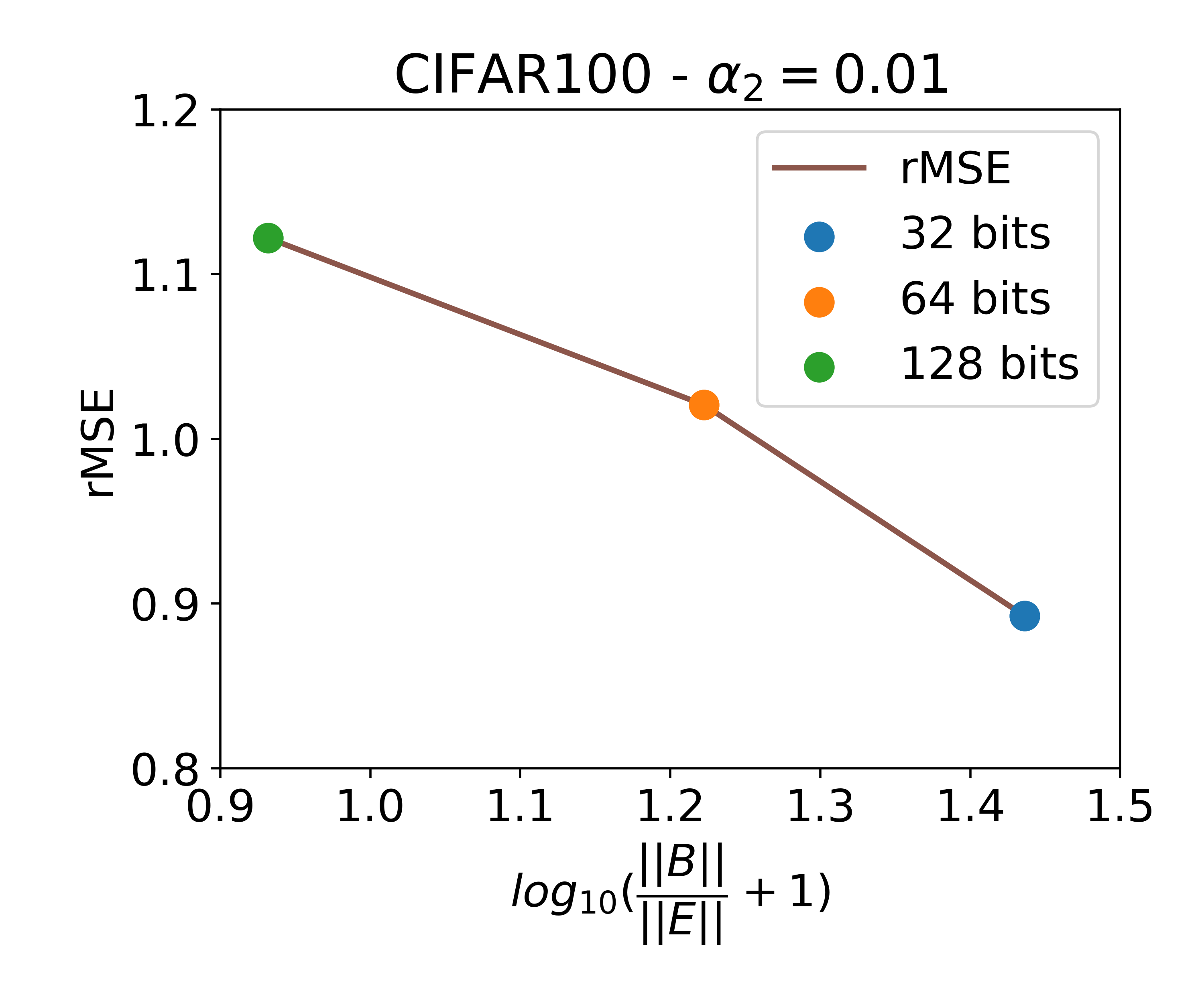}
		\caption{\textbf{Left}: CIFAR10; \textbf{Right}: CIFAR100.}
		\label{fig:numbits}
	\end{figure}

	\newpage
	\section*{Appendix F: Federated Learning}
	
	\subsection{Accuracies for Privacy Attack Analysis}
	
	Table \ref{tab:lenet-acc-supp} shows accuracies of different privacy-preserving mechanisms using DLNet as network architecture and using round robin for model aggregation. \textit{Accuracies are measured using test dataset on server model.} For SPN, we are using 64-bit.
	
	\begin{table}[H]
		\centering
		\adjustbox{max width=\textwidth}{
			\begin{tabular}{l|c|c|c|c}
				\toprule
				& MNIST & CIFAR10 & CIFAR100 & SVHN \\
				\midrule
				DP-0.5 & 0.9466 & 0.4205 & 0.0876 & 0.6882 \\
				DP-0.1 & 0.9757 & 0.5808 & 0.2084 & 0.8295 \\
				DP-0.01 & 0.9922 & 0.6853 & 0.3614 & 0.8950 \\
				DP-0.001 & 0.9942 & 0.7027 & 0.3803 & 0.8980 \\
				DP-0.0001 & 0.9925 & 0.7025 & \textbf{0.3902} & 0.9015 \\
				\midrule
				PPDL-0.05, DP-0.5 & 0.9305 & 0.3783 & 0.0469 & 0.5151 \\
				PPDL-0.05, DP-0.1 & 0.9708 & 0.5184 & 0.1497 & 0.8254 \\
				PPDL-0.05, DP-0.01 & 0.9881 & 0.6081 & 0.2323 & 0.8829 \\
				PPDL-0.05, DP-0.001 & 0.9886 & 0.5906 & 0.2237 & 0.8894 \\
				PPDL-0.05, DP-0.0001 & 0.9897 & 0.5814 & 0.2068 & 0.8740 \\
				\midrule
				PPDL-0.3, DP-0.5 & 0.9466 & 0.4250 & 0.0894 & 0.5132 \\
				PPDL-0.3, DP-0.1 & 0.9775 & 0.5672 & 0.1931 & 0.8424 \\
				PPDL-0.3, DP-0.01 & 0.9910 & 0.6783 & 0.3127 & 0.8992 \\
				PPDL-0.3, DP-0.001 & 0.9926 & 0.6672 & 0.3177 & 0.8992 \\
				PPDL-0.3, DP-0.0001 & 0.9934 & 0.6686 & 0.3223 & 0.8975 \\
				\midrule
				SPN-0.5 & 0.9936 & 0.6583 & 0.2990 & 0.8711 \\
				SPN-0.4 & 0.9939 & 0.6594 & 0.2999 & 0.8740 \\
				SPN-0.3 & 0.9933 & 0.6814 & 0.2999 & 0.8727 \\
				SPN-0.2 & 0.9937 & 0.6898 & 0.3230 & 0.8788 \\
				SPN-0.1 & 0.9927 & 0.6897 & 0.3457 & 0.8880 \\
				SPN-0.01 & 0.9928 & 0.6970 & 0.3843 & \textbf{0.9163} \\
				SPN-0.001 & 0.9938 & 0.6968 & 0.3768 & 0.9154 \\
				SPN-0.0001 & \textbf{0.9944} & \textbf{0.7154} & 0.3740 & 0.9126 \\
				\bottomrule
			\end{tabular}
		}
		\caption{Accuracies of different dataset on DLNet with different privacy-preserving mechanisms and their hyper-parameters. Bold values are highest accuracy among different mechanisms in the dataset.}
		\label{tab:lenet-acc-supp}
		\vspace{-10pt}
	\end{table}
	
	\newpage
	\subsection{Accuracies on I.I.D dataset}
	
		\begin{figure}[H]	
			%	\begin{subfigure}{0.95\linewidth}
			\centering
%			\includegraphics[scale=0.6]{imgs/legends/legend_flacc_v2_horizontal.png}
%			\\
%			\begin{subfigure}{0.99\linewidth}
%				\centering
%				\includegraphics[width=0.3\linewidth]{imgs/iidacc/alexnet_mnist.png}
%				\includegraphics[width=0.3\linewidth]{imgs/iidacc/vgg16_mnist.png}
%				\caption{MNIST}
%			\end{subfigure}
%			\\
			\includegraphics[scale=0.6]{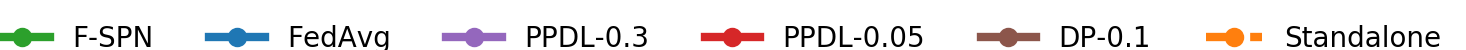}
			\begin{subfigure}{0.99\linewidth}
				\centering
				\includegraphics[width=0.3\linewidth]{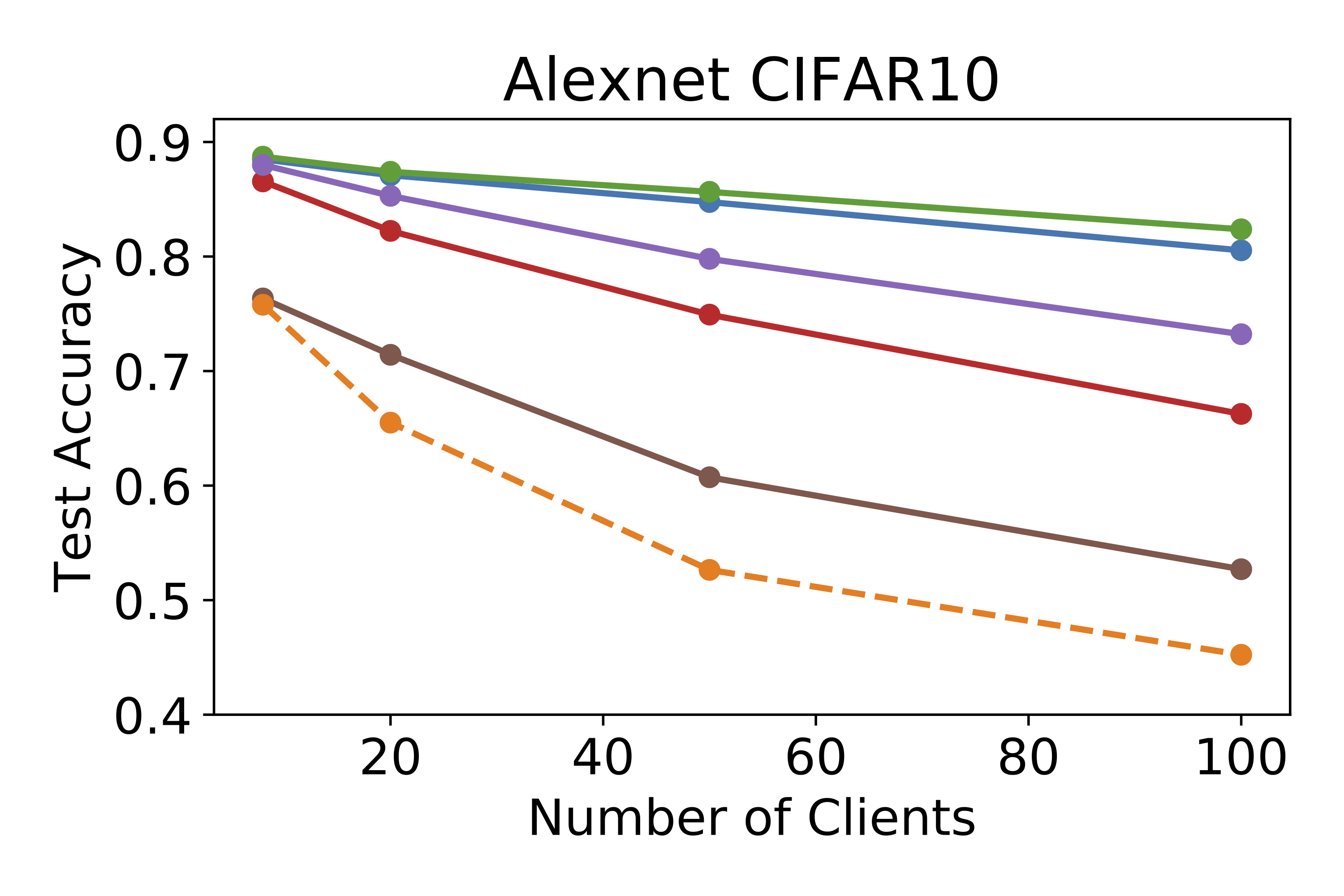}
				\includegraphics[width=0.3\linewidth]{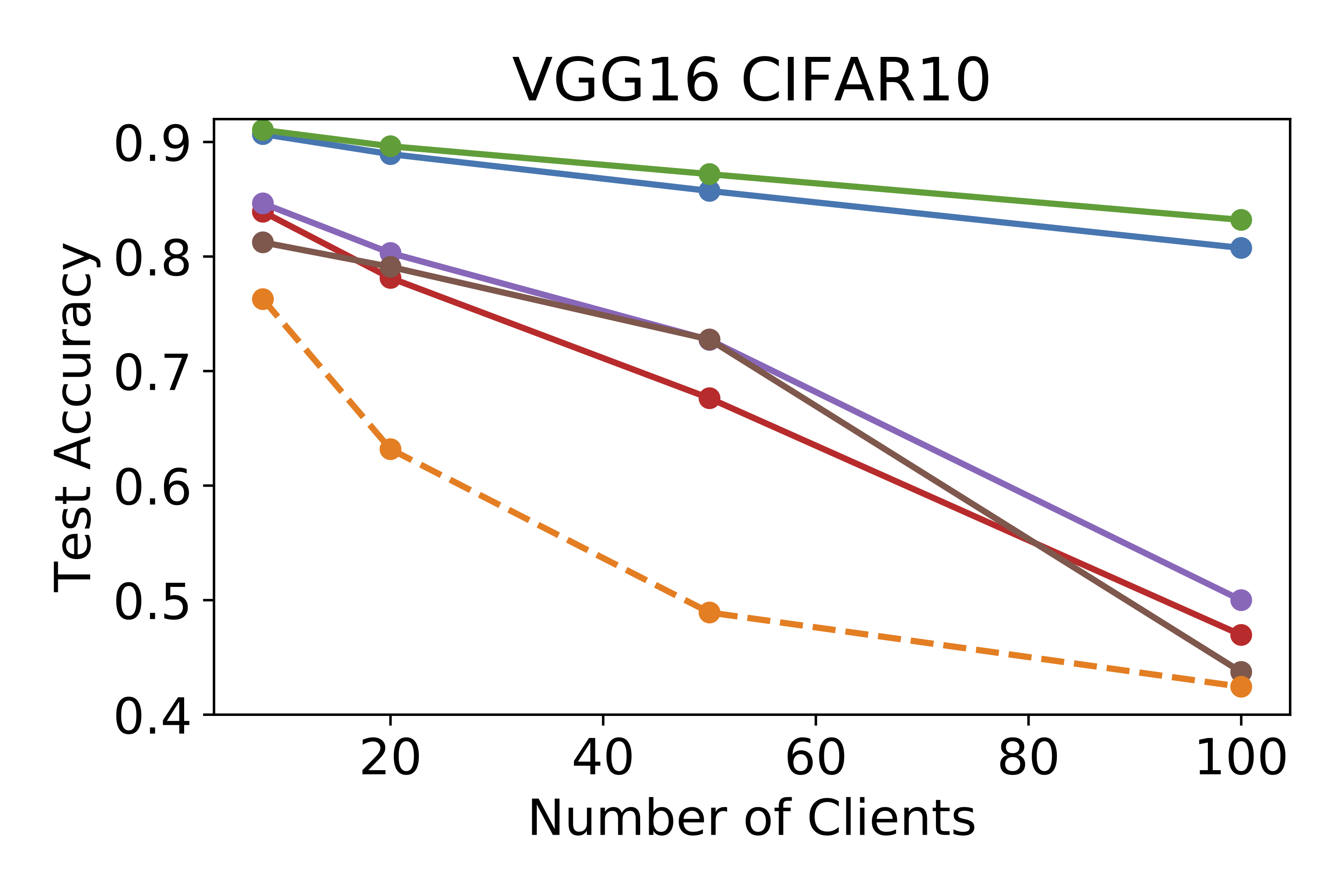}
				\caption{CIFAR10}
			\end{subfigure}
			\\
			\begin{subfigure}{0.99\linewidth}
				\centering
				\includegraphics[width=0.3\linewidth]{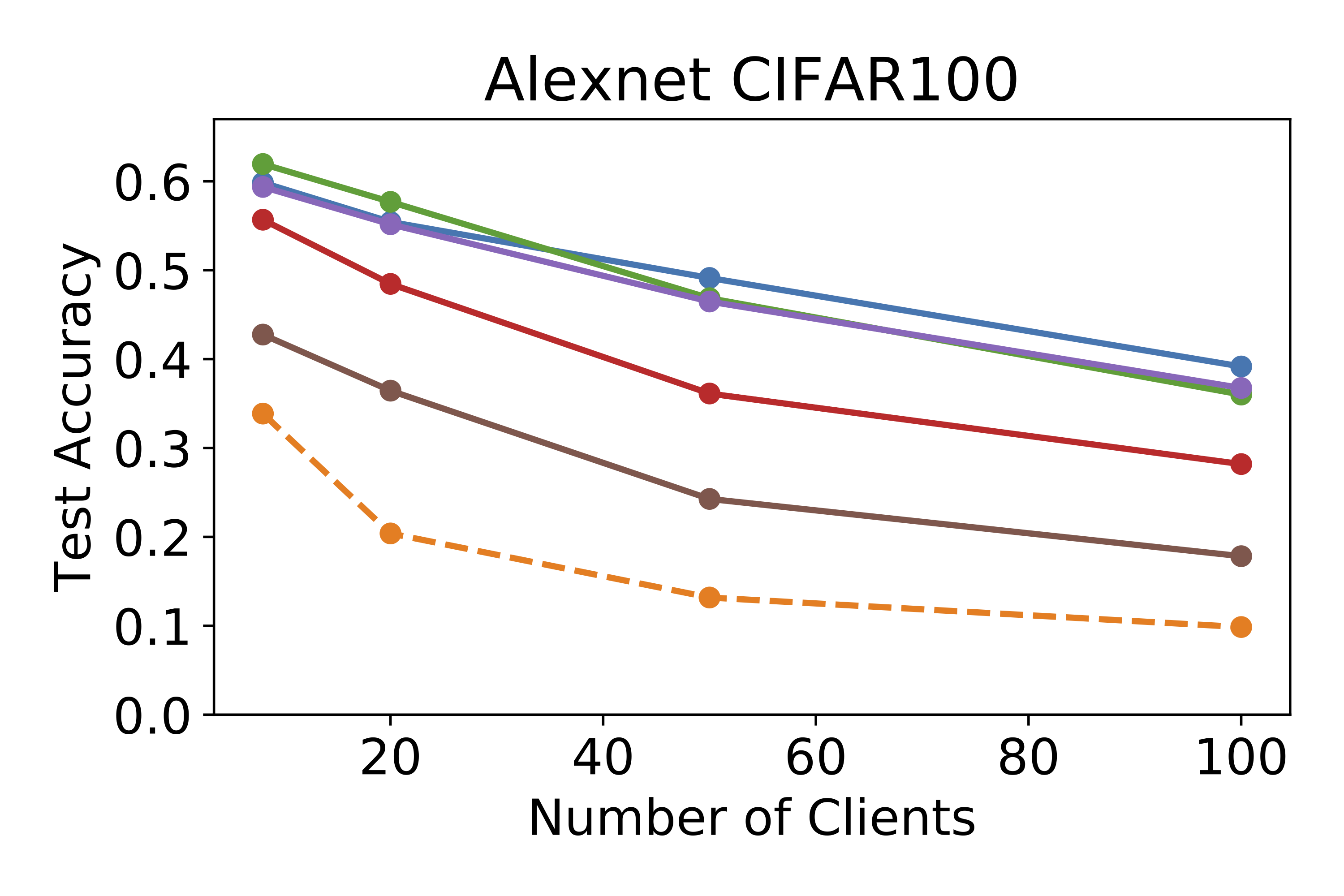}
				\includegraphics[width=0.3\linewidth]{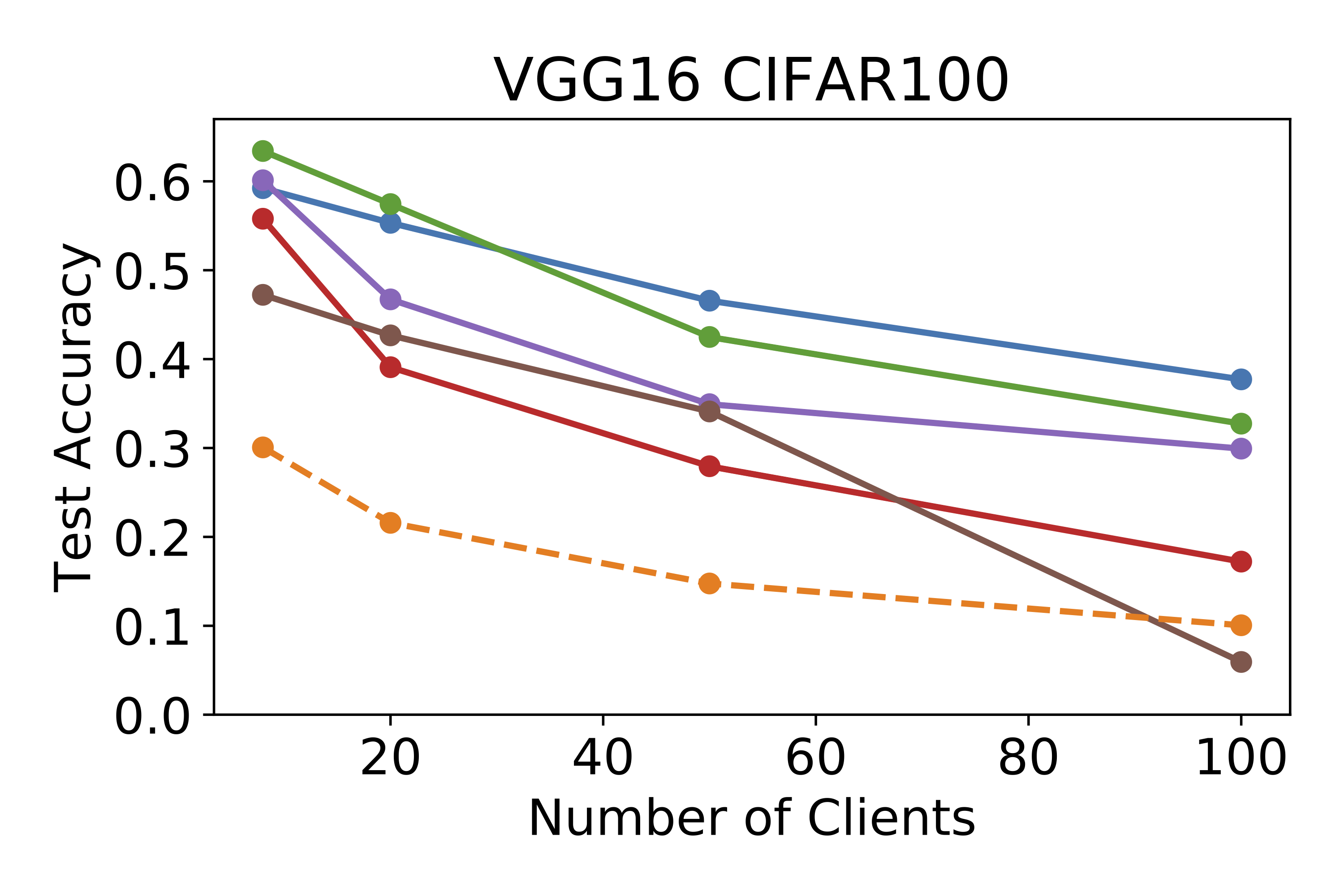}
				\caption{CIFAR100}
			\end{subfigure}
			\\
			\begin{subfigure}{0.99\linewidth}
				\centering
				\includegraphics[width=0.3\linewidth]{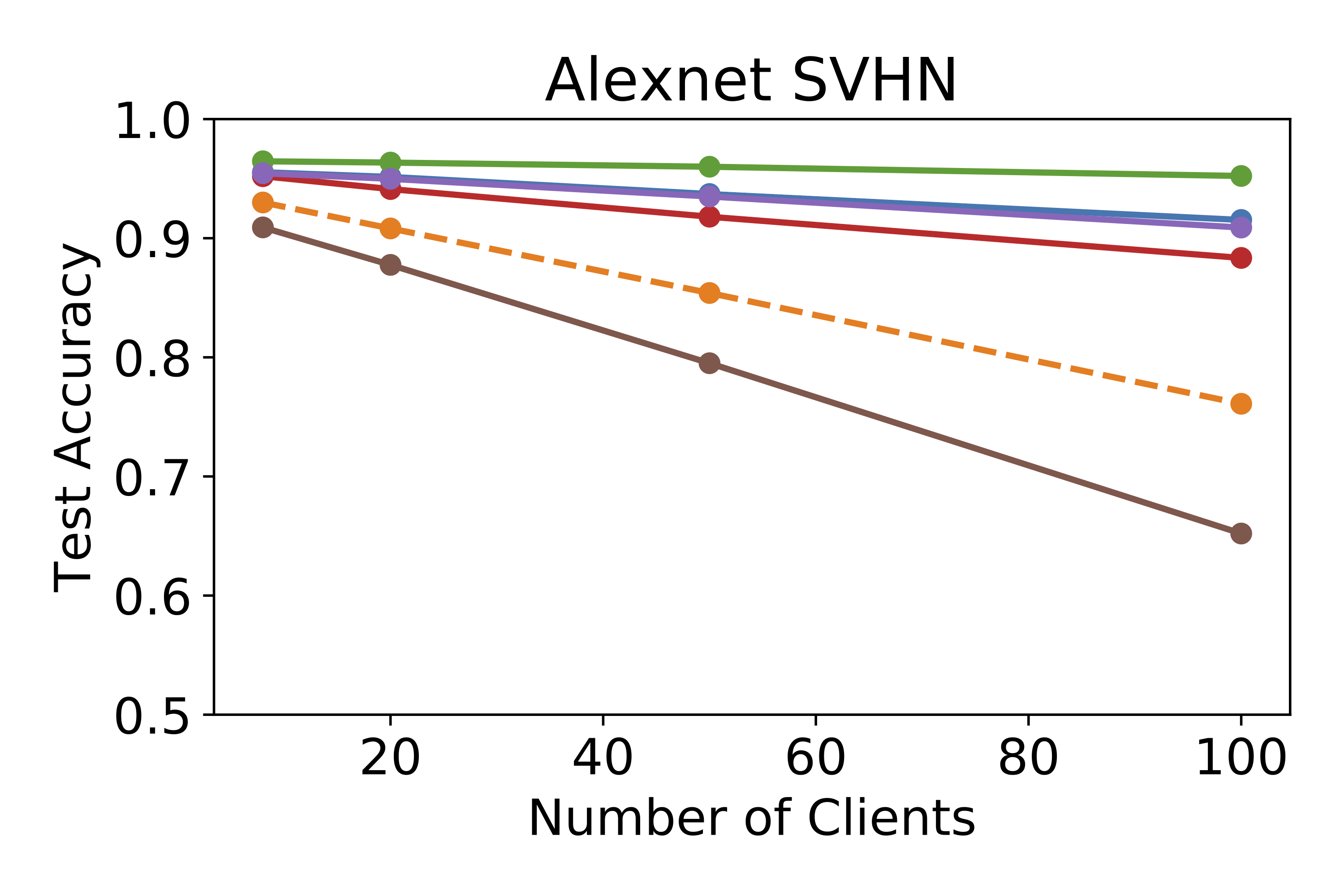}
				\includegraphics[width=0.3\linewidth]{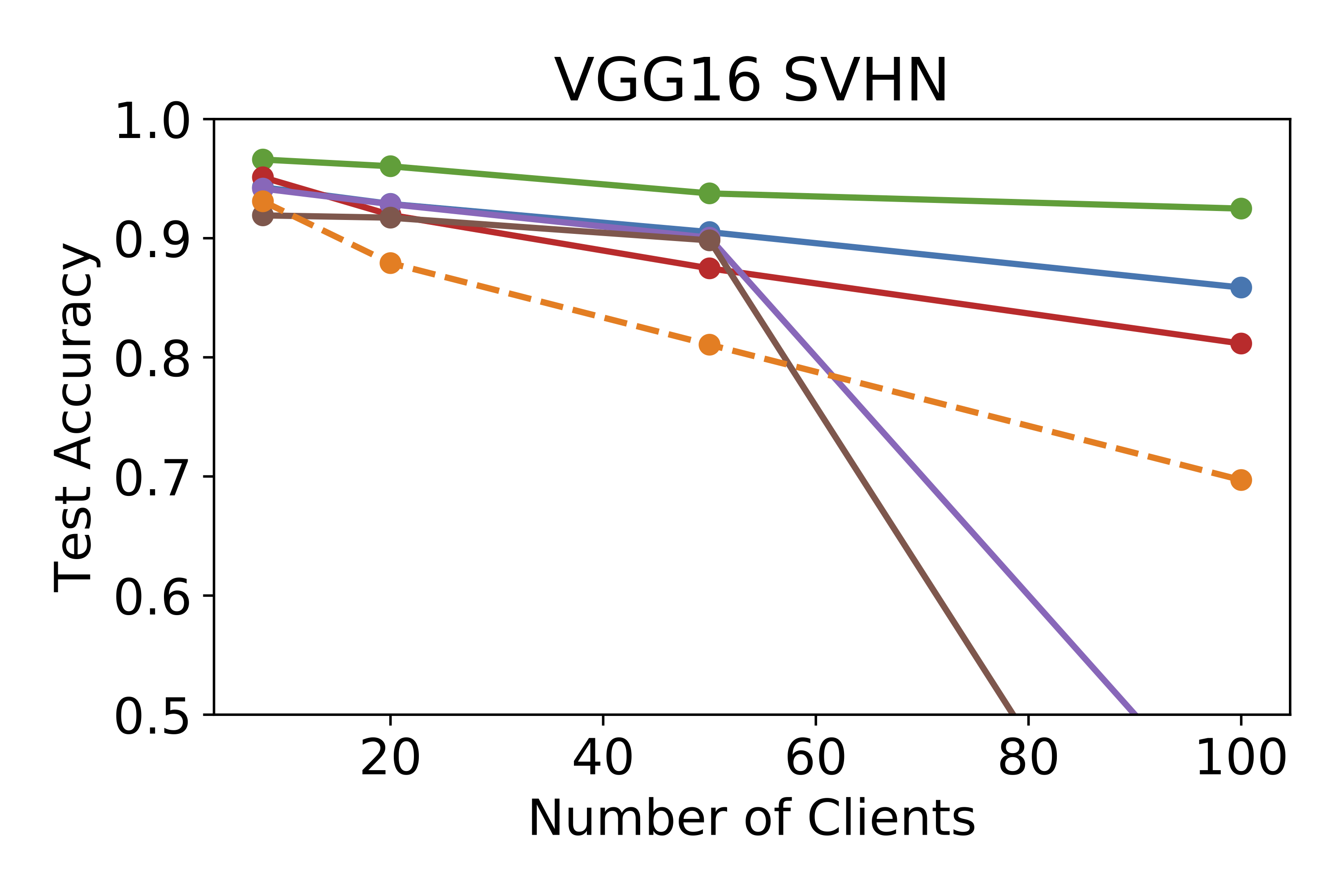}
				\caption{SVHN}
			\end{subfigure}
			\\
			\caption{Comparison of accuracies for \textit{standalone} local models, \textit{FedAvg} global model and models with different privacy-preserving mechanisms (\textit{Federated SPN}($\alpha_{2}=0.1$), \textit{PPDL-0.3}, \textit{PPDL-0.05} and \textit{DP-0.1}). 
				Improvements over standalone models increase with the number of clients. \textbf{Left}: AlexNet; \textbf{Right}: VGG16; F-SPN with $\alpha_{2}=0.1$ and DP-0.1 both having $\frac{||B_I||}{||E_B||} \approx 1$ which is considered borderline between green region and white region, F-SPN-0.1 outperforms DP-0.1 in terms of performance (e.g. ~25\% test accuracy improved in AlexNet CIFAR10.) while maintaining privacy guarantee. While comparing with PPDL with no DP added, F-SPN consistently performs better than PPDL. We observed that PPDL-0.3 and DP-0.1 on VGG16 SVHN are unstable which failed to train at 100 clients.}
			\label{fig:flacc-supp}
			%\vspace{-10pt}
		\end{figure}
	
	\subsection{Accuracies on Non-I.I.D dataset}
	
		\begin{figure}[H]
			\centering
			\includegraphics[scale=0.3]{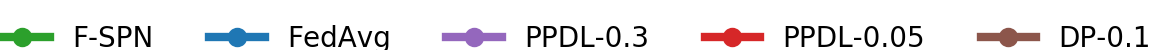}
			\\
			\includegraphics[width=0.3\linewidth]{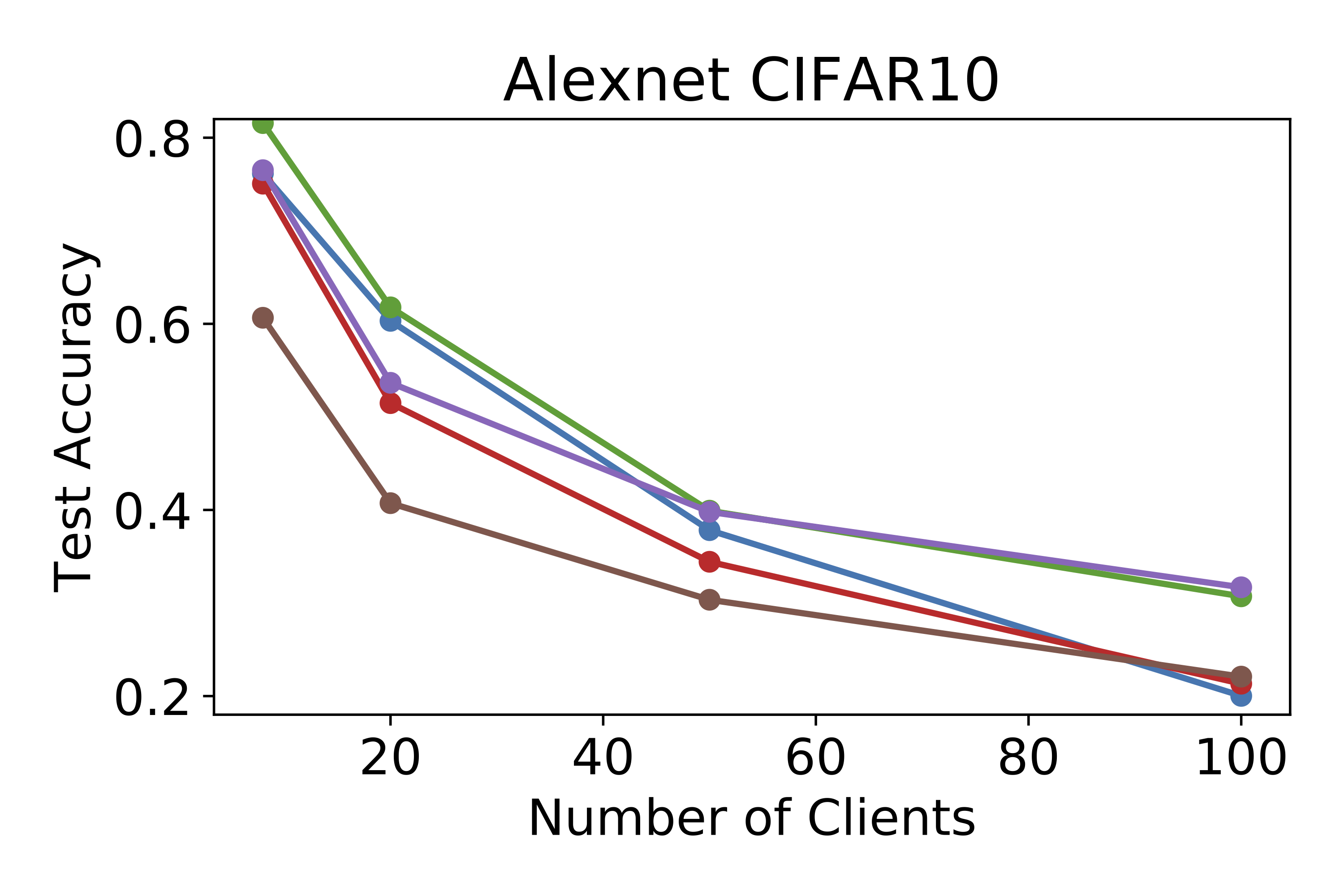}
			\includegraphics[width=0.3\linewidth]{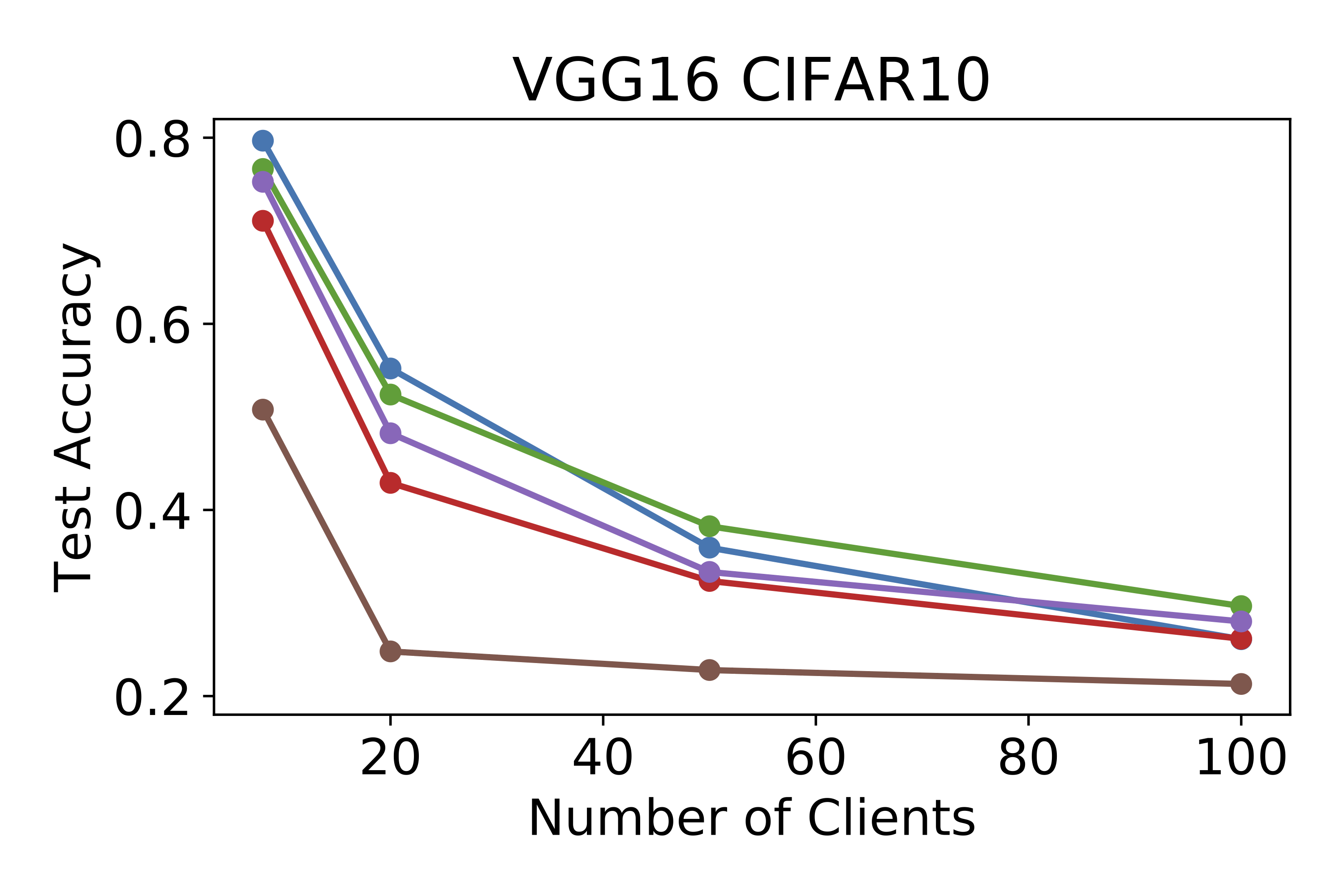}
			\caption{\textbf{Left}: AlexNet; \textbf{Right}: VGG16. For non-IID data, we split the dataset using Dirichlet distribution with $alpha=0.9$ as mentioned in Appendix C. F-SPN is also showing improvements over FedAvg increased with the number of clients and on-par with PPDL-0.3 at 100 clients.}
		\end{figure}
	
	\newpage
	\section*{Appendix F: Algorithm}
	
	\begin{algorithm}[H]
		\KwIn{local model $M_{i}$, private target $\bm{t}_{i}$, local dataset $D_{i}$, global model $G$} % \tcp*{Client Update}
		$M_{i}$ $\leftarrow$ download $G$ from server  \;
		\For{each $d_{j}, y_{j} \in D_{i}$}
		{
			$\bm{u}, \bm{v} = $ forward($M_{i}$, $d_{j}$)  \;
			$\nabla W_{i} = $ backprop($(\alpha_{1} * L_{CE}(W_{i}, \bm{u}, y_{j})) + (\alpha_{2} * L_{P}(W_{i}, \bm{v}, \bm{t}_{i}, y_{j}))$)  \;
			$W_{i} = W_{i} - lr * \nabla W_{i}$  \;
		}
		$\Delta W_{i}^{new} = W_{i} - G$  \;
		\Return $\Delta W_{i}^{new}$ 
		\caption{Client Update}
		\label{alg:client}
	\end{algorithm}	
	
	\begin{algorithm}[H]
		\KwIn{client $\Delta W_{i}^{new}$, global model $G$}   %\tcp*{Server Update}
		receive $\Delta W_{i}^{new}$ from $K$ clients \;
		$G^{new} = G + \frac{1}{K}\sum^{K}_{i=1}{\Delta W_{i}^{new}}$  \;
		\Return $G^{new}$ 
		\caption{Server Update}
	\end{algorithm}
	
	\begin{algorithm}[H]
		% \KwIn{client $\Delta W_i^{new}$, global model $G$}
		\textbf{Server:} initialize global model $G$ \;
		\For{$each\ client$}
		{
			initialize private target $\bm{t}_{i}$  \;
		}
		\For{$each\ communication$}
		{
			\For{$each\ client$}
			{
				$\Delta W_{i}^{new}$ = Client Update \tcp*{Client Update}
			}
			$G^{new} = $ Server Update \tcp*{Server Update}
		}
		\Return $G^{new}$
		
		\caption{Training Experiment (FedAvg)}
		\label{alg:training}
	\end{algorithm}

	\begin{algorithm}[H]
		% \KwIn{client $\Delta W_i^{new}$, global model $G$}
		\textbf{Server:} initialize global model $G$ \;
		\For{$each\ client$}
		{
			initialize private target $\bm{t}_{i}$  \;
		}
		\For{$each\ communication$}
		{
			\For{$each\ client$}
			{
				$\Delta W_{i}^{new}$ = Client Update \tcp*{Client Update}
				$G^{new} = $ Server Update \tcp*{Server Update}
			}
		}
		\Return $G^{new}$
		
		\caption{Training Experiment (Round Robin)}
		\label{alg:training-rr}
	\end{algorithm}
	
	\begin{algorithm}[H]
		% \KwIn{client $\Delta W_i^{new}$, global model $G$}
		\textbf{Input:} Global model $G$, dataset $D$, number of reconstruction $X$, private target $t_i$ \;
		
		\For{$1..X$}
		{
			$d_{x}, y_{x}$ $\leftarrow$ random sample from $D$ \;
			$\bm{u}, \bm{v} = $ forward($G$, $d_{x}$)  \;
			$\nabla W_G =(\alpha_{1} * L_{CE}(W_G, \bm{u}, y_{x})) + (\alpha_{2} * L_{P}(W_G, \bm{v}, \bm{t}_{i}, y_{x}))$  \;
			
			$d^{rec}_{x} = $ INIT(shape of $d_{x}$) \;
			$y^{rec}_{x} = $ INIT(shape of $y_{x}$) \;
			
			\While{not converged}{
				$\bm{u} = $ forward($G$, $d^{rec}_{x}$) \;
				$\nabla W^{rec}_G = L_{CE}(W_G, \bm{u}, y^{rec}_{x})$  \; \tcp*{Attacker doesn't know the existence of private term}
				minimize $||\nabla W_G - W^{rec}_G||$ \;
				update $d^{rec}_{x},y^{rec}_{x}$ \;
			}
		}
		\Return $D^{rec}, Y^{rec}$ \tcp*{Reconstructed Images and Labels (Membership)}
		
		\caption{Reconstruction and Membership Attack (Deep Leakage Attack)}
		\label{alg:recon}
	\end{algorithm}
	
	\newpage
	\begin{algorithm}[H]
		% \KwIn{client $\Delta W_i^{new}$, global model $G$}
		\textbf{Input:} Reconstructed images $D^{rec}$, query dataset $D$ \;
		
		Split $D^{rec}$ and $D$ into $N$ partitions \; \tcp*{Index of reconstructed image is same as index of groundtruth image.}
		
		\For{$d_{j}$, index$_{j} \in D$}
		{
			\For{$d^{rec}_{k}$, index$_{k} \in D^{rec}$}
			{
				MSE $ = ||d_{j} - d^{rec}_{k}||$ \;
				\If{found lowest MSE}{$t_{j} = $ index$_{k}$}
				\tcp*{Assign index of reconstructed image into tracing result}	
			}
			
			\If{$t_{j}$ == index$_{j}$}{Tracing successful}
		}
		\Return $T$ \tcp*{Tracing Indices}
		
		\caption{Tracing Attack}
		\label{alg:tracing}
	\end{algorithm}

%	\section{Appendix G: Efficiency of SPN}
%	
%	Our SPN is an additional loss term added into final loss function, therefore, the efficiency overhead is minor which only comes from calculating the loss term. With 64-bit private target and one extra fully connected layer (the layer with) on DLNet, the speed of forwarding and backpropagation is around 30\% more time than 
	
\end{document}